\documentclass{article}
\PassOptionsToPackage{numbers, compress}{natbib}
\usepackage[main, final]{neurips_2026}
\usepackage[utf8]{inputenc} % allow utf-8 input
\usepackage[T1]{fontenc}    % use 8-bit T1 fonts
\usepackage{url}            % simple URL typesetting
\usepackage{booktabs}       % professional-quality tables
\usepackage{amsfonts}       % blackboard math symbols
\usepackage{nicefrac}       % compact symbols for 1/2, etc.
\usepackage{microtype}      % microtypography
\usepackage{xcolor}         % colors

\usepackage{setspace}
\usepackage{graphicx}
\usepackage{wrapfig}
\usepackage{subcaption}
\usepackage{booktabs} % for professional tables
\usepackage[american]{babel}
\usepackage{setspace}

\usepackage{amsmath}
\usepackage{mathrsfs}
\usepackage{amssymb}
\usepackage{mathtools}
\usepackage{amsthm}
\usepackage{algorithm}
\usepackage{algorithmic}
\newenvironment{breakablealgorithm}
{%
  \par
  \addvspace{\topsep}
  \refstepcounter{algorithm}
  \hrule height 0.8pt depth 0pt
  \kern 2pt
  \renewcommand{\caption}[2][\relax]{%
    {\raggedright
    \textbf{Algorithm~\thealgorithm.} ##2\par}%
    \ifx\relax##1\relax
      \addcontentsline{loa}{algorithm}
      {\protect\numberline{\thealgorithm}##2}%
    \else
      \addcontentsline{loa}{algorithm}
      {\protect\numberline{\thealgorithm}##1}%
    \fi
    \kern 2pt
    \hrule
    \kern 2pt
  }%
}
{%
  \kern 2pt
  \hrule
  \par
  \addvspace{\topsep}
}
\usepackage{array}
\usepackage{multirow, bm}
\usepackage{float}
\usepackage{make cell}
\usepackage{pifont}

\usepackage{xcolor}         % colors
\definecolor{mydarkred}{rgb}{0.6,0,0}
\definecolor{mydarkgreen}{rgb}{0,0.6,0}
\usepackage[colorlinks,
linkcolor=mydarkred,
citecolor=mydarkgreen]{hyperref}
\usepackage{fontawesome5}   % [collaborator] GitHub icon in abstract
\usepackage[capitalize,noabbrev]{cleveref}
\theoremstyle{plain}
\newtheorem{theorem}{Theorem}
\newenvironment{thmbis}[1]
{\renewcommand{\thetheorem}{\ref{#1}$'$}%
\addtocounter{theorem}{-1}%
\begin{theorem}}
{\end{theorem}}

\usepackage[framemethod=TikZ]{mdframed}

\newmdenv[
  backgroundcolor=gray!15,
  linecolor=gray!15,
  linewidth=0pt,
  innertopmargin=2pt,
  innerbottommargin=2pt,
  innerleftmargin=1pt,
  innerrightmargin=1pt,
  skipabove=2pt,
  skipbelow=0pt
]{theoremshade}

\newtheorem{lemma}[theorem]{Lemma}

\newtheorem{definition}{Definition}
\newtheorem{assumption}{Assumption}

\newtheorem{remark}{Remark}
\usepackage[textsize=tiny]{todonotes}
\usepackage{bbding}
\usepackage{bm}
\usepackage{epstopdf}
\usepackage[font=small,labelfont=bf]{caption}

\usepackage{tikz}
\newcommand{\circled}[1]{%
\tikz[baseline=(X.base)]\node (X) [draw,circle,inner sep=1.2pt] {\normalsize #1};%
}

\usepackage{enumitem}
\setlist[itemize]{leftmargin=15pt}

\def \x {\bm{x}}

\def \bR {\mathbb{R}}
\def \bP {\mathbb{P}}

\def \bI {\mathbb{I}}
\def \bE {\mathbb{E}}

\def \cF {\mathcal{F}}

\def \cM {\mathcal{M}}
\def \cE {\mathcal{E}}

\def \cI {\mathcal{I}}
\def \cB {\mathcal{B}}

\def \cH {\mathcal{H}}

\def \cX {\mathcal{X}}
\def \cY {\mathcal{Y}}

\def \cJ {\mathcal{J}}

\def \cK {\mathcal{K}}

\newcommand{\ourmethod}{\textsc{Celeus}}
\newcommand{\nosurr}{\textsc{Celeus (w/o Surr)}}
\newcommand{\evbaseline}{\textsc{EValue}}
\newcommand{\cereval}{\textsc{Cer-Eval}}
\newcommand{\oracleacq}{\textsc{Oracle}}
\renewcommand{\eqref}[1]{Eqn.~\textup{(\ref{#1})}}

\makeatletter
\newcommand{\appendixtocname}{Appendix Contents}

\newcommand{\listofappendices}{%
  \section*{\appendixtocname}
  \@starttoc{atoc}
}

\newcommand{\appsection}[1]{%
  \section{#1}%
  \addcontentsline{atoc}{section}{\protect\numberline{\thesection}#1}%
}

\newcommand{\appsubsection}[1]{%
  \subsection{#1}%
  \addcontentsline{atoc}{subsection}{\protect\numberline{\thesubsection}#1}%
}

\usepackage{colortbl}
\definecolor{lightgray}{gray}{0.92}

\makeatother

\title{\ourmethod: Certifiable and Efficient LLM Evaluation \\via E-Processes}

\author{
Zhijian Zhou$^1$\thanks{Equal contribution. $^\dag$Corresponding author.}\qquad Zesheng Ye$^{1*}$\qquad Zhaorun Chen$^2$\qquad
\textbf{Bo Li$^{2,3}$ \qquad Feng Liu$^{1\dag}$} \\
The University of Melbourne$^1$ \qquad University of Chicago$^2$ \\
University of Illinois Urbana-Champaign$^3$ \\
\texttt{fengliu.ml@gmail.com}
}

\begin{document}

\maketitle

% TODO (T1) [ZJ]: OpenReview (outside the PDF): (1) replace the OpenReview abstract with the current abstract below;
% (2) fill in the lay summary; (3) check the author list and order; (4) upload the final PDF and supplementary material.
\begin{abstract}
Can we trust evaluation scores to capture an LLM's true real-world performance?
\emph{Certifiable} evaluation answers this question by providing a confidence interval (CI) that covers the LLM's true performance with high probability (e.g., 95\%).
In particular, existing methods sequentially select evaluation samples from a fixed pool and keep updating the CI until certain conditions are satisfied, e.g., the CI width reaches a target precision. 
However, existing methods are not generally anytime-valid: the claimed coverage (e.g., 95\%) may fail when CIs are repeatedly updated and used to decide when to stop, leaving a \emph{gap} between theoretical rigor and practice. This paper bridges~this gap by proposing \textbf{\ourmethod}, a \textbf{C}ertifiable framework for \textbf{E}fficient \textbf{L}LM evaluation, which leverages \emph{\textbf{E}-processes} to build anytime-valid CIs. 
Concretely, we propose \emph{a new signal} that combines two ingredients: (i) \textbf{U}ncertainty-guided sampling to select informative samples for evaluation, and (ii) \textbf{S}urrogate-assisted approximations for unevaluated samples.
We prove that the new signal remains unbiased for the evaluation score conditional on the past, enabling statistically-grounded and anytime-valid $e$-process CIs. 
More importantly, the two ingredients reduce estimation variance and help reach the target precision with fewer evaluated samples. We also prove that CIs obtained by \ourmethod~can shrink at a near-parametric rate up to logarithmic factors and analyze the oracle variance-optimal sampling rule that motivates the empirical uncertainty-guided one. Experiments show that \ourmethod~reaches the target precision using $54$-$62\%$ fewer evaluated samples than baselines, while preserving anytime-valid coverage. 
Code is available at \href{https://github.com/zyecs/celeus}{\faGithub\ \texttt{github.com/zyecs/celeus}}.
\end{abstract}

\section{Introduction}
Large language models (LLMs) are increasingly deployed as general-purpose systems for tasks such as question answering \cite{Hendrycks:Burns:Basart:Zou:Mazeika:Song2021,Lin:Hilton:Evans2022}, coding \cite{Chen:Tworek:Jun:Yuan:Kaplan:Edwards2021}, and interactive assistance \cite{Ouyang:Wu:Jiang:Almeida:Wainwright:Mishkin2022,Zheng:Chiang:Sheng:Zhuang:Wu:Zhuang2023}, making reliable evaluation essential for understanding their capabilities, limitations, and readiness for real-world deployment~\cite{Chang:Wang:Wang:Wu:Yang:Zhu2023}.
To meet this need, recent benchmark efforts~\cite{Srivastava:Rastogi:Rao:Shoeb:Abid:Fisch2023, Liang:Bommasani:Lee:Tsipras:Soylu:Yasunaga2022} have sought to assess LLMs systematically and at scale, rather than through a handful of narrow evaluation sets. 
However, scale alone is not enough for reliable deployment: simply collecting more evaluation samples and reporting an estimated performance score does not certify whether the estimate is precise enough for decision-making, leaving no statistically grounded criterion for deciding when enough samples have been evaluated. 

Certifiable evaluation addresses the above issue by sequentially selecting evaluation samples from a fixed pool and updating CIs~\cite{Wang:Chen:Bo:Xu2025}. The updating process will end until certain conditions are satisfied, e.g., the CI width reaches a target precision. A common line of evaluation with CIs considers static settings, where the evaluation samples are fixed in advance and a one-shot CI is reported for model performance~\cite{Angelopoulos:Bates:Fannjiang:Jordan:Zrnic2023,Angelopoulos:Duchi:Zrnic2023}. These guarantees apply to the final interval at a pre-specified endpoint, rather than to the evolving sequence of intervals monitored during evaluation, and therefore do not directly support stopping once the CI width reaches a target precision. Another line of work provide considers a different setting in which data arrive sequentially and CIs are updated as new samples become available~\cite{Zrnic:Candes2024,Csillag:Struchiner:Goedert2025,Kilian:Cortinovis:Caron2025}. In these settings, the evaluator has no access to future samples, and these methods do not support the sample selection required by certifiable evaluation with a fixed pool.

Within the fixed pool setting, existing approaches to certifiable evaluation are not generally anytime-valid: their claimed coverage (e.g., 95\%) may fail to hold when CIs are repeatedly updated and used to decide when to stop, leaving a \emph{gap} between theoretical rigor and practice.
For example, the state-of-the-art method Cer-Eval provides certification for a fixed partition, while its practical algorithm learns and updates the partition using evaluated data~\cite{Wang:Chen:Bo:Xu2025}. Using the same data to choose the partition and construct CIs introduces a dependence that is not addressed by the fixed-partition~analysis.

\textbf{Key Idea.} This paper bridges~this gap by proposing \textbf{\ourmethod}, a certifiable framework for efficient LLM evaluation, which leverages \emph{$e$-processes} to build anytime-valid CIs. Given a fixed data pool, \ourmethod~combines uncertainty-guided sampling, which prioritizes informative samples for evaluation, with surrogate-assisted approximations for unevaluated samples. These strategies reduce the number of costly evaluations needed to reach a target precision. Under such non-uniform selection, however, a naive signal based on the observed risk can be biased for the target performance. This creates the key \emph{obstacle} for $e$-process construction: the signal must remain conditionally unbiased for the evaluation risk as evaluation proceeds.

\textbf{Technical Contributions.} To cope with this obstacle, \ourmethod~builds \emph{inferential signals} for sampling without replacement from a fixed evaluation pool. At each round, the signal combines the observed losses of previously evaluated samples with surrogate scores for the remaining samples. After a new sample is selected and evaluated, it adds an inverse-probability-weighted correction for that sample's surrogate residual, using its actual selection probability. We prove that the resulting signal is conditionally unbiased for the finite-pool risk under history-dependent proposal distributions satisfying a probability floor (see Thm.~\ref{thm:signal_unbias}, Sec.~\ref{sec:theory}). This conditional unbiasedness enables finite-sample, anytime-valid $e$-process CIs through test inversion~\citep{Waudby-Smith:Ramdas2020}. We then lift these finite-pool CIs to population-risk CIs through a correction that depends on the full pool size.

We also justify the effectivenss of \ourmethod~from two theoretical perspectives: 1) We characterize how surrogate scores and sample selection can reduce estimation variance and help reach the target precision with fewer evaluated samples, with an oracle variance-optimal sampling rule motivating our empirical uncertainty-guided sampling strategy (see Thms.~\ref{thm:optimal} and~\ref{thm:conf_sequen}); 2) We prove that \ourmethod's CIs are anytime-valid, covering target risk with probability at least $1-\alpha$ at any stopping time, and that their widths shrink at a near-parametric rate up to logarithmic factors (see Thm.~\ref{thm:coverage_convergence}). 

\textbf{Empirical Contributions.} Extensive experiments validate both the statistical validity and efficiency of \ourmethod. In a diagnostic setting where \cereval{}'s partitioning can lose coverage, \ourmethod~preserves nominal anytime-valid coverage. On real LLM evaluation tasks, \ourmethod~reaches the target precision using $54$--$62\%$ fewer evaluated samples than baselines. Moreover, compared with the no-surrogate version of \ourmethod, leveraging pre-trained surrogate models with only $7$--$8$B parameters, without any task-specific training or fine-tuning, further improves evaluation efficiency by $10$--$23\%$ when evaluating larger models, including $67$--$72$B dense models and an $8\times7$B MoE model.

\section{Certifiable and Efficient Evaluation: Setup and Overview}
In this section, we introduce the evaluation setup and give an overview of certifiable and efficient evaluation based on $e$-processes. We characterize a task by a joint distribution $\bP_{\x,y}$ over input-label pairs $(\x,y)$, where $\x\in\cX$ denotes an input and $y\in\cY$ denotes the corresponding ground-truth label. Given an LLM $f:\cX\to\cY$ fixed before the evaluation pool is drawn and a bounded risk function $\ell:\cY\times\cY\to[L,U]$ for some constants $L,U\in\bR$, we aim to evaluate the LLM's population risk:
\[
R(f;\bP_{\x,y})=\bE_{(\x,y)\sim\bP_{\x,y}}\big[\ell(f(\x),y)\big],
\]
where, for brevity, $R=R(f;\bP_{\x,y})$.

In practice, the distribution $\bP_{\x,y}$ is unknown. Instead, we are given an evaluation pool $D_N=\{(\x_j,y_j)\}_{j=1}^N$, whose input pool is $X_N=\{\x_j\}_{j=1}^N$. For population-risk results, these input-label pairs are \emph{i.i.d.} draws from $\bP_{\x,y}$.\footnote{The i.i.d. assumption is used for the population variance formula, consistency, and the correction from finite-pool to population risk. The finite-pool results condition on the complete pool and do not require i.i.d. sampling; for non-i.i.d. pools, population-risk guarantees require an appropriate distribution-dependent generalization bound~\cite{Kontorovich:Ramanan2008,Mohri:Rostamizadeh2010}.}
The reference labels or evaluation judgments are initially withheld, and each sample has a fixed target loss throughout evaluation. \emph{Certifiable evaluation} sequentially evaluates samples and stops once the CI reaches a target precision. Because evaluating a sample can be costly, especially when human judgments are required, \emph{efficient evaluation} aims to reach this precision with as few evaluated samples as possible.

A natural idea is to \emph{select informative samples for evaluation}, which has been extensively studied in active testing~\cite{Berrada:Kossen:Freddie:Razzak:Gal:Rainforth2025,Kossen:Farquhar:Sebastian:Gal:Rainforth2021}. Rather than evaluating samples uniformly at random, active testing assigns larger selection probabilities to samples expected to be more informative, with the aim of reaching the target precision using fewer evaluations. At each round $t$, the sampling policy selects an unevaluated sample $i_t$ from the pool and queries the information required to determine its target loss $\ell(f(\x_{i_t}),y_{i_t})$. This query may reveal a reference label, obtain a human judgment, or invoke a model judge. We call a sample \emph{evaluated} once its target loss becomes available.\footnote{Our efficiency objective is to reduce the number of ground-truth labels needed to reach the target precision, as acquiring these labels is often costly in practice. Our sampling rules use target-model outputs computed over the full pool, together with surrogate-model outputs when applicable. These preprocessing costs are included in the end-to-end cost decomposition in App.~\ref{app:wallclock}, which also reports online timing with precomputed outputs and scores.}
For offline labeled benchmarks, we simulate this process by revealing labels only for selected samples. After round $t$, the evaluated data are $\{(\x_{i_m},y_{i_m})\}_{m=1}^{t}$, and the remaining index set is
\[
\cJ_t=[N]\setminus\{i_1,\ldots,i_t\},
\qquad
[N]=\{1,\ldots,N\},
\qquad
i_{1:t}=(i_1,\ldots,i_t).
\]

Another effective strategy is to use \emph{surrogate-assisted approximations for unevaluated samples}~\cite{Angelopoulos:Bates:Fannjiang:Jordan:Zrnic2023,Angelopoulos:Duchi:Zrnic2023}. Specifically, a surrogate score $\tilde{\ell}_f(\x_j)\in[L,U]$ approximates the unknown target loss $\ell(f(\x_j),y_j)$ without revealing $y_j$. For example, the score can be produced by a surrogate model that predicts the target loss or by a label-free proxy constructed from available model outputs. The surrogate scores are computed before evaluation and kept fixed throughout the procedure. They allow us to extract information from unevaluated samples: when the scores are close to the target losses, we can obtain an accurate estimate with fewer evaluated samples.

For the analysis, let $\mathcal P_0$ collect the label-free information computed before evaluation, including model outputs and surrogate scores. We write
\begin{equation}\label{eq:analysis_information}
\cF_t=\sigma(D_N,\mathcal P_0,i_{1:t}),
\qquad 0\le t\le N,
\end{equation}
where $\sigma(\cdot)$ denotes the information determined by the listed quantities and $i_{1:0}$ is empty. Thus, $\cF_t$ contains the complete pool, fixed preprocessing information, and the first $t$ selections. Conditioning on the complete pool isolates the randomness from sample selection; the evaluator still sees a label only after selecting its index. The observed information and sampling law are formalized in App.~\ref{app:assumptions}.

Given the selected evaluation samples and surrogate-assisted approximations, the key question is how to construct CIs that remain anytime-valid as evaluation proceeds. Standard static evaluation reports a CI for a fixed set of evaluated samples~\cite{Angelopoulos:Bates:Fannjiang:Jordan:Zrnic2023,Angelopoulos:Duchi:Zrnic2023} and can lose validity when the CI is repeatedly updated and used to decide when to stop. Sequential concentration inequalities and $e$-processes provide tools for inference under continuous monitoring~\cite{Wang:Chen:Bo:Xu2025,Zhao:Zhou:Sabharwal:Ermon2016,Waudby-Smith:Ramdas2020,Waudby-Smith:Ramdas2024}. Applying these tools to our setting requires a signal with the correct conditional mean under non-uniform sampling without replacement. We use $e$-processes and recall their definition and anytime-valid implication below.

\begin{definition}\cite{Ramdas:Wang2025}\label{def:eprocess}
Let $\mathcal H_0$ be a null hypothesis, and let $\{\cF_t\}_{0\le t\le N}$ describe the information used for conditioning. A nonnegative process $\{\cK_t\}_{0\le t\le N}$ is an \emph{$e$-process} for $\mathcal H_0$ if $\cK_0=1$, each $\cK_t$ is determined by $\cF_t$, and $\bE_P[\cK_\tau]\le1$ for every $P\in\mathcal H_0$ and every stopping time $\tau\le N$ with respect to $\{\cF_t\}$. A stopping time is a decision to stop based only on information available by that round. A sufficient construction is a nonnegative supermartingale satisfying $\bE_P[\cK_t\mid\cF_{t-1}]\le\cK_{t-1}$ at each round under every $P\in\mathcal H_0$.
Consequently, for any $\alpha\in(0,1)$, it satisfies the anytime-valid control
\[
\bP_P\left(\max_{0\le t\le N}\cK_t\ge\frac{1}{\alpha}\right)\le\alpha,
\qquad
\forall P\in\mathcal H_0.
\]
Thus, rejecting $\mathcal H_0$ whenever $\cK_t\ge1/\alpha$ gives an anytime-valid test: if $\mathcal H_0$ is true, the probability of falsely rejecting it at any round is at most $\alpha$.
\end{definition}

To turn $e$-processes into CIs, we use test inversion for the data-pool risk $R_N=\frac{1}{N}\sum_{j=1}^{N}\ell(f(\x_j),y_j)$. Conditional on $\cF_0$, this risk is fixed, and the remaining randomness comes from sample selection. For each candidate value $\upsilon$, we test $\cH_0(\upsilon):R_N=\upsilon$ by constructing an $e$-process $\{\cK_t(\upsilon)\}_{0\le t\le N}$. Intuitively, $\cK_t(\upsilon)$ measures how strongly the information collected up to round $t$ contradicts $\upsilon$. If $\cK_t(\upsilon)\ge1/\alpha$, then $\upsilon$ is rejected at that round. The CI keeps the remaining candidate values:
\[
\mathrm{CI}_t=\left\{\upsilon:\cK_t(\upsilon)<1/\alpha\right\}.
\]
Since $R_N$ corresponds to the true null hypothesis $\cH_0(R_N)$, non-coverage can occur only if this true null is falsely rejected:
\[
\bP\left(\exists 0\le t\le N:R_N\notin\mathrm{CI}_t\right)
=
\bP\left(\exists 0\le t\le N:\cK_t(R_N)\ge\frac{1}{\alpha}\right)
\le\alpha.
\]
The resulting CIs therefore cover $R_N$ simultaneously over all rounds with probability at least $1-\alpha$, allowing evaluation to stop once the target precision is reached.

For certifiable and efficient LLM evaluation, the key technical step is to construct valid inferential signals from the evaluated samples and surrogate approximations for unevaluated samples. These signals capture the new information obtained at each round and are designed to be conditionally unbiased for the data-pool risk, which allows us to construct an $e$-process satisfying the supermartingale condition above~\cite[Proposition 2]{Waudby-Smith:Ramdas2024}. Their variance also provides a natural criterion for sample selection, guiding the evaluator toward samples with informative surrogate-correction residuals that help the CI reach the target precision with fewer evaluations. We then lift this CI to the population risk $R$ using a finite-pool-to-population correction.

\section{The \ourmethod~Evaluation Method}
We now instantiate the overview above by giving the concrete construction of \ourmethod, a sequential evaluation procedure that yields \emph{anytime-valid and shrinking CIs} for $R$, defined as:
\begin{definition}
Fix a significance level $\alpha\in(0,1)$. For each pool size $N$, a sequence of random sets $\{C_{t,N}\}_{1\le t\le N}$ is called a sequence of $(1-\alpha)$ anytime-valid CIs for $R$ if
\[
\Pr\left(\exists\,1\le t\le N:R\notin C_{t,N}\right)\le\alpha.
\]
The family $\{C_{t,N}:1\le t\le N,\ N\ge1\}$ is additionally called shrinking if, for every sequence $N\to\infty$ and $1\le t_N\le N$ satisfying $t_N\to\infty$, $\operatorname{width}(C_{t_N,N})
\xrightarrow{\mathbb P}0.$
\end{definition}
% \begin{definition}\label{def:conf_sequen}
% Fix a significance level $\alpha\in(0,1)$. A sequence of random sets $\{C_t\}_{t\ge1}$ is called a sequence of \emph{$(1-\alpha)$ anytime-valid and shrinking CIs} for $R$ if, for every distribution $\bP_{\x,y}$,
% \[
% \bP_{\x,y}\!\left(\exists 1\le t\le N:\ R\notin C_t\right)\le \alpha,
% \]
% and its width $\mathrm{width}(C_t)=\sup C_t-\inf C_t$ satisfies $\mathrm{width}(C_t)\to 0$ as $t\to\infty$.
% \end{definition}
% The shrinkage condition means that, as more samples are evaluated, $C_t$ becomes increasingly informative: its width decreases toward zero, progressively reducing the uncertainty about $R$; hence, the target precision can be guaranteed to be achieved with enough samples.
% [ZJ] page limit
% The shrinkage condition concerns joint growth of the pool size $N$ and the number of evaluated samples $t$. It means that the CI width converges to zero in probability in this regime. For a fixed finite pool, evaluating more samples does not reduce the population correction $\Delta_N(\alpha_2)$, so the pool may be exhausted before the requested precision is reached.
The shrinkage condition concerns joint growth of the pool size $N$ and the number of evaluated samples $t$. For a fixed finite pool, evaluating more samples does not reduce the population correction $\Delta_N(\alpha_2)$, so the pool may be exhausted before the requested precision is reached.

% [ZJ] page limit
Constructing such CIs for $R$ involves two conceptually different sources of randomness. First, we only observe risks for a selected subset of the evaluation pool, which creates randomness in estimating the \emph{pool risk} $R_N$. Second, even if the entire pool were evaluated, $R_N$ itself may differ from $R$ because the pool is a finite i.i.d.\ sample from the population distribution. To handle these two sources of randomness separately, we split the overall error budget $\alpha$ into $\alpha_1,\alpha_2\in(0,1)$ with $\alpha_1+\alpha_2=\alpha$. Here, $\alpha_1$ controls anytime-valid inference for $R_N$ through test inversion of $e$-processes, while $\alpha_2$ controls the finite-pool-to-population correction from $R_N$ to $R$.
% Constructing such CIs for $R$ involves two conceptually different sources of randomness. First, we only observe risks for a selected subset of the evaluation pool, which creates randomness in estimating the \emph{pool risk} $R_N$. Second, even if the entire pool were evaluated, $R_N$ itself may differ from $R$ because the pool is a finite i.i.d.\ sample from the population distribution. To handle these two sources of randomness separately, we split the overall error budget $\alpha$ into $\alpha_1,\alpha_2\in(0,1)$ with $\alpha_1+\alpha_2=\alpha$.

\textbf{Anytime-valid CIs for $R_N$ at level $\alpha_1$ via $e$-process.} At the beginning of round $t$, after the first $t-1$ evaluations, the evaluator has access to two sources of information: the observed risks on the evaluated samples,
$\{\ell(f(\x_{i_m}),y_{i_m})\}_{m=1}^{t-1}$, and the surrogate risk scores on the remaining unevaluated samples,
$\{\tilde{\ell}_f(\x_j)\}_{j\in\cJ_{t-1}}$. 
% Based on this information, the sampling policy assigns selection probabilities to the remaining unevaluated samples; we denote by $q_t(i_t\mid i_{1:t-1},D_N)$ the conditional probability of selecting index $i_t$ from $\cJ_{t-1}$ at round $t$, given the previously selected indices $i_{1:t-1}$ and the data pool $D_N$. 
% Based on this information, the sampling policy assigns selection probabilities to the remaining unevaluated samples; we denote by $q_t(i_t\mid i_{1:t-1},D_N)$ the probability of selecting index $i_t$ from $\cJ_{t-1}$ at round $t$. This notation also includes dependence on the fixed preprocessing information and the labels revealed in previous rounds, which we suppress for brevity.
Based on this information, the sampling policy assigns selection probabilities to the remaining unevaluated samples; we denote by $q_t(i_t\mid i_{1:t-1},X_N)$ the probability of selecting index $i_t$ from $\cJ_{t-1}$ at round $t$. This notation also includes dependence on the fixed preprocessing information and the labels revealed in previous rounds, which we suppress for brevity.
After selecting $i_t$ according to this selection probability and observing its risk, we first derive the per-round \emph{inferential signal} for constructing $e$-processes as
\begin{equation}\label{eq:signal_estimator}
\hat{S}_t
=
\frac{
\sum_{m=1}^{t-1}\ell\left(f(\x_{i_m}),y_{i_m}\right)
+
\sum_{j\in\cJ_{t-1}}\tilde{\ell}_f(\x_j)
}{N}
+
\frac{
\ell\left(f(\x_{i_t}),y_{i_t}\right)
-
\tilde{\ell}_f(\x_{i_t})
}{
% Nq_t(i_t\mid i_{1:t-1},D_N)
Nq_t(i_t\mid i_{1:t-1},X_N)
}.
\end{equation}
The first term in $\hat S_t$ is a surrogate-completed estimate of the data-pool risk: it uses the observed risks for the previously evaluated samples and fills in the remaining unevaluated samples with their surrogate risk scores. The second term is an inverse-probability-weighted correction: the selected sample reveals the surrogate residual $\ell(f(\x_{i_t}),y_{i_t})-\tilde{\ell}_f(\x_{i_t})$, and the factor $1/q_t(i_t\mid i_{1:t-1},X_N)$ turns this residual into an unbiased estimate of the normalized total residual over the remaining pool, as in a conditional Horvitz--Thompson-type estimator~\cite{Horvitz:Thompson1952}. Hence, conditional on the complete pool, preprocessing information, and past selections, $\hat S_t$ is unbiased for the finite-pool risk $R_N$, as shown in Thm.~\ref{thm:signal_unbias}. 
% [ZJ] page limit: set this paragraph one line shorter if possible (text unchanged)
\looseness=-1

% [BjfK-W2]
In Rem.~\ref{rem:0-1}, we propose \emph{uncertainty}-guided sampling strategies that aim to reduce the remaining statistical \emph{uncertainty} in the per-round signal (i.e., the conditional variance of signal $\hat S_t$). This quantity appears explicitly in the width analysis of the $e$-process confidence interval: lower conditional variance yields signals that concentrate more tightly around the target risk, thereby providing stronger sequential evidence against incorrect candidate values and helping the CI reach the target precision with fewer evaluated samples. In particular, the \emph{uncertainty}-guided sampling strategies assign larger probabilities to samples with larger discrepancies between the observed risk and the surrogate risk. Formal details are provided in Thms.~\ref{thm:optimal} and~\ref{thm:conf_sequen} in Sec.~\ref{sec:theory}.

% [ZJ] page limit
% In Rem.~\ref{rem:0-1}, we propose \emph{uncertainty}-guided sampling strategies that aim to reduce the remaining statistical \emph{uncertainty} in the per-round signal (i.e., the conditional variance of signal $\hat S_t$). This quantity appears explicitly in the width analysis of the $e$-process confidence interval: lower conditional variance yields signals that concentrate more tightly around the target risk, thereby providing stronger sequential evidence against incorrect candidate values and helping the CI reach the target precision with fewer evaluated samples. In particular, the \emph{uncertainty}-guided sampling strategies assign larger probabilities to samples with larger discrepancies between the observed risk and the surrogate risk. Sec.~\ref{sec:theory} formalizes this intuition.
% In Rem.~\ref{rem:0-1}, we propose \emph{uncertainty}-guided sampling strategies that aim to reduce the remaining statistical \emph{uncertainty} in the per-round signal (i.e., the conditional variance of signal $\hat S_t$). In particular, the \emph{uncertainty}-guided sampling strategies assign larger probabilities to samples with larger discrepancies between the observed risk and the surrogate risk. Sec.~\ref{sec:theory} formalizes this intuition.
% \bo{minor: is a bit far for theorem 4 and 5 here. hmm... maybe just mention intuition here and mention formal detail in thm 4 and 5 in section 5?}
% ZJ: I have revised this part to present the intuition here and moved the formal details to Thms.~\ref{thm:optimal} and~\ref{thm:conf_sequen} in Sec.~5.

We now construct CIs for $R_N$ from the sequentially updated signals $\{\hat S_t\}_{1\le t\le N}$ using $e$-processes. Following standard $e$-process constructions~\cite{Waudby-Smith:Ramdas2024}, we map the signal to $[0,1]$ by setting $\bar S_t=(\hat S_t-a_t)/(b_t-a_t)$, where $a_t$ and $b_t$ are predictable\footnote{A sequence $(s_t)_{t\geq 1}$ is predictable if $s_t$ can only depend on the information available up to time $t-1$. The bounds $a_t$ and $b_t$ can also be chosen as constants, as long as they satisfy $a_t\le \hat S_t\le b_t$. See Rem.~\ref{rem:scaling} for further discussion.} lower and upper bounds of $\hat S_t$ given the information available before round $t$. Accordingly, for each candidate value $\upsilon$ of $R_N$, we define its scaled counterpart as $\bar\upsilon_t=(\upsilon-a_t)/(b_t-a_t)$. Since this is an affine transformation and $\hat S_t$ is conditionally unbiased for $R_N$ under the null hypothesis $\mathcal H_0(\upsilon): R_N=\upsilon$, we have
% $\mathbb E[\bar S_t\mid \mathcal F_{t-1}, \mathcal H_0(\upsilon)]=\bar\upsilon_t$. 
$\mathbb E[\bar S_t\mid \cF_{t-1}]=\bar\upsilon_t$.
% [ZJ] page limit
This leads to a centered sequential signal: for each candidate $\upsilon$, the sequence $\bar S_t-\bar\upsilon_t$ has conditional mean zero for all $t$ under $\mathcal H_0(\upsilon)$. This ensures that, when the candidate value is correct, the signal has no systematic drift against it, making it a valid sequential signal for testing $\mathcal H_0(\upsilon)$ on scaled space.
% This leads to a centered sequential signal: for each candidate $\upsilon$, the sequence $\bar S_t-\bar\upsilon_t$ has conditional mean zero for all $t$ under $\mathcal H_0(\upsilon)$.

To implement a sequential test for $\mathcal H_0(\upsilon)$ at each candidate value $\upsilon$, we use a \emph{hedged capital process}~\cite[Section 4.4]{Waudby-Smith:Ramdas2024}. For any $\theta\in(0,1)$ and candidate $\upsilon$, define
{
\setlength{\abovedisplayskip}{2pt}
\setlength{\belowdisplayskip}{-2pt}
\setlength{\abovedisplayshortskip}{2pt}
\setlength{\belowdisplayshortskip}{-2pt}
\begin{equation}\label{eq:hedhedCP}
\mathcal K_t^{\pm}(\upsilon)
=
\theta\prod_{i=1}^t\left(1+\lambda_i^+(\upsilon)\left(\bar S_i-\bar\upsilon_i\right)\right)
+
(1-\theta)\prod_{i=1}^t\left(1-\lambda_i^-(\upsilon
)\left(\bar S_i-\bar\upsilon_i\right)\right).
\end{equation}
}\par\noindent
Here, $\lambda_i^+(\upsilon)$ and $\lambda_i^-(\upsilon)$ are predictable fractions in $[0,1/\bar\upsilon_i)$ and $[0,1/(1-\bar\upsilon_i))$, respectively; see Remark~\ref{rem:betting-fractions} for implementation details. The resulting $(1-\alpha_1)$ confidence set for $R_N$ at time $t$ is\footnote{The experiments use grid-point inversion as described in App.~\ref{app:implementation}. App.~\ref{app:cell_inversion_discussion} gives a conservative alternative that returns an interval containing the exact confidence set and preserves its coverage guarantee. }
{
\setlength{\abovedisplayskip}{2pt}
\setlength{\belowdisplayskip}{-2pt}
\setlength{\abovedisplayshortskip}{2pt}
\setlength{\belowdisplayshortskip}{-2pt}
\begin{equation}\label{eq:M_t}
\mathcal M_t^{\pm}
=
\bigcap_{i\le t}
\left\{
\upsilon\in[L,U]:
\mathcal K_i^{\pm}(\upsilon)<1/\alpha_1
\right\}\ .
\end{equation}
}\par\noindent
\textbf{Finite-pool-to-population correction at level $\alpha_2$.} 
We next transfer the CI guarantee from the pool risk $R_N$ to the population risk $R$ by controlling the generalization gap between them at level $\alpha_2$. Specifically, by Hoeffding's inequality~\cite{Hoeffding1963}, with probability at least $1-\alpha_2$,
{
\setlength{\abovedisplayskip}{2pt}
\setlength{\belowdisplayskip}{0pt}
\setlength{\abovedisplayshortskip}{2pt}
\setlength{\belowdisplayshortskip}{0pt}
\[
|R_N-R|\le \Delta_N(\alpha_2)=(U-L)\sqrt{\frac{\log(2/\alpha_2)}{2N}}.
\]
}\par\noindent
Notably, this correction depends on the full data-pool size $N$, rather than the number of evaluated samples, and can therefore be small or even negligible for large pools.

% [ZJ] page limit
% Considering both (i) $R_N\in\mathcal M_t^{\pm}$ for all $1\le t\le N$ and (ii) $|R_N-R|\le \Delta_N(\alpha_2)$ hold
Considering both (i) $R_N\in\mathcal M_t^{\pm}$ for all $1\le t\le N$ and (ii) $|R_N-R|\le \Delta_N(\alpha_2)$ hold, we have for all $t$ that $R\in \mathcal M_t^{\pm}\oplus[-\Delta_N(\alpha_2),\Delta_N(\alpha_2)]$,
% [ZJ] page limit
% \[
% \left\{\forall 1\le t\le N:\ R_N\in\mathcal M_t^{\pm}\right\}\cap\left\{|R_N-R|\le \Delta_N(\alpha_2)\right\},
% \]
% we have for all $t$ that
% \[
% R\in \mathcal M_t^{\pm}\oplus[-\Delta_N(\alpha_2),\Delta_N(\alpha_2)],
% \]
where $\oplus$ denotes the Minkowski sum\footnote{For sets $A,B\subseteq\bR$, the Minkowski sum is $A\oplus B=\{a+b:\ a\in A,\ b\in B\}$.}. Therefore, we define the CI for $R$ at time $t$ as
\[
C_t=
\left(\mathcal M_t^{\pm}\oplus[-\Delta_N(\alpha_2),\Delta_N(\alpha_2)]\right)\cap[L,U]\ .
% \]
\]

% [ZJ] page limit
% \textbf{Numerical inversion.} App.~\ref{app:cell_inversion_discussion} gives a conservative finite-grid construction that preserves coverage.
% A grid cell is the interval between two adjacent grid points. The construction rejects a cell only when a bound valid throughout that cell certifies that every candidate in it is rejected, and returns the outside edges of the remaining cells.
% The resulting interval $I_t$ contains $\mathcal M_t^{\pm}$ at every round by \eqref{eq:cell_enclosure}. Applying the same population correction to $I_t$ therefore preserves the $(1-\alpha)$ anytime-valid coverage guarantee under the stated assumptions.
% App.~\ref{app:cell_inversion_discussion} also specifies how to account for rounding and empty-cell cases. The numerical interval may be wider than the exact set, so stopping must use a certified upper bound on its full width; the exact-inversion contraction rate does not automatically transfer to this enclosure.

\textbf{Performance Estimation at Stopping.}
% Given a target precision $\epsilon\in(0,1)$, the procedure stops at the first time $t$ such that $\mathrm{width}(C_t)\le \epsilon$, and returns $C_t$ as the final CI, which covers the true performance with probability at least $1-\alpha$. Besides the CI, \ourmethod~also reports a final performance estimate by aggregating the per-round inferential signals (see Appendix~\ref{appendix:signal_to_estimator} for details):
% Given a target precision $\epsilon\in(0,1)$, the procedure stops when $\mathrm{width}(C_t)\le\epsilon$ or the pool is exhausted. With exact test inversion and the stated assumptions, the returned CI covers $R$ with probability at least $1-\alpha$ under either stopping rule. If the pool is exhausted before the width criterion is met, the procedure reports that the target precision was not reached. Alongside the CI, \ourmethod~reports the following weighted surrogate-residual estimate (see Appendix~\ref{appendix:signal_to_estimator}):
Given a target precision $\epsilon\in(0,1)$, the procedure stops when $\mathrm{width}(C_t)\le\epsilon$ or the pool is exhausted. With exact test inversion or the conservative numerical construction in App.~\ref{app:cell_inversion_discussion}, and under the stated assumptions, the returned CI covers $R$ with probability at least $1-\alpha$ under either stopping rule. If the pool is exhausted before the width criterion is met, the procedure reports that the target precision was not reached. Alongside the CI, \ourmethod~reports the following weighted surrogate-residual estimate (see Appendix~\ref{appendix:signal_to_estimator}):

\begin{equation}\label{eq:lure_estimator}
\hat{R}_t
=
\frac{1}{N}\sum_{n=1}^{N}\tilde{\ell}_f\left(\x_n\right)
+
\frac{1}{t}\sum_{m=1}^{t}u_{m,t}
\left(
\ell\left(f(\x_{i_m}),y_{i_m}\right)-\tilde{\ell}_f(\x_{i_m})
\right),
\end{equation}
% where $u_{m,t}$, motivated by the LURE estimator~\cite{Farquhar:Gal:Rainforth2021}, adjusts the surrogate residual for sequential, non-uniform sample selection up to time $t$:
For $1\le m\le t<N$, the weight $u_{m,t}$, motivated by the LURE estimator~\cite{Farquhar:Gal:Rainforth2021}, adjusts the surrogate residual for sequential, non-uniform sample selection:
\[
u_{m,t}
=
1+\frac{N-t}{N-m}
\left(
% \frac{1}{(N-m+1)q_m(i_m\mid i_{1:m-1},D_N)}-1
\frac{1}{(N-m+1)q_m(i_m\mid i_{1:m-1},X_N)}-1
\right)\ .
\]
The first term in $\hat R_t$ provides a low-cost approximation to the evaluation performance by averaging the surrogate risk scores $\tilde{\ell}_f(\x_n)$ over all samples. The second term corrects this approximation using the observed surrogate residuals from the evaluated samples, yielding a fixed-time unbiased estimator. At pool exhaustion, set $u_{m,N}=1$ for all $m$, so that $\hat R_N=R_N$; the preceding weight formula is used only for $t<N$. Fixed-time unbiasedness does not imply unbiasedness at a data-dependent stopping time. The CI, rather than this scalar estimate, supplies the stopping-time coverage guarantee.

\textbf{Overall procedure.} Alg.~\ref{alg:celeus-main} summarizes the procedure of \ourmethod. The complete end-to-end implementation, including acquisition-score normalization, probability-floor mixing, sequential $e$-process updates, numerical test inversion, and precision-based stopping, is given in Alg.~\ref{alg:celeus} ( App.~\ref{app:implementation}).

% [BjfK-W2]
{
\begin{algorithm}[t]
\caption{Overall procedure of \ourmethod}
\label{alg:celeus-main}
\small
\begin{algorithmic}[1]
\REQUIRE Input pool $X_N=\{\x_1,\ldots,\x_N\}$, surrogate scores $\{\tilde{\ell}_f(\x_j)\}_{j=1}^N$, error split $\alpha_1+\alpha_2=\alpha$, parameters $\beta$ and $\theta$, and target precision $\epsilon$
\ENSURE Stopping time $\tau$, CI $C_\tau$, and performance estimate $\hat R_\tau$
\STATE $\cJ_0\gets\{1,\ldots,N\}$ and $\tau\gets N$
\FOR{$t=1,\ldots,N$}
    \STATE Compute an acquisition score for every $j\in\cJ_{t-1}$ using only information available before round $t$
    % [BjfK-W2]
    % \STATE Convert the scores into the sampling probabilities $q_t(j\mid i_{1:t-1},X_N)\ge\beta/(N-t+1)$, as in Assumption~\ref{ass:beta}
    \STATE Convert the scores into the sampling probabilities $q_t(j\mid i_{1:t-1},X_N)\ge\beta/(N-t+1)$ (Assumption~\ref{ass:beta})
    \STATE Draw $i_t\sim q_t(\cdot\mid i_{1:t-1},X_N)$ over $\cJ_{t-1}$, evaluate sample $i_t$, and observe $\ell(f(\x_{i_t}),y_{i_t})$
    \STATE Compute $\hat S_t$ in \eqref{eq:signal_estimator} and $\bar S_t=(\hat S_t-a_t)/(b_t-a_t)$
    \STATE Update $\mathcal K_t^{\pm}(\upsilon)$ in \eqref{eq:hedhedCP} for every candidate $\upsilon\in[L,U]$, and $\mathcal M_t^{\pm}$ in \eqref{eq:M_t}
    \STATE $C_t\gets\left(\mathcal M_t^{\pm}\oplus[-\Delta_N(\alpha_2),\Delta_N(\alpha_2)]\right)\cap[L,U]$ and $\cJ_t\gets\cJ_{t-1}\setminus\{i_t\}$
    \IF{$\mathrm{width}(C_t)\le\epsilon$}
        \STATE $\tau\gets t$ and \textbf{break}
    \ENDIF
\ENDFOR
\STATE \textbf{return} $(\tau,C_\tau,\hat R_\tau)$, with $\hat R_\tau$ in \eqref{eq:lure_estimator}
\end{algorithmic}
\end{algorithm}
}

\section{Theoretical Analysis}\label{sec:theory}
% Proofs and assumptions are deferred to App.~\ref{app:proof}. Assumption~\ref{ass:loss_bound} only requires bounded evaluation scores, which can usually be ensured by normalization or clipping. All other assumptions are algorithmic conditions satisfied by \ourmethod~and impose no additional restrictions on the application. 
Proofs and assumptions are deferred to App.~\ref{app:proof}. Assumption~\ref{ass:loss_bound} only requires bounded evaluation scores, which can usually be ensured by normalization or clipping. The population-risk results also use the i.i.d. input-label pairs and fixed target model specified in the evaluation setup. The remaining assumptions are algorithmic conditions satisfied by \ourmethod.

\textbf{Risk estimator.} We first establish the fixed-time unbiasedness and consistency of the risk estimator.
% \vspace{-0.3em}
\begin{theorem}\label{thm:lure_unbias}
% Under Assumption~\ref{ass:loss_bound}-\ref{ass:beta}, the $\hat R_t$ has $\bE[\hat R_t]=R$ with $\mathrm{Var}[\hat R_t]
% Under the population sampling setup and Assumptions~\ref{ass:loss_bound}-\ref{ass:beta}, the $\hat R_t$ has $\bE[\hat R_t]=R$ with $\mathrm{Var}[\hat R_t]
Under Assumptions~\ref{ass:loss_bound}-\ref{ass:beta}, for each fixed $1\le t<N$, $\hat R_t$ has $\bE[\hat R_t]=R$ with $\mathrm{Var}[\hat R_t]
=\mathrm{Var}(\ell(f(\x),y))/N
+
\bE\left[
\mathrm{Var}\left[
% [ZJ] notation: c_m, a_m -> w_m, \psi_m (c is the betting constant, a_t is the scaling bound)
% \sum_{m=1}^{t}c_m a_m/t
\sum_{m=1}^{t}w_m \psi_m/t
% \mid D_N,\tilde{\ell}_f
\mid\cF_0
\right]
\right]$,
% where $c_m$ and $a_m$ are defined in \eqref{eq:cm} and \eqref{eq:am}.
% [ZJ] notation: c_m, a_m -> w_m, \psi_m (c is the betting constant, a_t is the scaling bound)
% where $c_m$ and $a_m$ are defined in \eqref{eq:cm} and \eqref{eq:am}. At $t=N$, $\hat R_N=R_N$ has expectation $R$ and variance $\mathrm{Var}(\ell(f(\x),y))/N$.
where $w_m$ and $\psi_m$ are defined in \eqref{eq:cm} and \eqref{eq:am}. At $t=N$, $\hat R_N=R_N$ has expectation $R$ and variance $\mathrm{Var}(\ell(f(\x),y))/N$.
\end{theorem}

% [ZJ] page limit
The next theorem further shows that the estimator is consistent.

\begin{theorem}\label{thm:consistency}
% Under Assumption~\ref{ass:loss_bound}-\ref{ass:beta}, for any sequence $\{t_N\}_{N\ge 1}$ satisfying $1\le t_N\le N$, $N\to\infty$, and $(N-t_N+1)/N\to 0$, we have $\lim_{N\to\infty}\bE\left[\left(\hat R_{t_N}-R\right)^2\right]=0$, i.e., $\hat R_{t_N}\to R$ in $L^2$.
Under Assumptions~\ref{ass:loss_bound}-\ref{ass:beta}, for any sequence $\{t_N\}_{N\ge 1}$ satisfying $1\le t_N\le N$, $N\to\infty$, and $(N-t_N+1)/N\to 0$, we have $\lim_{N\to\infty}\bE\left[\left(\hat R_{t_N}-R\right)^2\right]=0$, i.e., $\hat R_{t_N}\to R$ in $L^2$.
\end{theorem}

\textbf{Conditionally unbiased signal.} We next study the signals used to construct $e$-processes. The key property is conditional unbiasedness, which enables anytime-valid CIs under sample selection.
\begin{theorem}\label{thm:signal_unbias}
\vspace{-0.1em}
Under Assumption~\ref{ass:loss_bound}-\ref{ass:beta}, the $\hat S_t$ satisfies $\bE[\hat S_t \mid \cF_{t-1}] = R_N$ with $\mathrm{Var}[\hat S_t \mid \cF_{t-1}]=$
\begin{align*}
% \sum_{j\in \cJ_{t-1}}\frac{(\ell(f(\x_j),y_j)-\tilde{\ell}_f(\x_j))^2}{N^2q_t(j\mid i_{1:t-1},D_N)}-\left(\frac{\sum_{j\in \cJ_{t-1}} \ell(f(\x_j),y_j)-\tilde{\ell}_f(\x_j)}{N}\right)^2.
% [ZJ] notation: parentheses around the summed residual
% \sum_{j\in \cJ_{t-1}}\frac{(\ell(f(\x_j),y_j)-\tilde{\ell}_f(\x_j))^2}{N^2q_t(j\mid i_{1:t-1},X_N)}-\left(\frac{\sum_{j\in \cJ_{t-1}} \ell(f(\x_j),y_j)-\tilde{\ell}_f(\x_j)}{N}\right)^2.
\sum_{j\in \cJ_{t-1}}\frac{(\ell(f(\x_j),y_j)-\tilde{\ell}_f(\x_j))^2}{N^2q_t(j\mid i_{1:t-1},X_N)}-\left(\frac{\sum_{j\in \cJ_{t-1}} (\ell(f(\x_j),y_j)-\tilde{\ell}_f(\x_j))}{N}\right)^2.
\end{align*}
\end{theorem}

\textbf{Variance-driven efficiency.} The variance expression above also indicates how the sampling distribution should be designed. Since the second term is fixed, $q_t$ only affects the first term, where samples with larger discrepancies between the observed risk and the surrogate risk score contribute more to the conditional variance and should therefore be sampled with higher probability. The following theorem formalizes this intuition by characterizing the variance-optimal sampling distribution.

\begin{theorem}\label{thm:optimal}
\vspace{-0.1em}
Under Assumptions~\ref{ass:loss_bound}-\ref{ass:beta}, fix a round $t$. If not all $\ell(f(\x_j),y_j)-\tilde{\ell}_f(\x_j), j\in \cJ_{t-1},$
are zero, the following proposal distribution minimizes the conditional variances of $\hat R_t$ and $\hat S_t$:
\footnote{This is stepwise optimality in round $t$. Optimizing the cumulative conditional variance term in Thm.~\ref{thm:conf_sequen} is a trajectory-level adaptive stochastic optimization problem; related adaptive optimization problems can be NP-hard~\citep{Golovin:Krause2011}.}
\[
q_t^*(j\mid i_{1:t-1},D_N)
=
\max\left\{
\beta/(N-t+1),
% [ZJ] notation: \tau_t -> \eta_t (\tau denotes the stopping time)
% \left|\ell(f(\x_j),y_j)-\tilde{\ell}_f(\x_j)\right|/\tau_t
\left|\ell(f(\x_j),y_j)-\tilde{\ell}_f(\x_j)\right|/\eta_t
\right\}.
\]
% Here, $\beta\in(0,1]$ controls the minimum sampling probability as specified in Assumption~\ref{ass:beta}, and $\tau_t>0$ is chosen so that $\sum_{j\in\cJ_{t-1}}q_t^*(j\mid i_{1:t-1},D_N)=1$. In contrast, if all $\ell(f(\x_j),y_j)-\tilde{\ell}_f(\x_j)$
% [ZJ] notation: \tau_t -> \eta_t (\tau denotes the stopping time)
% The oracle uses the complete labeled pool $D_N$. Here, $\beta\in(0,1]$ controls the minimum sampling probability as specified in Assumption~\ref{ass:beta}, and $\tau_t>0$ is chosen so that $\sum_{j\in\cJ_{t-1}}q_t^*(j\mid i_{1:t-1},D_N)=1$. In contrast, if all $\ell(f(\x_j),y_j)-\tilde{\ell}_f(\x_j)$
The oracle uses the complete labeled pool $D_N$. Here, $\beta\in(0,1]$ controls the minimum sampling probability as specified in Assumption~\ref{ass:beta}, and $\eta_t>0$ is chosen so that $\sum_{j\in\cJ_{t-1}}q_t^*(j\mid i_{1:t-1},D_N)=1$. In contrast, if all $\ell(f(\x_j),y_j)-\tilde{\ell}_f(\x_j)$
are zero, then every proposal distribution $q_t$ satisfying Assumption~\ref{ass:beta} is optimal.
\end{theorem}

% \begin{theorem}\label{thm:optimal}
% \vspace{-0.1em}
% Under Assumptions~\ref{ass:loss_bound}-\ref{ass:beta}, fix a round $t$. If not all $\ell(f(\x_j),y_j)-\tilde{\ell}_f(\x_j), j\in \cJ_{t-1},$
% are zero, the following proposal distribution minimizes the variance of $\hat R_t$ and $\hat S_t$:
% \footnote{This is stepwise optimality in round $t$. Optimizing the cumulative conditional variance term in Thm.~\ref{thm:conf_sequen} is a trajectory-level adaptive stochastic optimization problem; related adaptive optimization problems can be NP-hard~\citep{Golovin:Krause2011}.}
% {
% \setlength{\abovedisplayskip}{2pt}
% \setlength{\belowdisplayskip}{-8pt}
% \setlength{\abovedisplayshortskip}{2pt}
% \setlength{\belowdisplayshortskip}{-8pt}
% \[
% q_t^*(j\mid i_{1:t-1},D_N)
% =
% \left|\ell(f(\x_j),y_j)-\tilde{\ell}_f(\x_j)\right|/
% \sum_{s\in \cJ_{t-1}}\left|\ell(f(\x_s),y_s)-\tilde{\ell}_f(\x_s)\right|\ .
% \]
% }\par\noindent
% In contrast, if all $\ell(f(\x_j),y_j)-\tilde{\ell}_f(\x_j)$
% are zero, then every proposal distribution $q_t$ is optimal.
% \vspace{-0.2em}
% \end{theorem}

We next translate this variance reduction into a guarantee on efficiency. The following theorem shows that the CIs provide anytime-valid coverage for $R_N$ and shrink as more samples are evaluated.
\begin{theorem}\label{thm:conf_sequen}
Fix $\alpha_1\in(0,1)$ and $c\in(0,1)$. Under Assumptions~\ref{ass:loss_bound}-\ref{ass:betting}, for every $N\in\mathbb Z_+$, the sequence $\left(\cM_t^\pm\right)_{1\le t\le N}$, defined in \eqref{eq:M_t}, is a $(1-\alpha_1)$ anytime-valid confidence sequence for $R_N$. Moreover, for any $\omega\in(0,1)$, there exists a constant $G_6>0$ such that, with probability at least $1-\omega$, the following bound holds simultaneously for all $t_0\le t\le N$:
\[
\mathrm{width}\left(\cM_t^\pm\right)\le
G_6\sqrt{
\frac{\log t\log\log t}{t}
\left(
\log\log t+
\sum_{i=1}^t
\frac{\mathrm{Var}\left(\hat S_i\mid\cF_{i-1}\right)}
{i\log(i+1)}
\right)
}.
\]
Here, $t_0$ is defined in \eqref{eq:t_0_define}. Writing $\cM_{t,N}^\pm$ to make the dependence on $N$ explicit, for any integer sequence $\{t_N\}_{N\ge1}$ satisfying $1\le t_N\le N$ and $t_N\to\infty$ as $N\to\infty$,
\[
\mathrm{width}\left(\cM_{t_N,N}^\pm\right)
=
O_{\mathbb P}\left(
\log\log t_N\sqrt{\frac{\log t_N}{t_N}}
\right),
\]
and hence $\mathrm{width}\left(\cM_{t_N,N}^\pm\right)
\xrightarrow{\mathbb P}0
$ as $N\to\infty.$ Thus, the confidence sequences are shrinking in the joint asymptotic regime $N\to\infty$, $t=t_N\to\infty$, and $t_N\le N$.
\end{theorem}

% \begin{theorem}\label{thm:conf_sequen}
% Fix $\alpha_1\in(0,1)$ and $c\in(0,1)$. Under Assumptions~\ref{ass:loss_bound}-\ref{ass:betting}, the sequence $\left(\cM_{t}^\pm\right)_{1\le t\le N}$, defined in \eqref{eq:M_t}, is a sequence of $(1-\alpha_1)$ anytime-valid, shrinking CIs for $R_N$. Moreover,
% {
% \setlength{\abovedisplayskip}{2pt}
% \setlength{\belowdisplayskip}{-8pt}
% \setlength{\abovedisplayshortskip}{2pt}
% \setlength{\belowdisplayshortskip}{-8pt}
% \[
% \mathrm{width}\left(\cM_t^\pm\right)\le
% G_6\sqrt{
% \frac{\log t\log\log t}{t}
% \left(
% \log\log t+
% \sum_{i=1}^t
% \frac{\mathrm{Var}\left(\hat S_i\mid \cF_{i-1}\right)}{i\log(i+1)}
% \right)
% },
% \quad \exists G_6>0,\ \forall t\ge t_0,
% \]
% }\par\noindent
% where $t_0$ is defined in \eqref{eq:t_0_define}. It follows that $\mathrm{width}\left(\cM_t^\pm\right)=O\left(\log\log t\sqrt{\log t/t}\right)$.
% \vspace{-0.2em}
% \end{theorem}

The bound shows that the CI width is driven by the cumulative conditional variance of the per-round signals. By Theorem~\ref{thm:signal_unbias}, surrogate and acquisition quality affect this variance through different components: the surrogate score determines the residuals $\ell(f(\x_j),y_j)-\tilde{\ell}_f(\x_j)$, whereas the acquisition proxy determines how closely the resulting proposal distribution approaches the constrained variance-optimal proposal in Theorem~\ref{thm:optimal}. Smaller surrogate residuals and a smaller mismatch from the variance-optimal proposal tighten the width bound and can reduce the number of evaluations required to reach the target precision. These approximation choices affect efficiency but not anytime-valid coverage under the assumptions. A detailed theoretical and empirical analysis is provided in App.~\ref{app:surrogate_acquisition_quality}.

% The bound shows that the CI width is driven by the cumulative conditional variance of the per-round signals. By Thm.~\ref{thm:signal_unbias}, this variance is mainly affected by the surrogate residuals $\ell(f(\x),y)-\tilde{\ell}_f(\x)$ and the sampling distribution $q_t$. Without surrogate information, e.g., $\tilde{\ell}_f(\x)=0$, residuals can be large; under uniform sampling, the procedure cannot prioritize high-variance samples. Accurate surrogate scores and suitable non-uniform sampling reduce these residual and sampling effects, thereby shrinking the CI faster and reaching the target precision with fewer evaluated samples.

\textbf{Coverage for population risk and shrinkage rate.} We finally lift the guarantee to population risk.
\begin{theorem}\label{thm:coverage_convergence}
Fix a significance level $\alpha\in(0,1)$ and choose any split $\alpha_1,\alpha_2\in(0,\alpha)$ with
$\alpha_1+\alpha_2=\alpha$.
% Under Assumptions~\ref{ass:loss_bound}-\ref{ass:betting}, for every $N\in\mathbb Z_+$, the sequence $(C_t)_{1\le t\le N}$ is a $(1-\alpha)$ anytime-valid confidence sequence for the population risk $R$. Moreover, for all $1\le t\le N$,
Under the population sampling setup and Assumptions~\ref{ass:loss_bound}-\ref{ass:betting}, for every $N\in\mathbb Z_+$, the sequence $(C_t)_{1\le t\le N}$ is a $(1-\alpha)$ anytime-valid confidence sequence for the population risk $R$. Moreover, for all $1\le t\le N$,
\[
\mathrm{width}\left(C_t\right)
\le
\mathrm{width}\left(\cM_t^\pm\right)
+
2(U-L)\sqrt{\frac{\log(2/\alpha_2)}{2N}}.
\]
Writing $C_{t,N}$ to make the dependence on $N$ explicit, for any integer sequence $\{t_N\}_{N\ge1}$ satisfying $1\le t_N\le N$ and $t_N\to\infty$ as $N\to\infty$,
\[
\mathrm{width}\left(C_{t_N,N}\right)
=
O_{\mathbb P}\left(
\log\log t_N\sqrt{\frac{\log t_N}{t_N}}
\right)
+
O\left(\sqrt{\frac{1}{N}}\right),
\]
and hence $\mathrm{width}\left(C_{t_N,N}\right)
\xrightarrow{\mathbb P}0
\qquad\text{as }N\to\infty.$ Thus, the confidence intervals are shrinking in the joint asymptotic regime $N\to\infty$, $t=t_N\to\infty$, and $t_N\le N$.
\end{theorem}
% \begin{theorem}\label{thm:coverage_convergence}
% \vspace{-0.0em}
% Fix a significance level $\alpha\in(0,1)$ and choose any split $\alpha_1,\alpha_2\in(0,\alpha)$ with
% $\alpha_1+\alpha_2=\alpha$.
% Under Assumptions~\ref{ass:loss_bound}-\ref{ass:betting}, the sequence $(C_t)_{1\le t\le N}$ is a sequence of $(1-\alpha)$ anytime-valid, shrinking CIs for the population risk $R$ with width $\mathrm{width}\left(C_t\right)=O\left(\log\log t\sqrt{\log t/t}\right)
% +O\left(\sqrt{1/N}\right).$
% \vspace{-0.4em}
% \end{theorem}
% [ZJ] page limit
Together, these results explain why \ourmethod~provides certifiable and efficient evaluation. Thms.~\ref{thm:lure_unbias}-\ref{thm:consistency} show that the final performance estimator targets the true evaluation risk and is consistent. Thm.~\ref{thm:signal_unbias} ensures the conditional unbiasedness of signal $\hat S_t$ for $e$-processes, while Thm.~\ref{thm:optimal} shows how surrogate scores and sampling probabilities reduce variance, thereby tightening the CI in Thm.~\ref{thm:conf_sequen}. Finally, we show in Thm.~\ref{thm:coverage_convergence} that the resulting anytime-valid CIs shrink with $t$ and $N$, yielding tight certified bounds under the stated sampling conditions; we further provide empirical validation in our experiments.

\section{Experiments}\label{sec:experiments}
In this section, we evaluate \ourmethod~on real LLM evaluation tasks. Coverage at the reported stopping rule is an empirical diagnostic; it does not establish simultaneous coverage over all rounds or certify the grid approximation. Sec.~\ref{sec:exp-setup} describes the experimental setup, Sec.~\ref{sec:exp-main-results} reports the main results, and Sec.~\ref{sec:exp-additional} provides additional analyses. 

\begin{table}[t]
\centering
\vspace{-0.3em}
\begin{minipage}[t]{0.49\linewidth}
\centering
\captionof{table}{Empirical miscoverage at the stopping time when the CI width reaches $\epsilon=0.05$ in the synthetic setting of App.~\ref{app:synthetic}. Each entry reports the mean miscoverage rate over $50$ independent runs, with Wilson $95\%$ intervals shown in grey brackets. \cereval~exceeds the nominal miscoverage level $\alpha=0.05$, indicating a gap between its theoretical guarantees and its proposed adaptive-partition algorithm. In contrast, all $e$-process-based methods remain within the nominal level.}

% \captionof{table}{Empirical miscoverage at the data-dependent stopping
% time $\tau_\epsilon$ ($\epsilon = 0.05$), pooled across the four
% surrogate--target pairs within each dataset. Point estimates are
% accompanied by Wilson 95\% intervals (in grey). All methods based on
% $e$-value construction stay within the nominal $\alpha = 0.05$ on every
% dataset, consistent with the by-construction anytime-validity of
% $\widetilde{\cM}_t^\pm$ in \eqref{eq:M_t} and of $C_t$ via
% Thm.~\ref{thm:coverage_convergence}. \zs{caption to be updated}}
\label{tab:rq1-coverage}

\vspace{-0.05em}
\setlength{\tabcolsep}{1pt}
\renewcommand{\arraystretch}{1.25}
\resizebox{\linewidth}{!}{%
\begin{tabular}{c c c c c}
\toprule
\large \cereval 
& \large \evbaseline 
& \large \textsc{w/o Surr} 
& \large \ourmethod 
& \large \oracleacq \\
\midrule
\begin{tabular}{@{}c@{}}{\large $0.161$}\\[-1pt]{\color{gray!70}$[0.158,0.164]$}\end{tabular}
& \begin{tabular}{@{}c@{}}{\large $0.030$}\\[-1pt]{\color{gray!90}$[0.029,0.032]$}\end{tabular}
& \begin{tabular}{@{}c@{}}{\large $0.030$}\\[-1pt]{\color{gray!90}$[0.029,0.032]$}\end{tabular}
& \begin{tabular}{@{}c@{}}{\large $0.021$}\\[-1pt]{\color{gray!90}$[0.020,0.023]$}\end{tabular}
& \begin{tabular}{@{}c@{}}{\large $0.028$}\\[-1pt]{\color{gray!90}$[0.027,0.030]$}\end{tabular} \\
\bottomrule
\end{tabular}%
}
\end{minipage}%
\hfill
\begin{minipage}[t]{0.49\linewidth}
\centering
\captionof{table}{Median number of evaluated samples to reach
$\epsilon=0.05$, across the six surrogate-target
pairs and $50$ independent runs. We mark the relative savings of \ourmethod{} vs.\
\evbaseline, and the \textbf{top-performer} in bold. The \oracleacq{} is
included for reference only.}
\label{tab:labels-to-eps}

\vspace{-0.05em}
\setlength{\tabcolsep}{4pt}
\renewcommand{\arraystretch}{1.15}
\resizebox{\linewidth}{!}{%
\begin{tabular}{l l l l}
\toprule
& \textbf{SST-2} & \textbf{MMLU} & \textbf{AG\,News} \\
& {\scriptsize ($N{=}68{,}221$)}
& {\scriptsize ($N{=}115{,}700$)}
& {\scriptsize ($N{=}127{,}600$)} \\
\midrule
\evbaseline
& $6{,}889$
& $8{,}144$
& $15{,}336$ \\
\textsc{w/o Surr}
& $3{,}844_{\,(-44\%)}$
& $5{,}110_{\,(-37\%)}$
& $8{,}298_{\,(-46\%)}$ \\
\ourmethod
& $\mathbf{3{,}156}_{\,(-54\%)}$
& $\mathbf{3{,}242}_{\,(-60\%)}$
& $\mathbf{5{,}840}_{\,(-62\%)}$ \\
\midrule
{\color{gray}\oracleacq}
& {\color{gray}$2{,}698$}
& {\color{gray}$1{,}923$}
& {\color{gray}$4{,}238$} \\
\bottomrule
\end{tabular}%
}
\end{minipage}
\end{table}

\subsection{Setup}
\label{sec:exp-setup}
\textbf{Datasets.} 
Following~\citet{Berrada:Kossen:Freddie:Razzak:Gal:Rainforth2025}, we evaluate LLMs on downstream tasks. We use three benchmarks with different levels of task difficulty, label spaces, and evaluation-pool sizes: Stanford Sentiment Treebank 2 (SST-2)~\cite{socher-etal-2013-recursive} for binary sentiment classification, AG News Topic Classification (AG$\,$News)~\cite{zhang2015character} for multi-class topic classification, and Massive Multitask Language Understanding (MMLU)~\cite{Hendrycks:Burns:Basart:Zou:Mazeika:Song2021}, a broad multiple-choice benchmark across a wide range of academic subjects. For each benchmark, all available samples are treated as the evaluation pool $D_N$. Table~\ref{tab:main-config-summary} summarizes the pool sizes and model pairs used in the main classification experiments.

% Following~\citet{Berrada:Kossen:Freddie:Razzak:Gal:Rainforth2025}, we evaluate pre-trained LLMs on performing downstream tasks, with three benchmarks spanning task difficulty, label-space, and evaluation pool size: SST-2~\citep{socher-etal-2013-recursive}, MMLU~\cite{Hendrycks:Burns:Basart:Zou:Mazeika:Song2021}, and AG$\,$News~\cite{zhang2015character}.
% The full validation/test split of each benchmark is treated as the evaluation pool $D_N$, with $t_{\max}$ set to the pool size throughout.

\textbf{Surrogate-target pairs.}
% [ZJ] page limit
% The target model is the model whose performance we aim to evaluate, and the surrogate model is used to provide surrogate scores for unevaluated samples. We evaluate multiple target models, including Llama-3-70B, DeepSeek-67B, Mixtral-8$\times$7B, and Qwen-2.5-72B, covering different model families and generations. We instantiate \emph{six} surrogate-target pairs across model-family and generation axes. Specifically, Llama-3-8B is used as the surrogate for four targets: Llama-3-70B (same family, same generation), DeepSeek-67B, Mixtral-8$\times$7B, and Qwen-2.5-72B (all cross-family, same generation). We further use Llama-2-7B as the surrogate for Llama-3-70B (same family, cross-generation) and DeepSeek-67B (cross-family, cross-generation). In all pairs, the surrogate model is substantially smaller without any task-specific training or fine-tuning.
We instantiate \emph{six} surrogate-target pairs across model-family and generation axes. Specifically, Llama-3-8B is used as the surrogate for four targets: Llama-3-70B (same family, same generation), DeepSeek-67B, Mixtral-8$\times$7B, and Qwen-2.5-72B (all cross-family, same generation). We further use Llama-2-7B as the surrogate for Llama-3-70B (same family, cross-generation) and DeepSeek-67B (cross-family, cross-generation). In all pairs, the surrogate model is substantially smaller without any task-specific training or fine-tuning.

\textbf{Baselines.}
We first use a diagnostic case to examine \cereval~\citep{Wang:Chen:Bo:Xu2025}, showing that its adaptive partitioning can fail to maintain anytime-valid coverage; and in the following comparisons, we focus on methods designed to be anytime-valid:  \ding{172} \evbaseline~applies the standard $e$-processes for bounded mean estimation~\citep{Waudby-Smith:Ramdas2024} under uniform sampling without a surrogate.
% \ding{173} \nosurr~ablates the surrogate component of \ourmethod~by setting $\tilde{\ell}_f\equiv 0$ while retaining the uncertainty-guided sampling rule, isolating the contribution of surrogate-assisted approximation.
\ding{173} \nosurr~sets $\tilde{\ell}_f\equiv0$ and samples uniformly without replacement from the remaining pool, with $q_t(j)=1/(N-t+1)$. It retains the inferential signal and CI construction of \ourmethod, but uses neither surrogate assistance nor uncertainty-guided sampling.
% \ding{174} \oracleacq~uses the variance-optimal sampling rule $q_t^\star$ in Thm.~\ref{thm:optimal}, which requires the per-sample residuals and is unknown in practice; it serves as an oracle reference for the best sampling efficiency attainable by \ourmethod.
% [ZJ][M4Cr-W1]
% \ding{174} \oracleacq~uses the true surrogate-residual magnitudes and sets $\beta=0$, giving the oracle sampling rule in Remark~\ref{rem:oracle-experiments}. It provides an infeasible reference for one-step variance reduction without the probability-floor constraint; it is not claimed to minimize the stopping time.
\ding{174} \oracleacq~uses the true surrogate-residual magnitudes as sample-selection scores, giving the oracle sampling rule in Remark~\ref{rem:oracle-experiments}. It provides an infeasible reference for one-step variance reduction; it is not claimed to minimize the stopping time.
% [ZJ] page limit: set this paragraph one line shorter if possible (text unchanged)
\looseness=-1

% \textbf{Baselines.}
% We compare \ourmethod~against four \zs{or 3?} references. 
% \ding{172} \evbaseline~applies the betting-based confidence sequence of~\citep{Waudby-Smith:Ramdas2024} to uniform sampling without a surrogate, i.e., with $\tilde{\ell}_f\equiv 0$.
% \ding{173} \nosurr~ablates the surrogate component of our method by setting $\tilde{\ell}_f\equiv 0$ while retaining the uncertainty-guided sampling rule, isolating the contribution of surrogate-assisted approximation.
% efficiency attainable under our framework.
% \ding{174} \oracleacq~uses the variance-minimizing proposal $q^{\star}_{t}$ of Thm.~(\ref{thm:optimal}), which requires the unknown per-sample residuals and is therefore infeasible in practice; it serves as a performance upper bound on \ourmethod.
% \ding{175} \cereval~\citep{Wang:Chen:Bo:Xu2025} is a certified competitor and combines adaptive partitioning with concentration bounds.

% \input{./results/tab_setup_pairs.tex}

\textbf{Evaluation metric and sample selection.}
% [ZJ] page limit
% We report the evaluation results with $0$-$1$ risk in the main text, with additional results on cross-entropy risk provided in App.~\ref{app:ce-results}. For \ourmethod, we use the uncertainty-guided sampling strategies from Rems.~\ref{rem:0-1} and~\ref{rem:cross_entropy}. For the no-surrogate baselines, i.e., \evbaseline~and \nosurr, we use uniform sampling over the unevaluated pool. Detailed results for different sample-selection strategies are provided in App.~\ref{app:acquisition-strategy}.
We report the evaluation results with $0$-$1$ risk in the main text, with additional results on cross-entropy risk provided in App.~\ref{app:ce-results}. For \ourmethod, we use the uncertainty-guided sampling strategies from Rems.~\ref{rem:0-1} and~\ref{rem:cross_entropy}.

% \textbf{Evaluation metric and acquisition.}
% We report $0$-$1$ loss in the main text (more results on evaluating cross-entropy are in App.~\ref{app:ce-results}
% For \ourmethod's acquisition, we deploy Strategy~$4$ from Remark~\ref{rem:0-1} \zs{TODO: why}.
% For the no-surrogate baselines (i.e., \evbaseline~and \nosurr), we replace the acquisition rule with uniform sampling over the unevaluated pool.
% We sweep the remaining acquisition strategies of Rem.~\ref{rem:0-1} in App.~\ref{app:acq-ablation}.

\textbf{Hyperparameters.}
We set the overall significance level to $\alpha=0.05$, which controls the miscoverage probability of the CIs, and split it evenly as $\alpha_1=\alpha_2=0.025$ between anytime-valid inference for $R_N$ and the finite-pool-to-population correction. Following previous $e$-process ~\cite{Waudby-Smith:Ramdas2024}, we use predictable fractions in Rem.~\ref{rem:betting-fractions} with $c=0.5$, set $\theta=0.5$ in \eqref{eq:hedhedCP}, and use the scaling strategy in Rem.~\ref{rem:scaling}. For sampling, we set $\beta=0.4$ to control the lower bound of sampling probabilities, with sensitivity analyzed in App.~\ref{app:sensitivity}. We stop the evaluation at the first time $t$ such that $\mathrm{width}(C_t)\le\epsilon=0.05$.

% \textbf{Hyperparameters.} We set the overall significance level to $\alpha=0.05$, which controls the miscoverage probability of the CIs, and split it evenly as $\alpha_1=\alpha_2=0.025$ between anytime-valid inference for $R_N$ and the finite-pool-to-population correction. For the $e$-process, we use the predictable fractions in Rem.~\ref{rem:betting-fractions} with $c=0.5$, set $\beta=0.4$to control the lower bound of sampling probabilities along with the scaling strategy in Rem.~\ref{rem:caling}, whinch statisifies Assumption~\ref{ass:beta}, and set $\theta=0.5$ in the hedged capital process in \eqref{eq:hedhedCP}. We stop the evaluation at the first time $t$ such that $\mathrm{width}(C_t)\le\epsilon=0.05$. We analyze the sensitivity to these choices in App.~\ref{app:beta-sensitivity}.

% \paragraph{Trials and reporting.}
% We repeat each cell for $50$ random seeds (\zs{Cer-Eval: $3$ seeds per cell on the available cells}). 
% Pooled figures stack per-seed trajectories across the displayed cells and report mean width with bootstrap $95\%$ intervals; per-cell breakdowns are in App.~(\ref{app:per-cell-efficiency}).

\begin{figure}[t]
\centering
\includegraphics[width=\linewidth]{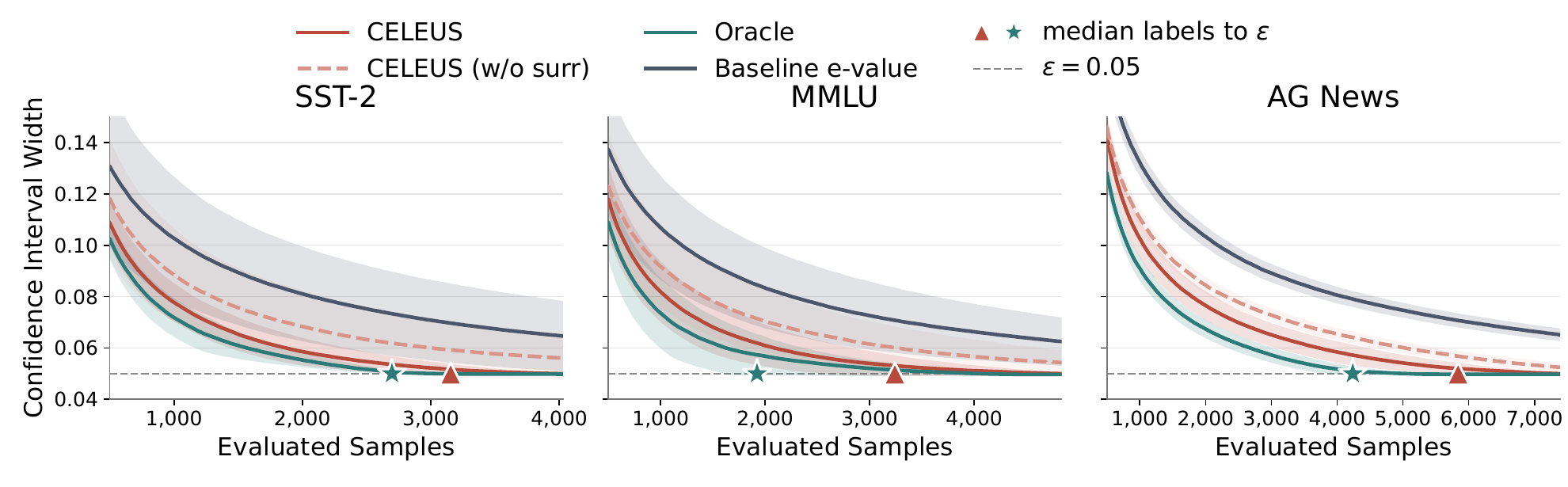}
\vspace{-2.4em}
\caption{
% CI width versus the number of evaluated samples. Each trajectory is averaged over $50$ seeds and surrogate--target pairs, with shaded bands showing $\pm 1$ standard deviation. The dashed line marks $\epsilon=0.05$. At the budget where \oracleacq~reaches $\epsilon$, \ourmethod~reduces the CI-width gap between \evbaseline~and \oracleacq~by $90\%/93\%/84\%$ on SST-2/MMLU/AG$\,$News, compared with $53\%/60\%/60\%$ for \nosurr.
% [ZJ] page limit
% CI width versus the number of evaluated samples. Each trajectory averages the mean widths of the six surrogate--target pairs, with each pair averaged over $50$ seeds. Shaded bands show $\pm 1$ standard deviation across these pair means. The dashed line marks $\epsilon=0.05$. At the budget where \oracleacq~reaches $\epsilon$, \ourmethod~reduces the CI-width gap between \evbaseline~and \oracleacq~by $90\%/93\%/84\%$ on SST-2/MMLU/AG$\,$News, compared with $53\%/60\%/60\%$ for \nosurr.
CI width versus the number of evaluated samples. Each trajectory averages the mean widths of the six surrogate--target pairs, with each pair averaged over $50$ seeds. Shaded bands show $\pm 1$ standard deviation across these pair means. The dashed line marks $\epsilon=0.05$.
}
% \caption{
% Confidence-interval width versus label used. Each trajectory is the mean over $50$ seeds and the four surrogate--target pairs; the shaded band is $\pm 1\sigma$ across the six pair-means. The dashed line marks the target precision $\epsilon=0.05$; each panel is capped at the budget where \ourmethod~reaches $\epsilon$.
% At the budget where \oracleacq~reaches $\epsilon$, the width gap between \ourmethod~and~\oracleacq~are $90\% / 93\% / 84\%$ {\em smaller} than the gap between \evbaseline~and~\oracleacq, on SST-2/MMLU/AG\,News, while the surrogate-free variant (\nosurr) covers $53\% /60\% /60\%$ of that distance.
% }
% \caption{Pooled confidence-interval width versus labels acquired (mean over $50$ seeds; shaded band is the bootstrap $95\%$ interval), for SAVE-Ada, the \textsc{IID+$e$-process} baseline, and the variance-minimising \textsc{Oracle} of Thm.~(\ref{thm:optimal}). \bo{it's a bit unclear the advantage of CELEUS here. Will log makes it more clear? The caption can help: mention that CELEUS here is xxx\% close to the Oracle, while the baselines is xxx.} 
% ZS: addressed by 1) zooming in the trajectories at each dataset and 2) explicitly mentioning the relative gain of CELEUS over baseline.
% Curves are pooled across the four pairs of Tab.~(\ref{tab:setup-pairs}); the dashed horizontal line marks the target precision $\epsilon=0.05$. Per-dataset and per-cell variants are in App.~(\ref{app:per-cell-efficiency}).}
\label{fig:rq2-efficiency-main}
\vspace{-0.5em}
\end{figure}

\subsection{Main Results}\label{sec:exp-main-results}
\textbf{RQ1: Empirical coverage at stopping.} For each method, we repeat the experiment $50$ times; each experiment consists of $1000$ Monte Carlo trials to estimate the miscoverage rate, where each trial regenerates a finite evaluation pool and runs the method (see App.~\ref{app:synthetic} for details). As shown in Table~\ref{tab:rq1-coverage}, \cereval~has a miscoverage rate of $0.161$, far above the level $\alpha=0.05$, indicating that its adaptive-partition algorithm can lose the claimed coverage. In contrast, all $e$-process-based methods stay below $0.05$, with miscoverage rates between $0.021$ and $0.030$.

% Tab.~(\ref{tab:rq1-coverage}) reports empirical miscoverage at $\tau_\epsilon$, pooled over the four pairs of Tab.~(\ref{tab:setup-pairs}) within each dataset. SAVE-Ada misses the true risk on at most $0.5\%$ of $753$--$800$ trajectories per cell, with the Wilson upper bound never above $1.3\%$. \textsc{IID+$e$-process} also stays within nominal, with slightly elevated miscoverage on \textsc{SST-2} ($1.3\%$, $[0.6,2.9]$). These numbers confirm the by-construction anytime-validity of $\widetilde\cM_t^\pm$ in \eqref{eq:M_t} and of $C_t$ via Thm.~(\ref{thm:coverage_convergence}). We hold the \textsc{Cer-Eval} comparison until the broader sweep completes: three seeds per cell is too few for a binomial interval.

% \subsection{Label efficiency: confidence-interval width and labels-to-$\epsilon$ (RQ2)}
% \label{sec:exp-efficiency}

\textbf{RQ2: Label Efficiency.} 
% [ZJ] page limit
% We compare \ourmethod~by CI width and the number of evaluated samples needed to reach $\epsilon=0.05$. 
Across SST-2, MMLU, and AG\,News, \ourmethod~uses $54\%$, $60\%$, and $62\%$ fewer labels than \evbaseline~(Table~\ref{tab:labels-to-eps}), evaluating only $4.6\%$, $2.8\%$, and $4.6\%$ of each pool. 
% Compared with \nosurr, it reduces the required labels by $10$-$23\%$ using $7$-$8$B surrogate models without task-specific training or fine-tuning. 
% Compared with \nosurr, it reduces the required labels by $10$-$23\%$ using $7$-$8$B surrogate models without task-specific training or fine-tuning. This comparison reflects the combined contribution of surrogate assistance and uncertainty-guided sampling.
Using the dataset-level counts in Table~\ref{tab:labels-to-eps}, the reductions relative to \nosurr~are approximately $18\%$, $37\%$, and $30\%$ on SST-2, MMLU, and AG News, respectively. The $7$--$8$B surrogate models receive no task-specific training or fine-tuning. This comparison reflects the combined contribution of surrogate assistance and uncertainty-guided sampling.
% Figure~\ref{fig:rq2-efficiency-main} shows that this gain holds anytime: \ourmethod's CI width remains below \evbaseline~and \nosurr~at every budget. 
Figure~\ref{fig:rq2-efficiency-main} shows that \ourmethod's mean CI width remains below \evbaseline~and \nosurr~across the displayed evaluation budgets.
At the budget where \oracleacq~first reaches $\epsilon$, \ourmethod~closes $90\%/93\%/84\%$ of the \evbaseline-to-\oracleacq~width gap, compared with $53\%/60\%/60\%$ for \nosurr. 
% This matches Thm.~\ref{thm:conf_sequen}, as the CI width is controlled by the cumulative conditional variance of the signals, which is minimized by the oracle variance-optimal strategy. 
% [ZJ][M4Cr-W1]
% This is consistent with the role of conditional variance in Thm.~\ref{thm:conf_sequen}. The oracle minimizes one-step conditional variance without a probability floor, subject to the numerical safeguard in Remark~\ref{rem:oracle-experiments}; this does not establish optimality for cumulative variance or stopping time.
This is consistent with the role of conditional variance in Thm.~\ref{thm:conf_sequen}. The oracle uses the true residual magnitudes to reduce one-step conditional variance; this does not establish optimality for cumulative variance or stopping time.
App.~\ref{app:01-coverage} reports miscoverage results, with all methods below $\alpha=0.05$.

% \textbf{RQ2: Efficiency.} We compare \ourmethod~with other methods in terms of CI width and the number of evaluated samples required to reach $\epsilon=0.05$. 
% Across SST-2, MMLU, and AG\,News, \ourmethod\ reaches $\epsilon$ with $54\%$, $60\%$, and $62\%$ fewer labels than \evbaseline\ (Table~\ref{tab:labels-to-eps}), corresponding to only $4.6\%$, $2.8\%$, and $4.6\%$ of each whole evaluation pool, that is, the {\em regime where certifiable evaluation is arguably useful}. Compared with \nosurr, \ourmethod~further improves evaluation efficiency by $10$-$23\%$ by leveraging pre-trained surrogate models with only $7$-$8$B parameters, without any task-specific training or fine-tuning. 
% The width trajectories in Figure~\ref{fig:rq2-efficiency-main} additionally show that this saving is achieved \emph{anytime}: \ourmethod's width sits
% below \evbaseline\ and \nosurr\ at every budget. 
% At the budget where \oracleacq\ first reaches $\epsilon$, \ourmethod\ has already closed $90\% / 93\% / 84\%$ of
% the \evbaseline-to-\oracleacq\ width gap, against $53\% / 60\% / 60\%$ for the surrogate-free variant \nosurr. The gap to \oracleacq~tracks Thm.~\ref{thm:conf_sequen}: the CI width is controlled by the cumulative conditional variance of the signals, which is minimized by using the oracle variance-optimal strategy. Empirical miscoverage results  under $0$-$1$ loss is reported in App.~\ref{app:01-coverage}, with all methods below $\alpha=0.05$.

\textbf{RQ3: What makes~\ourmethod\ work.}
% [ZJ] page limit
% We isolate three key designs of \ourmethod: the weighted final performance estimator, the surrogate-corrected inferential signal, and the uncertainty-guided sampling rule, and examine how they contribute to certifiable and efficiency LLM evaluation.
% 
\ding{172}~\textit{Estimator unbiasedness.} Fig.~\ref{fig:rq3-bias-mse}~(left) compares the empirical bias of $\hat R_t$ in~\eqref{eq:lure_estimator} with a naive unweighted estimator that removes the weight $u_{m,t}$, using $R_N$ as the reference since the population risk $R$ is unknown. Across all three datasets, $\hat R_t$ stays centered around zero throughout the evaluation process, consistent with the unbiasedness guarantee in Thm.~\ref{thm:lure_unbias}. In contrast, the naive estimator has a \emph{persistent} bias of order $10^{-2}$, about $1\%$ on SST-2 and MMLU and $2\%$ on AG$\,$News. This shows that the weight $u_{m,t}$ is necessary under non-uniform sample selection.
% [ZJ] page limit: set this paragraph one line shorter if possible (text unchanged)
\looseness=-1
% \ding{172}~\textit{Estimator unbiasedness.}
% Fig.~\ref{fig:rq3-bias-mse}~(left) plots the empirical bias of $\hat R_t$ in~\eqref{eq:lure_estimator} against an unweighted alternative that drops the LURE weight $u_{m,t}$. 
% Across all three datasets, $\hat R_t$ concentrates around zero at every~$t$, matching the conditional unbiasedness of Thms.~\ref{thm:lure_unbias} and~\ref{thm:signal_unbias}. 
% $\hat R_t^{\mathrm{naive}}$, by contrast, carries a \emph{persistent} bias of order $10^{-2}$ that shows no material decay with~$t$ ($\approx 1\%$ on SST-2 and MMLU, $\approx 2\%$ on AG\,News under the \ourmethod\ proposal).
% % The bias grows with how concentrated the proposal is: under the sharper \oracleacq\ proposal, it reaches roughly an order of magnitude larger than under \ourmethod, while $\hat R_t$ stays close to zero throughout. 
% The LURE correction is therefore essential precisely where active sampling departs most from uniform: the further the proposal moves from uniform, the larger the bias that $u_{m,t}$ must remove, and the more the unweighted estimator misrepresents the finite-pool risk it claims to estimate.

\ding{173}~\textit{Inferential signal.}
% [ZJ] page limit
% Fig.~\ref{fig:rq3-bias-mse}~(middle) compares $\hat S_t$ in~\eqref{eq:signal_estimator} with the naive single-step risk $\ell(f(\x_{i_t}),y_{i_t})$ as estimators of the finite-pool risk $R_N$. On SST-2 and MMLU, $\hat S_t$ achieves $1.5$-$2\times$ lower mean squared error (MSE) across $t$. This improvement comes from the structure of $\hat S_t$: it combines previously observed risks, surrogate scores for unevaluated samples, and an inverse-probability-weighted correction from the newly evaluated sample. Thus, unlike the naive single-step risk, $\hat S_t$ leverages richer information, making the signal more stable.
Fig.~\ref{fig:rq3-bias-mse}~(middle) compares $\hat S_t$ in~\eqref{eq:signal_estimator} with the naive single-step risk $\ell(f(\x_{i_t}),y_{i_t})$ as estimators of the finite-pool risk $R_N$. On SST-2 and MMLU, $\hat S_t$ achieves $1.5$-$2\times$ lower mean squared error (MSE) across $t$. Unlike the naive single-step risk, $\hat S_t$ leverages richer information, making the signal more stable.

\ding{174}~\textit{Conditional variance.}
Fig.~\ref{fig:rq3-bias-mse}~(right) reports $\mathrm{Var}(\hat S_t\mid \cF_{t-1})$ over $50$ seeds under uniform sampling, \ourmethod's uncertainty-guided sampling, and the oracle $q_t^{\mathrm{orc}}$ in Remark~\ref{rem:oracle-experiments}. Compared with uniform sampling, \ourmethod~reduces the variance by $27\%$, $20\%$, and $5\%$ on SST-2, MMLU, and AG News, respectively, and closes $\approx67\%$, $\approx53\%$, and $\approx18\%$ of the uniform-to-\oracleacq~variance gap. This is consistent with the variance-minimization principle in Thm.~\ref{thm:optimal}: sampling probabilities that better reflect surrogate-residual magnitudes reduce variance more effectively.

% \ding{174}~\textit{Conditional variance.}
% Fig.~\ref{fig:rq3-bias-mse}~(right) plots $\mathrm{Var}(\hat S_t\mid \cF_{t-1})$ averaged over $50$ seeds under three proposals: \ourmethod's active proposal, uniform sampling, and the \oracleacq\ $q_t^{\star}$ of Thm.~(\ref{thm:optimal}).
% The active proposal sits strictly below uniform on every panel, with mean reductions of $27\%$, $20\%$, and $5\%$ on SST-2, MMLU, and AG\,News respectively, and approaches \oracleacq\ where the surrogate is most informative: it closes $\approx 67\%$ of the uniform-to-\oracleacq\ variance gap on SST-2, $\approx 53\%$ on MMLU, and $\approx 18\%$ on AG\,News.
% \zs{todo: briefly discuss the reason}
% The residual on AG\,News tracks its weaker surrogate--target alignment (mean $\rho\!\approx\!0.34$, vs.\ $0.39$/$0.46$ on SST-2/MMLU): with a less informative $\hat\ell$, the active proposal has less variance to remove and therefore stays closer to uniform.
% We defer the surrogate-strategy ablation that closes this gap to Sec.~(\ref{sec:exp-pending}).

% [ZJ] page limit
% Taken together, the three sub-claims under RQ3 close the loop: LURE weights reduce bias, surrogate scores improve signal quality, and uncertainty-guided sampling reduces conditional variance. 
% Combining RQ1 and RQ2, they complete empirical confirmation of theoretical claims of \ourmethod.

\begin{figure}[t]
\centering
\includegraphics[width=\linewidth]{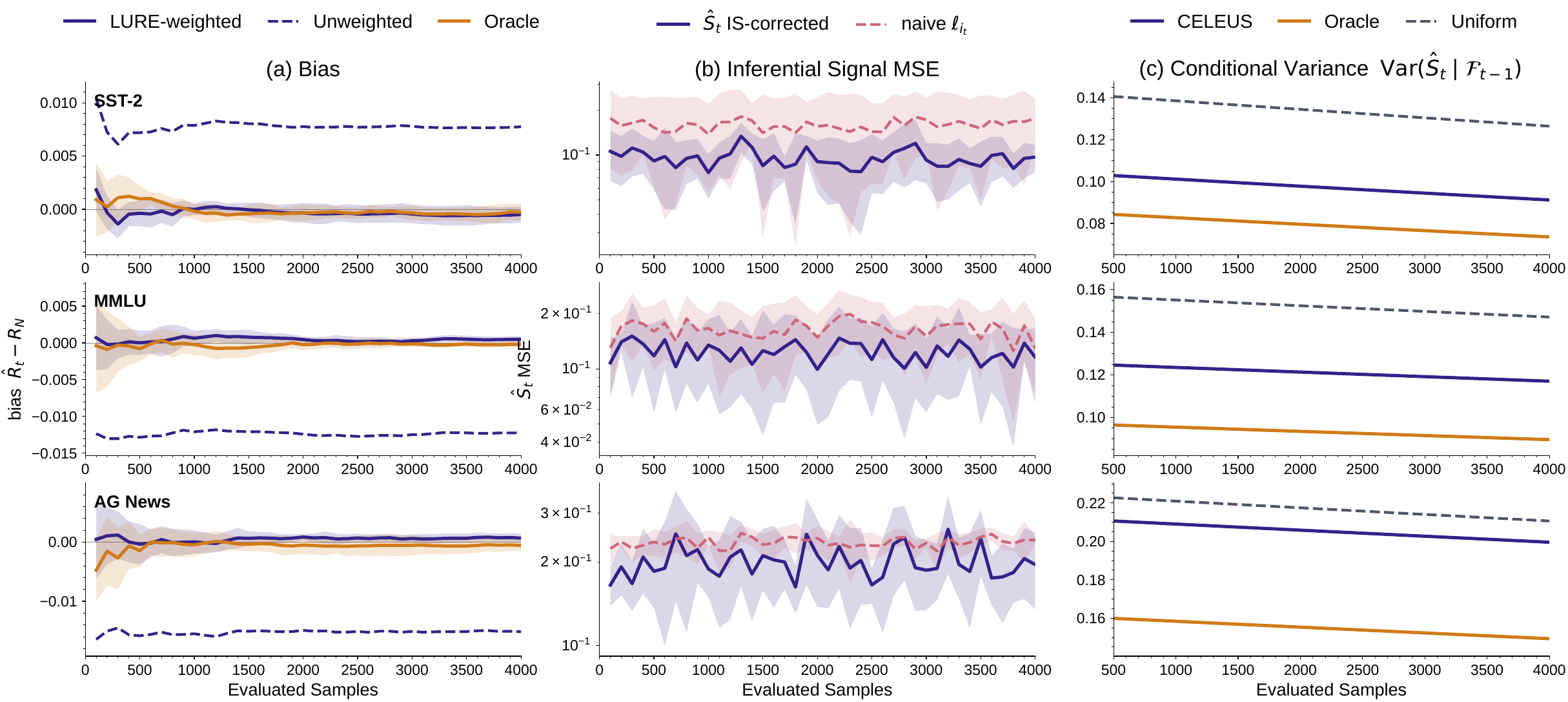}
  \vspace{-1.5em}
    \caption{
    Mechanism analysis for \ourmethod, pooled across the six surrogate-target pairs. 
    % [ZJ] page limit
    % \emph{Left:} empirical bias of the final estimator $\hat R_t$ in \eqref{eq:lure_estimator} versus the unweighted estimator that removes $u_{m,t}$. 
    \emph{Left:} empirical bias of $\hat R_t$ in \eqref{eq:lure_estimator} versus the unweighted estimator. 
    % [ZJ] page limit
    % \emph{Middle:} MSE of the proposed signal $\hat S_t$ in \eqref{eq:signal_estimator} versus the naive single-step observed risk of the newly evaluated sample.
    \emph{Middle:} MSE of $\hat S_t$ in \eqref{eq:signal_estimator} versus the naive single-step observed risk.
%     \emph{Right:} $\mathrm{Var}(\hat S_t\mid\cF_{t-1})$ under uniform sampling, \ourmethod's uncertainty-guided sampling, and the oracle $q_t^\star$ in Thm.~\ref{thm:optimal}.
    % [ZJ][M4Cr-W1]
    % \emph{Right:} $\mathrm{Var}(\hat S_t\mid\cF_{t-1})$ under uniform sampling, \ourmethod's uncertainty-guided sampling, and the $\beta=0$ oracle $q_t^{\mathrm{orc}}$ in Remark~\ref{rem:oracle-experiments}.
    \emph{Right:} $\mathrm{Var}(\hat S_t\mid\cF_{t-1})$ under uniform sampling, \ourmethod's uncertainty-guided sampling, and the oracle $q_t^{\mathrm{orc}}$ in Remark~\ref{rem:oracle-experiments}.
    }
\label{fig:rq3-bias-mse}
\vspace{-0.5em}
\end{figure}

% \begin{figure}[t]
% \centering
% \includegraphics[width=0.95\linewidth]{results/figures/rq6-variance/accuracy/pooled.pdf}
% \caption{Conditional variance of the signal $\mathrm{Var}(\hat S_t\mid\cF_{t-1})$ (top, linear scale) and empirical variance of the estimator $\mathrm{Var}(\hat R_t)$ (bottom, log scale), per dataset (columns) and per acquisition rule. The active proposal of SAVE-Ada sits below uniform sampling on every panel and approaches the \textsc{Oracle} of Thm.~(\ref{thm:optimal}) where the surrogate is most informative.}
% \label{fig:rq3-variance}
% \end{figure}

\subsection{Additional Experiments}
\label{sec:exp-additional}

% In addition to the three RQs above, we provide further analyses of sample-selection strategies, hyperparameter sensitivity, and practical overhead. Due to the page limit, we summarize their scope here and defer the full results to the appendix. App.~\ref{app:acquisition-strategy} compares the sample-selection strategies in Rem.~\ref{rem:0-1}. App.~\ref{app:hyperparameters} studies the sensitivity of \ourmethod~to key hyperparameters. App.~\ref{app:wallclock} compares the per-round runtime of \ourmethod~with the baselines. App.~\ref{app:per-configuration-efficiency} reports detailed results for each surrogate-target pair, corresponding to the summary results in Tab.~\ref{tab:labels-to-eps} and Fig.~\ref{fig:rq2-efficiency-main}.
% [ZJ] page limit
% In addition to the three RQs above, we provide further analyses of sample-selection strategies, hyperparameter sensitivity, and practical overhead. Due to the page limit, we summarize their scope here and defer the full results to the appendix. App.~\ref{app:acquisition-strategy} compares the sample-selection strategies in Rem.~\ref{rem:0-1}. App.~\ref{app:sensitivity} studies the sensitivity of \ourmethod~to key hyperparameters. App.~\ref{app:wallclock} compares the per-round runtime of \ourmethod~with the baselines. App.~\ref{app:per-configuration-efficiency} reports detailed results for each surrogate-target pair, corresponding to the summary results in Tab.~\ref{tab:labels-to-eps} and Fig.~\ref{fig:rq2-efficiency-main}.
App.~\ref{app:acquisition-strategy} compares the sample-selection strategies in Rem.~\ref{rem:0-1}. App.~\ref{app:sensitivity} studies the sensitivity of \ourmethod~to key hyperparameters. App.~\ref{app:wallclock} compares the per-round runtime of \ourmethod~with the baselines. App.~\ref{app:per-configuration-efficiency} reports detailed results for each surrogate-target pair, corresponding to the summary results in Tab.~\ref{tab:labels-to-eps} and Fig.~\ref{fig:rq2-efficiency-main}.
% [M4Cr-W5][AC]
% [ZJ]
% App.~\ref{sec:re-new-results} further compares \ourmethod~with \cereval, evaluates open-ended generation on RealToxicityPrompts, separates the effects of surrogate and acquisition quality, and studies a pre-specified cross-entropy bound.
App.~\ref{app:closely_related_methods} further compares \ourmethod~with \cereval, and App.~\ref{app:surrogate_acquisition_quality} separates the effects of surrogate and acquisition quality. App.~\ref{sec:re-new-results} evaluates open-ended generation on RealToxicityPrompts and studies a pre-specified cross-entropy bound.

\section{Conclusion}
This paper proposed \ourmethod, an $e$-process-based framework for certifiable and efficient LLM evaluation. \ourmethod~builds anytime-valid CIs from conditionally unbiased signals, uses surrogate-assisted approximation and uncertainty-guided sampling to reduce variance, and stops once the CI reaches a target precision. Theory and experiments show that \ourmethod~preserves coverage and reaches the target precision with fewer evaluated samples than baselines. An interesting direction for future work is to extend \ourmethod~to distribution-shift-aware evaluation scenarios. Additional limitations and the scope of the theoretical guarantees are discussed in App.~\ref{app:limitations}.

% TODO (T3) [ZJ]: NeurIPS camera-ready requires "Acknowledgments and Disclosure of Funding" (it does not count toward the page limit).
% Add funding sources and competing interests, then uncomment the environment below. (T3 status: funding added and environment uncommented; competing interests not stated yet.)

\newpage
\begin{ack}
This research was supported by the University of Melbourne’s Research Computing Services and the Petascale Campus Initiative. ZJZ is supported by the Melbourne Research Scholarship and the ARC with grant number DE240101089. FL is supported by the ARC with grant number DE240101089, LP240100101, and DP230101540. This work is partially supported by the National Science Foundation under grant No. 1910100, No. 2046726, NSF AI Institute ACTION No. IIS-2229876, DARPA Cooperative Agreement HR00112520013, the National Aeronautics and Space Administration (NASA) under grant No. 80NSSC20M0229, ARL Grant W911NF-23-2-0137, Alfred P. Sloan Fellowship, the research grant from eBay, AI Safety Fund, and Schmidt Science.
\end{ack}

\bibliographystyle{unsrtnat}
\bibliography{TETreference}

\newpage
\appendix
\onecolumn

\listofappendices
\newpage

\appsection{Additional discussions}
\label{app:additional_discussions}
\appsubsection{Related work}
\label{app:related_work}

\textbf{Efficient evaluation via sample selection.}
A growing line of work studies how to improve evaluation efficiency through active testing and adaptive sample selection. Existing approaches include cheaper proxy predictors, representative subset selection, and adaptive selection of informative examples \cite{kossen:Farquhar:Gal:Rainforth2022,Ash:Zhang:Krishnamurthy:Langford:Agarwal,Kossen:Farquhar:Sebastian:Gal:Rainforth2021,Berrada:Kossen:Freddie:Razzak:Gal:Rainforth2025}. Active testing and related adaptive evaluation methods show that evaluation can be made more efficient by prioritizing informative examples rather than evaluating the entire test pool uniformly at random \cite{Kossen:Farquhar:Sebastian:Gal:Rainforth2021,Farquhar:Gal:Rainforth2021}. These methods provide a natural framework for pool-based evaluation under limited evaluation budgets \cite{Berrada:Kossen:Freddie:Razzak:Gal:Rainforth2025}. However, while they offer principled sampling strategies and, in some cases, valid estimation under adaptive sampling, their interval guarantees are typically not designed for fully sequential use. In particular, the resulting intervals or stopping criteria are generally not \emph{anytime-valid}, and may fail to remain valid when the evaluator repeatedly monitors the current interval and stops at a data-dependent time.

\textbf{Prediction-powered and semi-supervised inference for evaluation.}
% Another closely related direction is prediction-powered inference (PPI) and, more broadly, semi-supervised inference with auxiliary predictors or automatic raters \cite{Angelopoulos:Bates:Fannjiang:Jordan:Zrnic2023,Angelopoulos:Duchi:Zrnic2023,Zheng:Chiang:Sheng:Zhuang:Wu:Zhuang2023}. These methods improve statistical efficiency by combining evaluated samples with auxiliary predictions on unevaluated samples. This paradigm is particularly relevant to LLM evaluation, where automatic raters or surrogate models can provide useful information on unevaluated examples and reduce the number of human evaluations needed \cite{Zhou:Lu:Mishra:Brahma:Basu:Luan2023}. However, existing PPI-style guarantees are typically derived under \emph{i.i.d.} sampling assumptions and rely on asymptotic normality or related large-sample approximations. As a result, they do not directly apply to the adaptive evaluation setting considered here, where evaluated examples are selected sequentially from a finite pool. More importantly, they do not by themselves provide anytime-valid guarantees under continuous monitoring and adaptive stopping.
Another closely related direction is prediction-powered inference (PPI), which combines a limited set of evaluated samples with auxiliary predictions on unevaluated samples~\cite{Angelopoulos:Bates:Fannjiang:Jordan:Zrnic2023,Angelopoulos:Duchi:Zrnic2023}. In LLM evaluation, auxiliary predictions can come from surrogate models or automatic raters~\cite{Zheng:Chiang:Sheng:Zhuang:Wu:Zhuang2023,Zhou:Lu:Mishra:Brahma:Basu:Luan2023}. Classical PPI primarily provides fixed-time inference, while active statistical inference and subsequent prediction-powered confidence sequences also address sequential validity~\cite{Zrnic:Candes2024,Csillag:Struchiner:Goedert2025,Kilian:Cortinovis:Caron2025}. Our setting differs in its adaptive, unequal-probability sampling without replacement from a fixed finite pool. This sampling structure requires a signal whose conditional mean remains the same finite-pool risk as the remaining set and proposal probabilities change.

\textbf{Sequential inference and anytime-valid CIs.}
Our work is also related to sequential inference, confidence sequences, and $e$-processes \cite{Howard:Ramdas:McAuliffe:Sekhon2021,Howard:Ramdas:McAuliffe:Sekhon2020,Waudby-Smith:Ramdas2024}. This line of work develops hypothesis tests and confidence intervals that remain valid under continuous monitoring, making them suitable for settings with adaptive stopping. In contrast to classical fixed-horizon intervals, anytime-valid procedures preserve statistical validity no matter when the analyst chooses to stop. However, many existing constructions are developed for online settings in which observations arrive sequentially from an external stochastic process~\cite{Zrnic:Candes2024,Park:Zecchin:Simeone2025,Xu:Karampatziakis:Mineiro2024,Sfyraki:Wang2026}. In our setting, evaluation begins with a fixed evaluation pool, and the evaluator sequentially selects which examples to evaluate without replacement. This pool-based structure induces a different dependence pattern and allows the procedure to use information from both evaluated and unevaluated samples. Existing sequential inference tools provide the foundation for anytime-valid inference, including confidence sequences for sampling without replacement \cite{Waudby-Smith:Ramdas2020}, but do not directly address the combination of surrogate-assisted estimation, adaptive sample selection, and certifiable LLM evaluation considered here.

\appsubsection{Comparison with Closely Related Methods}
\label{app:closely_related_methods}

% [ZJ]
% Recent prediction-powered methods are closely related to CELEUS in their use of auxiliary predictions to improve the efficiency of sequential inference. The main distinction lies in the underlying data-acquisition protocol. \citet{Csillag:Struchiner:Goedert2025} propose Prediction-Powered E-Values for an externally arriving data stream, where the procedure decides at each round whether to query the label of the current observation. Although the query probability may depend on previously observed information, the procedure does not control which observation arrives next. CELEUS instead begins with a fixed evaluation pool and adaptively selects the next item from the remaining unevaluated samples. This results in adaptive, unequal-probability sampling without replacement, which CELEUS explicitly accounts for through its conditionally unbiased signal construction.
Recent prediction-powered methods are closely related to \ourmethod~in their use of auxiliary predictions to improve the efficiency of sequential inference. The main distinction lies in the underlying data-acquisition protocol. \citet{Csillag:Struchiner:Goedert2025} propose Prediction-Powered E-Values for an externally arriving data stream, where the procedure decides at each round whether to query the label of the current observation. Although the query probability may depend on previously observed information, the procedure does not control which observation arrives next. \ourmethod~instead begins with a fixed evaluation pool and adaptively selects the next item from the remaining unevaluated samples. This results in adaptive, unequal-probability sampling without replacement, which \ourmethod~explicitly accounts for through its conditionally unbiased signal construction.

% [ZJ]
% \citet{Kilian:Cortinovis:Caron2025} develop anytime-valid, Bayes-assisted prediction-powered inference under independently sampled labelled and unlabelled data streams, with coverage established asymptotically uniformly over time. In contrast, CELEUS provides finite-sample, time-uniform coverage under adaptive sampling from a fixed evaluation pool. Therefore, although these methods share the goal of using inexpensive auxiliary predictions to reduce evaluation costs, the guarantees established in these works do not directly cover the adaptive sampling-without-replacement setting considered by CELEUS.
\citet{Kilian:Cortinovis:Caron2025} develop anytime-valid, Bayes-assisted prediction-powered inference under independently sampled labelled and unlabelled data streams, with coverage established asymptotically uniformly over time. In contrast, \ourmethod~provides finite-sample, time-uniform coverage under adaptive sampling from a fixed evaluation pool. Therefore, although these methods share the goal of using inexpensive auxiliary predictions to reduce evaluation costs, the guarantees established in these works do not directly cover the adaptive sampling-without-replacement setting considered by \ourmethod.

% [ZJ]
% \textbf{Cer-Eval and adaptive partitioning.}
\textbf{\cereval~and adaptive partitioning.}
% [ZJ]
% \citet{Wang:Chen:Bo:Xu2025} study a closely related fixed-pool evaluation setting and propose a partition-based strategy for reducing evaluation cost. Their theoretical result considers the algorithm when supplied with a fixed partition satisfying the required conditions, whereas the practical procedure repeatedly learns and updates the partition using evaluated samples and their observed risks. Consequently, the stated guarantee does not directly cover the complete adaptive partition-learning procedure, for which an additional argument accounting for the data-dependent updates would be required. This distinction concerns the scope of the current theoretical analysis rather than the validity of adaptive partitioning in general. By comparison, \ourmethod allows the sampling distribution to depend on the observed history while preserving the conditional unbiasedness of its inferential signal, thereby supporting anytime-valid inference under the adaptive sampling and stopping rules considered here.
\citet{Wang:Chen:Bo:Xu2025} study a closely related fixed-pool evaluation setting and propose a partition-based strategy for reducing evaluation cost. Their theoretical result considers the algorithm when supplied with a fixed partition satisfying the required conditions, whereas the practical procedure repeatedly learns and updates the partition using evaluated samples and their observed risks. Consequently, the stated guarantee does not directly cover the complete adaptive partition-learning procedure, for which an additional argument accounting for the data-dependent updates would be required. This distinction concerns the scope of the current theoretical analysis rather than the validity of adaptive partitioning in general. By comparison, \ourmethod~allows the sampling distribution to depend on the observed history while preserving the conditional unbiasedness of its inferential signal, thereby supporting anytime-valid inference under the adaptive sampling and stopping rules considered here.

% [ZJ]
% Cer-Eval is therefore a particularly relevant empirical baseline for \ourmethod. We compare the two methods under the same MMLU evaluation setting and target precision. Table~\ref{tab:cereval_comparison} reports the number of queried target losses required to reach the target precision, averaged over three random seeds, $42$, $43$, and $44$. The surrogate model listed in the first column is used only by \ourmethod, whereas Cer-Eval does not use a surrogate model.
\cereval~is therefore a particularly relevant empirical baseline for \ourmethod. We compare the two methods under the same MMLU evaluation setting and target precision. Table~\ref{tab:cereval_comparison} reports the number of queried target losses required to reach the target precision, averaged over three random seeds, $42$, $43$, and $44$. The surrogate model listed in the first column is used only by \ourmethod, whereas \cereval~does not use a surrogate model.

\begin{table*}[t]
\centering
\small
% [ZJ]
% \caption{Number of target labels required by CELEUS and Cer-Eval to reach the target precision on MMLU. Results are averaged over three random seeds. Label reduction is computed relative to Cer-Eval. Lower values are better.}
% [ZJ]
% \caption{Number of target labels required by \ourmethod~and Cer-Eval to reach the target precision on MMLU. Results are averaged over three random seeds. Label reduction is computed relative to Cer-Eval. Lower values are better.}
\caption{Number of target labels required by \ourmethod~and \cereval~to reach the target precision on MMLU. Results are averaged over three random seeds. Label reduction is computed relative to \cereval. Lower values are better.}
\label{tab:cereval_comparison}
\begin{tabular}{lccc}
\toprule
Surrogate $\rightarrow$ Target
% [ZJ]
% & CELEUS
& \ourmethod
% [ZJ]
% & Cer-Eval
& \cereval
& Label Reduction \\
\midrule
Llama-3-8B $\rightarrow$ Llama-3-70B
& $2{,}373.7$ & $3{,}771.0$ & $37.1\%$ \\
Llama-2-7B $\rightarrow$ Llama-3-70B
& $2{,}799.0$ & $3{,}771.0$ & $25.8\%$ \\
Llama-3-8B $\rightarrow$ Qwen2.5-72B
& $1{,}478.3$ & $2{,}349.0$ & $37.1\%$ \\
Llama-3-8B $\rightarrow$ Mixtral-8$\times$7B
& $3{,}856.0$ & $5{,}474.7$ & $29.6\%$ \\
Llama-2-7B $\rightarrow$ DeepSeek-67B
& $5{,}394.7$ & $6{,}451.7$ & $16.4\%$ \\
Llama-3-8B $\rightarrow$ DeepSeek-67B
& $4{,}734.7$ & $6{,}451.7$ & $26.6\%$ \\
\bottomrule
\end{tabular}
\end{table*}

% [ZJ]
% Across all six surrogate--target pairs, CELEUS requires fewer target labels than Cer-Eval to reach the same target precision. The reduction ranges from $16.4\%$ to $37.1\%$, with an average reduction of $28.8\%$. The improvement is observed across all four target models considered, including both dense and mixture-of-experts architectures, indicating that it is not driven by a single favorable model pair.
% [ZJ]
% Across all six surrogate--target pairs, \ourmethod~requires fewer target labels than Cer-Eval to reach the same target precision. The reduction ranges from $16.4\%$ to $37.1\%$, with an average reduction of $28.8\%$. The improvement is observed across all four target models considered, including both dense and mixture-of-experts architectures, indicating that it is not driven by a single favorable model pair.
Across all six surrogate--target pairs, \ourmethod~requires fewer target labels than \cereval~to reach the same target precision. The reduction ranges from $16.4\%$ to $37.1\%$, with an average reduction of $28.8\%$. The improvement is observed across all four target models considered, including both dense and mixture-of-experts architectures, indicating that it is not driven by a single favorable model pair.

The results also illustrate how the choice of surrogate can affect the efficiency of \ourmethod. For both Llama-3-70B and DeepSeek-67B, using Llama-3-8B as the surrogate results in a larger label reduction than using Llama-2-7B. Since the \cereval~result remains unchanged for a fixed target, this variation arises from the \ourmethod~surrogate-target pairing. This observation is consistent with our theoretical analysis: when a surrogate more accurately approximates the target losses, the resulting surrogate residuals and the conditional variance of the inferential signals can be reduced, allowing the confidence interval to reach the target precision with fewer evaluated labels. Overall, these results show that the label-efficiency improvement of \ourmethod~over \cereval~persists across the surrogate--target configurations considered in this MMLU experiment.

\appsubsection{Remarks on practical choices and implementation details}
\label{app:remarks}

We first describe how the uncertainty-guided sampling strategies are instantiated for the two loss families used in our experiments.

\begin{remark}[Surrogate scoring strategies and sample-selection scores for $0$-$1$ loss.]
\label{rem:0-1}
We describe one oracle strategy and three practical strategies for the case of $0$-$1$ loss.

Let $p_f(\cdot \mid \x)$ and $p_g(\cdot \mid \x)$ denote the predictive distributions of the target LLM $f$ and a surrogate model $g$, respectively. Define the true $0$-$1$ loss of the target model by
\[
\ell_{0\text{-}1}(f,\x,y)=\bI\{\hat y_f(\x)\neq y\},
\qquad
\hat y_f(\x)=\arg\max_{y'\in\cY}p_f(y'\mid\x).
\]
For any model $h\in\{f,g\}$, define the predictive dispersion score
\[
\ell_{\mathrm{disp}}(h,\x)=1-\sum_{y\in\cY}p_h^2(y\mid\x).
\]
For $h_1,h_2\in\{f,g\}$, define the hard pseudo-loss and soft pseudo-loss as
\[
\ell_{\mathrm{hard}}(h_1,h_2,\x)
=
\bI\{\hat y_{h_1}(\x)\neq \hat y_{h_2}(\x)\},
\]
and
\[
\ell_{\mathrm{soft}}(h_1,h_2,\x)
=
1-p_{h_1}(\hat y_{h_2}(\x)\mid \x),
\qquad
\hat y_h(\x)=\arg\max_{y\in\cY}p_h(y\mid\x).
\]
Here, the top-1 prediction of $h_2$ is used as a pseudo label for assessing the prediction of $h_1$.

\begin{itemize}
    \item \textbf{Oracle strategy} (true $0$-$1$ loss and surrogate score). If the true $0$-$1$ loss of the target model were accessible, we would set
    \[
    \ell(f(\x_j),y_j)=\ell_{0\text{-}1}(f,\x_j,y_j),
    \qquad
    \tilde{\ell}_f(\x_j)\in
    \left\{
    \ell_{\mathrm{disp}}(f,\x_j),
    \ell_{\mathrm{soft}}(f,g,\x_j)
    \right\}.
    \]
    The sample-selection score is
    \[
    |\ell(f(\x_j),y_j)-\tilde{\ell}_f(\x_j)|.
    \]
%     This oracle strategy is infeasible in practice because it requires the true loss values of unevaluated samples.
    % [ZJ][M4Cr-W1]
    % This oracle strategy is infeasible in practice because it requires the true loss values of unevaluated samples. In the oracle experiments, these scores are sampled with $\beta=0$ according to Remark~\ref{rem:oracle-experiments}.
    This oracle strategy is infeasible in practice because it requires the true loss values of unevaluated samples. In the oracle experiments, these scores are sampled with $\beta=0.4$ according to Remark~\ref{rem:oracle-experiments}.

    \item \textbf{Strategy A} (hard pseudo-loss as target-loss proxy, predictive dispersion as surrogate score). Set
    \[
    \ell(f(\x_j),y_j)\approx \ell_{\mathrm{hard}}(f,g,\x_j),
    \qquad
    \tilde{\ell}_f(\x_j)=\ell_{\mathrm{disp}}(f,\x_j).
    \]
    The sample-selection score is
    \[
    |\ell_{\mathrm{hard}}(f,g,\x_j)-\ell_{\mathrm{disp}}(f,\x_j)|.
    \]
    This compares the hard disagreement between the target and surrogate predictions with the dispersion of the target model's predictive distribution.

    \item \textbf{Strategy B} (predictive dispersion as target-loss proxy, soft pseudo-loss as surrogate score). Set
    \[
    \ell(f(\x_j),y_j)\approx \ell_{\mathrm{disp}}(f,\x_j),
    \qquad
    \tilde{\ell}_f(\x_j)=\ell_{\mathrm{soft}}(f,g,\x_j).
    \]
    The sample-selection score is
    \[
    |\ell_{\mathrm{disp}}(f,\x_j)-\ell_{\mathrm{soft}}(f,g,\x_j)|.
    \]
    This compares the dispersion of the target model's predictive distribution with its soft loss evaluated at the surrogate model's top-1 prediction.

    \item \textbf{Strategy C} (soft pseudo-loss as target-loss proxy, predictive dispersion as surrogate score). Set
    \[
    \ell(f(\x_j),y_j)\approx \ell_{\mathrm{soft}}(f,g,\x_j),
    \qquad
    \tilde{\ell}_f(\x_j)=\ell_{\mathrm{disp}}(f,\x_j).
    \]
    The sample-selection score is
    \[
    |\ell_{\mathrm{soft}}(f,g,\x_j)-\ell_{\mathrm{disp}}(f,\x_j)|,
    \]
    which is numerically identical to that of Strategy~B. However, Strategies~B and~C differ in which quantity plays the role of the target-loss proxy and which plays the role of the surrogate score inside the estimator $\hat R_t$ in \eqref{eq:lure_estimator}. This distinction changes the correction term and can affect the variance of the estimator and the resulting $e$-process.
\end{itemize}

The three practical strategies use the surrogate model in different ways. Strategy~A uses the surrogate model's top-1 prediction to form a hard pseudo-loss, while Strategies~B and~C use the surrogate model's top-1 prediction to form a soft pseudo-loss. In our main $0$-$1$ loss experiments, we use Strategy~A. The other strategies are also compatible with our framework, and we compare them empirically in App.~\ref{app:acquisition-strategy}; see Fig.~\ref{fig:acq-01}.
\end{remark}

\begin{remark}[Surrogate scoring strategies and sample-selection scores for cross-entropy loss.]
\label{rem:cross_entropy}
Recall that, when the residuals are not all zero, the oracle sampling distribution in Theorem~\ref{thm:optimal} is
\[
% [ZJ] notation: q_t^\star -> q_t^* (as in Theorem~\ref{thm:optimal}); notation: \tau_t -> \eta_t (\tau denotes the stopping time)
% q_t^\star(j)
% =
% \max\left\{
% \frac{\beta}{N-t+1},
% \frac{|\ell(f(\x_j),y_j)-\tilde{\ell}_f(\x_j)|}{\tau_t}
% \right\},
% \]
% where $\tau_t>0$ is chosen so that the probabilities sum to one. This distribution is unavailable in practice because it depends on the unknown label $y_j$. To obtain computable sample-selection scores, we replace the unknown target loss with label-free proxies when constructing the sampling distribution.
q_t^*(j)
=
\max\left\{
\frac{\beta}{N-t+1},
\frac{|\ell(f(\x_j),y_j)-\tilde{\ell}_f(\x_j)|}{\eta_t}
\right\},
\]
where $\eta_t>0$ is chosen so that the probabilities sum to one. This distribution is unavailable in practice because it depends on the unknown label $y_j$. To obtain computable sample-selection scores, we replace the unknown target loss with label-free proxies when constructing the sampling distribution.

Let $p_f(\cdot\mid\x)$ and $p_g(\cdot\mid\x)$ denote the predictive distributions of the target LLM $f$ and a surrogate model $g$, respectively. Before the sequential evaluation, the evaluator prespecifies a fixed clipping level $B>0$. Throughout the cross-entropy evaluation, we define the target loss as
\[
\ell_{\mathrm{ce}}(f,\x,y)
=
\min\left\{
-\log p_f(y\mid\x),
B
\right\}.
\]
For any model $h\in\{f,g\}$, define its clipped predictive entropy score as
\[
H(h,\x)
=
\min\left\{
-\sum_{y\in\cY}p_h(y\mid\x)\log p_h(y\mid\x),
B
\right\}.
\]
We also define the clipped mode loss of the target model as
\[
\ell_{\mathrm{mode}}(f,\x)
=
\min\left\{
-\log p_f(\hat y_f(\x)\mid\x),
B
\right\},
\qquad
\hat y_f(\x)=\arg\max_{y\in\cY}p_f(y\mid\x),
\]
which uses the target model's own top-1 prediction as a pseudo label.

In all strategies below, the target loss used in the estimator after a sample is evaluated remains $\ell_{\mathrm{ce}}(f,\x_j,y_j)$. The label-free proxy is used only to construct the sample-selection score.

\begin{itemize}
    \item \textbf{Oracle strategy} (true cross-entropy loss and surrogate score). If the true clipped cross-entropy loss were accessible, we would set
    \[
    \ell(f(\x_j),y_j)=\ell_{\mathrm{ce}}(f,\x_j,y_j),
    \qquad
    \tilde{\ell}_f(\x_j)\in
    \left\{
    H(g,\x_j), H(f,\x_j), \ell_{\mathrm{mode}}(f,\x_j)
    \right\}.
    \]
    The sample-selection score is
    \[
    |\ell_{\mathrm{ce}}(f,\x_j,y_j)-\tilde{\ell}_f(\x_j)|.
    \]
%     This oracle strategy is infeasible because it requires the true loss values of unevaluated samples.
    % [ZJ][M4Cr-W1]
    % This oracle strategy is infeasible because it requires the true loss values of unevaluated samples. In the oracle experiments, these scores are sampled with $\beta=0$ according to Remark~\ref{rem:oracle-experiments}.
    This oracle strategy is infeasible because it requires the true loss values of unevaluated samples. In the oracle experiments, these scores are sampled with $\beta=0.4$ according to Remark~\ref{rem:oracle-experiments}.

    \item \textbf{Strategy A} (mode loss as target-loss proxy, entropy as surrogate score). For sample selection, set
    \[
    \ell_{\mathrm{proxy}}(\x_j)=\ell_{\mathrm{mode}}(f,\x_j),
    \qquad
    \tilde{\ell}_f(\x_j)=H(f,\x_j).
    \]
    The sample-selection score is
    \[
    |\ell_{\mathrm{mode}}(f,\x_j)-H(f,\x_j)|.
    \]
    This compares the loss at the target model's top-1 prediction with the entropy of its predictive distribution.

    \item \textbf{Strategy B} (entropy as target-loss proxy, mode loss as surrogate score). For sample selection, set
    \[
    \ell_{\mathrm{proxy}}(\x_j)=H(f,\x_j),
    \qquad
    \tilde{\ell}_f(\x_j)=\ell_{\mathrm{mode}}(f,\x_j).
    \]
    The sample-selection score is
    \[
    |H(f,\x_j)-\ell_{\mathrm{mode}}(f,\x_j)|,
    \]
    which is numerically identical to that of Strategy~A. However, Strategies~A and~B use different surrogate scores $\tilde{\ell}_f$ in the estimator $\hat R_t$ in \eqref{eq:lure_estimator}. This changes the residual correction and can affect the variance of the estimator and the resulting $e$-process.

    \item \textbf{Strategy C} (target-surrogate entropy gap). For sample selection, set
    \[
    \ell_{\mathrm{proxy}}(\x_j)=H(f,\x_j),
    \qquad
    \tilde{\ell}_f(\x_j)=H(g,\x_j).
    \]
    The sample-selection score is
    \[
    |H(f,\x_j)-H(g,\x_j)|.
    \]
    This compares the predictive entropies of the target and surrogate models, capturing inter-model disagreement at the distributional level.
\end{itemize}

Computing the practical acquisition scores above requires the target predictive outputs $p_f(\cdot\mid\x_j)$ over the full pool. Thus, the current implementation uses these outputs to determine which target losses should be revealed, but does not reduce the number of target forward passes. Its direct efficiency benefit is a reduction in the labels, human judgments, or model-judge evaluations required for certification. Acquisition rules based entirely on surrogate or previously available information could avoid full-pool target inference, but are not the instantiation evaluated in this paper.

All target losses, surrogate scores, and target-loss proxies defined above lie in $[0,B]$. Therefore, Assumption~\ref{ass:loss_bound} holds with $L=0$ and $U=B$, and every sample-selection score lies in $[0,B]$. In this setting, $R_N$ and $R$ denote the finite-pool and population means of the clipped target cross-entropy loss, respectively.

Strategy~C is the only strategy that uses a separate surrogate model $g$, while Strategies~A and~B are self-referential and use only the target model $f$. In our main cross-entropy experiments, we use Strategy~A. The other strategies are also compatible with our framework, and we compare them empirically in App.~\ref{app:acquisition-strategy}; see Fig.~\ref{fig:acq-ce}.
\end{remark}

Overall, these practical strategies are uncertainty-guided sample-selection rules: they use label-free quantities to approximate the oracle residual magnitude $|\ell(f(\x_j),y_j)-\tilde{\ell}_f(\x_j)|$ in Theorem~\ref{thm:optimal}. Since this residual determines the variance-optimal sampling distribution, assigning larger probabilities to samples with larger proxy residuals aims to reduce the conditional variance of $\hat S_t$ and hence accelerate CI-width shrinkage.
% [3HvM-Q4]
These label-free quantities follow sample-selection scores commonly used in active testing and active statistical inference, such as the expected surrogate loss, target--surrogate disagreement, and predictive uncertainty~\cite{Kossen:Farquhar:Sebastian:Gal:Rainforth2021,kossen:Farquhar:Gal:Rainforth2022,Berrada:Kossen:Freddie:Razzak:Gal:Rainforth2025,Zrnic:Candes2024}. The choice among them affects efficiency but not validity, and the most effective strategy depends on the task, the evaluation loss, and the available model outputs. Principled strategy selection remains an open problem.

Having specified the practical sample-selection strategies, we next describe the implementation choices used to construct the $e$-processes.
\begin{remark}\label{rem:betting-fractions}
Following~\cite{Waudby-Smith:Ramdas2024}, we use predictable, time-adaptive betting fractions to control the Type-I error uniformly over time, ensuring that the resulting confidence sequence remains valid under data-dependent stopping. Since the scaling bounds $(a_t,b_t)$ may vary across rounds, each candidate value $\upsilon$ is transformed using the same round-specific scaling as the signal:
$$
\bar\upsilon_t=\frac{\upsilon-a_t}{b_t-a_t}.
$$

To construct candidate-independent base fractions, define the regularized running mean and variability estimate
$$
\hat\mu_t
=
\frac{1/2+\sum_{i=1}^t\bar S_i}{t+1},
\qquad
\hat\sigma_t^2
=
\frac{1/4+\sum_{i=1}^t(\bar S_i-\hat\mu_i)^2}{t+1}.
$$

Using the constants $\underline\gamma$ and $\overline\gamma$ from Assumption~\ref{ass:betting}, we set
$$
\gamma_t
=
\left(
\sqrt{\frac{2\log(2/\alpha_1)}{\hat\sigma_{t-1}^2}}
\vee\underline\gamma
\right)
\wedge\overline\gamma,
\qquad
\tilde\lambda_t^+
=
\tilde\lambda_t^-
=
\frac{\gamma_t}{\sqrt{t\log(t+1)}}.
$$

For each candidate value $\upsilon$, the final betting fractions are
$$
\lambda_t^+(\upsilon)
=
\tilde\lambda_t^+
\wedge
\frac{1}{\bar\upsilon_t+c},
\qquad
\lambda_t^-(\upsilon)
=
\tilde\lambda_t^-
\wedge
\frac{1}{1-\bar\upsilon_t+c},
$$
where $c\in(0,1)$ is a small constant.

Because $\hat\sigma_{t-1}^2$, $a_t$, and $b_t$ depend only on information available before round $t$, the resulting fractions are predictable. Under $\mathcal H_0(\upsilon)$, the centered scaled signal $\bar S_t-\bar\upsilon_t$ has conditional mean zero. The truncation of $\gamma_t$ ensures the envelope condition in Assumption~\ref{ass:betting}, while the candidate-specific clipping keeps the final fractions within their admissible ranges.
\end{remark}

% \begin{remark}\label{rem:betting-fractions}
% Following~\cite{Waudby-Smith:Ramdas2024}, we use predictable, time-adaptive fractions to obtain Type-I error control uniformly over time, so that the resulting confidence sequence remains valid under data-dependent stopping. Since the scaling bounds $(a_t,b_t)$ may vary across rounds, each candidate value $\upsilon$ is also scaled at round $t$ as $\bar\upsilon_t=(\upsilon-a_t)/(b_t-a_t)$. Thus, the fractions are chosen based on the centered scaled signal $\bar S_i-\bar\upsilon_i$, which has conditional mean zero under $\mathcal H_0(\upsilon)$. Specifically, we set
% \[
% \lambda_t^+(\upsilon)=\sqrt{\frac{2\log(2/\alpha_1)}{\hat{\sigma}_{t-1}^2(\upsilon)\, t\log(t+1)}}\wedge\frac{1}{\bar\upsilon_t+c},
% \qquad
% \lambda_t^-(\upsilon)=\sqrt{\frac{2\log(2/\alpha_1)}{\hat{\sigma}_{t-1}^2(\upsilon)\, t\log(t+1)}}\wedge\frac{1}{1-\bar\upsilon_t+c},
% \]
% where $c>0$ is a small constant, and
% \[
% \hat{\sigma}_t^2(\upsilon)=\frac{1/4+\sum_{i=1}^t(\bar S_i-\bar\upsilon_i)^2}{t+1}.
% \]
% The clipping terms ensure that $\lambda_t^+(\upsilon)$ and $\lambda_t^-(\upsilon)$ stay within their admissible ranges, while $\hat{\sigma}_{t-1}^2(\upsilon)$ adapts the fractions to the past variability of the centered scaled signals.
% \end{remark}

Finally, we describe how the scaling bounds are chosen so that the per-round signal can be mapped to a bounded range before applying the $e$-process construction.
\begin{remark}\label{rem:scaling}
The scaling bounds $a_t$ and $b_t$ need to be predictable and satisfy $a_t\le \hat S_t\le b_t$ almost surely. In practice, one simple choice is to use fixed bounds $a_t=a$ and $b_t=b$ for all rounds. Suppose that
\[
% q_t(j\mid i_{1:t-1},D_N)\ge q_{\min}>0
q_t(j\mid i_{1:t-1},X_N)\ge q_{\min}>0
\]
for all $t$ and all $j\in\cJ_{t-1}$. We can rewrite the signal as
\begin{eqnarray*}
\hat S_t
&=&
\frac{
\sum_{m=1}^{t-1}\ell(f(\x_{i_m}),y_{i_m})
+
\sum_{j\in\cJ_t}\tilde{\ell}_f(\x_j)
+
\ell(f(\x_{i_t}),y_{i_t})
}{N}
\\
&&+
% \frac{1-q_t(i_t\mid i_{1:t-1},D_N)}
\frac{1-q_t(i_t\mid i_{1:t-1},X_N)}
% {Nq_t(i_t\mid i_{1:t-1},D_N)}
{Nq_t(i_t\mid i_{1:t-1},X_N)}
\left(
\ell(f(\x_{i_t}),y_{i_t})-\tilde{\ell}_f(\x_{i_t})
\right).
\end{eqnarray*}
The first term is an average of observed losses and surrogate loss scores, and therefore lies in $[L,U]$. The second term is bounded in absolute value by
\[
\frac{(U-L)(1-q_{\min})}{Nq_{\min}}.
\]
Thus, a valid fixed choice is
\[
a = L - \frac{(U-L)(1-q_{\min})}{Nq_{\min}},
\qquad
b = U + \frac{(U-L)(1-q_{\min})}{Nq_{\min}},
\]
which ensures $\hat S_t\in[a,b]$ almost surely for all $t$. Under Assumption~\ref{ass:beta}, we can take $q_{\min}=\beta/N$.
% [ZJ] Remark~\ref{rem:scaling} is closed here; it was previously closed after Remark~\ref{rem:sampling-probabilities} (nested remarks)
\end{remark}

\begin{remark}\label{rem:sampling-probabilities}
% For notational simplicity, we suppress the conditioning on $(i_{1:t-1},D_N)$ below. At round $t$, let $s_t(j)\ge0$ denote the sample-selection score for $j\in\cJ_{t-1}$. For the experimental oracle benchmark, we use the true residual score
For notational simplicity, we write $q_t(j)$ below, suppressing the conditioning on past selections and on $X_N$ for practical sampling or $D_N$ for the oracle benchmark. At round $t$, let $s_t(j)\ge0$ denote the sample-selection score for $j\in\cJ_{t-1}$. For the experimental oracle benchmark, we use the true residual score
\[
s_t^{\mathrm{oracle}}(j)
=
\left|
\ell(f(\x_j),y_j)-\tilde{\ell}_f(\x_j)
\right|.
\]
This score is unavailable to the practical method because $y_j$ has not yet been observed. For the practical samplers, we instead use
\[
s_t^{\mathrm{proxy}}(j)
=
\left|
\ell_{\mathrm{proxy}}(\x_j)-\tilde{\ell}_f(\x_j)
\right|,
\]
where the label-free target-loss proxies and surrogate scores are defined in Remarks~\ref{rem:0-1} and~\ref{rem:cross_entropy}.

% Given either score, we first define
For the practical scores, we first define
\[
p_t(j)
=
\begin{cases}
\displaystyle
\frac{s_t(j)}
{\sum_{k\in\cJ_{t-1}}s_t(k)},
&
\displaystyle
\sum_{k\in\cJ_{t-1}}s_t(k)>10^{-12},
\\
\displaystyle
\frac{1}{N-t+1},
&
\textnormal{otherwise}.
\end{cases}
\]
% The threshold $10^{-12}$ is a numerical safeguard: when the total score is negligible, we revert to uniform sampling. The final sampling distribution used in the experiments is
The threshold $10^{-12}$ is a numerical safeguard: when the total score is negligible, we revert to uniform sampling. The final practical sampling distribution is
\[
% q_t(j\mid i_{1:t-1},D_N)
q_t(j)
=
\frac{\beta}{N-t+1}
+
(1-\beta)p_t(j).
\]
This construction satisfies
\[
% q_t(j\mid i_{1:t-1},D_N)
q_t(j)
\ge
\frac{\beta}{N-t+1},
\]
% as required by Assumption~\ref{ass:beta}. We set $\beta=0.4$ in all experiments to balance uniform exploration and score-guided sampling. The probability floor prevents very small sampling probabilities and unstable inverse-probability corrections when the practical scores only approximate the unknown residuals.
as required by Assumption~\ref{ass:beta}. For practical sampling, we use $\beta=0.4$ by default, unless otherwise specified, to balance uniform exploration and score-guided sampling. The probability floor prevents very small sampling probabilities and unstable inverse-probability corrections when the practical scores only approximate the unknown residuals.

% Here, ``oracle'' refers to access to the true residual scores when constructing the experimental benchmark. The floor-mixture above is the stable sampling rule used in our experiments, while Theorem~\ref{thm:optimal} separately characterizes the exact constrained variance-minimizing distribution. For the practical method, the proxy scores affect only sample selection; after a sample is selected, its observed target loss is used in the estimator and the $e$-process. Provided that the scores are available before selection and the resulting sampling probabilities are used in the inverse-probability corrections, the choice of the scores and $\beta$ affects statistical efficiency but not the validity guarantees.
% [ZJ][M4Cr-W1]
% The oracle experiments use the separate $\beta=0$ rule in Remark~\ref{rem:oracle-experiments}. For the practical method, the proxy scores affect only sample selection; after a sample is selected, its observed target loss is used in the estimator and the $e$-process. Provided that the scores are available before selection, the probability-floor and scaling assumptions hold, and the actual sampling probabilities are used in the inverse-probability corrections, the proxy scores and positive $\beta$ affect efficiency without invalidating the coverage guarantee.
Here, ``oracle'' refers to access to the true residual scores when constructing the experimental benchmark (Remark~\ref{rem:oracle-experiments}). The floor-mixture above is the stable sampling rule used in our experiments, while Theorem~\ref{thm:optimal} separately characterizes the exact constrained variance-minimizing distribution. For the practical method, the proxy scores affect only sample selection; after a sample is selected, its observed target loss is used in the estimator and the $e$-process. Provided that the scores are available before selection, the probability-floor and scaling assumptions hold, and the actual sampling probabilities are used in the inverse-probability corrections, the proxy scores and positive $\beta$ affect efficiency without invalidating the coverage guarantee.
\end{remark}

\begin{remark}[Oracle experiments]\label{rem:oracle-experiments}
% [ZJ][M4Cr-W1]
% In the oracle experiments, the complete labeled pool $D_N$ is available to the sampling rule. We therefore set $\beta=0$: unlike a practical proxy, the true residual identifies exactly which samples contribute to the correction. A sample with zero residual contributes nothing, so it does not require a positive sampling probability while nonzero residuals remain. This is an idealized benchmark; the practical method retains a positive probability floor.
In the oracle experiments, the complete labeled pool $D_N$ is available to the sampling rule. The oracle uses the true residual magnitudes as the sample-selection scores in \eqref{eq:practical_sampling_probability}, with the same $\beta=0.4$ as the practical method. This is an idealized benchmark.

Write $m_t=N-t+1$, $r_{t,j}=\ell(f(\x_j),y_j)-\tilde{\ell}_f(\x_j)$, and $A_t=\sum_{j\in\cJ_{t-1}}|r_{t,j}|$. Retaining the score-normalization safeguard $\delta_{\mathrm{num}}=10^{-12}$, the experimental oracle uses
\begin{equation}\label{eq:experimental_oracle}
% [ZJ][M4Cr-W1]
% q_t^{\mathrm{orc}}(j)
% =
% \begin{cases}
% |r_{t,j}|/A_t, & A_t>\delta_{\mathrm{num}},\\
% 1/m_t, & A_t\le\delta_{\mathrm{num}}.
% \end{cases}
q_t^{\mathrm{orc}}(j)
=
\frac{\beta}{m_t}+(1-\beta)p_t^{\mathrm{res}}(j),
\qquad
p_t^{\mathrm{res}}(j)
=
\begin{cases}
|r_{t,j}|/A_t, & A_t>\delta_{\mathrm{num}},\\
1/m_t, & A_t\le\delta_{\mathrm{num}}.
\end{cases}
\end{equation}
Since $q_t^{\mathrm{orc}}(j)\ge\beta/m_t$, the oracle satisfies Assumption~\ref{ass:beta}. Here, $q_t^{\mathrm{orc}}$ is the implemented oracle-information reference based on unavailable true residuals; it is distinct from the exact constrained optimizer $q_t^*$ in Theorem~\ref{thm:optimal} unless the two constructions coincide.
\end{remark}

\appsubsection{Effects of surrogate and acquisition quality on CI width}
\label{app:surrogate_acquisition_quality}

We now clarify how surrogate quality and acquisition quality affect the confidence-interval width. The two quantities play different roles: the surrogate score determines the residuals appearing in the inferential signal, whereas the acquisition proxy determines how the sampling probabilities allocate evaluations across these residuals.

\paragraph{Theoretical analysis.}
Fix a round $t$ and write
\[
m_t=N-t+1,
\qquad
r_{t,j}=\ell(f(\x_j),y_j)-\tilde{\ell}_f(\x_j),
\qquad
j\in\cJ_{t-1}.
\]
If $r_{t,j}=0$ for every $j\in\cJ_{t-1}$, then Theorem~\ref{thm:signal_unbias} gives
\[
\mathrm{Var}(\hat S_t\mid\cF_{t-1})=0,
\]
and every proposal satisfying Assumption~\ref{ass:beta} is optimal. We therefore consider the case in which the residuals are not all zero. Define
\[
% [ZJ] notation: a_{t,j} -> |r_{t,j}| (a_t is the scaling bound)
% a_{t,j}=|r_{t,j}|,
% \qquad
% A_t=\sum_{j\in\cJ_{t-1}}a_{t,j},
% \qquad
% \pi_t(j)=\frac{a_{t,j}}{A_t}.
A_t=\sum_{j\in\cJ_{t-1}}|r_{t,j}|,
\qquad
\pi_t(j)=\frac{|r_{t,j}|}{A_t}.
\]
Here, $\pi_t$ is the residual-normalized distribution obtained without imposing the probability floor.

% [ZJ] notation: V_t(q) -> \mathrm{Var}_q(\hat S_t\mid\cF_{t-1}) (V_t is used in Theorem~\ref{thm:freedman}); notation: a_{t,j} -> |r_{t,j}| (a_t is the scaling bound)
% For a feasible proposal $q$, let $V_t(q)$ denote the conditional variance of $\hat S_t$ when $q$ is used at round $t$. The part of this variance that depends on the proposal is
% \[
% \Phi_t(q)
% =
% \frac{1}{N^2}\sum_{j\in\cJ_{t-1}}
% \frac{a_{t,j}^2}{q(j)}.
For a feasible proposal $q$, let $\mathrm{Var}_q(\hat S_t\mid\cF_{t-1})$ denote the conditional variance of $\hat S_t$ when $q$ is used at round $t$. The part of this variance that depends on the proposal is
\[
\Phi_t(q)
=
\frac{1}{N^2}\sum_{j\in\cJ_{t-1}}
\frac{r_{t,j}^2}{q(j)}.
\]
Since Assumption~\ref{ass:beta} ensures $q(j)>0$, we have the exact decomposition
\begin{align*}
\Phi_t(q)
&=
\frac{A_t^2}{N^2}
\sum_{j\in\cJ_{t-1}}
\frac{\pi_t(j)^2}{q(j)}
\\
&=
\frac{A_t^2}{N^2}
\left(1+\chi^2(\pi_t\|q)\right),
\end{align*}
where
\[
\chi^2(\pi_t\|q)
=
\sum_{j\in\cJ_{t-1}}
\frac{(\pi_t(j)-q(j))^2}{q(j)}.
\]
Consequently,
\[
% [ZJ] notation: V_t(q) -> \mathrm{Var}_q(\hat S_t\mid\cF_{t-1}) (V_t is used in Theorem~\ref{thm:freedman})
% V_t(q)
\mathrm{Var}_q(\hat S_t\mid\cF_{t-1})
=
\frac{A_t^2}{N^2}
\left(1+\chi^2(\pi_t\|q)\right)
-
\left(
\frac{1}{N}\sum_{j\in\cJ_{t-1}}r_{t,j}
\right)^2.
\]

To account for the probability floor, define
\[
\mathcal Q_{t,\beta}
=
\left\{
q:
\sum_{j\in\cJ_{t-1}}q(j)=1,
\quad
q(j)\ge\frac{\beta}{m_t}
\right\}.
\]
Because the second term in the preceding variance expression is independent of $q$, the constrained optimizer in Theorem~\ref{thm:optimal} equivalently satisfies
\[
q_t^*
=
\arg\min_{q\in\mathcal Q_{t,\beta}}
\chi^2(\pi_t\|q).
\]
Thus, for every feasible proxy-induced proposal $q_t$,
\begin{align*}
% [ZJ] notation: V_t(q) -> \mathrm{Var}_q(\hat S_t\mid\cF_{t-1}) (V_t is used in Theorem~\ref{thm:freedman})
% V_t(q_t)-V_t(q_t^*)
\mathrm{Var}_{q_t}(\hat S_t\mid\cF_{t-1})-\mathrm{Var}_{q_t^*}(\hat S_t\mid\cF_{t-1})
&=
\frac{A_t^2}{N^2}
\left(
\chi^2(\pi_t\|q_t)
-
\chi^2(\pi_t\|q_t^*)
\right)
\\
&\ge 0.
\end{align*}
If $\pi_t$ itself satisfies the probability floor, then $q_t^*=\pi_t$. Otherwise, $q_t^*$ is the water-filling proposal characterized in Theorem~\ref{thm:optimal}.

This decomposition separates the two approximation effects. First, the surrogate score $\tilde{\ell}_f$ determines the residuals $r_{t,j}$. In particular, Assumption~\ref{ass:beta} gives
\begin{align*}
\mathrm{Var}(\hat S_t\mid\cF_{t-1})
&\le
\frac{1}{N^2}\sum_{j\in\cJ_{t-1}}
\frac{r_{t,j}^2}{q_t(j)}
\\
&\le
\frac{m_t}{\beta N^2}
\sum_{j\in\cJ_{t-1}}r_{t,j}^2
\\
&=
\frac{m_t^2}{\beta N^2}\mathrm{MSE}_t,
\end{align*}
where
\[
\mathrm{MSE}_t
=
\frac{1}{m_t}\sum_{j\in\cJ_{t-1}}
\left(
\ell(f(\x_j),y_j)-\tilde{\ell}_f(\x_j)
\right)^2.
\]
Hence, a more accurate surrogate, in the sense of smaller residual mean squared error, tightens this upper bound on the conditional variance.

% [ZJ] notation: a_{t,j} -> |r_{t,j}| (a_t is the scaling bound)
% Second, the acquisition proxy is used to approximate the unknown residual magnitudes $a_{t,j}$ and thereby construct $q_t$. Holding the surrogate residuals fixed, its effect is measured exactly by
Second, the acquisition proxy is used to approximate the unknown residual magnitudes $|r_{t,j}|$ and thereby construct $q_t$. Holding the surrogate residuals fixed, its effect is measured exactly by
\[
\chi^2(\pi_t\|q_t)-\chi^2(\pi_t\|q_t^*).
\]
A proposal with a smaller divergence gap incurs less variance relative to the constrained oracle. Through the cumulative conditional-variance term in Theorem~\ref{thm:conf_sequen}, improvements in either surrogate accuracy or acquisition quality can therefore tighten the confidence-width bound.

In Remarks~\ref{rem:0-1} and~\ref{rem:cross_entropy}, the target-loss proxy is used only to construct the sample-selection score, whereas $\tilde{\ell}_f$ is the surrogate score used in $\hat S_t$ and $\hat R_t$. Therefore, two strategies may produce the same acquisition score and hence the same sampling probabilities, while still producing different residuals and conditional variances because they use different surrogate scores.

% These approximation choices affect efficiency rather than validity. No quantitative approximation-quality assumption is required for conditional unbiasedness. Provided that the surrogate scores and sampling probabilities are predictable, the boundedness assumptions hold, the actual proposal probabilities satisfy Assumption~\ref{ass:beta} and are used in the inverse-probability correction, and the observed target loss is used after selection, the coverage guarantees in Theorems~\ref{thm:conf_sequen} and~\ref{thm:coverage_convergence} remain unchanged. Less accurate approximations may increase the conditional variance and delay stopping, but they do not invalidate the confidence sequence.
These approximation choices affect efficiency rather than validity. No quantitative approximation-quality assumption is required for conditional unbiasedness. Provided that the surrogate scores are fixed before evaluation, the sampling probabilities depend only on observed information, the boundedness assumptions hold, the actual proposal probabilities satisfy Assumption~\ref{ass:beta} and are used in the inverse-probability correction, and the observed target loss is used after selection, the coverage guarantees in Theorems~\ref{thm:conf_sequen} and~\ref{thm:coverage_convergence} remain unchanged. The population guarantee additionally uses the i.i.d. input-label pairs specified in the setup. Less accurate approximations may increase the conditional variance and delay stopping, but they do not invalidate the confidence sequence.

\paragraph{Controlled semi-synthetic analysis.}
We further conduct a controlled semi-synthetic analysis that varies surrogate and acquisition quality independently. Let
\[
\ell_j=\ell(f(\x_j),y_j),
\qquad
\mathrm{MSE}(h,\ell)
=
\frac{1}{N}\sum_{j=1}^N
\left(h(\x_j)-\ell_j\right)^2.
\]
We measure surrogate quality using the normalized MSE skill relative to a fixed baseline predictor $\tilde{\ell}^{\mathrm{base}}$:
\[
Q_{\mathrm{surr}}
=
1-
\frac{
\mathrm{MSE}(\tilde{\ell}_f,\ell)
}{
\mathrm{MSE}(\tilde{\ell}^{\mathrm{base}},\ell)
}.
\]
A larger value of $Q_{\mathrm{surr}}$ therefore corresponds to a smaller surrogate residual MSE. We consider
\[
Q_{\mathrm{surr}}\in\{0.8,0.4,0.0\}.
\]

To control acquisition quality, we first construct an infeasible oracle-information reference proposal using the true residual magnitudes. Specifically, define
\[
p_t^{\mathrm{res}}(j)
=
\begin{cases}
\displaystyle
% [ZJ] notation: a_{t,j} -> |r_{t,j}| (a_t is the scaling bound)
% \frac{a_{t,j}}
% {\sum_{k\in\cJ_{t-1}}a_{t,k}},
% &
% \displaystyle
% \sum_{k\in\cJ_{t-1}}a_{t,k}>10^{-12},
\frac{|r_{t,j}|}
{\sum_{k\in\cJ_{t-1}}|r_{t,k}|},
&
\displaystyle
\sum_{k\in\cJ_{t-1}}|r_{t,k}|>10^{-12},
\\[1.2em]
\displaystyle
\frac{1}{m_t},
&
\textnormal{otherwise},
\end{cases}
\]
and
\[
% q_t^{\mathrm{orc}}(j)
% =
% \frac{\beta}{m_t}
% +
% (1-\beta)p_t^{\mathrm{res}}(j).
% [ZJ][M4Cr-W1]
% q_t^{\mathrm{orc}}(j)
% =
% p_t^{\mathrm{res}}(j).
q_t^{\mathrm{orc}}(j)
=
\frac{\beta}{m_t}
+
(1-\beta)p_t^{\mathrm{res}}(j).
\]
% We use $\beta=0.4$, as in the practical experiments. Here, $q_t^{\mathrm{orc}}$ is the implemented oracle-information reference based on unavailable true residuals; it is distinct from the exact constrained optimizer $q_t^*$ in Theorem~\ref{thm:optimal} unless the two constructions coincide.
% [ZJ][M4Cr-W1]
% We set $\beta=0$ for this oracle reference, as in the other oracle experiments. Its residual-proportional branch is the unconstrained one-step variance optimum described in Remark~\ref{rem:oracle-experiments}, rather than the positive-floor optimizer $q_t^*$ in Theorem~\ref{thm:optimal}.
We use $\beta=0.4$, as in the practical experiments. Here, $q_t^{\mathrm{orc}}$ is the implemented oracle-information reference based on unavailable true residuals; it is distinct from the exact constrained optimizer $q_t^*$ in Theorem~\ref{thm:optimal} unless the two constructions coincide.

Let
\[
q_t^{\mathrm{unif}}(j)=\frac{1}{m_t}.
\]
% [ZJ] notation: mixture weight \gamma -> \kappa (\gamma_t, \underline\gamma, \overline\gamma are betting constants)
% For $\gamma\in[0,1]$, we then define
% \[
% q_t^{(\gamma)}
% =
% \gamma q_t^{\mathrm{orc}}
% +
% (1-\gamma)q_t^{\mathrm{unif}}.
For $\kappa\in[0,1]$, we then define
\[
q_t^{(\kappa)}
=
\kappa q_t^{\mathrm{orc}}
+
(1-\kappa)q_t^{\mathrm{unif}}.
\]
% Both component distributions satisfy the probability floor, so every $q_t^{(\gamma)}$ also satisfies Assumption~\ref{ass:beta}. We measure acquisition quality by its normalized total-variation proximity to the oracle-information reference:
% [ZJ][M4Cr-W1]
% For $\gamma<1$, the mixture satisfies $q_t^{(\gamma)}(j)\ge(1-\gamma)/m_t$. At $\gamma=1$, it is the $\beta=0$ oracle, with zero residuals and scaling handled as in Remark~\ref{rem:oracle-experiments}. We measure acquisition quality by its normalized total-variation proximity to the oracle reference:
% [ZJ] notation: mixture weight \gamma -> \kappa (\gamma_t, \underline\gamma, \overline\gamma are betting constants)
% Both component distributions satisfy the probability floor, so every $q_t^{(\gamma)}$ also satisfies Assumption~\ref{ass:beta}. We measure acquisition quality by its normalized total-variation proximity to the oracle-information reference:
% \[
% Q_{\mathrm{acq}}
% =
% 1-
% \frac{
% \mathrm{TV}\left(q_t^{(\gamma)},q_t^{\mathrm{orc}}\right)
% }{
% \mathrm{TV}\left(q_t^{\mathrm{unif}},q_t^{\mathrm{orc}}\right)
% }
% =
% \gamma,
Both component distributions satisfy the probability floor, so every $q_t^{(\kappa)}$ also satisfies Assumption~\ref{ass:beta}. We measure acquisition quality by its normalized total-variation proximity to the oracle-information reference:
\[
Q_{\mathrm{acq}}
=
1-
\frac{
\mathrm{TV}\left(q_t^{(\kappa)},q_t^{\mathrm{orc}}\right)
}{
\mathrm{TV}\left(q_t^{\mathrm{unif}},q_t^{\mathrm{orc}}\right)
}
=
\kappa,
\]
where
\[
\mathrm{TV}(p,q)
=
\frac{1}{2}\sum_{j\in\cJ_{t-1}}|p(j)-q(j)|.
\]
The equality holds whenever $q_t^{\mathrm{orc}}\ne q_t^{\mathrm{unif}}$; when the two distributions coincide, there is no acquisition-quality gap to control. We consider
\[
Q_{\mathrm{acq}}\in\{1.0,0.5,0.0\},
\]
representing oracle-information sampling, an equal oracle--uniform mixture, and uniform sampling, respectively.

For each combination of surrogate and acquisition quality, we generate five surrogate constructions and run ten sampling repetitions, giving $50$ trajectories per cell. Table~\ref{tab:controlled-quality} reports the number of evaluated samples required to reach CI width $\epsilon=0.05$ as the median and interquartile range.

\begin{table*}[t]
\centering
\small
\caption{Number of evaluated samples required to reach CI width $\epsilon=0.05$, reported as median [interquartile range] across $50$ trajectories. Lower is better.}
\label{tab:controlled-quality}
\begin{tabular}{llccc}
\toprule
Task
&
Acquisition quality
&
$Q_{\mathrm{surr}}=0.8$
&
$Q_{\mathrm{surr}}=0.4$
&
$Q_{\mathrm{surr}}=0.0$
\\
\midrule
SST-2
&
$Q_{\mathrm{acq}}=1.0$
&
$1{,}830$ [$1{,}801$--$1{,}863$]
&
$4{,}542$ [$4{,}490$--$4{,}629$]
&
$8{,}060$ [$7{,}983$--$8{,}172$]
\\
SST-2
&
$Q_{\mathrm{acq}}=0.5$
&
$2{,}133$ [$2{,}060$--$2{,}168$]
&
$6{,}062$ [$5{,}890$--$6{,}148$]
&
$11{,}040$ [$10{,}881$--$11{,}244$]
\\
SST-2
&
$Q_{\mathrm{acq}}=0.0$
&
$2{,}718$ [$2{,}681$--$2{,}776$]
&
$8{,}842$ [$8{,}586$--$9{,}079$]
&
$16{,}854$ [$16{,}529$--$17{,}287$]
\\
\midrule
MMLU
&
$Q_{\mathrm{acq}}=1.0$
&
$1{,}339$ [$1{,}328$--$1{,}346$]
&
$1{,}741$ [$1{,}712$--$1{,}766$]
&
$2{,}434$ [$2{,}404$--$2{,}510$]
\\
MMLU
&
$Q_{\mathrm{acq}}=0.5$
&
$1{,}387$ [$1{,}377$--$1{,}399$]
&
$2{,}236$ [$2{,}176$--$2{,}275$]
&
$3{,}532$ [$3{,}410$--$3{,}636$]
\\
MMLU
&
$Q_{\mathrm{acq}}=0.0$
&
$1{,}678$ [$1{,}628$--$1{,}714$]
&
$4{,}054$ [$3{,}902$--$4{,}270$]
&
$7{,}348$ [$7{,}148$--$7{,}664$]
\\
\bottomrule
\end{tabular}
\end{table*}

Holding surrogate quality fixed, improving acquisition quality from $Q_{\mathrm{acq}}=0$ to $Q_{\mathrm{acq}}=1$ reduces the median sample requirement by $32.7\%$--$52.2\%$ on SST-2 and $20.2\%$--$66.9\%$ on MMLU. Holding acquisition quality fixed, improving surrogate quality from $Q_{\mathrm{surr}}=0$ to $Q_{\mathrm{surr}}=0.8$ reduces the median sample requirement by $77.3\%$--$83.9\%$ on SST-2 and $45.0\%$--$77.2\%$ on MMLU.

These results are consistent with the theoretical decomposition above. Better surrogate scores reduce the residual-error component, whereas better acquisition scores reduce the variance inflation caused by allocating sampling probability differently from the oracle-information reference. Both improvements therefore allow the confidence interval to reach the target width using fewer evaluated samples.

\appsubsection{From per-round signals to performance estimation}
\label{appendix:signal_to_estimator}

In this appendix, we explain how the final performance estimator in \eqref{eq:lure_estimator} is obtained by aggregating the correction information exposed by the inferential signal in \eqref{eq:signal_estimator}. The inferential signal is designed for $e$-process construction, while the final estimator uses the same surrogate-correction structure to report a scalar performance estimate at stopping.

Recall that the per-round inferential signal is
\[
\hat{S}_t
=
\frac{
\sum_{m=1}^{t-1}\ell\left(f(\x_{i_m}),y_{i_m}\right)
+
\sum_{j\in\cJ_{t-1}}\tilde{\ell}_f(\x_j)
}{N}
+
\frac{
\ell\left(f(\x_{i_t}),y_{i_t}\right)
-
\tilde{\ell}_f(\x_{i_t})
}{
% Nq_t(i_t\mid i_{1:t-1},D_N)
Nq_t(i_t\mid i_{1:t-1},X_N)
}.
\]
Equivalently, it can be written as
\[
\hat S_t
=
\frac{1}{N}\sum_{n=1}^{N}\tilde{\ell}_f(\x_n)
+
\frac{1}{N}\sum_{m=1}^{t-1}
\left(
\ell\left(f(\x_{i_m}),y_{i_m}\right)-\tilde{\ell}_f(\x_{i_m})
\right)
+
\frac{
\ell\left(f(\x_{i_t}),y_{i_t}\right)-\tilde{\ell}_f(\x_{i_t})
}{
% Nq_t(i_t\mid i_{1:t-1},D_N)
Nq_t(i_t\mid i_{1:t-1},X_N)
}.
\]
This form shows that each round contributes a surrogate correction term
\[
\ell\left(f(\x_{i_t}),y_{i_t}\right)-\tilde{\ell}_f(\x_{i_t}),
\]
which measures the discrepancy between the observed risk and the surrogate risk score on the selected sample. The signal uses this correction in an inverse-probability-weighted way so that it remains conditionally unbiased for $R_N$ at each round.

To obtain the final estimator at horizon $t$, we aggregate the correction terms revealed over the first $t$ rounds. Before observing the risk at round $t$, the horizon-$t$ estimator has the predictable component
\[
B_{t-1}^{(t)}
=
\frac{1}{N}\sum_{n=1}^{N}\tilde{\ell}_f(\x_n)
+
\frac{1}{t}\sum_{m=1}^{t-1}u_{m,t}
\left(
\ell\left(f(\x_{i_m}),y_{i_m}\right)-\tilde{\ell}_f(\x_{i_m})
\right).
\]
Relative to this predictable component, the inferential signal can be decomposed as
\[
\hat S_t = B_{t-1}^{(t)}+\hat Z_t,
\]
where
\begin{align*}
\hat Z_t
&=
\left[
\frac{1}{N}\sum_{n=1}^{N}\tilde{\ell}_f(\x_n)
+
\frac{1}{N}\sum_{m=1}^{t-1}
\left(
\ell\left(f(\x_{i_m}),y_{i_m}\right)-\tilde{\ell}_f(\x_{i_m})
\right)
-
B_{t-1}^{(t)}
\right]
\\
&\qquad
+
\frac{
\ell\left(f(\x_{i_t}),y_{i_t}\right)-\tilde{\ell}_f(\x_{i_t})
}{
% Nq_t(i_t\mid i_{1:t-1},D_N)
Nq_t(i_t\mid i_{1:t-1},X_N)
}.
\end{align*}
Here, $\hat Z_t$ is the one-step correction used to turn the predictable component into a conditionally unbiased signal for $e$-process construction.

For performance estimation, we use the same predictable component but aggregate the observed correction terms with horizon-dependent weights. After observing the risk at round $t$, the final estimator is
\[
\hat{R}_t
=
B_{t-1}^{(t)}
+
\frac{1}{t}u_{t,t}
\left(
\ell\left(f(\x_{i_t}),y_{i_t}\right)-\tilde{\ell}_f(\x_{i_t})
\right).
\]
Substituting the definition of $B_{t-1}^{(t)}$ yields
\[
\hat{R}_t
=
\frac{1}{N}\sum_{n=1}^{N}\tilde{\ell}_f(\x_n)
+
\frac{1}{t}\sum_{m=1}^{t}u_{m,t}
\left(
\ell\left(f(\x_{i_m}),y_{i_m}\right)-\tilde{\ell}_f(\x_{i_m})
\right),
\]
where
\[
u_{m,t}
=
1+\frac{N-t}{N-m}
\left(
% \frac{1}{(N-m+1)q_m(i_m\mid i_{1:m-1},D_N)}-1
\frac{1}{(N-m+1)q_m(i_m\mid i_{1:m-1},X_N)}-1
\right).
\]
Thus, the estimator in \eqref{eq:lure_estimator} is obtained by aggregating the same surrogate correction terms that appear in the inferential signal. The signal $\hat S_t$ uses a one-step correction for valid $e$-process updates, whereas $\hat R_t$ aggregates the correction information to produce the final performance estimate.
% [BjfK-Q4]
This construction also clarifies the connection to Horvitz--Thompson estimation~\cite{Horvitz:Thompson1952}. Conditional on $\cF_{t-1}$, the inverse-probability-weighted correction in $\hat S_t$ is an unbiased estimate of the normalized total residual over the remaining pool, so $\hat S_t$ adds a conditional Horvitz--Thompson-type correction to the surrogate-completed estimate. The estimator $\hat R_t$ also uses inverse-probability weighting, but it is not the standard Horvitz--Thompson estimator: its horizon-dependent weights $u_{m,t}$ follow the LURE construction for sequential, non-uniform sample selection~\cite{Farquhar:Gal:Rainforth2021}.

\appsection{End-to-end implementation of \ourmethod}
\label{app:implementation}

% This section presents the complete implementation of \ourmethod, connecting the practical sample-selection strategies in Remarks~\ref{rem:0-1} and~\ref{rem:cross_entropy} to the proposal probabilities, inferential signals, $e$-processes, confidence intervals, stopping rule, and final performance estimate. At each round, the acquisition scores, proposal distribution, scaling bounds, and betting fractions are determined before the newly selected target loss is observed.
% [ZJ][M4Cr-W1] consistency: the oracle uses the true residual scores in the same floor mixture (Remark~\ref{rem:oracle-experiments}); "separate sampling rule" referred to the earlier beta=0 oracle
% This section presents the complete implementation of \ourmethod, connecting the practical sample-selection strategies in Remarks~\ref{rem:0-1} and~\ref{rem:cross_entropy} to the proposal probabilities, inferential signals, $e$-processes, confidence intervals, stopping rule, and final performance estimate. At each round, the acquisition scores, proposal distribution, scaling bounds, and betting fractions are determined before the newly selected target loss is observed. The oracle experiments use the separate sampling rule in Remark~\ref{rem:oracle-experiments}; Alg.~\ref{alg:celeus} describes practical sampling with a positive probability floor.
This section presents the complete implementation of \ourmethod, connecting the practical sample-selection strategies in Remarks~\ref{rem:0-1} and~\ref{rem:cross_entropy} to the proposal probabilities, inferential signals, $e$-processes, confidence intervals, stopping rule, and final performance estimate. At each round, the acquisition scores, proposal distribution, scaling bounds, and betting fractions are determined before the newly selected target loss is observed. The oracle experiments use the true residual magnitudes as the sample-selection scores (Remark~\ref{rem:oracle-experiments}); Alg.~\ref{alg:celeus} describes practical sampling with label-free acquisition scores.

At round $t$, let $s_t(j)\ge0$ denote the label-free acquisition score assigned to each remaining sample $j\in\cJ_{t-1}$. For the residual-based strategies considered in this paper, the score takes the form
\[
s_t(j)
=
\left|
\ell_{\mathrm{proxy},t}(\x_j)
-
\tilde{\ell}_f(\x_j)
\right|,
\]
where $\ell_{\mathrm{proxy},t}(\x_j)$ is a label-free proxy for the unknown target loss. The score may depend on information available before round $t$, including previously observed losses and surrogate information, but not on the unrevealed label $y_j$. Define
\[
m_t
=
|\cJ_{t-1}|
=
N-t+1,
\qquad
Z_t^{\mathrm{score}}
=
\sum_{k\in\cJ_{t-1}}s_t(k).
\]
We first normalize the acquisition scores according to
\begin{equation}\label{eq:score_normalization}
p_t(j)
=
\begin{cases}
\displaystyle
\frac{s_t(j)}{Z_t^{\mathrm{score}}},
&
Z_t^{\mathrm{score}}>\delta_{\mathrm{num}},
\\
\displaystyle
\frac{1}{m_t},
&
Z_t^{\mathrm{score}}\le\delta_{\mathrm{num}},
\end{cases}
\qquad
j\in\cJ_{t-1},
\end{equation}
where $\delta_{\mathrm{num}}>0$ is a numerical tolerance, set to $10^{-12}$ in our implementation. Thus, if all scores are zero, or their sum is numerically negligible, the normalized distribution falls back to uniform sampling. The actual proposal distribution used to draw the next sample is
\begin{equation}\label{eq:practical_sampling_probability}
% q_t(j\mid i_{1:t-1},D_N)
q_t(j\mid i_{1:t-1},X_N)
=
\frac{\beta}{m_t}
+
(1-\beta)p_t(j),
\qquad
j\in\cJ_{t-1}.
\end{equation}
This construction defines a probability distribution satisfying
\[
% q_t(j\mid i_{1:t-1},D_N)
q_t(j\mid i_{1:t-1},X_N)
\ge
\frac{\beta}{N-t+1},
\]
as required by Assumption~\ref{ass:beta}. If all scores are zero, both $p_t$ and $q_t$ are uniform over the remaining pool. If only some scores are zero, the probability floor ensures that the corresponding samples remain selectable. The proposal in \eqref{eq:practical_sampling_probability} is the practical score-based sampling rule motivated by Theorem~\ref{thm:optimal}; it is not, in general, identical to the constrained oracle proposal $q_t^*$.

For numerical test inversion, let
\[
\mathcal G
=
\{\upsilon_1,\ldots,\upsilon_K\}
\subset[L,U]
\]
% be an equally spaced candidate grid containing $L$ and $U$. For compactness in Alg.~\ref{alg:celeus}, we write
% be an equally spaced candidate grid with $K\ge2$, $\upsilon_1=L$, and $\upsilon_K=U$. Fix $K$ before evaluation and record the cell width $h=(U-L)/(K-1)$, so $\upsilon_j=L+(j-1)h$. We retain ranges between adjacent grid points, called grid cells, rather than individual candidate points. For compactness in Alg.~\ref{alg:celeus}, we write
be an equally spaced candidate grid containing $L$ and $U$. For compactness in Alg.~\ref{alg:celeus}, we write
\[
q_t(j)
=
% q_t(j\mid i_{1:t-1},D_N),
q_t(j\mid i_{1:t-1},X_N),
\qquad
z_t
=
\ell(f(\x_{i_t}),y_{i_t}),
\qquad
\tilde\ell_j
=
\tilde{\ell}_f(\x_j).
\]
% The predictable scaling bounds and betting fractions are selected according to Remarks~\ref{rem:scaling} and~\ref{rem:betting-fractions}, respectively. Algorithm~\ref{alg:celeus} summarizes the complete sequential procedure.
% The predictable scaling bounds and betting fractions are selected according to Remarks~\ref{rem:scaling} and~\ref{rem:betting-fractions}, respectively. Write $\cK_{t,+}$ and $\cK_{t,-}$ for the two products in \eqref{eq:hedhedCP}. For each cell $[\upsilon_j,\upsilon_{j+1}]$, define
% \begin{equation}\label{eq:cell_capital_bound}
% \underline{\cK}_{t,j}
% =
% \theta\cK_{t,+}(\upsilon_{j+1})
% +(1-\theta)\cK_{t,-}(\upsilon_j).
% \end{equation}
% As justified below, this is a lower bound on $\cK_t^\pm(\upsilon)$ throughout the cell. Thus, we remove a cell only when this bound reaches $1/\alpha_1$, and report the outside edges of all remaining cells. This changes only the numerical inversion of the existing hedged process.
% 
% Algorithm~\ref{alg:celeus} uses certified numerical bounds, meaning bounds guaranteed to contain the exact values despite rounding. Specifically, $T^{\mathrm{up}}\ge1/\alpha_1$, $\Delta_N^{\mathrm{up}}\ge\Delta_N$, and $\underline{\cK}_{t,j}^{\mathrm{lo}}\le\underline{\cK}_{t,j}$. These bounds equal their exact values in exact arithmetic; the finite-precision rules are given after the algorithm.
The predictable scaling bounds and betting fractions are selected according to Remarks~\ref{rem:scaling} and~\ref{rem:betting-fractions}, respectively. Algorithm~\ref{alg:celeus} summarizes the complete sequential procedure.

% \begin{breakablealgorithm}
% \Needspace{8\baselineskip}
% \begin{breakablealgorithm}
\begin{breakablealgorithm}
\caption{End-to-end implementation of \ourmethod}
\label{alg:celeus}
\small
\begin{algorithmic}[1]
% \REQUIRE Evaluation pool $D_N=\{\x_1,\ldots,\x_N\}$, surrogate scores $\{\tilde\ell_j\}_{j=1}^N$, and acquisition-score rule $s_t$
\REQUIRE Input pool $X_N=\{\x_1,\ldots,\x_N\}$, surrogate scores $\{\tilde\ell_j\}_{j=1}^N$, and acquisition-score rule $s_t$
\REQUIRE Loss bounds $[L,U]$, grid $\mathcal G$, error split $\alpha_1+\alpha_2=\alpha$, parameters $\beta$, $\theta$, $c$, $\delta_{\mathrm{num}}$, and target width $\epsilon$
\ENSURE Stopping time $\tau$, grid-based population interval $\widehat C_\tau$, performance estimate $\hat R_\tau$, and stopping status $\mathtt{reached}$
% \STATE $\cJ_0\gets\{1,\ldots,N\}$, $\mathcal A_0\gets\mathcal G$, $\tau\gets N$, and $\mathtt{reached}\gets\mathrm{false}$
% \STATE $\cJ_0\gets\{1,\ldots,N\}$, $\mathcal A_0\gets\{1,\ldots,K-1\}$, $\tau\gets N$, and $\mathtt{reached}\gets\mathrm{false}$
\STATE $\cJ_0\gets\{1,\ldots,N\}$, $\mathcal A_0\gets\mathcal G$, $\tau\gets N$, and $\mathtt{reached}\gets\mathrm{false}$
% \STATE $\displaystyle \Delta_N\gets(U-L)\sqrt{\frac{\log(2/\alpha_2)}{2N}}$
% \STATE Compute an upper bound $T^{\mathrm{up}}$ on $1/\alpha_1$
% \STATE Compute an upper bound $\Delta_N^{\mathrm{up}}$ on $\displaystyle (U-L)\sqrt{\frac{\log(2/\alpha_2)}{2N}}$
\STATE $\displaystyle \Delta_N\gets(U-L)\sqrt{\frac{\log(2/\alpha_2)}{2N}}$
\FOR{each $\upsilon\in\mathcal G$}
    \STATE $\cK_{0,+}(\upsilon)\gets1$ and $\cK_{0,-}(\upsilon)\gets1$
\ENDFOR
\FOR{$t=1,\ldots,N$}
    \STATE $m_t\gets N-t+1$
    \FOR{each $j\in\cJ_{t-1}$}
        \STATE Compute $s_t(j)\ge0$ using only information available before round $t$
    \ENDFOR
    \STATE $\displaystyle Z_t^{\mathrm{score}}\gets\sum_{k\in\cJ_{t-1}}s_t(k)$
    \IF{$Z_t^{\mathrm{score}}>\delta_{\mathrm{num}}$}
        \STATE $p_t(j)\gets s_t(j)/Z_t^{\mathrm{score}}$ for every $j\in\cJ_{t-1}$
    \ELSE
        \STATE $p_t(j)\gets1/m_t$ for every $j\in\cJ_{t-1}$
    \ENDIF
    \STATE $q_t(j)\gets\beta/m_t+(1-\beta)p_t(j)$ for every $j\in\cJ_{t-1}$
    \STATE Before observing $z_t$, choose predictable scaling bounds $a_t,b_t$ according to Remark~\ref{rem:scaling}
    \FOR{each $\upsilon\in\mathcal G$}
%         \STATE $\bar\upsilon_t\gets(\upsilon-a_t)/(b_t-a_t)$
%         \STATE Store $\bar\upsilon_t(\upsilon)\gets(\upsilon-a_t)/(b_t-a_t)$ for this candidate
        \STATE $\bar\upsilon_t\gets(\upsilon-a_t)/(b_t-a_t)$
        \STATE Compute predictable $\lambda_t^+(\upsilon)$ and $\lambda_t^-(\upsilon)$ according to Remark~\ref{rem:betting-fractions}
    \ENDFOR
    \STATE Draw $i_t\sim q_t(\cdot)$ over $\cJ_{t-1}$ without replacement
    \STATE Evaluate sample $i_t$ and observe $z_t=\ell(f(\x_{i_t}),y_{i_t})$
%     \STATE $q_t^{\mathrm{obs}}\gets q_t(i_t\mid i_{1:t-1},D_N)$
    \STATE $q_t^{\mathrm{obs}}\gets q_t(i_t\mid i_{1:t-1},X_N)$
    % [ZJ] notation: B_t written out (B is the clipping level in Remark~\ref{rem:cross_entropy})
    % \STATE $\displaystyle B_t\gets\sum_{m=1}^{t-1}z_m+\sum_{j\in\cJ_{t-1}}\tilde\ell_j$
    % \STATE $\displaystyle \hat S_t\gets\frac{B_t}{N}+\frac{z_t-\tilde\ell_{i_t}}{Nq_t^{\mathrm{obs}}}$
    \STATE $\displaystyle \hat S_t\gets\frac{1}{N}\left(\sum_{m=1}^{t-1}z_m+\sum_{j\in\cJ_{t-1}}\tilde\ell_j\right)+\frac{z_t-\tilde\ell_{i_t}}{Nq_t^{\mathrm{obs}}}$
    \STATE $\bar S_t\gets(\hat S_t-a_t)/(b_t-a_t)$
    \STATE $\cJ_t\gets\cJ_{t-1}\setminus\{i_t\}$
    \FOR{each $\upsilon\in\mathcal G$}
%         \STATE $d_t(\upsilon)\gets\bar S_t-\bar\upsilon_t$
%         \STATE $d_t(\upsilon)\gets\bar S_t-\bar\upsilon_t(\upsilon)$
        \STATE $d_t(\upsilon)\gets\bar S_t-\bar\upsilon_t$
        \STATE $\cK_{t,+}(\upsilon)\gets\cK_{t-1,+}(\upsilon)\left[1+\lambda_t^+(\upsilon)d_t(\upsilon)\right]$
        \STATE $\cK_{t,-}(\upsilon)\gets\cK_{t-1,-}(\upsilon)\left[1-\lambda_t^-(\upsilon)d_t(\upsilon)\right]$
        \STATE $\cK_t^\pm(\upsilon)\gets\theta\cK_{t,+}(\upsilon)+(1-\theta)\cK_{t,-}(\upsilon)$
    \ENDFOR
%     \STATE $\displaystyle
%     \mathcal A_t
%     \gets
%     \mathcal A_{t-1}
%     \cap
%     \left\{
%     \upsilon\in\mathcal G:
%     \cK_t^\pm(\upsilon)<1/\alpha_1
%     \right\}$
%     \FOR{each $j\in\mathcal A_{t-1}$}
%         \STATE Compute a lower bound $\underline{\cK}_{t,j}^{\mathrm{lo}}$ on \eqref{eq:cell_capital_bound}
%     \ENDFOR
%     \STATE $\mathcal A_t\gets\{j\in\mathcal A_{t-1}:\underline{\cK}_{t,j}^{\mathrm{lo}}<T^{\mathrm{up}}\}$
    \STATE $\displaystyle
    \mathcal A_t
    \gets
    \mathcal A_{t-1}
    \cap
    \left\{
    \upsilon\in\mathcal G:
    \cK_t^\pm(\upsilon)<1/\alpha_1
    \right\}$
    \IF{$\mathcal A_t\neq\varnothing$}
%         \STATE $\widehat{\cM}_{t,\mathcal G}^\pm\gets[\min\mathcal A_t,\max\mathcal A_t]$
%         \STATE $\widehat{\cM}_{t,\mathcal G}^\pm\gets[\upsilon_{\min\mathcal A_t},\upsilon_{1+\max\mathcal A_t}]$, with endpoints rounded outward
        \STATE $\widehat{\cM}_{t,\mathcal G}^\pm\gets[\min\mathcal A_t,\max\mathcal A_t]$
%         \STATE $\widehat C_t\gets\left(\widehat{\cM}_{t,\mathcal G}^\pm\oplus[-\Delta_N,\Delta_N]\right)\cap[L,U]$
%         \STATE $\widehat C_t\gets\left(\widehat{\cM}_{t,\mathcal G}^\pm\oplus[-\Delta_N^{\mathrm{up}},\Delta_N^{\mathrm{up}}]\right)\cap[L,U]$, with endpoints rounded outward
        \STATE $\widehat C_t\gets\left(\widehat{\cM}_{t,\mathcal G}^\pm\oplus[-\Delta_N,\Delta_N]\right)\cap[L,U]$
    \ELSE
        \STATE $\widehat{\cM}_{t,\mathcal G}^\pm\gets[L,U]$ and $\widehat C_t\gets[L,U]$
    \ENDIF
%     \IF{$\mathcal A_t\neq\varnothing$ and $\mathrm{width}(\widehat C_t)\le\epsilon$}
%     \STATE Compute an upper bound $W_t^{\mathrm{up}}$ on $\mathrm{width}(\widehat C_t)$
%     \IF{$\mathcal A_t\neq\varnothing$ and $W_t^{\mathrm{up}}\le\epsilon$}
    \IF{$\mathcal A_t\neq\varnothing$ and $\mathrm{width}(\widehat C_t)\le\epsilon$}
        \STATE $\tau\gets t$ and $\mathtt{reached}\gets\mathrm{true}$
        \STATE \textbf{break}
    \ENDIF
\ENDFOR
\IF{$\tau<N$}
    \FOR{$m=1,\ldots,\tau$}
        \STATE $\displaystyle
        u_{m,\tau}
        \gets
        1+\frac{N-\tau}{N-m}
        \left(
        \frac{1}{(N-m+1)q_m^{\mathrm{obs}}}-1
        \right)$
    \ENDFOR
    \STATE $\displaystyle
    \hat R_\tau
    \gets
    \frac{1}{N}\sum_{n=1}^N\tilde\ell_n
    +
    \frac{1}{\tau}\sum_{m=1}^{\tau}
    u_{m,\tau}(z_m-\tilde\ell_{i_m})$
\ELSE
    \STATE $\displaystyle
    \hat R_N
    \gets
    \frac{1}{N}\sum_{m=1}^Nz_m
    =
    R_N$
\ENDIF
\STATE \textbf{return} $(\tau,\widehat C_\tau,\hat R_\tau,\mathtt{reached})$
\end{algorithmic}
\end{breakablealgorithm}

The distinction between $p_t$ and $q_t$ is essential. The distribution $p_t$ only normalizes the acquisition scores, whereas $q_t$ in \eqref{eq:practical_sampling_probability} is the distribution actually used to sample $i_t$. Consequently, the realized probability $q_t^{\mathrm{obs}}$ is used both in the inverse-probability correction in \eqref{eq:signal_estimator} and in the weight $u_{m,\tau}$ in \eqref{eq:lure_estimator}. Using $p_t$, the unnormalized scores, or the oracle proposal in either correction would not correspond to the implemented sampling distribution.

For each grid candidate, the algorithm recursively updates the two one-sided capital processes and their hedged combination in \eqref{eq:hedhedCP}. The set $\mathcal A_t$ retains only the candidates that have remained unrejected at every round up to $t$, thereby implementing the running intersection in \eqref{eq:M_t} on the grid. The numerical interval $\widehat{\cM}_{t,\mathcal G}^\pm$ spans the smallest and largest surviving grid points, after which $\widehat C_t$ applies the finite-pool-to-population correction. If no grid point survives, the full interval $[L,U]$ is used as a numerical safeguard, so an empty grid is not interpreted as attaining zero width. Theorems~\ref{thm:conf_sequen} and~\ref{thm:coverage_convergence} concern exact inversion over the continuum $[L,U]$, while the grid provides its numerical implementation.

The procedure stops at the first round for which the reported interval reaches the target width $\epsilon$. If this does not occur by round $N$, the procedure terminates when the pool is exhausted and reports the resulting interval with $\mathtt{reached}=\mathrm{false}$ rather than declaring that the target precision has been attained. At $t=N$, only one sample remains and \eqref{eq:practical_sampling_probability} assigns it probability one, regardless of the acquisition scores. Consequently, $\hat S_N=R_N$. We also set $\hat R_N=R_N$ directly, avoiding the undefined expression obtained by substituting $m=t=N$ into $u_{m,t}$. For an earlier stopping time $\tau<N$, the final estimator uses the actual proposal probabilities stored at the corresponding sampling rounds. The confidence interval is the certification output of the procedure, while $\hat R_\tau$ is reported as an accompanying scalar performance estimate.

\appsection{Detailed Proofs for Our Theoretical Results}\label{app:proof}

\appsubsection{Assumptions}\label{app:assumptions}

% In this appendix, we collect the standing assumptions used in the theoretical analysis. The first assumption is an application-level boundedness condition on the evaluation scores. In most evaluation tasks, this condition is mild and can be satisfied by normalizing the metric or clipping the scores to a fixed range. The remaining assumptions are algorithmic conditions: they concern how the procedure selects samples, scales the inferential signals, and chooses the predictable fractions used in the $e$-process construction. In \ourmethod, these algorithmic conditions are designed to be satisfied and impose no additional restrictions on the application scenario.
In this appendix, we collect the standing assumptions used in the theoretical analysis. The first assumption is an application-level boundedness condition on the evaluation scores. In most evaluation tasks, this condition is mild and can be satisfied by normalizing the metric or clipping the scores to a fixed range. The population-risk results additionally use the i.i.d. input-label pairs and fixed target model specified in the evaluation setup. The remaining assumptions are algorithmic conditions: they concern how the procedure selects samples, scales the inferential signals, and chooses the predictable fractions used in the $e$-process construction.

We distinguish the information available to the evaluator from the complete-pool information used in the proofs. The observed information after round $t$ is
\[
% \mathcal O_t=\sigma\left(D_N,\mathcal P_0,\{(i_m,y_{i_m})\}_{m=1}^t\right).
% \mathcal O_t=\sigma\left(D_N,\mathcal P_0,(i_m,y_{i_m})_{m=1}^t\right).
\mathcal O_t=\sigma\left(X_N,\mathcal P_0,(i_m,y_{i_m})_{m=1}^t\right).
\]
The sequences $(\mathcal O_t)$ and $(\cF_t)$ are filtrations, meaning that each records information accumulated up to a given round. Here, $\cF_t$ is defined in \eqref{eq:analysis_information}, and $\mathcal O_t\subseteq\cF_t$. Practical sampling probabilities, scaling bounds, and betting fractions are predictable with respect to $(\mathcal O_t)$: their round-$t$ values are determined using only $\mathcal O_{t-1}$. The surrogate scores in $\mathcal P_0$ remain fixed throughout evaluation.

% Each index is drawn using fresh randomization independent of the complete pool and preprocessing information. Consequently, for every $j\in\cJ_{t-1}$,
Each index is drawn using fresh randomization independent of $\cF_{t-1}$, including the complete pool, preprocessing information, and past selections. Consequently, for every $j\in\cJ_{t-1}$,
\begin{equation}\label{eq:conditional_sampling_law}
\bP(i_t=j\mid\cF_{t-1})
=
\bP(i_t=j\mid\mathcal O_{t-1})
=
% q_t(j\mid i_{1:t-1},D_N).
q_t(j\mid i_{1:t-1},X_N).
\end{equation}
Thus, conditioning on the complete pool fixes its losses without changing the probabilities used to draw the next index. Conditional expectations and variances given $\cF_0$ concern only sample-selection randomness, whereas unconditional quantities also average over the pool and preprocessing information.

\begin{assumption}\label{ass:loss_bound}
The evaluation risk score and the surrogate risk score are uniformly bounded: there exist constants $L,U\in\bR$ such that
\[
L\le \ell(f(\x),y)\le U,
\qquad
L\le \tilde{\ell}_f(\x)\le U.
\]
\end{assumption}

\begin{assumption}\label{ass:beta}
There exists a constant $\beta\in(0,1]$\footnote{More generally, the sampling floor may use a time-varying sequence $\{\beta_t\}_{t\ge1}$ with $\beta_t\in(0,1]$. If $\beta_t$ is predictable and there exists a deterministic constant $\underline{\beta}>0$ such that $\beta_t\ge\underline{\beta}$ for all $t$, all theoretical guarantees continue to hold. The round-$t$ variance-optimal proposal uses $\beta_t$, while the uniform variance and confidence-width bounds use $\underline{\beta}$. This changes only the constants in the bounds, not the validity or convergence rates. We use a time-independent $\beta$ for notational simplicity.} such that
\[
% \min_{j\in\cJ_{t-1}} q_t(j\mid i_{1:t-1},D_N)\ge \frac{\beta}{N-t+1},
\min_{j\in\cJ_{t-1}} q_t(j\mid i_{1:t-1},X_N)\ge \frac{\beta}{N-t+1},
\qquad
\forall N\in\mathbb Z_+,\ \forall 1\le t\le N.
\]
\end{assumption}

Assumption~\ref{ass:beta} ensures that every remaining sample has a nonzero chance of being selected. This prevents inverse-probability weights from becoming ill-defined or arbitrarily large.

\begin{assumption}\label{ass:scaling}
% For every $N\in\mathbb Z_+$ and $1\le t\le N$, the scaling bounds $a_t$ and $b_t$ are $\cF_{t-1}$-measurable. Moreover, there exist deterministic constants $0<\underline\Delta\le\overline\Delta<\infty$, independent of $N$ and $t$, such that, almost surely, $a_t\le \hat S_t\le b_t
For every $N\in\mathbb Z_+$ and $1\le t\le N$, the scaling bounds $a_t$ and $b_t$ are $\mathcal O_{t-1}$-measurable and hence $\cF_{t-1}$-measurable. Moreover, there exist deterministic constants $0<\underline\Delta\le\overline\Delta<\infty$, independent of $N$ and $t$, such that, almost surely, $a_t\le \hat S_t\le b_t
$, $[L,U]\subseteq[a_t,b_t]$, $\underline\Delta\le b_t-a_t\le\overline\Delta.$
\end{assumption}

Assumption~\ref{ass:scaling} is an algorithmic condition on the scaling step. Predictability preserves the conditional mean property after scaling. The containment conditions ensure that both the inferential signal and every candidate value in $[L,U]$ are mapped to $[0,1]$, while the uniform width bounds make the transformation non-degenerate and provide the scale control required for the confidence-width bound.

% \begin{assumption}\label{ass:scaling}
% The scaling intervals are non-degenerate and their widths are uniformly bounded above and below: there exist constants $0<\underline\Delta\le \overline\Delta<\infty$ such that
% \[
% \underline\Delta\le b_t-a_t\le \overline\Delta,
% \qquad
% \forall 1\le t\le N.
% \]
% \end{assumption}

% Assumption~\ref{ass:scaling} is an algorithmic condition on the scaling step. It ensures that the affine transformation used to map the inferential signal to the scaled space is well-defined and stable.

\begin{assumption}\label{ass:betting}
Fix a constant $c\in(0,1)$. There exist deterministic constants
% $0<\underline\gamma\le\overline\gamma<\infty$, independent of $N$ and $t$, such that, for every $N\in\mathbb Z_+$, the evaluator chooses real-valued predictable sequences
$0<\underline\gamma\le\overline\gamma<\infty$, independent of $N$ and $t$, such that, for every $N\in\mathbb Z_+$, the evaluator chooses real-valued sequences predictable with respect to $(\mathcal O_t)$
$\bigl(\tilde\lambda_t^{+}\bigr)_{1\le t\le N}$ and
$\bigl(\tilde\lambda_t^{-}\bigr)_{1\le t\le N}$ that do not depend on the candidate value $\upsilon$ and satisfy, almost surely,
\[
\frac{\underline\gamma}{\sqrt{t\log(t+1)}}
\le
|\tilde\lambda_t^{\pm}|
\le
\frac{\overline\gamma}{\sqrt{t\log(t+1)}},
\qquad
\forall 1\le t\le N.
\]
For each $\upsilon\in[L,U]$, define
\[
\lambda_t^{+}(\upsilon)
=
|\tilde\lambda_t^{+}|
\wedge
\frac{1}{\bar\upsilon_t+c},
\qquad
\lambda_t^{-}(\upsilon)
=
|\tilde\lambda_t^{-}|
\wedge
\frac{1}{1-\bar\upsilon_t+c}.
\]
\end{assumption}

Assumption~\ref{ass:betting} specifies the predictable fractions used in the $e$-process. Predictability and clipping ensure the validity of the capital processes, while the candidate-independent envelope condition provides the uniform control over $\upsilon\in[L,U]$ required for the confidence-width bound. The constants are uniform in $N$ and $t$, as required for the joint asymptotic result.

% \begin{assumption}\label{ass:betting}
% Fix a constant $c>0$. For each candidate value $\upsilon$, choose
% \[
% \lambda_t^{+}(\upsilon)
% =
% |\tilde\lambda_t^{+}(\upsilon)|
% \wedge
% \frac{1}{\bar\upsilon_t+c},
% \qquad
% \lambda_t^{-}(\upsilon)
% =
% |\tilde\lambda_t^{-}(\upsilon)|
% \wedge
% \frac{1}{1-\bar\upsilon_t+c},
% \]
% where $\bigl(\tilde\lambda_t^{+}(\upsilon)\bigr)_{1\le t\le N}$ and
% $\bigl(\tilde\lambda_t^{-}(\upsilon)\bigr)_{1\le t\le N}$ are real-valued predictable sequences. The clipping ensures that the resulting fractions remain in the admissible ranges required for the $e$-process construction. We assume that there exist constants
% $0<\underline\gamma\le \overline\gamma<\infty$ such that, for every candidate value $\upsilon$ and all $1\le t\le N$,
% \[
% \frac{\underline\gamma}{\sqrt{t\log(t+1)}}
% \le
% |\tilde\lambda_t^{\pm}(\upsilon)|
% \le
% \frac{\overline\gamma}{\sqrt{t\log(t+1)}} .
% \]
% \end{assumption}

% Assumption~\ref{ass:betting} specifies the predictable fractions used in the $e$-process. This is also an algorithmic condition: the fractions are chosen by the evaluator before observing the risk at each round and can therefore be designed to satisfy the stated bounds.

\appsubsection{Proofs of Theorem~\ref{thm:lure_unbias}}
\begin{thmbis}{thm:lure_unbias}
% Under Assumption~\ref{ass:loss_bound}-\ref{ass:beta}, the estimator $\hat R_t$ as defined in \eqref{eq:lure_estimator} satisfies $\bE[\hat R_t]=R$ with
% Under the population sampling setup and Assumptions~\ref{ass:loss_bound}-\ref{ass:beta}, the estimator $\hat R_t$ as defined in \eqref{eq:lure_estimator} satisfies $\bE[\hat R_t]=R$ with
Under the population sampling setup and Assumptions~\ref{ass:loss_bound}-\ref{ass:beta}, for each fixed $1\le t<N$, the estimator $\hat R_t$ as defined in \eqref{eq:lure_estimator} satisfies $\bE[\hat R_t]=R$ with
\begin{align*}
\mathrm{Var}[\hat R_t]
=
\frac{\mathrm{Var}(\ell(f(\x),y))}{N}
+
\bE\left[
\mathrm{Var}\left[
% [ZJ] notation: c_m, a_m -> w_m, \psi_m (c is the betting constant, a_t is the scaling bound)
% \frac{1}{t}\sum_{m=1}^{t}c_m a_m
\frac{1}{t}\sum_{m=1}^{t}w_m \psi_m
% \mid D_N,\tilde{\ell}_f
\mid\cF_0
\right]
\right],
\end{align*}
where
\begin{equation}\label{eq:cm}
% [ZJ] notation: c_m, a_m -> w_m, \psi_m (c is the betting constant, a_t is the scaling bound)
% c_m=\frac{N(N-t)}{(N-m)(N-m+1)},
% \end{equation}
% and
% \begin{equation}\label{eq:am}
% a_m=
w_m=\frac{N(N-t)}{(N-m)(N-m+1)},
\end{equation}
and
\begin{equation}\label{eq:am}
\psi_m=
\frac{\ell\left(f(\x_{i_m}),y_{i_m}\right)-\tilde{\ell}_f(\x_{i_m})}
% {Nq_m(i_m\mid i_{1:m-1},D_N)}
{Nq_m(i_m\mid i_{1:m-1},X_N)}
+
\frac{1}{N}\sum_{s=1}^{m-1}
\left(
\ell\left(f(\x_{i_s}),y_{i_s}\right)-\tilde{\ell}_f\left(\x_{i_s}\right)
\right).
\end{equation}
% [ZJ] proof: the restatement also covers t=N, as in Theorem~\ref{thm:lure_unbias} and the first line of the proof
At $t=N$, $\hat R_N=R_N$ has expectation $R$ and variance $\mathrm{Var}(\ell(f(\x),y))/N$.
\end{thmbis}

We present the detailed proof of Theorem~\ref{thm:lure_unbias} as follows.

\begin{proof}
% Under Assumption~\ref{ass:loss_bound}-\ref{ass:beta}, we first prove the unbiasedness of $\hat R_t$, and then investigate its variance. Recall the estimator in \eqref{eq:lure_estimator}:
At $t=N$, the convention $\hat R_N=R_N$ gives the stated expectation and variance directly under the population sampling setup. We therefore prove the remaining case $1\le t<N$. Under Assumptions~\ref{ass:loss_bound}--\ref{ass:beta}, we first prove unbiasedness and then compute the variance, starting from \eqref{eq:lure_estimator}:
\begin{equation}\label{eq:lure_mean_proof}
\bE[\hat R_t]
=
\bE\left[
\frac{1}{N}\sum_{n=1}^{N}\tilde{\ell}_f\left(\x_n\right)
\right]
+
\bE\Bigg[
\underbrace{
\frac{1}{t}\sum_{m=1}^{t}u_{m,t}
\left(
\ell\left(f(\x_{i_m}),y_{i_m}\right)-\tilde{\ell}_f(\x_{i_m})
\right)
}_{\circled{a}}
\Bigg],
\end{equation}
where
\[
u_{m,t}
=
1+\frac{N-t}{N-m}
\left(
% \frac{1}{(N-m+1)q_m(i_m\mid i_{1:m-1},D_N)}-1
\frac{1}{(N-m+1)q_m(i_m\mid i_{1:m-1},X_N)}-1
\right).
\]
By Assumption~\ref{ass:beta}, the proposal probabilities are strictly positive on the remaining pool, so the inverse-probability terms above are well defined.

The key step is to show that the residual-correction term $\circled{a}$ exactly matches, in expectation, the average residual over the entire pool. To this end, we invoke the unbiasedness results for adaptive sampling without replacement from \citet{Farquhar:Gal:Rainforth2021}, which imply that $\circled{a}$ can be rewritten as
\[
\circled{a}
=
% [ZJ] notation: c_m, a_m -> w_m, \psi_m (c is the betting constant, a_t is the scaling bound)
% \frac{1}{t}\sum_{m=1}^{t}c_m a_m,
% \qquad
% c_m=\frac{N(N-t)}{(N-m)(N-m+1)},
% \]
% and
% \[
% a_m=
\frac{1}{t}\sum_{m=1}^{t}w_m \psi_m,
\qquad
w_m=\frac{N(N-t)}{(N-m)(N-m+1)},
\]
and
\[
\psi_m=
\frac{\ell\left(f(\x_{i_m}),y_{i_m}\right)-\tilde{\ell}_f(\x_{i_m})}
% {Nq_m(i_m\mid i_{1:m-1},D_N)}
{Nq_m(i_m\mid i_{1:m-1},X_N)}
+
\frac{1}{N}\sum_{s=1}^{m-1}
\left(
\ell\left(f(\x_{i_s}),y_{i_s}\right)-\tilde{\ell}_f\left(\x_{i_s}\right)
\right).
\]

% We now compute $\bE[a_m]$ via the tower property. Conditioning on the pool $D_N$, the surrogate predictor $\tilde{\ell}_f$, and the first $m-1$ selections, we obtain
% [ZJ] notation: c_m, a_m -> w_m, \psi_m (c is the betting constant, a_t is the scaling bound)
% We now compute $\bE[a_m]$ via the tower property, which averages a conditional expectation over the information being conditioned on. Conditioning on $\cF_{m-1}$ fixes the complete pool, preprocessing information, and the first $m-1$ selections, giving
% \begin{eqnarray*}
% \bE[a_m]
We now compute $\bE[\psi_m]$ via the tower property, which averages a conditional expectation over the information being conditioned on. Conditioning on $\cF_{m-1}$ fixes the complete pool, preprocessing information, and the first $m-1$ selections, giving
\begin{eqnarray*}
\bE[\psi_m]
&=&
% \bE_{D_N,\tilde{\ell}_f,i_{1:m-1}}
% \bE_{\mathcal Z_N,\mathcal P_0,i_{1:m-1}}
\bE_{D_N,\mathcal P_0,i_{1:m-1}}
\left[
% \bE_m\left[a_m\mid D_N,\tilde{\ell}_f,i_{1:m-1}\right]
% [ZJ] notation: c_m, a_m -> w_m, \psi_m (c is the betting constant, a_t is the scaling bound)
% \bE_m\left[a_m\mid\cF_{m-1}\right]
\bE_m\left[\psi_m\mid\cF_{m-1}\right]
\right]\\
&=&
\underbrace{
% \bE_{D_N,\tilde{\ell}_f,i_{1:m-1}}
% \bE_{\mathcal Z_N,\mathcal P_0,i_{1:m-1}}
\bE_{D_N,\mathcal P_0,i_{1:m-1}}
\left[
\bE_m\left[
\frac{\ell\left(f(\x_{i_m}),y_{i_m}\right)-\tilde{\ell}_f(\x_{i_m})}
% {Nq_m(i_m\mid i_{1:m-1},D_N)}
{Nq_m(i_m\mid i_{1:m-1},X_N)}
% \mid D_N,\tilde{\ell}_f,i_{1:m-1}
\mid\cF_{m-1}
\right]
\right]
}_{\circled{b}}\\
&&+
% \bE_{D_N,\tilde{\ell}_f,i_{1:m-1}}
% \bE_{\mathcal Z_N,\mathcal P_0,i_{1:m-1}}
\bE_{D_N,\mathcal P_0,i_{1:m-1}}
\left[
\frac{1}{N}\sum_{s=1}^{m-1}
\left(
\ell\left(f(\x_{i_s}),y_{i_s}\right)-\tilde{\ell}_f\left(\x_{i_s}\right)
\right)
\right].
\end{eqnarray*}

% For $\circled{b}$, the conditional expectation over the $m$-th draw expands as a finite sum over the remaining unlabeled indices. Specifically, given $(D_N,\tilde{\ell}_f,i_{1:m-1})$, the next selected index takes values $j\in\cJ_{m-1}$ with conditional probability $q_m(j\mid i_{1:m-1},D_N)$, hence
% For $\circled{b}$, the conditional expectation over the $m$-th draw expands as a finite sum over the remaining unlabeled indices. Specifically, given $\cF_{m-1}$ and the sampling law in \eqref{eq:conditional_sampling_law}, the next selected index takes values $j\in\cJ_{m-1}$ with conditional probability $q_m(j\mid i_{1:m-1},D_N)$, hence
For $\circled{b}$, the conditional expectation over the $m$-th draw expands as a finite sum over the remaining unlabeled indices. Specifically, given $\cF_{m-1}$ and the sampling law in \eqref{eq:conditional_sampling_law}, the next selected index takes values $j\in\cJ_{m-1}$ with conditional probability $q_m(j\mid i_{1:m-1},X_N)$, hence
\begin{eqnarray*}
\circled{b}
&=&
% \bE_{D_N,\tilde{\ell}_f,i_{1:m-1}}
% \bE_{\mathcal Z_N,\mathcal P_0,i_{1:m-1}}
\bE_{D_N,\mathcal P_0,i_{1:m-1}}
\left[
\sum_{j\in\cJ_{m-1}}
% q_m(j\mid i_{1:m-1},D_N)
q_m(j\mid i_{1:m-1},X_N)
\frac{\ell\left(f(\x_j),y_j\right)-\tilde{\ell}_f\left(\x_j\right)}
% {Nq_m(j\mid i_{1:m-1},D_N)}
{Nq_m(j\mid i_{1:m-1},X_N)}
\right]\\
&=&
% \bE_{D_N,\tilde{\ell}_f,i_{1:m-1}}
% \bE_{\mathcal Z_N,\mathcal P_0,i_{1:m-1}}
\bE_{D_N,\mathcal P_0,i_{1:m-1}}
\left[
\sum_{j\in\cJ_{m-1}}
\frac{1}{N}
\left(
\ell\left(f(\x_j),y_j\right)-\tilde{\ell}_f\left(\x_j\right)
\right)
\right].
\end{eqnarray*}

Combining the two terms yields
\begin{eqnarray*}
% [ZJ] notation: c_m, a_m -> w_m, \psi_m (c is the betting constant, a_t is the scaling bound)
% \bE[a_m]
\bE[\psi_m]
&=&
% \bE_{D_N,\tilde{\ell}_f,i_{1:m-1}}
% \bE_{\mathcal Z_N,\mathcal P_0,i_{1:m-1}}
\bE_{D_N,\mathcal P_0,i_{1:m-1}}
\left[
\frac{1}{N}\sum_{n=1}^{N}
\left(
\ell\left(f(\x_n),y_n\right)-\tilde{\ell}_f\left(\x_n\right)
\right)
\right]\\
&=&
% \bE_{D_N,\tilde{\ell}_f}
% \bE_{\mathcal Z_N,\mathcal P_0}
\bE_{D_N,\mathcal P_0}
\left[
\frac{1}{N}\sum_{n=1}^{N}
\left(
\ell\left(f(\x_n),y_n\right)-\tilde{\ell}_f\left(\x_n\right)
\right)
\right],
\end{eqnarray*}
where the last equality holds because the expression inside the brackets no longer depends on the selection history $i_{1:m-1}$.

Substituting this back into the representation of $\circled{a}$, we have
\begin{eqnarray*}
\bE[\circled{a}]
&=&
% [ZJ] notation: c_m, a_m -> w_m, \psi_m (c is the betting constant, a_t is the scaling bound)
% \frac{1}{t}\sum_{m=1}^{t}c_m\bE[a_m]\\
% &=&
% \frac{1}{t}\sum_{m=1}^{t}c_m
\frac{1}{t}\sum_{m=1}^{t}w_m\bE[\psi_m]\\
&=&
\frac{1}{t}\sum_{m=1}^{t}w_m
\bE\left[
\frac{1}{N}\sum_{n=1}^{N}
\left(
\ell\left(f(\x_n),y_n\right)-\tilde{\ell}_f\left(\x_n\right)
\right)
\right]\\
&=&
\bE\left[
\frac{1}{N}\sum_{n=1}^{N}
\left(
\ell\left(f(\x_n),y_n\right)-\tilde{\ell}_f\left(\x_n\right)
\right)
\right]
\times
% [ZJ] notation: c_m, a_m -> w_m, \psi_m (c is the betting constant, a_t is the scaling bound)
% \frac{1}{t}\sum_{m=1}^{t}c_m\\
\frac{1}{t}\sum_{m=1}^{t}w_m\\
&=&
\bE\left[
\frac{1}{N}\sum_{n=1}^{N}
\left(
\ell\left(f(\x_n),y_n\right)-\tilde{\ell}_f\left(\x_n\right)
\right)
\right].
\end{eqnarray*}
The last equality holds since
% [ZJ] layout: display set on two lines to fit the text width; notation: c_m, a_m -> w_m, \psi_m (c is the betting constant, a_t is the scaling bound)
% \[
% \sum_{m=1}^{t}c_m
% =
% N(N-t)\sum_{m=1}^{t}\frac{1}{(N-m)(N-m+1)}
% =
% N(N-t)\sum_{m=1}^{t}\left(\frac{1}{N-m}-\frac{1}{N-m+1}\right)
% =
% t.
% \]
\begin{align*}
\sum_{m=1}^{t}w_m
&=
N(N-t)\sum_{m=1}^{t}\frac{1}{(N-m)(N-m+1)}\\
&=
N(N-t)\sum_{m=1}^{t}\left(\frac{1}{N-m}-\frac{1}{N-m+1}\right)
=
t.
\end{align*}

Substituting $\bE[\circled{a}]$ back into \eqref{eq:lure_mean_proof}, we obtain
\begin{eqnarray*}
\bE[\hat R_t]
&=&
\bE\left[
\frac{1}{N}\sum_{n=1}^{N}\tilde{\ell}_f\left(\x_n\right)
\right]
+
\bE\left[
\frac{1}{N}\sum_{n=1}^{N}
\left(
\ell\left(f(\x_n),y_n\right)-\tilde{\ell}_f\left(\x_n\right)
\right)
\right]\\
&=&
\bE\left[
\frac{1}{N}\sum_{n=1}^{N}\ell\left(f(\x_n),y_n\right)
\right]\\
&=&
R,
\end{eqnarray*}
which proves the unbiasedness of $\hat R_t$.

We next investigate the variance of $\hat R_t$ under Assumption~\ref{ass:loss_bound}. By the law of total variance, we have
\begin{equation}\label{eq:lure_var_decomp}
\mathrm{Var}[\hat R_t]
=
\bE\Bigg[
\underbrace{
% \mathrm{Var}\left[\hat R_t\mid D_N,\tilde{\ell}_f\right]
\mathrm{Var}\left[\hat R_t\mid\cF_0\right]
}_{\circled{c}}
\Bigg]
+
\mathrm{Var}\Bigg[
\underbrace{
% \bE\left[\hat R_t\mid D_N,\tilde{\ell}_f\right]
\bE\left[\hat R_t\mid\cF_0\right]
}_{\circled{d}}
\Bigg].
\end{equation}

% We first investigate the term $\circled{d}$. From the unbiasedness argument above, conditional on $D_N$ and $\tilde{\ell}_f$, we have
We first investigate the term $\circled{d}$. The same unbiasedness calculation holds with the complete pool and preprocessing information fixed, so conditional on $\cF_0$ we have
\[
% \bE\left[\hat R_t\mid D_N,\tilde{\ell}_f\right]
\bE\left[\hat R_t\mid\cF_0\right]
=
\frac{1}{N}\sum_{n=1}^{N}\ell\left(f(\x_n),y_n\right).
\]
% Therefore,
Since the input-label pairs are i.i.d. and $f$ is fixed, the pool losses are i.i.d. Therefore,
\[
\mathrm{Var}[\circled{d}]
=
\mathrm{Var}\left[
\frac{1}{N}\sum_{n=1}^{N}\ell\left(f(\x_n),y_n\right)
\right]
=
\frac{\mathrm{Var}(\ell(f(\x),y))}{N}.
\]

% We now consider the term $\circled{c}$. Since the first term in $\hat R_t$ is measurable with respect to $(D_N,\tilde{\ell}_f)$, it does not contribute to the conditional variance. Hence,
We now consider the term $\circled{c}$. Since the first term in $\hat R_t$ is measurable with respect to $\cF_0$, it does not contribute to the conditional variance. Hence,
\[
\circled{c}
=
\mathrm{Var}\Bigg[
\underbrace{
% [ZJ] notation: c_m, a_m -> w_m, \psi_m (c is the betting constant, a_t is the scaling bound)
% \frac{1}{t}\sum_{m=1}^{t}c_m a_m
\frac{1}{t}\sum_{m=1}^{t}w_m \psi_m
}_{\circled{a}}
% \mid D_N,\tilde{\ell}_f
\mid\cF_0
\Bigg].
\]

Write 
\[
\bE\left[
\circled{a}
% \mid D_N,\tilde{\ell}_f
\mid\cF_0
\right]=\frac{1}{N}\sum_{n=1}^{N}
\left(
\ell\left(f(\x_n),y_n\right)-\tilde{\ell}_f\left(\x_n\right)
\right)
\]
for the average residual over the pool. Then
\begin{eqnarray*}
\circled{c}
&=&
\bE\left[
\circled{a}
^2
% \mid D_N,\tilde{\ell}_f
\mid\cF_0
\right]
-
\left(
\frac{1}{N}\sum_{n=1}^{N}
\left(
\ell\left(f(\x_n),y_n\right)-\tilde{\ell}_f\left(\x_n\right)
\right)
\right)^2\\
&=&
\bE\Bigg[
\left(
% [ZJ] notation: c_m, a_m -> w_m, \psi_m (c is the betting constant, a_t is the scaling bound)
% \frac{1}{t}\sum_{m=1}^{t}c_m a_m
\frac{1}{t}\sum_{m=1}^{t}w_m \psi_m
-
\frac{1}{N}\sum_{n=1}^{N}
\left(
\ell\left(f(\x_n),y_n\right)-\tilde{\ell}_f\left(\x_n\right)
\right)
\right)^2
% \mid D_N,\tilde{\ell}_f
\mid\cF_0
\Bigg]\\
&=&
\bE\Bigg[
\left(
% [ZJ] notation: c_m, a_m -> w_m, \psi_m (c is the betting constant, a_t is the scaling bound)
% \frac{1}{t}\sum_{m=1}^{t}c_m
% \left(
% a_m-
\frac{1}{t}\sum_{m=1}^{t}w_m
\left(
\psi_m-
\frac{1}{N}\sum_{n=1}^{N}
\left(
\ell\left(f(\x_n),y_n\right)-\tilde{\ell}_f\left(\x_n\right)
\right)
\right)
\right)^2
% \mid D_N,\tilde{\ell}_f
\mid\cF_0
\Bigg]\\
&=&
\frac{1}{t^2}
\sum_{m=1}^{t}\sum_{s=1}^{t}
% [ZJ] notation: c_m, a_m -> w_m, \psi_m (c is the betting constant, a_t is the scaling bound)
% c_mc_s
% \bE\left[
% \underbrace{\left(
% a_m-
w_mw_s
\bE\left[
\underbrace{\left(
\psi_m-
\frac{1}{N}\sum_{n=1}^{N}
\left(
\ell\left(f(\x_n),y_n\right)-\tilde{\ell}_f\left(\x_n\right)
\right)
\right)}_{\circled{e}}\right.
\\
&&\qquad\qquad\qquad\qquad\qquad\left.
\cdot
\underbrace{\left(
% [ZJ] notation: c_m, a_m -> w_m, \psi_m (c is the betting constant, a_t is the scaling bound)
% a_s-
\psi_s-
\frac{1}{N}\sum_{n=1}^{N}
\left(
\ell\left(f(\x_n),y_n\right)-\tilde{\ell}_f\left(\x_n\right)
\right)
\right)}_{\circled{f}}
% \mid D_N,\tilde{\ell}_f
\mid\cF_0
\right].
\end{eqnarray*}

% For $s<m$, since $a_s$ is measurable with respect to $(D_N,\tilde{\ell}_f,i_{1:m-1})$, we have
% [ZJ] notation: c_m, a_m -> w_m, \psi_m (c is the betting constant, a_t is the scaling bound)
% For $s<m$, since $a_s$ is measurable with respect to $\cF_{m-1}$, we have
For $s<m$, since $\psi_s$ is measurable with respect to $\cF_{m-1}$, we have
\begin{eqnarray*}
% \bE\left[a_m\cdot\circled{f}\mid D_N,\tilde{\ell}_f\right]
% [ZJ] notation: c_m, a_m -> w_m, \psi_m (c is the betting constant, a_t is the scaling bound)
% \bE\left[a_m\cdot\circled{f}\mid\cF_0\right]
\bE\left[\psi_m\cdot\circled{f}\mid\cF_0\right]
&=&
\bE\left[
% \bE\left[a_m\cdot\circled{f}\mid D_N,\tilde{\ell}_f,i_{1:m-1}\right]
% [ZJ] notation: c_m, a_m -> w_m, \psi_m (c is the betting constant, a_t is the scaling bound)
% \bE\left[a_m\cdot\circled{f}\mid\cF_{m-1}\right]
\bE\left[\psi_m\cdot\circled{f}\mid\cF_{m-1}\right]
% \mid D_N,\tilde{\ell}_f
\mid\cF_0
\right]\\
&=&
\bE\left[
% \circled{f}\cdot \bE\left[a_m\mid D_N,\tilde{\ell}_f,i_{1:m-1}\right]
% [ZJ] notation: c_m, a_m -> w_m, \psi_m (c is the betting constant, a_t is the scaling bound)
% \circled{f}\cdot \bE\left[a_m\mid\cF_{m-1}\right]
\circled{f}\cdot \bE\left[\psi_m\mid\cF_{m-1}\right]
% \mid D_N,\tilde{\ell}_f
\mid\cF_0
\right].
\end{eqnarray*}
Moreover,
\begin{eqnarray*}
% \bE\left[a_m\mid D_N,\tilde{\ell}_f,i_{1:m-1}\right]
% [ZJ] notation: c_m, a_m -> w_m, \psi_m (c is the betting constant, a_t is the scaling bound)
% \bE\left[a_m\mid\cF_{m-1}\right]
\bE\left[\psi_m\mid\cF_{m-1}\right]
&=&
\sum_{j\in\cJ_{m-1}}
% q_m(j\mid i_{1:m-1},D_N)
q_m(j\mid i_{1:m-1},X_N)
\frac{\ell\left(f(\x_j),y_j\right)-\tilde{\ell}_f\left(\x_j\right)}
% {Nq_m(j\mid i_{1:m-1},D_N)}\\
{Nq_m(j\mid i_{1:m-1},X_N)}\\
&&+
% [ZJ] notation: summation index u -> k (u_{m,t} is the LURE weight)
% \frac{1}{N}\sum_{u=1}^{m-1}
% \left(
% \ell\left(f(\x_{i_u}),y_{i_u}\right)-\tilde{\ell}_f\left(\x_{i_u}\right)
\frac{1}{N}\sum_{k=1}^{m-1}
\left(
\ell\left(f(\x_{i_k}),y_{i_k}\right)-\tilde{\ell}_f\left(\x_{i_k}\right)
\right)\\
&=&
\frac{1}{N}\sum_{n=1}^{N}
\left(
\ell\left(f(\x_n),y_n\right)-\tilde{\ell}_f\left(\x_n\right)
\right).
\end{eqnarray*}
Therefore,
\begin{eqnarray*}
% \bE\left[a_m\cdot\circled{f}\mid D_N,\tilde{\ell}_f\right]
% [ZJ] notation: c_m, a_m -> w_m, \psi_m (c is the betting constant, a_t is the scaling bound)
% \bE\left[a_m\cdot\circled{f}\mid\cF_0\right]
\bE\left[\psi_m\cdot\circled{f}\mid\cF_0\right]
&=&
\bE\left[
\circled{f}\cdot
\frac{1}{N}\sum_{n=1}^{N}
\left(
\ell\left(f(\x_n),y_n\right)-\tilde{\ell}_f\left(\x_n\right)
\right)
% \mid D_N,\tilde{\ell}_f
\mid\cF_0
\right].
\end{eqnarray*}

% [ZJ] proof: it is the conditional expectation of the cross term given \cF_0 that vanishes; notation: c_m, a_m -> w_m, \psi_m (c is the betting constant, a_t is the scaling bound)
% Hence, when $s<m$, the cross term $\circled{e}\times\circled{f}$ is equal to zero. By symmetry, the same conclusion holds when $m<s$. Therefore, only the diagonal terms remain, and
% \begin{eqnarray*}
% \circled{c}
% &=&
% \frac{1}{t^2}\sum_{m=1}^{t}c_m^2
% \bE\left[
% \left(
% a_m-
Hence, when $s<m$, the conditional expectation of the cross term $\circled{e}\times\circled{f}$ given $\cF_0$ is equal to zero. By symmetry, the same conclusion holds when $m<s$. Therefore, only the diagonal terms remain, and
\begin{eqnarray*}
\circled{c}
&=&
\frac{1}{t^2}\sum_{m=1}^{t}w_m^2
\bE\left[
\left(
\psi_m-
\frac{1}{N}\sum_{n=1}^{N}
\left(
\ell\left(f(\x_n),y_n\right)-\tilde{\ell}_f\left(\x_n\right)
\right)
\right)^2
% \mid D_N,\tilde{\ell}_f
\mid\cF_0
\right].
\end{eqnarray*}

Using the tower property once more, we can rewrite the remaining term as
\begin{eqnarray*}
&&
\bE\left[
\left(
% [ZJ] notation: c_m, a_m -> w_m, \psi_m (c is the betting constant, a_t is the scaling bound)
% a_m-
\psi_m-
\frac{1}{N}\sum_{n=1}^{N}
\left(
\ell\left(f(\x_n),y_n\right)-\tilde{\ell}_f\left(\x_n\right)
\right)
\right)^2
% \mid D_N,\tilde{\ell}_f
\mid\cF_0
\right]\\
&=&
\bE\Bigg[
\mathrm{Var}\left[
% [ZJ] notation: c_m, a_m -> w_m, \psi_m (c is the betting constant, a_t is the scaling bound)
% a_m
\psi_m
% \mid D_N,\tilde{\ell}_f,i_{1:m-1}
\mid\cF_{m-1}
\right]
% \mid D_N,\tilde{\ell}_f
\mid\cF_0
\Bigg].
\end{eqnarray*}
Therefore,
\[
% \mathrm{Var}\left[\hat R_t\mid D_N,\tilde{\ell}_f\right]
\mathrm{Var}\left[\hat R_t\mid\cF_0\right]
=
% [ZJ] notation: c_m, a_m -> w_m, \psi_m (c is the betting constant, a_t is the scaling bound)
% \frac{1}{t^2}\sum_{m=1}^{t}c_m^2
% \bE\Bigg[
% \mathrm{Var}\left[
% a_m
\frac{1}{t^2}\sum_{m=1}^{t}w_m^2
\bE\Bigg[
\mathrm{Var}\left[
\psi_m
% \mid D_N,\tilde{\ell}_f,i_{1:m-1}
\mid\cF_{m-1}
\right]
% \mid D_N,\tilde{\ell}_f
\mid\cF_0
\Bigg].
\]

Substituting this and the expression for $\mathrm{Var}[\circled{d}]$ back into \eqref{eq:lure_var_decomp}, we obtain
\begin{eqnarray*}
\mathrm{Var}[\hat R_t]
&=&
\frac{\mathrm{Var}(\ell(f(\x),y))}{N}
+
\bE\left[
% [ZJ] notation: c_m, a_m -> w_m, \psi_m (c is the betting constant, a_t is the scaling bound)
% \frac{1}{t^2}\sum_{m=1}^{t}c_m^2
% \bE\Bigg[
% \mathrm{Var}\left[
% a_m
\frac{1}{t^2}\sum_{m=1}^{t}w_m^2
\bE\Bigg[
\mathrm{Var}\left[
\psi_m
% \mid D_N,\tilde{\ell}_f,i_{1:m-1}
\mid\cF_{m-1}
\right]
% \mid D_N,\tilde{\ell}_f
\mid\cF_0
\Bigg]
\right]\\
&=&
\frac{\mathrm{Var}(\ell(f(\x),y))}{N}
+
\bE\left[
\mathrm{Var}\left[
% [ZJ] notation: c_m, a_m -> w_m, \psi_m (c is the betting constant, a_t is the scaling bound)
% \frac{1}{t}\sum_{m=1}^{t}c_m a_m
\frac{1}{t}\sum_{m=1}^{t}w_m \psi_m
% \mid D_N,\tilde{\ell}_f
\mid\cF_0
\right]
\right],
\end{eqnarray*}
which completes the proof.
\end{proof}

\appsubsection{Proofs of Theorem~\ref{thm:consistency}}
We present the detailed proof of Theorem~\ref{thm:consistency} as follows.
\begin{proof}
Under Assumption~\ref{ass:loss_bound}-\ref{ass:beta}, Theorem~\ref{thm:lure_unbias} establishes the unbiasedness of the estimator $\hat R_t$ at each round $t$. Hence
\[
\bE\left[\left(\hat R_t-R\right)^2\right]
=
\mathrm{Var}[\hat R_t].
\]
To establish convergence, it suffices to show that the variance vanishes along the sequence $\{t_N\}_{N\ge 1}$.

\paragraph{Step 1: Control the variance of the residual-correction term.}
Recall from Theorem~\ref{thm:lure_unbias} that, for $t<N$,
\[
\mathrm{Var}[\hat R_t]
=
\frac{\mathrm{Var}(\ell(f(\x),y))}{N}
+
\bE\left[
\mathrm{Var}\left[
% [ZJ] notation: c_m, a_m -> w_m, \psi_m (c is the betting constant, a_t is the scaling bound)
% \frac{1}{t}\sum_{m=1}^{t}c_m a_m
\frac{1}{t}\sum_{m=1}^{t}w_m \psi_m
% \mid D_N,\tilde{\ell}_f
\mid\cF_0
\right]
\right],
\]
where
\[
% [ZJ] notation: c_m, a_m -> w_m, \psi_m (c is the betting constant, a_t is the scaling bound)
% c_m=\frac{N(N-t)}{(N-m)(N-m+1)},
% \]
% and
% \[
% a_m=
w_m=\frac{N(N-t)}{(N-m)(N-m+1)},
\]
and
\[
\psi_m=
\frac{\ell\left(f(\x_{i_m}),y_{i_m}\right)-\tilde{\ell}_f(\x_{i_m})}
% {Nq_m(i_m\mid i_{1:m-1},D_N)}
{Nq_m(i_m\mid i_{1:m-1},X_N)}
+
\frac{1}{N}\sum_{s=1}^{m-1}
\left(
\ell\left(f(\x_{i_s}),y_{i_s}\right)-\tilde{\ell}_f\left(\x_{i_s}\right)
\right).
\]
Moreover, by the variance analysis in the proof of Theorem~\ref{thm:lure_unbias},
\[
\mathrm{Var}\left[
% [ZJ] notation: c_m, a_m -> w_m, \psi_m (c is the betting constant, a_t is the scaling bound)
% \frac{1}{t}\sum_{m=1}^{t}c_m a_m
\frac{1}{t}\sum_{m=1}^{t}w_m \psi_m
% \mid D_N,\tilde{\ell}_f
\mid\cF_0
\right]
=
% [ZJ] notation: c_m, a_m -> w_m, \psi_m (c is the betting constant, a_t is the scaling bound)
% \frac{1}{t^2}\sum_{m=1}^{t}c_m^2
% \bE\Bigg[
% \mathrm{Var}\left[
% a_m
\frac{1}{t^2}\sum_{m=1}^{t}w_m^2
\bE\Bigg[
\mathrm{Var}\left[
\psi_m
% \mid D_N,\tilde{\ell}_f,i_{1:m-1}
\mid\cF_{m-1}
\right]
% \mid D_N,\tilde{\ell}_f
\mid\cF_0
\Bigg].
\]

% We now bound the conditional variance of $a_m$. Since the second term in $a_m$ is measurable with respect to $(D_N,\tilde{\ell}_f,i_{1:m-1})$, it does not contribute to the conditional variance. Hence,
% [ZJ] notation: c_m, a_m -> w_m, \psi_m (c is the betting constant, a_t is the scaling bound)
% We now bound the conditional variance of $a_m$. Since the second term in $a_m$ is measurable with respect to $\cF_{m-1}$, it does not contribute to the conditional variance. Hence,
% \begin{eqnarray*}
% \lefteqn{
% \mathrm{Var}\left[
% a_m
We now bound the conditional variance of $\psi_m$. Since the second term in $\psi_m$ is measurable with respect to $\cF_{m-1}$, it does not contribute to the conditional variance. Hence,
\begin{eqnarray*}
\lefteqn{
\mathrm{Var}\left[
\psi_m
% \mid D_N,\tilde{\ell}_f,i_{1:m-1}
\mid\cF_{m-1}
\right]
}\\
&=&
\mathrm{Var}\left[
\frac{\ell\left(f(\x_{i_m}),y_{i_m}\right)-\tilde{\ell}_f(\x_{i_m})}
% {Nq_m(i_m\mid i_{1:m-1},D_N)}
{Nq_m(i_m\mid i_{1:m-1},X_N)}
% \mid D_N,\tilde{\ell}_f,i_{1:m-1}
\mid\cF_{m-1}
\right]\\
&\le&
\bE\left[
\left(
\frac{\ell\left(f(\x_{i_m}),y_{i_m}\right)-\tilde{\ell}_f(\x_{i_m})}
% {Nq_m(i_m\mid i_{1:m-1},D_N)}
{Nq_m(i_m\mid i_{1:m-1},X_N)}
\right)^2
% \mid D_N,\tilde{\ell}_f,i_{1:m-1}
\mid\cF_{m-1}
\right]\\
&=&
\sum_{j\in \cJ_{m-1}}
% q_m(j\mid i_{1:m-1},D_N)
q_m(j\mid i_{1:m-1},X_N)
\frac{\left(\ell\left(f(\x_j),y_j\right)-\tilde{\ell}_f\left(\x_j\right)\right)^2}
% {N^2q_m^2(j\mid i_{1:m-1},D_N)}\\
{N^2q_m^2(j\mid i_{1:m-1},X_N)}\\
&=&
\sum_{j\in \cJ_{m-1}}
\frac{\left(\ell\left(f(\x_j),y_j\right)-\tilde{\ell}_f\left(\x_j\right)\right)^2}
% {N^2q_m(j\mid i_{1:m-1},D_N)}\ .
{N^2q_m(j\mid i_{1:m-1},X_N)}.
\end{eqnarray*}

Under Assumption~\ref{ass:loss_bound}, $L\le \ell,\tilde{\ell}_f\le U$, so
\[
\left|
\ell\left(f(\x_j),y_j\right)-\tilde{\ell}_f\left(\x_j\right)
\right|
\le U-L,
\qquad
j\in\cJ_{m-1}.
\]
Using the positivity assumption (Assumption~\ref{ass:beta})
\[
% q_m(j\mid i_{1:m-1},D_N)\ge \beta/(N-m+1),
q_m(j\mid i_{1:m-1},X_N)\ge \beta/(N-m+1),
\qquad
j\in\cJ_{m-1},
\]
we obtain
\[
% \frac{1}{q_m(j\mid i_{1:m-1},D_N)}
\frac{1}{q_m(j\mid i_{1:m-1},X_N)}
\le
\frac{N-m+1}{\beta}.
\]
Therefore,
\begin{eqnarray*}
\mathrm{Var}\left[
% [ZJ] notation: c_m, a_m -> w_m, \psi_m (c is the betting constant, a_t is the scaling bound)
% a_m
\psi_m
% \mid D_N,\tilde{\ell}_f,i_{1:m-1}
\mid\cF_{m-1}
\right]
&\le&
\sum_{j\in \cJ_{m-1}}
% \frac{(U-L)^2}{N^2q_m(j\mid i_{1:m-1},D_N)}\\
\frac{(U-L)^2}{N^2q_m(j\mid i_{1:m-1},X_N)}\\
&\le&
\frac{(U-L)^2}{N^2}
\sum_{j\in \cJ_{m-1}}
\frac{N-m+1}{\beta}\\
&\leq&
\frac{(U-L)^2}{\beta}\frac{|\cJ_{m-1}|}{N}\\
&=&
\frac{(U-L)^2}{\beta}\frac{N-m+1}{N}\ .
\end{eqnarray*}

Substituting this into the above display yields
\begin{equation}\label{eq:lure_consistency_bound_1}
\bE\left[
\mathrm{Var}\left[
% [ZJ] notation: c_m, a_m -> w_m, \psi_m (c is the betting constant, a_t is the scaling bound)
% \frac{1}{t}\sum_{m=1}^{t}c_m a_m
\frac{1}{t}\sum_{m=1}^{t}w_m \psi_m
% \mid D_N,\tilde{\ell}_f
\mid\cF_0
\right]
\right]
\le
\frac{(U-L)^2}{\beta}
\frac{1}{t^2}
% [ZJ] notation: c_m, a_m -> w_m, \psi_m (c is the betting constant, a_t is the scaling bound)
% \sum_{m=1}^{t}c_m^2\frac{N-m+1}{N}.
\sum_{m=1}^{t}w_m^2\frac{N-m+1}{N}.
\end{equation}

We next bound the coefficient sum. For $t<N$,
\begin{eqnarray*}
\lefteqn{
\frac{1}{t^2}
% [ZJ] notation: c_m, a_m -> w_m, \psi_m (c is the betting constant, a_t is the scaling bound)
% \sum_{m=1}^{t}c_m^2\frac{N-m+1}{N}
\sum_{m=1}^{t}w_m^2\frac{N-m+1}{N}
}\\
&=&
\frac{1}{t^2}
\sum_{m=1}^{t}
\left(
\frac{N(N-t)}{(N-m)(N-m+1)}
\right)^2
\frac{N-m+1}{N}\\
&=&
\frac{N(N-t)^2}{t^2}
\sum_{m=1}^{t}
\frac{1}{(N-m)^2(N-m+1)}\\
&\le&
\frac{N(N-t)^2}{t^2}
\sum_{m=1}^{t}
\frac{1}{(N-t)(N-m)(N-m+1)}\\
&=&
\frac{N(N-t)}{t^2}
\sum_{m=1}^{t}
\frac{1}{(N-m)(N-m+1)}\\
&=&
\frac{N(N-t)}{t^2}
\sum_{m=1}^{t}
\left(
\frac{1}{N-m}-\frac{1}{N-m+1}
\right)\\
&=&
\frac{N(N-t)}{t^2}
\left(
\frac{1}{N-t}-\frac{1}{N}
\right)\\
&=&
\frac{1}{t}.
\end{eqnarray*}
Combining this with \eqref{eq:lure_consistency_bound_1}, we obtain
\begin{equation}\label{eq:lure_consistency_bound_2}
\bE\left[
\mathrm{Var}\left[
% [ZJ] notation: c_m, a_m -> w_m, \psi_m (c is the betting constant, a_t is the scaling bound)
% \frac{1}{t}\sum_{m=1}^{t}c_m a_m
\frac{1}{t}\sum_{m=1}^{t}w_m \psi_m
% \mid D_N,\tilde{\ell}_f
\mid\cF_0
\right]
\right]
\le
\frac{(U-L)^2}{\beta t}.
\end{equation}

\paragraph{Step 2: Conclude the $L^2$ convergence.}
If $t<N$, combining Theorem~\ref{thm:lure_unbias} with \eqref{eq:lure_consistency_bound_2} gives
\[
\bE\left[\left(\hat R_t-R\right)^2\right]
=
\mathrm{Var}[\hat R_t]
\le
\frac{\mathrm{Var}\left(\ell(f(\x),y)\right)}{N}
+
\frac{(U-L)^2}{\beta t}.
\]
If $t=N$, then $\hat R_N=R_N$ pathwise, so
\[
\bE\left[\left(\hat R_N-R\right)^2\right]
=
\bE\left[\left(R_N-R\right)^2\right]
=
\frac{\mathrm{Var}\left(\ell(f(\x),y)\right)}{N}.
\]
Therefore, for all $1\le t\le N$,
\[
\bE\left[\left(\hat R_t-R\right)^2\right]
\le
\frac{\mathrm{Var}\left(\ell(f(\x),y)\right)}{N}
+
\frac{(U-L)^2}{\beta t}.
\]

Finally, if $t=t_N\le N$ is any sequence such that $N\to\infty$ and
\[
\frac{N-t_N+1}{N}\to 0,
\]
then $t_N/N\to 1$, and hence $t_N\to\infty$. Therefore both terms on the right-hand side vanish, and we obtain
\[
\bE\left[\left(\hat R_{t_N}-R\right)^2\right]\to 0.
\]
Hence $\hat R_{t_N}\to R$ in $L^2$, and therefore in probability.
\end{proof}

\appsubsection{Proofs of Theorem~\ref{thm:signal_unbias}}
We present the detailed proof of Theorem~\ref{thm:signal_unbias} as follows.

\begin{proof}
Recall the signal in \eqref{eq:signal_estimator}:
\begin{equation*}
\hat{S}_t
=
\frac{
\sum_{m=1}^{t-1}\ell\left(f(\x_{i_m}),y_{i_m}\right)
+
\sum_{j\in\cJ_{t-1}}\tilde{\ell}_f\left(\x_j\right)
}{N}
+
\frac{
\ell\left(f(\x_{i_t}),y_{i_t}\right)
-
\tilde{\ell}_f\left(\x_{i_t}\right)
}{
% Nq_t(i_t\mid i_{1:t-1},D_N)
Nq_t(i_t\mid i_{1:t-1},X_N)
}\ .
\end{equation*}
By Assumption~\ref{ass:beta}, the proposal probabilities are strictly positive on the remaining pool, so the inverse-probability terms above are well defined.

We first rewrite $\hat S_t$ around the finite-pool risk
\[
R_N=\frac{1}{N}\sum_{j=1}^N \ell(f(\x_j),y_j)=
\frac{\sum_{m=1}^{t-1}\ell(f(\x_{i_m}),y_{i_m})
+\sum_{j\in \mathcal J_{t-1}}\ell(f(\x_j),y_j)}{N}\ ,
\]
subtracting this identity from \eqref{eq:signal_estimator} yields
\begin{equation}\label{eq:decomp_hatS}
\hat S_t
=
R_N
+
\underbrace{
\frac{1}{N}
\left(
\frac{\ell(f(\x_{i_t}),y_{i_t})-\tilde{\ell}_f(\x_{i_t})}
% {q_t(i_t\mid i_{1:t-1},D_N)}
{q_t(i_t\mid i_{1:t-1},X_N)}
-
\sum_{j\in \mathcal J_{t-1}}
\left(\ell(f(\x_j),y_j)-\tilde{\ell}_f(\x_j)\right)
\right)
}_{\circled{a}}.
\end{equation}

% Under Assumption~\ref{ass:loss_bound}-\ref{ass:beta}, we first prove the conditional unbiasedness. Since $R_N$ is measurable with respect to $\cF_{t-1}$, taking conditional expectation on both sides of \eqref{eq:decomp_hatS} gives
% Under Assumption~\ref{ass:loss_bound}-\ref{ass:beta}, we first prove the conditional unbiasedness. Since $\cF_{t-1}$ includes the complete pool $\mathcal Z_N$, $R_N$ is measurable with respect to $\cF_{t-1}$. Thus, taking conditional expectation on both sides of \eqref{eq:decomp_hatS} gives
Under Assumption~\ref{ass:loss_bound}-\ref{ass:beta}, we first prove the conditional unbiasedness. Since $\cF_{t-1}$ includes the complete pool $D_N$, $R_N$ is measurable with respect to $\cF_{t-1}$. Thus, taking conditional expectation on both sides of \eqref{eq:decomp_hatS} gives
\[
\bE[\hat S_t\mid \cF_{t-1}]
=
R_N+\bE[\circled{a}\mid \cF_{t-1}].
\]
We now compute $\bE[\circled{a}\mid \cF_{t-1}]$. Conditional on $\cF_{t-1}$, the next queried index $i_t$ takes values in $\mathcal J_{t-1}$ with probabilities
% $q_t(j\mid i_{1:t-1},D_N)$. Therefore,
$q_t(j\mid i_{1:t-1},X_N)$. Therefore,
\begin{eqnarray*}
\bE[\circled{a}\mid \cF_{t-1}]
&=&
\frac{1}{N}
\sum_{j\in \mathcal J_{t-1}}
% q_t(j\mid i_{1:t-1},D_N)
q_t(j\mid i_{1:t-1},X_N)
\frac{\ell(f(\x_j),y_j)-\tilde{\ell}_f(\x_j)}
% {q_t(j\mid i_{1:t-1},D_N)}
{q_t(j\mid i_{1:t-1},X_N)}
\\
&&
\ -\ \frac{1}{N}
\sum_{j\in \mathcal J_{t-1}}
% q_t(j\mid i_{1:t-1},D_N)
q_t(j\mid i_{1:t-1},X_N)
\sum_{k\in \mathcal J_{t-1}}
\left(\ell(f(\x_k),y_k)-\tilde{\ell}_f(\x_k)\right)\\
&=&
\frac{1}{N}
\sum_{j\in \mathcal J_{t-1}}
\left(\ell(f(\x_j),y_j)-\tilde{\ell}_f(\x_j)\right)
\\
&&
\ -\ \frac{1}{N}
\left(
% \sum_{j\in \mathcal J_{t-1}} q_t(j\mid i_{1:t-1},D_N)
\sum_{j\in \mathcal J_{t-1}} q_t(j\mid i_{1:t-1},X_N)
\right)
\left(
\sum_{k\in \mathcal J_{t-1}}
\left(\ell(f(\x_k),y_k)-\tilde{\ell}_f(\x_k)\right)
\right)\\
&=&0,
\end{eqnarray*}
where the last equality uses
\[
% \sum_{j\in \mathcal J_{t-1}} q_t(j\mid i_{1:t-1},D_N)=1.
\sum_{j\in \mathcal J_{t-1}} q_t(j\mid i_{1:t-1},X_N)=1.
\]
Hence
\[
\bE[\hat S_t\mid \cF_{t-1}] = R_N,
\]
which proves the conditional unbiasedness.

We next investigate the conditional variance of $\hat S_t$. Since $R_N$ is $\cF_{t-1}$-measurable, by \eqref{eq:decomp_hatS},
\[
\mathrm{Var}[\hat S_t\mid \cF_{t-1}]
=
\mathrm{Var}[\circled{a}\mid \cF_{t-1}].
\]
Moreover, since $\bE[\circled{a}\mid \cF_{t-1}]=0$,
\[
\mathrm{Var}[\circled{a}\mid \cF_{t-1}]
=
\bE[\circled{a}^2\mid \cF_{t-1}].
\]
Expanding $\circled{a}^2$ from \eqref{eq:decomp_hatS}, we obtain
\begin{eqnarray*}
\lefteqn{\circled{a}^2\ =\ }\\
&&\underbrace{\frac{\left(\ell(f(\x_{i_t}),y_{i_t})-\tilde{\ell}_f(\x_{i_t})\right)^2}
% {N^2 q_t^2(i_t\mid i_{1:t-1},D_N)}}_{\circled{b}}\\
{N^2 q_t^2(i_t\mid i_{1:t-1},X_N)}}_{\circled{b}}\\
% &&\ -\ \underbrace{\frac{2}{N^2}\cdot\frac{\ell(f(\x_{i_t}),y_{i_t})-\tilde{\ell}_f(\x_{i_t})}{q_t(i_t\mid i_{1:t-1},D_N)}\left(\sum_{k\in \mathcal J_{t-1}}\ell(f(\x_k),y_k)-\tilde{\ell}_f(\x_k)\right)}_{\circled{c}}\\
% [ZJ] notation: parentheses around the summed residual
% &&-\underbrace{\frac{2}{N^2}\cdot\frac{\ell(f(\x_{i_t}),y_{i_t})-\tilde{\ell}_f(\x_{i_t})}{q_t(i_t\mid i_{1:t-1},X_N)}\left(\sum_{k\in \mathcal J_{t-1}}\ell(f(\x_k),y_k)-\tilde{\ell}_f(\x_k)\right)}_{\circled{c}}\\
% &&\ +\ \underbrace{\frac{1}{N^2}\left(\sum_{k\in \mathcal J_{t-1}}\ell(f(\x_k),y_k)-\tilde{\ell}_f(\x_k)\right)^2}_{\circled{d}}\ .
&&-\underbrace{\frac{2}{N^2}\cdot\frac{\ell(f(\x_{i_t}),y_{i_t})-\tilde{\ell}_f(\x_{i_t})}{q_t(i_t\mid i_{1:t-1},X_N)}\left(\sum_{k\in \mathcal J_{t-1}}\left(\ell(f(\x_k),y_k)-\tilde{\ell}_f(\x_k)\right)\right)}_{\circled{c}}\\
&&\ +\ \underbrace{\frac{1}{N^2}\left(\sum_{k\in \mathcal J_{t-1}}\left(\ell(f(\x_k),y_k)-\tilde{\ell}_f(\x_k)\right)\right)^2}_{\circled{d}}\ .
\end{eqnarray*}

Then, we have
\[
\bE[\circled{a}^2\mid \cF_{t-1}]
=
\bE[\circled{b}\mid \cF_{t-1}]
-
\bE[\circled{c}\mid \cF_{t-1}]
+
\bE[\circled{d}\mid \cF_{t-1}].
\]
For the first term,
\begin{eqnarray*}
\lefteqn{\bE[\circled{b}\mid \cF_{t-1}]}\\
&=&
% \sum_{j\in \mathcal J_{t-1}}q_t(j\mid i_{1:t-1},D_N)
\sum_{j\in \mathcal J_{t-1}}q_t(j\mid i_{1:t-1},X_N)
\frac{\left(\ell(f(\x_j),y_j)-\tilde{\ell}_f(\x_j)\right)^2}
% {N^2 q_t^2(j\mid i_{1:t-1},D_N)}\\
{N^2 q_t^2(j\mid i_{1:t-1},X_N)}\\
&=&
\sum_{j\in \mathcal J_{t-1}}
\frac{\left(\ell(f(\x_j),y_j)-\tilde{\ell}_f(\x_j)\right)^2}
% {N^2 q_t(j\mid i_{1:t-1},D_N)}\ .
{N^2 q_t(j\mid i_{1:t-1},X_N)}.
\end{eqnarray*}

For the second term, in a similar manner, we have
\begin{eqnarray*}
\lefteqn{\bE[\circled{c}\mid \cF_{t-1}]}\\
&=&
\sum_{j\in \mathcal J_{t-1}}
\frac{2\left(\ell(f(\x_j),y_j)-\tilde{\ell}_f(\x_j)\right)
% [ZJ] notation: parentheses around the summed residual
% \left(\sum_{k\in \cJ_{t-1}}\ell(f(\x_k),y_k)-\tilde{\ell}_f(\x_k)\right)}
\left(\sum_{k\in \cJ_{t-1}}\left(\ell(f(\x_k),y_k)-\tilde{\ell}_f(\x_k)\right)\right)}
% {q^{-1}_t(j\mid i_{1:t-1},D_N) N^2 q_t(j\mid i_{1:t-1},D_N)}\\
{q^{-1}_t(j\mid i_{1:t-1},X_N) N^2 q_t(j\mid i_{1:t-1},X_N)}\\
&=&
\sum_{j\in \mathcal J_{t-1}}
\frac{2\left(\ell(f(\x_j),y_j)-\tilde{\ell}_f(\x_j)\right)
% [ZJ] notation: parentheses around the summed residual
% \left(\sum_{k\in \cJ_{t-1}}\ell(f(\x_k),y_k)-\tilde{\ell}_f(\x_k)\right)}
% {N^2}\\
% &=&
% 2\left(
% \frac{\sum_{k\in \cJ_{t-1}}\ell(f(\x_k),y_k)-\tilde{\ell}_f(\x_k)}{N}
\left(\sum_{k\in \cJ_{t-1}}\left(\ell(f(\x_k),y_k)-\tilde{\ell}_f(\x_k)\right)\right)}
{N^2}\\
&=&
2\left(
\frac{\sum_{k\in \cJ_{t-1}}\left(\ell(f(\x_k),y_k)-\tilde{\ell}_f(\x_k)\right)}{N}
\right)^2\ .
\end{eqnarray*}

Finally, since $\circled{d}$ is $\cF_{t-1}$-measurable, we have
\[
\bE[\circled{d}\mid \cF_{t-1}]
=
\left(
% [ZJ] notation: parentheses around the summed residual
% \frac{\sum_{k\in \cJ_{t-1}}\ell(f(\x_k),y_k)-\tilde{\ell}_f(\x_k)}{N}
\frac{\sum_{k\in \cJ_{t-1}}\left(\ell(f(\x_k),y_k)-\tilde{\ell}_f(\x_k)\right)}{N}
\right)^2.
\]
Therefore,
\begin{align*}
\mathrm{Var}[\hat S_t \mid \cF_{t-1}]
=
&
\sum_{j\in \mathcal J_{t-1}}
\frac{\left(\ell(f(\x_j),y_j)-\tilde{\ell}_f(\x_j)\right)^2}
% {N^2 q_t(j\mid i_{1:t-1},D_N)}\\
{N^2 q_t(j\mid i_{1:t-1},X_N)}\\
&-
\left(
\frac{\sum_{j\in \mathcal J_{t-1}}
\left(\ell(f(\x_j),y_j)-\tilde{\ell}_f(\x_j)\right)}
{N}
\right)^2\ ,
\end{align*}
which completes the proof.
\end{proof}

\appsubsection{Proofs of Theorem~\ref{thm:optimal}}
We present the detailed proof of Theorem~\ref{thm:optimal} as follows.
\begin{proof}
% Fix $t$. Throughout this proof, we condition on the full finite pool, including its loss values, and consider the randomness of the sampling procedure. Let
% \[
% \cF_{t-1}=\sigma(D_N,i_{1:t-1},\tilde{\ell}_f).
% \]
Fix $t$. Throughout this proof, we condition on $\cF_{t-1}$ from \eqref{eq:analysis_information}, which fixes the complete pool, preprocessing information, and previous selections. The oracle proposal may use the unobserved pool losses; it is a mathematical benchmark and need not be determined by the observed information $\mathcal O_{t-1}$.
By Assumption~\ref{ass:loss_bound}, both $\ell(f(\x_j),y_j)$ and $\tilde{\ell}_f(\x_j)$ are uniformly bounded, hence all quantities
\[
\ell(f(\x_j),y_j)-\tilde{\ell}_f(\x_j),
\qquad
j\in\cJ_{t-1},
\]
are finite. Under the conditioning above, the pool losses are fixed, and $\tilde{\ell}_f$ is measurable with respect to $\cF_{t-1}$. Therefore, these residuals are fixed conditional on $\cF_{t-1}$.

If $t=N$, there is only one remaining sample, so the only feasible proposal assigns it probability one. With the terminal convention $\hat R_N=R_N$, both $\hat R_N$ and $\hat S_N$ equal $R_N$, and the claimed optimality follows immediately. We therefore consider $t<N$ below.

We first consider the inferential signal $\hat S_t$. Recall that
\[
\hat S_t
=
\frac{\sum_{m=1}^{t-1}\ell(f(\x_{i_m}),y_{i_m})+\sum_{j\in\cJ_{t-1}}\tilde{\ell}_f(\x_j)}{N}
+
\frac{\ell(f(\x_{i_t}),y_{i_t})-\tilde{\ell}_f(\x_{i_t})}
% {Nq_t(i_t\mid i_{1:t-1},D_N)}.
{Nq_t(i_t\mid i_{1:t-1},X_N)}.
\]
By Assumption~\ref{ass:beta}, the proposal probabilities are strictly positive on the remaining pool, so the inverse-probability terms above are well defined.

We also recall that
\[
R_N
=
\frac{\sum_{m=1}^{t-1}\ell(f(\x_{i_m}),y_{i_m})+\sum_{j\in\cJ_{t-1}}\ell(f(\x_j),y_j)}{N}.
\]
Hence
\begin{equation}\label{eq:optimal_signal_decomp}
\hat S_t-R_N
=
\frac{\ell(f(\x_{i_t}),y_{i_t})-\tilde{\ell}_f(\x_{i_t})}
% {Nq_t(i_t\mid i_{1:t-1},D_N)}
{Nq_t(i_t\mid i_{1:t-1},X_N)}
-
\frac{1}{N}\sum_{j\in\cJ_{t-1}}
\left(\ell(f(\x_j),y_j)-\tilde{\ell}_f(\x_j)\right).
\end{equation}

By Theorem~\ref{thm:signal_unbias},
\[
\bE[\hat S_t\mid\cF_{t-1}]=R_N,
\]
and
\begin{align*}
\mathrm{Var}(\hat S_t\mid\cF_{t-1})
&=
\sum_{j\in\cJ_{t-1}}
\frac{\left(\ell(f(\x_j),y_j)-\tilde{\ell}_f(\x_j)\right)^2}
% {N^2q_t(j\mid i_{1:t-1},D_N)}
{N^2q_t(j\mid i_{1:t-1},X_N)}
\\
&\qquad
-
\left(
\frac{\sum_{j\in\cJ_{t-1}}
\left(\ell(f(\x_j),y_j)-\tilde{\ell}_f(\x_j)\right)}
{N}
\right)^2.
\end{align*}
The second term does not depend on the proposal distribution, so minimizing
$\mathrm{Var}(\hat S_t\mid\cF_{t-1})$ is equivalent to minimizing
\[
\sum_{j\in\cJ_{t-1}}
\frac{\left(\ell(f(\x_j),y_j)-\tilde{\ell}_f(\x_j)\right)^2}
% {q_t(j\mid i_{1:t-1},D_N)}
{q_t(j\mid i_{1:t-1},X_N)}
\]
over proposal distributions satisfying
\[
% \sum_{j\in\cJ_{t-1}}q_t(j\mid i_{1:t-1},D_N)=1,
\sum_{j\in\cJ_{t-1}}q_t(j\mid i_{1:t-1},X_N)=1,
\qquad
% q_t(j\mid i_{1:t-1},D_N)\ge\frac{\beta}{N-t+1}.
q_t(j\mid i_{1:t-1},X_N)\ge\frac{\beta}{N-t+1}.
\]

Suppose first that the residuals are not all zero. For this optimization, write
\[
M=N-t+1,
\qquad
r_j=\ell(f(\x_j),y_j)-\tilde{\ell}_f(\x_j),
\qquad
% q_j=q_t(j\mid i_{1:t-1},D_N).
q_j=q_t(j\mid i_{1:t-1},X_N).
\]
If $\beta=1$, the only feasible proposal is uniform:
\[
q_j=\frac{1}{M},
\qquad
j\in\cJ_{t-1}.
\]
This agrees with the claimed formula for any
\[
% [ZJ] notation: \tau_t -> \eta_t (\tau denotes the stopping time)
% \tau_t\ge M\max_{j\in\cJ_{t-1}}|r_j|.
\eta_t\ge M\max_{j\in\cJ_{t-1}}|r_j|.
\]

Now consider $0<\beta<1$. The objective is convex, the feasible set is compact, and the uniform proposal strictly satisfies the lower-bound constraints. Thus, the Karush--Kuhn--Tucker conditions are necessary and sufficient for optimality.

The Lagrangian is
\[
% [ZJ] notation: Lagrange multiplier \lambda -> \zeta (\lambda_t^\pm are the betting fractions)
% \mathcal L(q,\lambda,\mu)
% =
% \sum_{j\in\cJ_{t-1}}\frac{r_j^2}{q_j}
% +
% \lambda\left(\sum_{j\in\cJ_{t-1}}q_j-1\right)
% +
% \sum_{j\in\cJ_{t-1}}\mu_j
% \left(\frac{\beta}{M}-q_j\right),
% \]
% where $\lambda\in\mathbb R$ and $\mu_j\ge0$. Stationarity and complementary slackness give
% \[
% -\frac{r_j^2}{q_j^2}+\lambda-\mu_j=0,
\mathcal L(q,\zeta,\mu)
=
\sum_{j\in\cJ_{t-1}}\frac{r_j^2}{q_j}
+
\zeta\left(\sum_{j\in\cJ_{t-1}}q_j-1\right)
+
\sum_{j\in\cJ_{t-1}}\mu_j
\left(\frac{\beta}{M}-q_j\right),
\]
where $\zeta\in\mathbb R$ and $\mu_j\ge0$. Stationarity and complementary slackness give
\[
-\frac{r_j^2}{q_j^2}+\zeta-\mu_j=0,
\qquad
\mu_j\left(q_j-\frac{\beta}{M}\right)=0.
\]

Since at least one residual is nonzero, there exists some $j$ such that $r_j\ne0$. For this $j$, stationarity gives
\[
% [ZJ] notation: Lagrange multiplier \lambda -> \zeta (\lambda_t^\pm are the betting fractions)
% \lambda
% =
% \frac{r_j^2}{q_j^2}+\mu_j
% >0.
% \]
% Therefore, $\sqrt{\lambda}$ is well defined and positive.
\zeta
=
\frac{r_j^2}{q_j^2}+\mu_j
>0.
\]
Therefore, $\sqrt{\zeta}$ is well defined and positive.

If $q_j>\beta/M$, complementary slackness gives $\mu_j=0$. The stationarity condition then becomes
\[
% [ZJ] notation: Lagrange multiplier \lambda -> \zeta (\lambda_t^\pm are the betting fractions); notation: \tau_t -> \eta_t (\tau denotes the stopping time)
% \lambda=\frac{r_j^2}{q_j^2}.
% \]
% Since $q_j>0$, taking the positive square root gives
% \[
% \sqrt{\lambda}=\frac{|r_j|}{q_j},
% \]
% or equivalently,
% \[
% q_j=\frac{|r_j|}{\sqrt{\lambda}}.
% \]
% Thus, the common normalizing denominator of the residual-proportional probabilities is $\sqrt{\lambda}$. We denote this denominator by
% \[
% \tau_t=\sqrt{\lambda}.
% \]
% Consequently,
% \[
% q_j=\frac{|r_j|}{\tau_t}
\zeta=\frac{r_j^2}{q_j^2}.
\]
Since $q_j>0$, taking the positive square root gives
\[
\sqrt{\zeta}=\frac{|r_j|}{q_j},
\]
or equivalently,
\[
q_j=\frac{|r_j|}{\sqrt{\zeta}}.
\]
Thus, the common normalizing denominator of the residual-proportional probabilities is $\sqrt{\zeta}$. We denote this denominator by
\[
\eta_t=\sqrt{\zeta}.
\]
Consequently,
\[
q_j=\frac{|r_j|}{\eta_t}
\]
whenever $q_j>\beta/M$.

If $q_j=\beta/M$, stationarity and $\mu_j\ge0$ imply
\[
% [ZJ] notation: Lagrange multiplier \lambda -> \zeta (\lambda_t^\pm are the betting fractions); notation: \tau_t -> \eta_t (\tau denotes the stopping time)
% \lambda-\frac{r_j^2}{q_j^2}\ge0,
% \]
% and hence
% \[
% \frac{|r_j|}{\tau_t}\le\frac{\beta}{M}.
\zeta-\frac{r_j^2}{q_j^2}\ge0,
\]
and hence
\[
\frac{|r_j|}{\eta_t}\le\frac{\beta}{M}.
\]
Therefore, the optimal proposal must satisfy
\[
q_t^*(j\mid i_{1:t-1},D_N)
=
\max\left\{
\frac{\beta}{M},
% [ZJ] notation: \tau_t -> \eta_t (\tau denotes the stopping time)
% \frac{|r_j|}{\tau_t}
\frac{|r_j|}{\eta_t}
\right\}.
\]

% [ZJ] notation: \tau_t -> \eta_t (\tau denotes the stopping time); notation: a_{(s)} -> |r|_{(s)} (a_t is the scaling bound)
% We next derive $\tau_t$ explicitly. Let the residual magnitudes be ordered as
% \[
% a_{(1)}\ge a_{(2)}\ge\cdots\ge a_{(M)},
% \]
% where $\{a_{(1)},\ldots,a_{(M)}\}=\{|r_j|:j\in\cJ_{t-1}\}$, and set $a_{(M+1)}=0$. For each $k\in\{1,\ldots,M\}$, define
% \[
% \tau_{t,k}
% =
% \frac{\sum_{s=1}^{k}a_{(s)}}
We next derive $\eta_t$ explicitly. Let the residual magnitudes be ordered as
\[
|r|_{(1)}\ge |r|_{(2)}\ge\cdots\ge |r|_{(M)},
\]
where $\{|r|_{(1)},\ldots,|r|_{(M)}\}=\{|r_j|:j\in\cJ_{t-1}\}$, and set $|r|_{(M+1)}=0$. For each $k\in\{1,\ldots,M\}$, define
\[
\eta_{t,k}
=
\frac{\sum_{s=1}^{k}|r|_{(s)}}
{1-(M-k)\beta/M}.
\]
Let
\[
k_t
=
\max\left\{
k\in\{1,\ldots,M\}:
% [ZJ] notation: \tau_t -> \eta_t (\tau denotes the stopping time); notation: a_{(s)} -> |r|_{(s)} (a_t is the scaling bound)
% a_{(k)}>\frac{\beta}{M}\tau_{t,k}
% \right\}.
% \]
% This set is nonempty because $a_{(1)}>0$ and $\beta<1$. The normalizing constant is
% \[
% \tau_t
% =
% \tau_{t,k_t}
% =
% \frac{\sum_{s=1}^{k_t}a_{(s)}}
|r|_{(k)}>\frac{\beta}{M}\eta_{t,k}
\right\}.
\]
This set is nonempty because $|r|_{(1)}>0$ and $\beta<1$. The normalizing constant is
\[
\eta_t
=
\eta_{t,k_t}
=
\frac{\sum_{s=1}^{k_t}|r|_{(s)}}
{1-(M-k_t)\beta/M}.
\]

By the ordering of the residual magnitudes and the maximality of $k_t$,
\[
% [ZJ] notation: \tau_t -> \eta_t (\tau denotes the stopping time); notation: a_{(s)} -> |r|_{(s)} (a_t is the scaling bound); proof: justification of the second inequality (maximality of k_t)
% a_{(s)}>\frac{\beta}{M}\tau_t
% \quad\text{for }s\le k_t,
% \qquad
% a_{(s)}\le\frac{\beta}{M}\tau_t
% \quad\text{for }s>k_t.
% \]
% Thus, the first $k_t$ samples receive probabilities proportional to their residual magnitudes, while the remaining samples receive the lower-bound probability $\beta/M$. It follows that
% \begin{align*}
% \sum_{j\in\cJ_{t-1}}q_t^*(j\mid i_{1:t-1},D_N)
% &=
% \sum_{s=1}^{k_t}\frac{a_{(s)}}{\tau_t}
% +
% (M-k_t)\frac{\beta}{M}
% \\
% &=
% \frac{\sum_{s=1}^{k_t}a_{(s)}}{\tau_t}
% +
% (M-k_t)\frac{\beta}{M}
% \\
% &=
% 1.
% \end{align*}
% Hence this explicit value of $\tau_t$ normalizes the proposal distribution.
|r|_{(s)}>\frac{\beta}{M}\eta_t
\quad\text{for }s\le k_t,
\qquad
|r|_{(s)}\le\frac{\beta}{M}\eta_t
\quad\text{for }s>k_t.
\]
Indeed, if $k_t<M$, the maximality of $k_t$ gives $|r|_{(k_t+1)}\le\beta\eta_{t,k_t+1}/M$, which, by the definition of $\eta_{t,k}$, is equivalent to $|r|_{(k_t+1)}\le\beta\eta_{t,k_t}/M=\beta\eta_t/M$.
Thus, the first $k_t$ samples receive probabilities proportional to their residual magnitudes, while the remaining samples receive the lower-bound probability $\beta/M$. It follows that
\begin{align*}
\sum_{j\in\cJ_{t-1}}q_t^*(j\mid i_{1:t-1},D_N)
&=
\sum_{s=1}^{k_t}\frac{|r|_{(s)}}{\eta_t}
+
(M-k_t)\frac{\beta}{M}
\\
&=
\frac{\sum_{s=1}^{k_t}|r|_{(s)}}{\eta_t}
+
(M-k_t)\frac{\beta}{M}
\\
&=
1.
\end{align*}
Hence this explicit value of $\eta_t$ normalizes the proposal distribution.

The resulting proposal satisfies the KKT conditions with
\[
% [ZJ] notation: \tau_t -> \eta_t (\tau denotes the stopping time); notation: Lagrange multiplier \lambda -> \zeta (\lambda_t^\pm are the betting fractions)
% \lambda=\tau_t^2,
% \qquad
% \mu_j=\tau_t^2-\frac{r_j^2}{(q_j^*)^2}.
% \]
% If $q_j^*>\beta/M$, then $q_j^*=|r_j|/\tau_t$, so $\mu_j=0$. If $q_j^*=\beta/M$, then $|r_j|/\tau_t\le\beta/M$, so $\mu_j\ge0$. Thus, stationarity, primal feasibility, dual feasibility, and complementary slackness all hold, and $q_t^*$ is optimal.
\zeta=\eta_t^2,
\qquad
\mu_j=\eta_t^2-\frac{r_j^2}{(q_j^*)^2}.
\]
If $q_j^*>\beta/M$, then $q_j^*=|r_j|/\eta_t$, so $\mu_j=0$. If $q_j^*=\beta/M$, then $|r_j|/\eta_t\le\beta/M$, so $\mu_j\ge0$. Thus, stationarity, primal feasibility, dual feasibility, and complementary slackness all hold, and $q_t^*$ is optimal.

Since the objective is strictly convex in every coordinate with $r_j\ne0$, and every coordinate with $r_j=0$ must receive the lower-bound probability $\beta/M$, the optimal proposal distribution is unique.

We next consider the estimator $\hat R_t$. Recall that
\[
\hat R_t
=
\frac{1}{N}\sum_{n=1}^{N}\tilde{\ell}_f(\x_n)
+
\frac{1}{t}\sum_{m=1}^{t}u_{m,t}
\left(
\ell(f(\x_{i_m}),y_{i_m})-\tilde{\ell}_f(\x_{i_m})
\right).
\]
Keeping all previous sampling decisions and probabilities fixed, only the $m=t$ term depends on the current proposal. Therefore, we can write
\begin{eqnarray*}
\hat R_t
&=&
\underbrace{
\frac{1}{N}\sum_{n=1}^{N}\tilde{\ell}_f(\x_n)
+
\frac{1}{t}\sum_{m=1}^{t-1}u_{m,t}
\left(
\ell(f(\x_{i_m}),y_{i_m})-\tilde{\ell}_f(\x_{i_m})
\right)
}_{\circled{a}}
\\
&&
+
\frac{1}{t}u_{t,t}
\left(
\ell(f(\x_{i_t}),y_{i_t})-\tilde{\ell}_f(\x_{i_t})
\right).
\end{eqnarray*}
Since $t<N$, the definition of $u_{m,t}$ gives
\[
u_{t,t}
=
1+\frac{N-t}{N-t}
\left(
% \frac{1}{(N-t+1)q_t(i_t\mid i_{1:t-1},D_N)}-1
\frac{1}{(N-t+1)q_t(i_t\mid i_{1:t-1},X_N)}-1
\right)
=
% \frac{1}{(N-t+1)q_t(i_t\mid i_{1:t-1},D_N)}.
\frac{1}{(N-t+1)q_t(i_t\mid i_{1:t-1},X_N)}.
\]
Hence
\begin{equation}\label{eq:optimal_lure_decomp}
\hat R_t
=
\circled{a}
+
\frac{
\ell(f(\x_{i_t}),y_{i_t})-\tilde{\ell}_f(\x_{i_t})
}{
% t(N-t+1)q_t(i_t\mid i_{1:t-1},D_N)
t(N-t+1)q_t(i_t\mid i_{1:t-1},X_N)
}.
\end{equation}

Since $\circled{a}$ is measurable with respect to $\cF_{t-1}$, conditioning on $\cF_{t-1}$ gives
\begin{align*}
\mathrm{Var}(\hat R_t\mid\cF_{t-1})
&=
\sum_{j\in\cJ_{t-1}}
\frac{\left(\ell(f(\x_j),y_j)-\tilde{\ell}_f(\x_j)\right)^2}
% {t^2(N-t+1)^2q_t(j\mid i_{1:t-1},D_N)}
{t^2(N-t+1)^2q_t(j\mid i_{1:t-1},X_N)}
\\
&\qquad
-
\left(
\frac{\sum_{j\in\cJ_{t-1}}
\left(\ell(f(\x_j),y_j)-\tilde{\ell}_f(\x_j)\right)}
{t(N-t+1)}
\right)^2.
\end{align*}
Again, the second term does not depend on the proposal distribution, so minimizing
$\mathrm{Var}(\hat R_t\mid\cF_{t-1})$ is equivalent to minimizing
\[
\sum_{j\in\cJ_{t-1}}
\frac{\left(\ell(f(\x_j),y_j)-\tilde{\ell}_f(\x_j)\right)^2}
% {q_t(j\mid i_{1:t-1},D_N)}
{q_t(j\mid i_{1:t-1},X_N)}
\]
subject to the same probability constraints. This is exactly the optimization problem above. Therefore, the same proposal distribution
\[
q_t^*(j\mid i_{1:t-1},D_N)
=
\max\left\{
\frac{\beta}{N-t+1},
% [ZJ] notation: \tau_t -> \eta_t (\tau denotes the stopping time)
% \frac{\left|\ell(f(\x_j),y_j)-\tilde{\ell}_f(\x_j)\right|}{\tau_t}
\frac{\left|\ell(f(\x_j),y_j)-\tilde{\ell}_f(\x_j)\right|}{\eta_t}
\right\}
\]
also minimizes $\mathrm{Var}(\hat R_t\mid\cF_{t-1})$.

Furthermore, with all sampling policies before round $t$ held fixed, the law of total variance gives
\[
\mathrm{Var}(\hat R_t)
=
\mathrm{Var}\bigl(\bE[\hat R_t\mid\cF_{t-1}]\bigr)
+
\bE\bigl[\mathrm{Var}(\hat R_t\mid\cF_{t-1})\bigr].
\]
From \eqref{eq:optimal_lure_decomp},
\[
\bE[\hat R_t\mid\cF_{t-1}]
=
\circled{a}
+
\frac{1}{t(N-t+1)}
\sum_{j\in\cJ_{t-1}}
\left(
\ell(f(\x_j),y_j)-\tilde{\ell}_f(\x_j)
\right),
\]
which does not depend on the current proposal. Consequently, applying the conditional minimizer at each history also minimizes $\mathrm{Var}(\hat R_t)$ with respect to the round-$t$ proposal, while the earlier policies remain fixed.

\medskip

\textbf{In contrast}, suppose that
\[
\ell(f(\x_j),y_j)-\tilde{\ell}_f(\x_j)=0,
\qquad
\forall j\in\cJ_{t-1}.
\]
For the signal $\hat S_t$, \eqref{eq:optimal_signal_decomp} immediately gives
\[
\hat S_t-R_N=0,
\]
and hence
\[
\hat S_t=R_N.
\]
Therefore,
\[
\mathrm{Var}(\hat S_t\mid\cF_{t-1})=0,
\]
so the conditional variance is minimized by every proposal distribution satisfying Assumption~\ref{ass:beta}.

For the estimator $\hat R_t$, since $i_t\in\cJ_{t-1}$, \eqref{eq:optimal_lure_decomp} yields
\[
\hat R_t
=
\circled{a}
+
\frac{
\ell(f(\x_{i_t}),y_{i_t})-\tilde{\ell}_f(\x_{i_t})
}{
% t(N-t+1)q_t(i_t\mid i_{1:t-1},D_N)
t(N-t+1)q_t(i_t\mid i_{1:t-1},X_N)
}
=
\circled{a},
\]
which is $\cF_{t-1}$-measurable and does not depend on the current proposal. Hence
\[
\mathrm{Var}(\hat R_t\mid\cF_{t-1})=0.
\]
Thus, this conditional variance is minimized by every proposal distribution satisfying Assumption~\ref{ass:beta}.

In summary, when the residuals are not all zero, $q_t^*$ uniquely minimizes both conditional variances among proposal distributions satisfying Assumption~\ref{ass:beta}. When all residuals are zero, both conditional variances vanish, and every feasible proposal distribution is optimal. This completes the proof.
\end{proof}

\appsubsection{Proofs of Theorem~\ref{thm:conf_sequen}}
Before presenting the detailed proof of Theorem~\ref{thm:conf_sequen}, we first recall the notion of a test martingale and a time-uniform concentration inequality for bounded martingale differences.

\begin{definition}\label{def:test_martingale}
Let $(\Omega,\cF,(\cF_t)_{t\ge 0},\bP)$ be a filtered probability space, where $(\cF_t)_{t\ge 0}$ is a filtration satisfying $\cF_t\subseteq\cF$ for all $t$. A nonnegative, $(\cF_t)$-adapted process $(Y_t)_{t\ge 0}$ is called a \emph{test martingale} if $Y_0=1$ and, for all $1\le t\le N$,
\[
\bE_{\bP}\left[Y_t\mid\cF_{t-1}\right]=Y_{t-1}
\quad\text{a.s.}
\]
\end{definition}

\begin{theorem}\cite[Corollary 8.2]{Hsu:Kontorovich:Levin:Peres:Szepesvari:Wolfer2019}\label{thm:freedman}
Let $(Y_t)_{1\le t\le N}$ be a bounded martingale difference sequence with respect to a filtration $\cF_0\subset\cF_1\subset\cF_2\subset\cdots$. Assume that there exists $b>0$ such that $|Y_t|\le b$ almost surely for all $t\in\mathbb N$. For each $t\in\mathbb N$, define
\[
V_t=\sum_{i=1}^t\bE[Y_i^2\mid\cF_{i-1}],
\qquad
W_t=\sum_{i=1}^tY_i.
\]
Then, for any $T\ge1$, $\delta>0$, and $\rho>1$,
\[
\Pr\left(
\exists t\in[T]:
W_t>\sqrt{2\rho V_t\delta}+\frac{4b\delta}{3}
\right)
\le
\left(
1+\left\lceil\log_\rho\left(\frac{2T}{\delta}\right)\right\rceil_{+}
\right)\exp(-\delta).
\]
\end{theorem}

We further introduce the following elementary inequality.

\begin{theorem}\label{thm:elem_ine}
% [ZJ] notation: u -> x (u_{m,t} is the LURE weight)
% For any $r>0$ and any $u\ge-1/(1+r)$,
% \[
% \log(1+u)\ge u-\frac{u^2}{2(1-1/(1+r))}.
For any $r>0$ and any $x\ge-1/(1+r)$,
\[
\log(1+x)\ge x-\frac{x^2}{2(1-1/(1+r))}.
\]
\end{theorem}

\begin{proof}[Proof of Theorem~\ref{thm:conf_sequen}]
Fix $N\in\mathbb Z_+$. The proof consists of two steps.
\begin{itemize}
    \item First, we establish the anytime-valid coverage of the inverted hedged capital process, showing that $\left(\cM_t^\pm\right)_{1\le t\le N}$ is a $(1-\alpha_1)$ confidence sequence for $R_N$.
    \item Second, we derive a time-uniform width bound for $\cM_t^\pm$ and use it to establish shrinkage in the joint asymptotic regime.
\end{itemize}

\textbf{Step 1: Anytime-validity via test martingales and inversion.}

% Let
% \[
% \cF_{t-1}=\sigma(D_N,i_{1:t-1},\tilde{\ell}_f).
% \]
We use $\cF_t$ from \eqref{eq:analysis_information}. First condition on $\cF_0$, so the complete pool and preprocessing information are fixed and only sample-selection randomness remains.
By Theorem~\ref{thm:signal_unbias},
\[
\bE[\hat S_t\mid\cF_{t-1}]=R_N.
\]
Recall that we use the predictable scaling
\[
\bar S_t=\frac{\hat S_t-a_t}{b_t-a_t},
\qquad
\bar\upsilon_t=\frac{\upsilon-a_t}{b_t-a_t},
\qquad
\bar R_{N,t}=\frac{R_N-a_t}{b_t-a_t},
\]
where $a_t$ and $b_t$ are $\cF_{t-1}$-measurable. By Assumption~\ref{ass:loss_bound}, $R_N\in[L,U]$. Moreover, Assumption~\ref{ass:scaling} gives $b_t-a_t>0$, $\hat S_t\in[a_t,b_t]$ almost surely, and $[L,U]\subseteq[a_t,b_t]$. Hence
\[
\bar S_t,\bar\upsilon_t,\bar R_{N,t}\in[0,1].
\]
Since the scaling is affine,
\[
\bE[\bar S_t\mid\cF_{t-1}]=\bar R_{N,t}.
\]
In particular, under the null hypothesis $\cH_0(\upsilon):R_N=\upsilon$,
\[
% \bE[\bar S_t\mid\cF_{t-1},\cH_0(\upsilon)]
\bE[\bar S_t\mid\cF_{t-1}]
=
\bar\upsilon_t.
\]

Recall the one-sided capital processes
\[
\cK_t^+(\upsilon)
=
\prod_{i=1}^t
\left(
1+\lambda_i^+(\upsilon)(\bar S_i-\bar\upsilon_i)
\right),
\]
and
\[
\cK_t^-(\upsilon)
=
\prod_{i=1}^t
\left(
1-\lambda_i^-(\upsilon)(\bar S_i-\bar\upsilon_i)
\right).
\]
By Assumption~\ref{ass:betting},
\[
\lambda_t^+(\upsilon)
=
|\tilde\lambda_t^+|
\wedge
\frac{1}{\bar\upsilon_t+c},
\qquad
\lambda_t^-(\upsilon)
=
|\tilde\lambda_t^-|
\wedge
\frac{1}{1-\bar\upsilon_t+c}.
\]
Since $\bar S_t,\bar\upsilon_t\in[0,1]$,
\[
1+\lambda_t^+(\upsilon)(\bar S_t-\bar\upsilon_t)
\ge
1-\lambda_t^+(\upsilon)\bar\upsilon_t
\ge
\frac{c}{\bar\upsilon_t+c}
>0,
\]
and similarly,
\[
1-\lambda_t^-(\upsilon)(\bar S_t-\bar\upsilon_t)
\ge
1-\lambda_t^-(\upsilon)(1-\bar\upsilon_t)
\ge
\frac{c}{1-\bar\upsilon_t+c}
>0.
\]
Thus both capital processes are nonnegative.

Because $\lambda_t^+(\upsilon)$ is $\cF_{t-1}$-measurable, under $\cH_0(\upsilon)$,
\begin{align*}
\bE\left[\cK_t^+(\upsilon)\mid\cF_{t-1}\right]
&=
\cK_{t-1}^+(\upsilon)
\bE\left[
1+\lambda_t^+(\upsilon)(\bar S_t-\bar\upsilon_t)
% \mid\cF_{t-1},\cH_0(\upsilon)
\mid\cF_{t-1}
\right]\\
&=
\cK_{t-1}^+(\upsilon).
\end{align*}
Therefore, $\{\cK_t^+(\upsilon)\}_{1\le t\le N}$ is a nonnegative martingale under $\cH_0(\upsilon)$. The same argument gives
\[
\bE\left[\cK_t^-(\upsilon)\mid\cF_{t-1}\right]
=
\cK_{t-1}^-(\upsilon),
\]
so $\{\cK_t^-(\upsilon)\}_{1\le t\le N}$ is also a nonnegative martingale under $\cH_0(\upsilon)$.

Hence, for $\theta\in(0,1)$,
\[
\cK_t^\pm(\upsilon)
=
\theta\cK_t^+(\upsilon)
+
(1-\theta)\cK_t^-(\upsilon)
\]
is a nonnegative martingale under $\cH_0(\upsilon)$ with $\cK_0^\pm(\upsilon)=1$.

% By Ville's inequality~\cite{Ville1939},
% \[
% \Pr_{R_N=\upsilon}\left(
% \exists 1\le t\le N:
% \cK_t^\pm(\upsilon)\ge\frac{1}{\alpha_1}
% \right)
% \le\alpha_1.
% \]
For each fixed complete pool and preprocessing information, take $\upsilon=R_N$. By Ville's inequality~\cite{Ville1939},
\[
\Pr\left(
\exists 1\le t\le N:
\cK_t^\pm(R_N)\ge\frac{1}{\alpha_1}
\mid\cF_0
\right)
\le\alpha_1.
\]
Therefore, defining
\[
\cB_t^\pm
=
\left\{
\upsilon\in[L,U]:
\cK_t^\pm(\upsilon)<\frac{1}{\alpha_1}
\right\},
\]
we have
\[
\Pr\left(
\exists 1\le t\le N:
R_N\notin\cB_t^\pm
% \right)
\mid\cF_0
\right)
\le\alpha_1.
\]

Recall that
\[
\cM_t^\pm
=
\bigcap_{i\le t}\cB_i^\pm.
\]
For every $t$,
\[
R_N\notin\cM_t^\pm
\Longleftrightarrow
\exists i\le t:
R_N\notin\cB_i^\pm.
\]
Consequently,
\[
\left\{
\exists 1\le t\le N:
R_N\notin\cM_t^\pm
\right\}
=
\left\{
\exists 1\le t\le N:
R_N\notin\cB_t^\pm
\right\}.
\]
It follows that
\[
\Pr\left(
\exists 1\le t\le N:
R_N\notin\cM_t^\pm
% \right)
\mid\cF_0
\right)
\le\alpha_1.
\]
% Thus, $\left(\cM_t^\pm\right)_{1\le t\le N}$ is a $(1-\alpha_1)$ anytime-valid confidence sequence for $R_N$.
Taking expectations over $\cF_0$ gives the same bound unconditionally. Thus, $\left(\cM_t^\pm\right)_{1\le t\le N}$ is a $(1-\alpha_1)$ anytime-valid confidence sequence for $R_N$.

\medskip

\textbf{Step 2: Shrinkage and width bound via a time-uniform martingale concentration argument.}

Fix an auxiliary failure probability $\omega\in(0,1)$ and any $\rho>1$. Under Assumption~\ref{ass:betting},
\[
\frac{\underline\gamma}{\sqrt{t\log(t+1)}}
\le
|\tilde\lambda_t^\pm|
\le
\frac{\overline\gamma}{\sqrt{t\log(t+1)}}.
\]
Define
\begin{equation}\label{eq:t_0_define}
t_c
=
\min\left\{
t\in\mathbb Z_+:
t\log(t+1)\ge((1+c)\overline\gamma)^2
\right\},
\qquad
t_0=\max\{2t_c,8\}.
\end{equation}
For every $i\ge t_c$,
\[
|\tilde\lambda_i^\pm|
\le
\frac{1}{1+c}.
\]
Since $\bar\upsilon_i\in[0,1]$,
\[
\frac{1}{\bar\upsilon_i+c}\ge\frac{1}{1+c},
\qquad
\frac{1}{1-\bar\upsilon_i+c}\ge\frac{1}{1+c}.
\]
Therefore, simultaneously for all $\upsilon\in[L,U]$ and $i\ge t_c$,
\begin{equation}\label{eq:uniform_unclipping}
\lambda_i^+(\upsilon)=|\tilde\lambda_i^+|,
\qquad
\lambda_i^-(\upsilon)=|\tilde\lambda_i^-|.
\end{equation}

Let
\[
Z_i=\bar S_i-\bar R_{N,i}.
\]
Then
\[
\bE[Z_i\mid\cF_{i-1}]=0,
\qquad
|Z_i|\le1.
\]
Define
\[
Y_i^+
=
\begin{cases}
0, & i<t_c,\\
|\tilde\lambda_i^+|Z_i, & i\ge t_c,
\end{cases}
\qquad
Y_i^-
=
\begin{cases}
0, & i<t_c,\\
|\tilde\lambda_i^-|Z_i, & i\ge t_c,
\end{cases}
\]
and
\[
W_t^+=\sum_{i=1}^tY_i^+,
\qquad
W_t^-=\sum_{i=1}^tY_i^-.
\]
Because $|\tilde\lambda_i^\pm|$ is predictable and does not depend on $\upsilon$, $(Y_i^+)_{1\le i\le N}$ and $(Y_i^-)_{1\le i\le N}$ are candidate-independent martingale difference sequences. Moreover,
\[
|Y_i^\pm|\le\frac{1}{c}
\quad\text{almost surely}.
\]

Their predictable quadratic variations satisfy
\begin{align}
V_t^\pm
&=
\sum_{i=1}^t
\bE\left[(Y_i^\pm)^2\mid\cF_{i-1}\right]\nonumber\\
&\le
\sum_{i=1}^t
\frac{\overline\gamma^2}{i\log(i+1)}
\mathrm{Var}\left(\bar S_i\mid\cF_{i-1}\right)\nonumber\\
&=
\sum_{i=1}^t
\frac{\overline\gamma^2}{i\log(i+1)(b_i-a_i)^2}
\mathrm{Var}\left(\hat S_i\mid\cF_{i-1}\right)\nonumber\\
&\le
\frac{\overline\gamma^2}{\underline\Delta^2}
\sum_{i=1}^t
\frac{\mathrm{Var}\left(\hat S_i\mid\cF_{i-1}\right)}
{i\log(i+1)}.
\label{eq:scale_Vt_uniform}
\end{align}

We now apply Theorem~\ref{thm:freedman} using a stitching argument. For
$j=0,1,\ldots,\lceil\log_2N\rceil$, define
\[
% [ZJ] notation: H_j -> 2^j (H denotes the entropy score in Remark~\ref{rem:cross_entropy})
% H_j=2^j,
% \qquad
% T_j=\min\{H_j,N\},
T_j=\min\{2^j,N\},
\qquad
\omega_j=\frac{3\omega}{\pi^2(j+1)^2},
\]
and
\[
\delta_j
=
\max\left\{
1,
\log\left(
\frac{
% [ZJ] notation: H_j -> 2^j (H denotes the entropy score in Remark~\ref{rem:cross_entropy})
% \left(1+\lceil\log_\rho(2H_j)\rceil_{+}\right)
\left(1+\lceil\log_\rho(2^{j+1})\rceil_{+}\right)
\pi^2(j+1)^2
}{3\omega}
\right)
\right\}.
\]
Define
\[
\cE_j^+
=
\left\{
\forall t\le T_j:
W_t^+
\ge
-\sqrt{2\rho V_t^+\delta_j}
-\frac{4\delta_j}{3c}
\right\},
\]
and
\[
\cE_j^-
=
\left\{
\forall t\le T_j:
W_t^-
\le
\sqrt{2\rho V_t^-\delta_j}
+\frac{4\delta_j}{3c}
\right\}.
\]
Applying Theorem~\ref{thm:freedman} to $-W_t^+$ and $W_t^-$ gives
\[
\Pr\left((\cE_j^+)^c\right)\le\omega_j,
\qquad
\Pr\left((\cE_j^-)^c\right)\le\omega_j.
\]
% [ZJ] notation: H_j -> 2^j (H denotes the entropy score in Remark~\ref{rem:cross_entropy})
% Indeed, $\delta_j\ge1$ and $T_j\le H_j$ imply
% \[
% \left\lceil
% \log_\rho\left(\frac{2T_j}{\delta_j}\right)
% \right\rceil_{+}
% \le
% \left\lceil\log_\rho(2H_j)\right\rceil_{+},
Indeed, $\delta_j\ge1$ and $T_j\le 2^j$ imply
\[
\left\lceil
\log_\rho\left(\frac{2T_j}{\delta_j}\right)
\right\rceil_{+}
\le
\left\lceil\log_\rho(2^{j+1})\right\rceil_{+},
\]
so the conclusion follows from the definition of $\delta_j$.

Let
\[
\cE^+
=
\bigcap_{j=0}^{\lceil\log_2N\rceil}\cE_j^+,
\qquad
\cE^-
=
\bigcap_{j=0}^{\lceil\log_2N\rceil}\cE_j^-.
\]
Since
\[
\sum_{j=0}^{\infty}\omega_j
=
\frac{\omega}{2},
\]
a union bound gives
\[
\Pr(\cE^+)\ge1-\frac{\omega}{2},
\qquad
\Pr(\cE^-)\ge1-\frac{\omega}{2},
\]
and hence
\[
\Pr(\cE^+\cap\cE^-)\ge1-\omega.
\]
On $\cE^+$,
\[
W_t^+
\ge
-\sqrt{2\rho V_t^+\delta_{\lceil\log_2t\rceil}}
-\frac{4\delta_{\lceil\log_2t\rceil}}{3c},
\]
and on $\cE^-$,
\[
W_t^-
\le
\sqrt{2\rho V_t^-\delta_{\lceil\log_2t\rceil}}
+\frac{4\delta_{\lceil\log_2t\rceil}}{3c},
\]
simultaneously for all $1\le t\le N$. Since these events are constructed from candidate-independent martingales, they hold simultaneously for every $\upsilon\in[L,U]$.

We first consider $\upsilon<R_N$. For every $i$,
\[
\bar S_i-\bar\upsilon_i
=
Z_i+\frac{R_N-\upsilon}{b_i-a_i},
\qquad
\left(\bar S_i-\bar\upsilon_i\right)^2\le1.
\]
Moreover,
\[
\lambda_i^+(\upsilon)(\bar S_i-\bar\upsilon_i)
\ge
-\frac{1}{1+c},
\]
so Theorem~\ref{thm:elem_ine} gives
\begin{align}
\log\cK_t^+(\upsilon)
&\ge
(R_N-\upsilon)
\sum_{i=1}^t
\frac{\lambda_i^+(\upsilon)}{b_i-a_i}
+
\sum_{i=1}^t
\lambda_i^+(\upsilon)Z_i\nonumber\\
&\qquad
-
\frac{1}{2(1-1/(1+c))}
\sum_{i=1}^t
\left(\lambda_i^+(\upsilon)\right)^2.
\label{eq:logK_plus_uniform}
\end{align}

Let
\[
% [ZJ] notation: C_0, C_1, C_c -> G_0, G_1, G_c (C_t denotes the CI)
% C_c=\frac{t_c-1}{c}.
% \]
% For $i<t_c$, $\lambda_i^+(\upsilon)\le1/c$ and $Z_i\ge-1$. Together with \eqref{eq:uniform_unclipping}, this gives
% \[
% \sum_{i=1}^t
% \lambda_i^+(\upsilon)Z_i
% \ge
% -C_c+W_t^+.
G_c=\frac{t_c-1}{c}.
\]
For $i<t_c$, $\lambda_i^+(\upsilon)\le1/c$ and $Z_i\ge-1$. Together with \eqref{eq:uniform_unclipping}, this gives
\[
\sum_{i=1}^t
\lambda_i^+(\upsilon)Z_i
\ge
-G_c+W_t^+.
\]
Moreover, for every $2t_c\le t\le N$,
\begin{align}
\sum_{i=1}^t
\frac{\lambda_i^+(\upsilon)}{b_i-a_i}
&\ge
\sum_{i=\lfloor t/2\rfloor}^{t}
\frac{|\tilde\lambda_i^+|}{b_i-a_i}\nonumber\\
&\ge
\frac{\underline\gamma\sqrt t}
{2\overline\Delta\sqrt{\log(t+1)}}.
\label{eq:At_bound_new}
\end{align}
Also,
\begin{align}
\sum_{i=1}^t
\left(\lambda_i^+(\upsilon)\right)^2
&\le
\overline\gamma^2
\sum_{i=1}^t
\frac{1}{i\log(i+1)}\nonumber\\
&\le
\overline\gamma^2
% [ZJ] notation: C_0, C_1, C_c -> G_0, G_1, G_c (C_t denotes the CI)
% \left(C_0+\log\log(t+1)\right)
% \label{eq:Bt_bound_new}
% \end{align}
% for some absolute constant $C_0>0$. Here we use
\left(G_0+\log\log(t+1)\right)
\label{eq:Bt_bound_new}
\end{align}
for some absolute constant $G_0>0$. Here we use
\[
\sum_{i=2}^t\frac{1}{i\log i}
\le
\frac{1}{2\log2}
+
\int_2^t\frac{dx}{x\log x}.
\]

Combining \eqref{eq:scale_Vt_uniform}--\eqref{eq:Bt_bound_new}, we obtain on $\cE^+$, simultaneously for all $\upsilon<R_N$ and $2t_c\le t\le N$,
\begin{align}
\log\cK_t^+(\upsilon)
&\ge
(R_N-\upsilon)
\frac{\underline\gamma\sqrt t}
{2\overline\Delta\sqrt{\log(t+1)}}
% [ZJ] notation: C_0, C_1, C_c -> G_0, G_1, G_c (C_t denotes the CI)
% -C_c\nonumber\\
% &\qquad
% -
% \frac{
% \overline\gamma^2\left(C_0+\log\log(t+1)\right)
-G_c\nonumber\\
&\qquad
-
\frac{
\overline\gamma^2\left(G_0+\log\log(t+1)\right)
}{2(1-1/(1+c))}\nonumber\\
&\qquad
-
\sqrt{
\frac{2\rho\overline\gamma^2}{\underline\Delta^2}
\sum_{i=1}^t
\frac{\mathrm{Var}\left(\hat S_i\mid\cF_{i-1}\right)}
{i\log(i+1)}
\delta_{\lceil\log_2t\rceil}
}
-
\frac{4\delta_{\lceil\log_2t\rceil}}{3c}.
\label{eq:logK_lower_simple}
\end{align}

The stitching construction gives
\[
\delta_{\lceil\log_2t\rceil}
=
O\left(\log(1/\omega)+\log\log t\right).
\]
To make the resulting bound explicit, let
\[
% [ZJ] notation: \ell_t -> \log\log(t+1) (\ell denotes the loss); notation: C_0, C_1, C_c -> G_0, G_1, G_c (C_t denotes the CI); notation: C_2 merged into G_1 (C_t denotes the CI)
% \ell_t=\log\log(t+1),
% \qquad
% Q_t=
% \sum_{i=1}^t
% \frac{\mathrm{Var}\left(\hat S_i\mid\cF_{i-1}\right)}
% {i\log(i+1)}.
% \]
% For the fixed $\omega$ and all $t\ge t_0$, there exists a constant $C_1>0$ such that
% \[
% \delta_{\lceil\log_2t\rceil}\le C_1\ell_t.
% \]
% Since $t_0\ge8$, $\ell_t$ is bounded away from zero, and hence there exists a constant $C_2>0$ such that
% \[
% 1+\ell_t+\sqrt{Q_t\ell_t}
% \le
% C_2\sqrt{\ell_t(\ell_t+Q_t)}.
% \]
% Thus, the fixed terms, the quadratic term, and the Freedman term in \eqref{eq:logK_lower_simple}, together with the target threshold $\log(1/(\theta\alpha_1))$, are dominated by a constant multiple of $\sqrt{\ell_t(\ell_t+Q_t)}$. Therefore, there exists a constant $G_2>0$ such that every $\upsilon<R_N$ satisfying
Q_t=
\sum_{i=1}^t
\frac{\mathrm{Var}\left(\hat S_i\mid\cF_{i-1}\right)}
{i\log(i+1)}.
\]
For the fixed $\omega$ and all $t\ge t_0$, there exists a constant $G_1>0$ such that
\[
\delta_{\lceil\log_2t\rceil}\le G_1\log\log(t+1).
\]
Since $t_0\ge8$, $\log\log(t+1)$ is bounded away from zero, and hence, after enlarging $G_1$ if necessary,
\[
1+\log\log(t+1)+\sqrt{Q_t\log\log(t+1)}
\le
G_1\sqrt{\log\log(t+1)\left(\log\log(t+1)+Q_t\right)}.
\]
Thus, the fixed terms, the quadratic term, and the Freedman term in \eqref{eq:logK_lower_simple}, together with the target threshold $\log(1/(\theta\alpha_1))$, are dominated by a constant multiple of $\sqrt{\log\log(t+1)(\log\log(t+1)+Q_t)}$. Therefore, there exists a constant $G_2>0$ such that every $\upsilon<R_N$ satisfying
\[
R_N-\upsilon
\ge
G_2
\sqrt{
\frac{\log(t+1)\log\log(t+1)}{t}
\left(
\log\log(t+1)
+
\sum_{i=1}^t
\frac{\mathrm{Var}\left(\hat S_i\mid\cF_{i-1}\right)}
{i\log(i+1)}
\right)
}
\]
also satisfies
\[
\cK_t^+(\upsilon)\ge\frac{1}{\theta\alpha_1}.
\]
Since
\[
\cK_t^\pm(\upsilon)
\ge
\theta\cK_t^+(\upsilon),
\]
every such candidate is excluded from $\cB_t^\pm$. Hence, on $\cE^+$,
\[
\cB_t^\pm
\subseteq
\left[
R_N-
G_2
\sqrt{
\frac{\log(t+1)\log\log(t+1)}{t}
\left(
\log\log(t+1)
+
\sum_{i=1}^t
\frac{\mathrm{Var}\left(\hat S_i\mid\cF_{i-1}\right)}
{i\log(i+1)}
\right)
},
U
\right]
\]
simultaneously for all $t_0\le t\le N$.

We next consider $\upsilon>R_N$. Using the same decomposition and applying the elementary inequality to $-\lambda_i^-(\upsilon)(\bar S_i-\bar\upsilon_i)$ gives
\begin{align}
\log\cK_t^-(\upsilon)
&\ge
(\upsilon-R_N)
\sum_{i=1}^t
\frac{\lambda_i^-(\upsilon)}{b_i-a_i}
-
\sum_{i=1}^t
\lambda_i^-(\upsilon)Z_i\nonumber\\
&\qquad
-
\frac{1}{2(1-1/(1+c))}
\sum_{i=1}^t
\left(\lambda_i^-(\upsilon)\right)^2.
\label{eq:logK_minus_uniform}
\end{align}
For $i<t_c$, $\lambda_i^-(\upsilon)\le1/c$ and $-Z_i\ge-1$. Hence, as above,
\[
-\sum_{i=1}^t
\lambda_i^-(\upsilon)Z_i
\ge
% [ZJ] notation: C_0, C_1, C_c -> G_0, G_1, G_c (C_t denotes the CI)
% -C_c-W_t^-,
-G_c-W_t^-,
\]
and, for every $2t_c\le t\le N$,
\[
\sum_{i=1}^t
\frac{\lambda_i^-(\upsilon)}{b_i-a_i}
\ge
\frac{\underline\gamma\sqrt t}
{2\overline\Delta\sqrt{\log(t+1)}},
\]
while
\[
\sum_{i=1}^t
\left(\lambda_i^-(\upsilon)\right)^2
\le
\overline\gamma^2
% [ZJ] notation: C_0, C_1, C_c -> G_0, G_1, G_c (C_t denotes the CI)
% \left(C_0+\log\log(t+1)\right).
\left(G_0+\log\log(t+1)\right).
\]
It follows that, on $\cE^-$, simultaneously for all $\upsilon>R_N$ and $t_0\le t\le N$,
\begin{align}
\log\cK_t^-(\upsilon)
&\ge
(\upsilon-R_N)
\frac{\underline\gamma\sqrt t}
{2\overline\Delta\sqrt{\log(t+1)}}
% [ZJ] notation: C_0, C_1, C_c -> G_0, G_1, G_c (C_t denotes the CI)
% -C_c\nonumber\\
% &\qquad
% -
% \frac{
% \overline\gamma^2\left(C_0+\log\log(t+1)\right)
-G_c\nonumber\\
&\qquad
-
\frac{
\overline\gamma^2\left(G_0+\log\log(t+1)\right)
}{2(1-1/(1+c))}\nonumber\\
&\qquad
-
\sqrt{
\frac{2\rho\overline\gamma^2}{\underline\Delta^2}
\sum_{i=1}^t
\frac{\mathrm{Var}\left(\hat S_i\mid\cF_{i-1}\right)}
{i\log(i+1)}
\delta_{\lceil\log_2t\rceil}
}
-
\frac{4\delta_{\lceil\log_2t\rceil}}{3c}.
\label{eq:logK_upper_simple}
\end{align}
Therefore, there exists a constant $G_4>0$ such that every $\upsilon>R_N$ satisfying
\[
\upsilon-R_N
\ge
G_4
\sqrt{
\frac{\log(t+1)\log\log(t+1)}{t}
\left(
\log\log(t+1)
+
\sum_{i=1}^t
\frac{\mathrm{Var}\left(\hat S_i\mid\cF_{i-1}\right)}
{i\log(i+1)}
\right)
}
\]
also satisfies
\[
\cK_t^-(\upsilon)
\ge
\frac{1}{(1-\theta)\alpha_1}.
\]
Since
\[
\cK_t^\pm(\upsilon)
\ge
(1-\theta)\cK_t^-(\upsilon),
\]
every such candidate is excluded from $\cB_t^\pm$. Hence, on $\cE^-$,
\[
\cB_t^\pm
\subseteq
\left[
L,R_N+
G_4
\sqrt{
\frac{\log(t+1)\log\log(t+1)}{t}
\left(
\log\log(t+1)
+
\sum_{i=1}^t
\frac{\mathrm{Var}\left(\hat S_i\mid\cF_{i-1}\right)}
{i\log(i+1)}
\right)
}
\right]
\]
simultaneously for all $t_0\le t\le N$.

On $\cE^+\cap\cE^-$, which has probability at least $1-\omega$, the preceding two bounds hold simultaneously. Assigning width zero to the empty set, there exists a constant $G_5>0$ such that
\[
\mathrm{width}\left(\cB_t^\pm\right)
\le
G_5
\sqrt{
\frac{\log(t+1)\log\log(t+1)}{t}
\left(
\log\log(t+1)
+
\sum_{i=1}^t
\frac{\mathrm{Var}\left(\hat S_i\mid\cF_{i-1}\right)}
{i\log(i+1)}
\right)
}
\]
simultaneously for all $t_0\le t\le N$. Since
\[
\cM_t^\pm
=
\bigcap_{i\le t}\cB_i^\pm
\subseteq
\cB_t^\pm,
\]
the same width bound holds for $\cM_t^\pm$.

For every $t\ge t_0$,
\[
\log(t+1)\le2\log t,
\qquad
\log\log(t+1)\le2\log\log t.
\]
Consequently, after enlarging the constant if necessary, with probability at least $1-\omega$,
\[
\mathrm{width}\left(\cM_t^\pm\right)
\le
G_6
\sqrt{
\frac{\log t\log\log t}{t}
\left(
\log\log t+
\sum_{i=1}^t
\frac{\mathrm{Var}\left(\hat S_i\mid\cF_{i-1}\right)}
{i\log(i+1)}
\right)
},
\]
simultaneously for all $t_0\le t\le N$.

The constant $G_6$ can be chosen independently of $N$ and $t$.

Finally, Theorem~\ref{thm:signal_unbias} and Assumptions~\ref{ass:loss_bound} and~\ref{ass:beta} imply
\begin{align*}
\mathrm{Var}\left(\hat S_i\mid\cF_{i-1}\right)
&\le
\frac{1}{N^2}
\sum_{j\in\cJ_{i-1}}
\frac{(U-L)^2}
% {q_i(j\mid i_{1:i-1},D_N)}\\
{q_i(j\mid i_{1:i-1},X_N)}\\
&\le
\frac{(U-L)^2(N-i+1)^2}{\beta N^2}\\
&\le
\frac{(U-L)^2}{\beta}.
\end{align*}
Therefore,
\begin{equation}\label{eq:cond_var_O}
\sum_{i=1}^t
\frac{\mathrm{Var}\left(\hat S_i\mid\cF_{i-1}\right)}
{i\log(i+1)}
=
O(\log\log t),
\end{equation}
where the implied constant is independent of $N$.

Writing $\cM_{t,N}^\pm$ to make the dependence on $N$ explicit, consider any integer sequence $\{t_N\}_{N\ge1}$ satisfying $1\le t_N\le N$ and $t_N\to\infty$ as $N\to\infty$. Then $t_N\ge t_0$ for all sufficiently large $N$. Since $\omega\in(0,1)$ was arbitrary, the preceding time-uniform bound and \eqref{eq:cond_var_O} imply that, for some constant $G_7>0$ independent of $N$,
\[
\Pr\left(
\mathrm{width}\left(\cM_{t_N,N}^\pm\right)
>
G_7\log\log t_N
\sqrt{\frac{\log t_N}{t_N}}
\right)
\le
\omega
\]
for all sufficiently large $N$. Consequently,
\[
\mathrm{width}\left(\cM_{t_N,N}^\pm\right)
=
O_{\mathbb P}\left(
\log\log t_N
\sqrt{\frac{\log t_N}{t_N}}
\right).
\]
Since
\[
\log\log t_N
\sqrt{\frac{\log t_N}{t_N}}
\to0,
\]
we conclude that
\[
\mathrm{width}\left(\cM_{t_N,N}^\pm\right)
\xrightarrow{\mathbb P}0
\qquad\text{as }N\to\infty.
\]
This completes the proof.
\end{proof}

\appsubsection{Proofs of Theorem~\ref{thm:coverage_convergence}}
We present the detailed proof of Theorem~\ref{thm:coverage_convergence} as follows.
\begin{proof}
Fix $N\in\mathbb Z_+$. We decompose the proof into two high-probability events: one for the sequential uncertainty in estimating $R_N$ from adaptive labels, and one for the finite-pool approximation error $|R_N-R|$.

Let
\[
\cE_1=\left\{\forall 1\le t\le N,\ R_N\in\cM_t^{\pm}\right\},
\qquad
\cE_2=\left\{|R_N-R|\le\Delta_N(\alpha_2)\right\}.
\]
By Theorem~\ref{thm:conf_sequen}, $\left(\cM_t^{\pm}\right)_{1\le t\le N}$ is a $(1-\alpha_1)$ anytime-valid confidence sequence for $R_N$, hence
\[
\Pr\left(\cE_1\right)
=
\Pr\left(\forall 1\le t\le N,\ R_N\in\cM_t^{\pm}\right)
\ge 1-\alpha_1.
\]

% Next, by Assumption~\ref{ass:loss_bound}, we have $\ell\left(f(\x),y\right)\in[L,U]$. Since the pool examples are i.i.d., Hoeffding's inequality gives, for any $r>0$,
Next, by Assumption~\ref{ass:loss_bound}, we have $\ell\left(f(\x),y\right)\in[L,U]$. Since the input-label pairs are i.i.d. and the target model is fixed, the pool losses are i.i.d. Hence, Hoeffding's inequality gives, for any $r>0$,
\[
\Pr\left(|R_N-R|>r\right)
\le
2\exp\left(-\frac{2Nr^2}{(U-L)^2}\right).
\]
Choosing
\[
\Delta_N(\alpha_2)
=
(U-L)\sqrt{\frac{\log\left(2/\alpha_2\right)}{2N}}
\]
yields
\[
\Pr\left(\cE_2\right)\ge 1-\alpha_2.
\]

Assume $\cE_1\cap\cE_2$ holds. Fix any $1\le t\le N$. From $\cE_1$, we have
\[
R_N\in\cM_t^{\pm}.
\]
From $\cE_2$, we have
\[
R\in
\left[
R_N-\Delta_N(\alpha_2),
R_N+\Delta_N(\alpha_2)
\right].
\]
Combining these two statements implies
\[
R
=
R_N+\left(R-R_N\right)
\in
\cM_t^{\pm}\oplus
\left[
-\Delta_N(\alpha_2),
\Delta_N(\alpha_2)
\right].
\]
Assumption~\ref{ass:loss_bound} also implies $R\in[L,U]$. Therefore, by the definition of $C_t$, it follows that $R\in C_t$. Since $1\le t\le N$ was arbitrary, we conclude
\[
\cE_1\cap\cE_2
\subseteq
\left\{
\forall 1\le t\le N,\ R\in C_t
\right\}.
\]

Taking complements, we obtain
\[
\left\{
\exists 1\le t\le N:\ R\notin C_t
\right\}
\subseteq
\cE_1^c\cup\cE_2^c.
\]
Therefore, by a union bound,
\[
\Pr\left(
\exists 1\le t\le N:\ R\notin C_t
\right)
\le
\Pr\left(\cE_1^c\right)
+
\Pr\left(\cE_2^c\right)
\le
\alpha_1+\alpha_2
=
\alpha.
\]
Thus, for every $N\in\mathbb Z_+$, $\left(C_t\right)_{1\le t\le N}$ is a $(1-\alpha)$ anytime-valid confidence sequence for $R$.

For any nonempty bounded set $A\subset\mathbb R$ and any $r\ge0$, the Minkowski sum satisfies
\[
\sup\left(A\oplus[-r,r]\right)=\sup A+r,
\qquad
\inf\left(A\oplus[-r,r]\right)=\inf A-r.
\]
Hence
\[
\mathrm{width}\left(A\oplus[-r,r]\right)
=
\left(\sup A+r\right)-\left(\inf A-r\right)
=
\mathrm{width}\left(A\right)+2r.
\]
% [ZJ] proof: the case \cM_t^\pm=\varnothing is treated separately (the width identity requires a nonempty set)
% Applying this with $A=\cM_t^{\pm}$ and $r=\Delta_N(\alpha_2)$ gives
If $\cM_t^{\pm}=\varnothing$, then $C_t=\varnothing$, and the width bound below holds trivially, assigning width zero to the empty set. Otherwise, applying this with $A=\cM_t^{\pm}$ and $r=\Delta_N(\alpha_2)$ gives
\[
\mathrm{width}\left(
\cM_t^{\pm}\oplus
[-\Delta_N(\alpha_2),\Delta_N(\alpha_2)]
\right)
=
\mathrm{width}\left(\cM_t^{\pm}\right)
+
2\Delta_N(\alpha_2).
\]
Since
\[
C_t
=
\left(
\cM_t^{\pm}\oplus
[-\Delta_N(\alpha_2),\Delta_N(\alpha_2)]
\right)\cap[L,U],
\]
and intersecting with $[L,U]$ can only reduce the width, for every $1\le t\le N$ we have
\begin{align*}
\mathrm{width}\left(C_t\right)
&\le
\mathrm{width}\left(\cM_t^{\pm}\right)
+
2\Delta_N(\alpha_2)
\\
&=
\mathrm{width}\left(\cM_t^{\pm}\right)
+
2(U-L)\sqrt{\frac{\log(2/\alpha_2)}{2N}}.
\end{align*}

Writing $C_{t,N}$ and $\cM_{t,N}^{\pm}$ to make their dependence on $N$ explicit, let $\{t_N\}_{N\ge1}$ be any integer sequence satisfying
\[
1\le t_N\le N,
\qquad
t_N\to\infty
\quad\text{as }N\to\infty.
\]
The preceding width inequality gives
\[
\mathrm{width}\left(C_{t_N,N}\right)
\le
\mathrm{width}\left(\cM_{t_N,N}^{\pm}\right)
+
2\Delta_N(\alpha_2).
\]
By Theorem~\ref{thm:conf_sequen},
\[
\mathrm{width}\left(\cM_{t_N,N}^{\pm}\right)
=
O_{\mathbb P}\left(
\log\log t_N
\sqrt{\frac{\log t_N}{t_N}}
\right).
\]
Moreover, since $\alpha_2$ is fixed,
\[
2\Delta_N(\alpha_2)
=
2(U-L)\sqrt{\frac{\log(2/\alpha_2)}{2N}}
=
O\left(\sqrt{\frac{1}{N}}\right).
\]
Consequently,
\[
\mathrm{width}\left(C_{t_N,N}\right)
=
O_{\mathbb P}\left(
\log\log t_N
\sqrt{\frac{\log t_N}{t_N}}
\right)
+
O\left(\sqrt{\frac{1}{N}}\right).
\]
Since $t_N\to\infty$ and $N\to\infty$, both terms converge to zero. Therefore,
\[
\mathrm{width}\left(C_{t_N,N}\right)
\xrightarrow{\mathbb P}0
\qquad\text{as }N\to\infty.
\]
Since the sequence $\{t_N\}_{N\ge1}$ was arbitrary, the confidence intervals are shrinking in the joint asymptotic regime $N\to\infty$, $t=t_N\to\infty$, and $t_N\le N$. This completes the proof.
\end{proof}

\appsection{Synthetic Experimental Setting for Table~\ref{tab:rq1-coverage}}
\label{app:synthetic}

\cereval~\cite{Wang:Chen:Bo:Xu2025} adaptively selects evaluation samples and constructs CIs by partitioning the evaluation pool and applying adaptive concentration inequalities within each partition~\cite{Zhao:Zhou:Sabharwal:Ermon2016}. This approach can be effective when the partition is fixed and aligned with the true risk structure, since a low-variance partition can yield tighter CIs with fewer evaluated samples. However, its certification analysis relies on concentration bounds for a fixed or well-behaved partition, whereas the algorithm repeatedly learns and updates the partition from the evaluated samples, their embeddings, and their observed risk values. Consequently, the stated guarantee does not directly cover the complete adaptive partition-learning procedure; this distinction concerns the scope of the current theoretical analysis rather than the validity of adaptive partitioning in general. In this procedure, the same evaluated data are used both to choose the partition and to quantify the within-partition variation. When the embedding used for partitioning is uninformative about the target risk, the adaptive partition can overfit the evaluated samples and underestimate the variation among unevaluated samples, leading to overly narrow CIs. 

% [BjfK-Q3]
% We present a synthetic setting to further illustrate the gap between the theoretical guarantees and the proposed algorithm of Cer-Eval~\cite{Wang:Chen:Bo:Xu2025}. Specifically, we construct a setting where the evaluation pool contains genuine risk heterogeneity, but this heterogeneity is independent of the embeddings used for partitioning. Therefore, any embedding-based partition cannot reliably transfer the risk structure observed on evaluated samples to the remaining unevaluated samples.
% [ZJ]
% We present a synthetic setting to further illustrate this distinction between the fixed-partition guarantee and the adaptive procedure of Cer-Eval~\cite{Wang:Chen:Bo:Xu2025}. Specifically, we construct a setting where the evaluation pool contains genuine risk heterogeneity, but this heterogeneity is independent of the embeddings used for partitioning. Therefore, any embedding-based partition cannot reliably transfer the risk structure observed on evaluated samples to the remaining unevaluated samples.
We present a synthetic setting to further illustrate this distinction between the fixed-partition guarantee and the adaptive procedure of \cereval~\cite{Wang:Chen:Bo:Xu2025}. Specifically, we construct a setting where the evaluation pool contains genuine risk heterogeneity, but this heterogeneity is independent of the embeddings used for partitioning. Therefore, any embedding-based partition cannot reliably transfer the risk structure observed on evaluated samples to the remaining unevaluated samples.

\textbf{Data-generating process.}
% We generate a finite evaluation pool $D_N=\{\x_j\}_{j=1}^N$. For each sample $\x_j$, we first draw a latent group $G_j\in\{1,\ldots,G\}$. The latent group determines the risk level of the sample, but it is not encoded in the embedding used for partitioning. Specifically, each group $g$ is assigned a risk probability $p_g$, and we choose these $G$ probabilities to be evenly spaced between $0.1$ and $0.9$. Given $G_j=g$, the realized per-sample risk value is generated as
% [ZJ] notation: G_j, G, g, Z_j -> \Gamma_j, \Gamma, k, \ell_j (G_6 is the constant in Theorem~\ref{thm:conf_sequen}, g is the surrogate model, Z_i is used in the proof of Theorem~\ref{thm:conf_sequen})
% We generate a finite input pool $X_N=\{\x_j\}_{j=1}^N$. For each sample $\x_j$, we first draw a latent group $G_j\in\{1,\ldots,G\}$. The latent group determines the risk level of the sample, but it is not encoded in the embedding used for partitioning. Specifically, each group $g$ is assigned a risk probability $p_g$, and we choose these $G$ probabilities to be evenly spaced between $0.1$ and $0.9$. Given $G_j=g$, the realized per-sample risk value is generated as
% \[
% Z_j \mid G_j=g \sim \mathrm{Bernoulli}(p_g),
% \]
% so samples from different latent groups have different risk distributions. The finite-pool risk is
% \[
% R_N=\frac{1}{N}\sum_{j=1}^N Z_j .
We generate a finite input pool $X_N=\{\x_j\}_{j=1}^N$. For each sample $\x_j$, we first draw a latent group $\Gamma_j\in\{1,\ldots,\Gamma\}$. The latent group determines the risk level of the sample, but it is not encoded in the embedding used for partitioning. Specifically, each group $k$ is assigned a risk probability $p_k$, and we choose these $\Gamma$ probabilities to be evenly spaced between $0.1$ and $0.9$. Given $\Gamma_j=k$, the realized per-sample risk value is generated as
\[
\ell_j \mid \Gamma_j=k \sim \mathrm{Bernoulli}(p_k),
\]
so samples from different latent groups have different risk distributions. The finite-pool risk is
\[
R_N=\frac{1}{N}\sum_{j=1}^N \ell_j .
\]

The crucial design choice is that the representation used for partitioning is independent of the latent risk structure:
\[
% [ZJ] notation: G_j, G, g, Z_j -> \Gamma_j, \Gamma, k, \ell_j (G_6 is the constant in Theorem~\ref{thm:conf_sequen}, g is the surrogate model, Z_i is used in the proof of Theorem~\ref{thm:conf_sequen})
% \x_j \perp (G_j,Z_j).
\x_j \perp (\Gamma_j,\ell_j).
\]
Thus, the pool contains genuine risk heterogeneity through the latent groups, but this heterogeneity is invisible from the embeddings. In other words, two samples may have very different risk probabilities because they belong to different latent groups, but their embeddings provide no information about this difference. Therefore, an embedding-based partition cannot reliably infer the risk behavior of unevaluated samples from the evaluated samples.

This setting is not adversarial in the sense of manipulating the evaluation order or using unbounded risks. The samples are generated i.i.d., and the realized risk values remain bounded in $[0,1]$. The only difficulty is that the representation used for partitioning is statistically unrelated to the target risk. This isolates a realistic failure mode: an embedding may capture semantic similarity, while being weakly related or unrelated to the particular evaluation risk being certified.

% [ZJ]
% \textbf{Why this setting is challenging for Cer-Eval.}
\textbf{Why this setting is challenging for \cereval.}
% [BjfK-Q3]
% This setting is challenging for Cer-Eval because the partition is learned from the evaluated samples, while the embeddings used to construct the partition are independent of the target risk values. Hence, any apparent risk structure discovered from the evaluated samples cannot be reliably transferred to the unevaluated samples. The key issue is a double-dipping effect: the observed risk values are used both to learn the partition and to estimate the within-partition variation. As a result, the adaptive partition may look highly informative on the evaluated samples, for example by separating samples with observed risk value $0$ from those with observed risk value $1$, even though the embeddings themselves contain no information about the risk.
% [ZJ]
% This setting is challenging for Cer-Eval because the partition is learned from the evaluated samples, while the embeddings used to construct the partition are independent of the target risk values. Hence, any apparent risk structure discovered from the evaluated samples cannot be reliably transferred to the unevaluated samples. The key issue is that the observed risk values are used both to learn the partition and to estimate the within-partition variation. As a result, the adaptive partition may look highly informative on the evaluated samples, for example by separating samples with observed risk value $0$ from those with observed risk value $1$, even though the embeddings themselves contain no information about the risk.
This setting is challenging for \cereval~because the partition is learned from the evaluated samples, while the embeddings used to construct the partition are independent of the target risk values. Hence, any apparent risk structure discovered from the evaluated samples cannot be reliably transferred to the unevaluated samples. The key issue is that the observed risk values are used both to learn the partition and to estimate the within-partition variation. As a result, the adaptive partition may look highly informative on the evaluated samples, for example by separating samples with observed risk value $0$ from those with observed risk value $1$, even though the embeddings themselves contain no information about the risk.

This creates a mismatch between the fixed-partition analysis and the adaptive procedure used in practice. \cereval~constructs its CI as if the learned partition were fixed or well behaved, so a small empirical within-partition variation on the evaluated samples leads to a small reported CI width. However, in this setting, the learned partition provides no predictive information about the risk values of unevaluated samples, because the embeddings used to assign samples to partition cells are independent of $\ell_j$. Therefore, the small within-partition variation reflects overfitting to the evaluated risks rather than genuine risk reduction on the remaining pool. Consequently, \cereval~can produce overly narrow CIs and suffer undercoverage, especially when the CI is repeatedly updated and used to stop once the target precision is reached.
% [BjfK-Q3]
% [ZJ]
% This undercoverage arises because the adaptive procedure lies outside the scope of the fixed-partition guarantee; it does not indicate that the theoretical analysis of Cer-Eval is incorrect under its stated assumptions.
This undercoverage arises because the adaptive procedure lies outside the scope of the fixed-partition guarantee; it does not indicate that the theoretical analysis of \cereval~is incorrect under its stated assumptions.

\textbf{Setups for \ourmethod~on synthetic data.}
% [ZJ]
% We compare three variants of \ourmethod in this synthetic setting. The first is the standard \ourmethod, for which we generate simple noisy surrogate risk scores as auxiliary information:
We compare three variants of \ourmethod~in this synthetic setting. The first is the standard \ourmethod, for which we generate simple noisy surrogate risk scores as auxiliary information:
\[
\tilde{\ell}_f(\x_j)
=
% [ZJ] notation: G_j, G, g, Z_j -> \Gamma_j, \Gamma, k, \ell_j (G_6 is the constant in Theorem~\ref{thm:conf_sequen}, g is the surrogate model, Z_i is used in the proof of Theorem~\ref{thm:conf_sequen}); notation: N(0,1) -> \mathcal N(0,1) (N is the pool size)
% \Pi_{[0,1]}(p_{G_j}+\sigma_s \xi_j),
% \quad
% \xi_j\sim N(0,1),
% \]
% where $\Pi_{[0,1]}$ denotes clipping to $[0,1]$. In this variant, the sampling probabilities are computed using only the surrogate risk scores, since the true risk values $Z_j$ are unobserved before evaluation. Specifically, we use a pre-evaluation uncertainty proxy proportional to $\sqrt{\tilde{\ell}_f(\x_j)(1-\tilde{\ell}_f(\x_j))}$, which prioritizes samples whose surrogate risk scores are closer to $1/2$.
\Pi_{[0,1]}(p_{\Gamma_j}+\sigma_s \xi_j),
\quad
\xi_j\sim \mathcal N(0,1),
\]
where $\Pi_{[0,1]}$ denotes clipping to $[0,1]$. In this variant, the sampling probabilities are computed using only the surrogate risk scores, since the true risk values $\ell_j$ are unobserved before evaluation. Specifically, we use a pre-evaluation uncertainty proxy proportional to $\sqrt{\tilde{\ell}_f(\x_j)(1-\tilde{\ell}_f(\x_j))}$, which prioritizes samples whose surrogate risk scores are closer to $1/2$.

% The second variant is \ourmethod~without surrogate assistance, where we set $\tilde{\ell}_f(\x_j)=0$ for all samples and use uniform sampling. This variant illustrates that the certification remains valid even without informative surrogate risk scores. The third variant is Oracle \ourmethod, which uses the same surrogate risk scores but samples according to the infeasible oracle proposal proportional to the true residual magnitude:
The second variant is \ourmethod~without surrogate assistance, where we set $\tilde{\ell}_f(\x_j)=0$ for all samples and use uniform sampling. This variant illustrates that the certification remains valid even without informative surrogate risk scores. The third variant is Oracle \ourmethod, which uses the same surrogate risk scores but samples according to the infeasible oracle proposal in \eqref{eq:experimental_oracle}, with the same $\beta=0.4$ as the practical method:
\[
% q_t^\star(j) \propto |Z_j-\tilde{\ell}_f(\x_j)|\ .
% [ZJ] notation: G_j, G, g, Z_j -> \Gamma_j, \Gamma, k, \ell_j (G_6 is the constant in Theorem~\ref{thm:conf_sequen}, g is the surrogate model, Z_i is used in the proof of Theorem~\ref{thm:conf_sequen})
% q_t^{\mathrm{orc}}(j) \propto |Z_j-\tilde{\ell}_f(\x_j)|.
% q_t^{\mathrm{orc}}(j) \propto |\ell_j-\tilde{\ell}_f(\x_j)|.
q_t^{\mathrm{orc}}(j)=\frac{\beta}{m_t}+(1-\beta)p_t^{\mathrm{res}}(j),\qquad p_t^{\mathrm{res}}(j)\propto|\ell_j-\tilde{\ell}_f(\x_j)|.
\]
% This oracle rule uses the unobserved realized risk values $Z_j$ before evaluation, and is therefore included only as an idealized benchmark.
% TODO (T4) [ZJ]: confirm whether this synthetic Oracle also uses the beta=0.4 floor; if so, replace the proportional rule above by Eqn.~(\ref{eq:experimental_oracle}).
% RESOLVED (T4, collaborator, 2026-11-04): confirmed by the author: the synthetic Oracle uses the beta=0.4 floor of Eqn. (eq:experimental_oracle); rule above replaced.
% [ZJ][M4Cr-W1]
% This oracle rule uses the realized risk values $Z_j$ before evaluation and sets $\beta=0$. The normalization and uniform fallback follow Remark~\ref{rem:oracle-experiments}. It is included only as an idealized benchmark.
% [ZJ] notation: G_j, G, g, Z_j -> \Gamma_j, \Gamma, k, \ell_j (G_6 is the constant in Theorem~\ref{thm:conf_sequen}, g is the surrogate model, Z_i is used in the proof of Theorem~\ref{thm:conf_sequen})
% This oracle rule uses the unobserved realized risk values $Z_j$ before evaluation, and is therefore included only as an idealized benchmark.
This oracle rule uses the unobserved realized risk values $\ell_j$ before evaluation, and is therefore included only as an idealized benchmark.

\appsection{Experimental Setup Details}
\label{app:setup}

This appendix expands the main-text setup of Section~\ref{sec:exp-setup}.
% Throughout, $\mathcal{D}_N = \{x_j\}_{j=1}^{N}$ denotes the evaluation pool containing all available samples from the benchmark split used for evaluation.
Throughout, $D_N=\{(\x_j,y_j)\}_{j=1}^{N}$ denotes the complete evaluation pool containing all available input-label pairs from the benchmark split used for evaluation, and $X_N=\{\x_j\}_{j=1}^{N}$ denotes its inputs.

\appsubsection{Benchmarks}
\label{app:benchmarks}

We evaluate on three publicly available classification benchmarks chosen to span label-space size, task difficulty, and pool size.
% For every benchmark, the bounded-loss requirement of Ass.~\ref{ass:loss_bound} is satisfied directly.
% \zs{We use the dataset maintained on Huggingface (better to give links)}.

% \textbf{SST-2}~\cite{socher-etal-2013-recursive}\footnote{\url{https://huggingface.co/datasets/stanfordnlp/sst2}}. Single-sentence binary sentiment classification, derived from the Stanford Sentiment Treebank by restricting the phrase-level annotations on the Pang--Lee movie review corpus to the polarity-only sentence labels.
\textbf{SST-2}~\cite{socher-etal-2013-recursive}\footnote{\url{https://huggingface.co/datasets/stanfordnlp/sst2}}. Binary sentiment classification on movie-review sentences and phrases from the Stanford Sentiment Treebank.
% We use the validation split as the evaluation pool, giving $N_{\text{SST2}} = 67{,}349$. 
We use all $67{,}349$ training examples and $872$ validation examples as the evaluation pool, giving $N_{\text{SST2}} = 68{,}221$.
The label space $\mathcal{Y} = \{\text{negative}, \text{positive}\}$ is approximately balanced.  
% SST-2 instantiates the $|\mathcal{Y}| = 2$ regime: under $0$-$1$ loss the residual
% $\ell - \widetilde{\ell}_f$ takes values in a discrete two-modal set whenever the surrogate is reasonably calibrated, which produces a sharply separated acquisition score and a highly non-uniform variance-optimal proposal $q_t^\star$ in Thm.~\ref{thm:optimal}.
% [ZJ] notation: q_t^\star -> q_t^* (as in Theorem~\ref{thm:optimal})
% Under $0$-$1$ loss, the true loss is either $0$ or $1$, while the surrogate score may vary continuously over $[0,1]$. The residual $\ell-\widetilde{\ell}_f$ therefore lies in $[-1,1]$; its distribution and the resulting proposal $q_t^\star$ depend on the surrogate scores as well as the labels.
Under $0$-$1$ loss, the true loss is either $0$ or $1$, while the surrogate score may vary continuously over $[0,1]$. The residual $\ell-\widetilde{\ell}_f$ therefore lies in $[-1,1]$; its distribution and the resulting proposal $q_t^*$ depend on the surrogate scores as well as the labels.
% The benchmark therefore probes the regime in which \ourmethod\ extracts the largest gains from inverse-probability weighting at fixed pool size.

\textbf{MMLU}~\cite{Hendrycks:Burns:Basart:Zou:Mazeika:Song2021}\footnote{\url{https://huggingface.co/datasets/cais/mmlu}}. 
Four-way multiple-choice question answering across $57$ academic subjects distributed over four broad strata: STEM (e.g., college mathematics, physics, chemistry, computer science), humanities (e.g., philosophy, formal logic, world history, moral disputes), social sciences (e.g., econometrics, sociology, public relations), and applied/professional domains (e.g., professional medicine, professional law, jurisprudence).
Every question is presented with four answer options labelled $\{\text{A}, \text{B}, \text{C}, \text{D}\}$ and a single correct choice, giving label-space cardinality $|\mathcal{Y}| = 4$. 
Following the benchmark's evaluation convention, we obtain a single evaluation pool with $N_{\text{MMLU}} = 115{,}700$.

\textbf{AG\,News}~\citep{zhang2015character}\footnote{\url{https://huggingface.co/datasets/fancyzhx/ag_news}}. 
Four-way topic classification from the headline and lead paragraph of news articles in the AG corpus, with label space $\mathcal{Y} = \{\text{World}, \text{Sports}, \text{Business}, \text{Sci/Tech}\}$. 
The standard release combines $120{,}000$ training and $7{,}600$ test examples for a total pool of $N_{\text{AGNews}} = 127{,}600$, with classes balanced exactly by construction.
AG News probes a third regime distinct from SST-2 and MMLU. 
Topic classification on lead paragraphs relies on dispersed lexical cues that smaller surrogate models capture less reliably than they do sentiment polarity or short-form factual recall.

% Together the three benchmarks span the three axes whose effect on \ourmethod\'s behaviour is captured by the analysis. The label-space cardinality varies as $|\mathcal{Y}| \in \{2, 4, 4\}$, controlling the granularity of the surrogate prediction and therefore the reachable
% correlation between $\ell$ and $\widetilde{\ell}_f$. 
% The within-pool difficulty distribution varies from approximately homogeneous (SST-2, AG News) to substantially heterogeneous across $57$ subjects (MMLU), which controls the tail of the residual distribution and the leverage available to $q_t^\star$. 
The three benchmarks vary in label-space size, task content, and pool size. Their label-space cardinalities are $|\mathcal{Y}|\in\{2,4,4\}$. The variance analysis depends on the resulting surrogate residuals and sampling probabilities, rather than on label-space size alone.
% In every case the full pool plays the role of $\mathcal{D}_N$, so the finite-pool risk $R_N$ is the average loss over the entire benchmark.
In every case the full pool plays the role of $D_N$, so the finite-pool risk $R_N$ is the average loss over the entire benchmark.
The pool size varies and controls the magnitude of the population correction $\Delta_N(\alpha_2)$.

\appsubsection{Models and surrogate-target pairs}
\label{app:models}

The target model is the LLM whose population risk we certify; the surrogate model produces the auxiliary score $\widetilde{\ell}_f$ for unevaluated samples. 
% We instantiate {\em six} surrogate-target pairs that vary along the model-family and model-generation axes, holding the surrogate substantially smaller than the target so that surrogate scoring is strictly cheaper than target scoring.
We instantiate {\em six} surrogate-target pairs that vary along the model-family and model-generation axes. Each surrogate is smaller than its target; the computational benefit depends on the inference setup and the cost of obtaining target losses.

\begin{enumerate}
\item \textbf{Llama-3-8B $\to$ Llama-3-70B}: same family, same generation.
\item \textbf{Llama-3-8B $\to$ DeepSeek-67B}: cross-family, same generation.
% \item \textbf{Llama-3-8B $\to$ Mixtral-8$\times$7B}: cross-family, same
% generation\footnote{Mixtral parameter counts are reported as active parameters per
% token.}.
\item \textbf{Llama-3-8B $\to$ Mixtral-8$\times$7B}: cross-family, same generation.
\item \textbf{Llama-3-8B $\to$ Qwen-2.5-72B}: cross-family, same generation.
\item \textbf{Llama-2-7B $\to$ Llama-3-70B}: same family, cross-generation.
\item \textbf{Llama-2-7B $\to$ DeepSeek-67B}: cross-family, cross-generation.
\end{enumerate}
\begin{table}[t]
\centering
\small
\caption{Datasets and surrogate-target pairs in the main classification experiments. Each of the six pairs is evaluated on each dataset. Additional studies state their own subsets and preprocessing.}
\label{tab:main-config-summary}
\begin{tabular}{lll}
\toprule
Dataset & Task & Pool size $N$ \\
\midrule
SST-2 & Binary sentiment classification & $68{,}221$ \\
MMLU & Four-way question answering & $115{,}700$ \\
AG News & Four-way topic classification & $127{,}600$ \\
\midrule
Pair & Surrogate & Target \\
\midrule
1 & Llama-3-8B & Llama-3-70B \\
2 & Llama-3-8B & DeepSeek-67B \\
3 & Llama-3-8B & Mixtral-8$\times$7B \\
4 & Llama-3-8B & Qwen-2.5-72B \\
5 & Llama-2-7B & Llama-3-70B \\
6 & Llama-2-7B & DeepSeek-67B \\
\bottomrule
\end{tabular}
\end{table}
% [ZJ]
% For all open-weights models we use the publicly released checkpoints without further fine-tuning. 
For all open-weights models, Llama-2-7B~\cite{Touvron:Martin:Stone:Albert:Almahairi:Babaei2023}, Llama-3-8B and Llama-3-70B~\cite{Grattafiori:Dubey:Jauhri:Pandey:Kadian:Al-Dahle2024}, DeepSeek-67B~\cite{Bi:Chen:Chen:Chen:Dai:Deng2024}, Mixtral-8$\times$7B~\cite{Jiang:Sablayrolles:Roux:Mensch:Savary:Bamford2024}, and Qwen-2.5-72B~\cite{Yang:Yang:Zhang:Hui:Zheng:Yu2024}, we use the publicly released checkpoints without further fine-tuning. 
All predictions are obtained at temperature $0$. 
Following~\citet{Berrada:Kossen:Freddie:Razzak:Gal:Rainforth2025}, for the multi-choice tasks (MMLU, AG News), we score each candidate label by its next-token log-probability and predict $\arg\max$; 
SST-2 follows the same scoring convention restricted to the two label tokens.

\appsubsection{Baselines}
\label{app:baselines}

\textbf{\evbaseline}~\cite{Waudby-Smith:Ramdas2024}. Applies the hedged-capital confidence sequence of~\eqref{eq:hedhedCP} to uniform sampling without a surrogate ($\tilde{\ell}\equiv 0$). 
% All other components (prediction fractions, scaling, $\alpha$-split) match \ourmethod\ for fairness; the only differences are the absence of the surrogate and the choice of sampler.
The betting fractions, scaling, and $\alpha$-split follow the common experimental settings.

% \textbf{\nosurr.} Ablates the surrogate-assisted approximation by setting $\tilde{\ell}\equiv 0$, while retaining the uncertainty-guided sampling strategy. 
\textbf{\nosurr.} Sets $\tilde{\ell}_f\equiv0$ and samples uniformly without replacement over the remaining pool, with $q_t(j)=1/(N-t+1)$. It retains the inferential signal and CI construction of \ourmethod.
% Comparison against \evbaseline\ isolates the contribution of active sampling at $\tilde{\ell}\equiv 0$; comparison against \ourmethod\ isolates the contribution of the surrogate at the same uncertainty-guided sampling strategy.
Both \nosurr~and \evbaseline~use uniform sampling, so their comparison does not isolate active sample selection. Relative to full \ourmethod, \nosurr~removes both surrogate assistance and uncertainty-guided sampling. The controlled study in App.~\ref{app:surrogate_acquisition_quality} examines the two contributions separately.

% \textbf{Oracle.} Uses the variance-minimizing proposal $q_t^\star(j)\propto |\ell(f(\x_j),y_j)-\tilde{\ell}_f(\x_j)|$ from Thm.~\ref{thm:optimal}. Computing $q_t^\star$ requires the true per-sample residuals, which are unknown for unevaluated samples, so \oracleacq~is not deployable in practice. We include it as an oracle reference for the best variance reduction attainable under the inferential signal in~\eqref{eq:signal_estimator}; the gap between \ourmethod~and \oracleacq~therefore measures the suboptimality of the practical sample-selection score.
% [ZJ][M4Cr-W1]
% \textbf{Oracle.} Uses the true per-sample residuals and sets $\beta=0$, sampling according to $q_t^{\mathrm{orc}}$ in Remark~\ref{rem:oracle-experiments}. Its residual-proportional branch minimizes the one-step conditional variance without a probability-floor constraint. Computing these probabilities requires the losses of unevaluated samples, so \oracleacq~is not deployable in practice. It provides an idealized variance reference; comparisons with the practical method reflect both unavailable residual information and removal of the sampling floor.
\textbf{Oracle.} Uses the true per-sample residuals as sample-selection scores with $\beta=0.4$, sampling according to $q_t^{\mathrm{orc}}$ in Remark~\ref{rem:oracle-experiments}. Computing these probabilities requires the losses of unevaluated samples, so \oracleacq~is not deployable in practice. It provides an idealized variance reference; comparisons with the practical method reflect unavailable residual information.

% [ZJ]
% \textbf{Cer-Eval}~\cite{Wang:Chen:Bo:Xu2025}.
\textbf{\cereval}~\cite{Wang:Chen:Bo:Xu2025}.
% [BjfK-Q3]
% At each round, \cereval~ partitions the evaluation pool using the evaluated samples, computes within-partition statistics and CI radii through Bernstein-style or adaptive concentration bounds, selects the partition that contributes most to the current uncertainty, and evaluates a new sample from that partition. This strategy can be effective when the partition is fixed or well aligned with the true risk structure, since a low-variance partition can reduce the number of evaluated samples. However, there is a gap between its theoretical certification and its proposed algorithm: the certification relies on concentration arguments for a fixed or well-behaved partition, whereas the algorithm repeatedly learns and updates the partition from the evaluated samples, their embeddings, and their observed risks. This creates a double-dipping issue, where the same evaluated data are used both to choose the partition and to quantify the within-partition variation. As shown in our diagnostic coverage study~Tab.~\ref{tab:rq1-coverage}, such adaptive partitioning can fail to maintain anytime-valid coverage under repeated monitoring and stopping. Therefore, we report \cereval~separately to illustrate this gap, and focus the main efficiency comparisons on methods designed to be anytime-valid.
At each round, \cereval~ partitions the evaluation pool using the evaluated samples, computes within-partition statistics and CI radii through Bernstein-style or adaptive concentration bounds, selects the partition that contributes most to the current uncertainty, and evaluates a new sample from that partition. This strategy can be effective when the partition is fixed or well aligned with the true risk structure, since a low-variance partition can reduce the number of evaluated samples. However, its certification relies on concentration arguments for a fixed or well-behaved partition, whereas the algorithm repeatedly learns and updates the partition from the evaluated samples, their embeddings, and their observed risks. Consequently, the stated guarantee does not directly cover the complete adaptive partition-learning procedure; this distinction concerns the scope of the current theoretical analysis rather than the validity of adaptive partitioning in general. In this procedure, the same evaluated data are used both to choose the partition and to quantify the within-partition variation. As shown in our diagnostic coverage study~Tab.~\ref{tab:rq1-coverage}, such adaptive partitioning can fail to maintain anytime-valid coverage under repeated monitoring and stopping. Therefore, we report \cereval~separately to illustrate this distinction, and focus the main efficiency comparisons on methods designed to be anytime-valid.

% \subsection{Acquisition strategies}
% \label{app:acquisition-sweep}

% The optimal proposal of Theorem~\ref{thm:optimal} requires per-sample true losses and is therefore infeasible in practice. Remark~\ref{rem:01loss} introduces \zs{three} label-free instantiations for the 0–1 loss, substituting label-free proxies (predictive uncertainty, hard/soft pseudo-label loss with respect to a surrogate) for the unobserved residual. 
% We deploy Strategy~4 (hard pseudo-label loss against target predictive uncertainty) in the main results because it is the most label-free of the five. Strategies~1–3 and~5 are evaluated in this appendix; results are in Table~\ref{tab:app-acquisition}.

\appsubsection{Hyperparameters}
\label{app:hyperparameters}

\textbf{Significance level and split.}
% We set the total significance level to $\alpha=0.05$, which controls the miscoverage probability of the final CI. We split it evenly as $\alpha_1=\alpha_2=0.025$: $\alpha_1$ is used for the anytime-valid $e$-process inference of the finite-pool risk $R_N$, while $\alpha_2$ is used for the finite-pool-to-population correction via Hoeffding's inequality. App.~\ref{app:hyperparameters} studies the sensitivity of the number of evaluated samples required to reach $\epsilon$ to this split.
We set the total significance level to $\alpha=0.05$, which controls the miscoverage probability of the final CI. We split it evenly as $\alpha_1=\alpha_2=0.025$: $\alpha_1$ is used for the anytime-valid $e$-process inference of the finite-pool risk $R_N$, while $\alpha_2$ is used for the finite-pool-to-population correction via Hoeffding's inequality. App.~\ref{app:sensitivity} studies the sensitivity of the number of evaluated samples required to reach $\epsilon$ to this split.

\textbf{$e$-process construction.}
Following the standard $e$-process construction of~\citet{Waudby-Smith:Ramdas2024}, we use the predictable fractions in Rem.~\ref{rem:betting-fractions} with constant $c=0.5$. We set the hedge weight to $\theta=0.5$ in the hedged capital process in~\eqref{eq:hedhedCP}, treating the upper and lower directions symmetrically.

\textbf{Scaling.}
% We use the global scaling strategy in Rem.~\ref{rem:scaling}. Specifically, we choose fixed predictable bounds $a$ and $b$ that satisfy $a\le \hat S_t\le b$ for all rounds, and scale the signal as $\bar S_t=(\hat S_t-a)/(b-a)$ before constructing the $e$-process.
% [ZJ][M4Cr-W1]
% We use the global scaling strategy in Rem.~\ref{rem:scaling}. Specifically, we choose fixed predictable bounds $a$ and $b$ that satisfy $a\le \hat S_t\le b$ for all rounds, and scale the signal as $\bar S_t=(\hat S_t-a)/(b-a)$ before constructing the $e$-process. For the $\beta=0$ oracle, valid finite signal bounds are given in Remark~\ref{rem:oracle-experiments}; the positive-floor formula involving $1/\beta$ cannot be used with $\beta=0$.
We use the global scaling strategy in Rem.~\ref{rem:scaling}. Specifically, we choose fixed predictable bounds $a$ and $b$ that satisfy $a\le \hat S_t\le b$ for all rounds, and scale the signal as $\bar S_t=(\hat S_t-a)/(b-a)$ before constructing the $e$-process.

% TODO (T5) [ZJ][M4Cr-W2b]: report the grid size $K$ (or the spacing $h=(U-L)/(K-1)$) of the candidate grid $\mathcal G$ in App.~\ref{app:implementation} used for numerical test inversion, e.g. one sentence here.
% RESOLVED (T5, collaborator, 2026-11-04): K = max(2000, ceil(2000(U-L))), capped at 200,000, interior points; evidence: src/save/benchmark.py:152-154 (adaptive-bound path used by the reported methods), src/save/baselines/evalue.py:72-73, src/save/core/confidence.py:66-80, configs/paper_experiment_v0504.yaml (cs_grid_size: 2000)
\textbf{Grid for numerical test inversion.}
The candidate grid $\mathcal G$ in App.~\ref{app:implementation} consists of $K=\max\{2000,\lceil 2000(U-L)\rceil\}$ equally spaced interior points $\upsilon_i=L+i(U-L)/(K+1)$, $i=1,\ldots,K$, with $K$ capped at $200{,}000$; the same grid is used for \ourmethod\ and \evbaseline. Thus $K=2000$ for the $0$-$1$ loss and for the toxicity score in App.~\ref{sec:re-new-safety}, while under the cross-entropy loss $K$ increases with the range $U-L$.

\textbf{Sampling lower bound and stopping rule.}
% For sampling, we set $\beta=0.4$ to control the lower bound of the sampling probabilities, as required by Assumption~\ref{ass:beta}. The sensitivity to this choice is also studied in App.~\ref{app:hyperparameters}. We stop the evaluation at the first time $t$ such that $\mathrm{width}(C_t)\le \epsilon$, with $\epsilon=0.05$ in all main experiments.
% For practical sampling, the default $\beta=0.4$ controls the probability floor in Assumption~\ref{ass:beta}. The oracle experiments set $\beta=0$ as described in Remark~\ref{rem:oracle-experiments}. The sensitivity to the practical choice of $\beta$ is also studied in App.~\ref{app:hyperparameters}. We stop the evaluation at the first time $t$ such that $\mathrm{width}(C_t)\le \epsilon$, with $\epsilon=0.05$ in all main experiments.
% [ZJ][M4Cr-W1]
% For practical sampling, the default $\beta=0.4$ controls the probability floor in Assumption~\ref{ass:beta}. The oracle experiments set $\beta=0$ as described in Remark~\ref{rem:oracle-experiments}. The sensitivity to the practical choice of $\beta$ is also studied in App.~\ref{app:sensitivity}. We stop the evaluation at the first time $t$ such that $\mathrm{width}(C_t)\le \epsilon$, with $\epsilon=0.05$ in all main experiments.
For sampling, we set $\beta=0.4$ to control the lower bound of the sampling probabilities, as required by Assumption~\ref{ass:beta}. The sensitivity to this choice is also studied in App.~\ref{app:sensitivity}. We stop the evaluation at the first time $t$ such that $\mathrm{width}(C_t)\le \epsilon$, with $\epsilon=0.05$ in all main experiments.
% The global constant-bound alternative is benchmarked in Appendix~\ref{app:sensitivity} and yields qualitatively identical behavior at slightly worse constants.

% \paragraph{Candidate grid.} 
% Test inversion is implemented over a finite grid on $[L,U]$. Following
% Remark~\ref{rem:inversion}, we initialize the search at the running
% estimate $\hat{R}_t$ and expand outward to locate the acceptance
% boundary, refining the endpoints by bisection. The grid resolution and
% its effect on width are reported in Appendix~\ref{app:sensitivity}.

\appsubsection{Trials and reporting}
\label{app:reporting}
% Each experimental configuration (i.e., a surrogate-target pair on a particular dataset) is run for 50 random seeds.
Each configuration in the main classification experiments (a surrogate-target pair on a dataset) is run for $50$ random seeds. Additional studies report their own numbers of seeds.
Pooled trajectories in Figs.~\ref{fig:rq2-efficiency-main}~and~\ref{fig:width-ce} stack per-seed widths across independent runs and report mean width with a $\pm 1\,\sigma$ band across the per-pair means.
% The per-seed wall-clock cost of each method is reported in Appendix~\ref{app:wallclock}. 
% The dominant cost is target-LLM forward passes; \ourmethod\'s sampling and CI bookkeeping add negligible overhead per round.

% [ZJ] warning fix: hyperref "Token not allowed in a PDF string": the PDF bookmark of this title uses \texorpdfstring{$0$-$1$}{0-1} (typeset title unchanged)
% \appsection{Empirical coverage on real LLM evaluation under $0$-$1$ loss}
\appsection{Empirical coverage on real LLM evaluation under \texorpdfstring{$0$-$1$}{0-1} loss}
\label{app:01-coverage}

In addition to the coverage investigation with synthetic data reported in Tab.~\ref{tab:rq1-coverage}, we also report the empirical coverage results for the main real LLM evaluation experiments under $0$-$1$ loss (Tab.~\ref{tab:labels-to-eps} and Fig.~\ref{fig:rq2-efficiency-main}). The evaluation follows the same sequential protocol as in Sec.~\ref{sec:experiments}: after each round, the CI is updated, and the procedure stops at the time when the CI width reaches the target precision $\epsilon=0.05$.
% Table~\ref{tab:01-coverage} reports the empirical miscoverage rate at this stopping time. Both \ourmethod~and \evbaseline~remain below the target level $\alpha=0.05$ on all datasets, supporting the anytime-valid coverage guarantee under adaptive stopping.
Table~\ref{tab:01-coverage} reports empirical miscoverage at this stopping time. The observed rates for \ourmethod~and \evbaseline~are below $\alpha=0.05$ on all datasets. These results describe coverage at the tested stopping rule; they do not establish simultaneous coverage at every round or a formal guarantee for the numerical grid approximation.

\begin{table}[h]
\centering
\small
\caption{Empirical miscoverage rate at the stopping time when the CI width reaches $\epsilon=0.05$, pooled across the surrogate-target pairs within each dataset. Brackets give Wilson 95\% intervals. Both \ourmethod~and \evbaseline~remain below the nominal level $\alpha=0.05$ on all datasets.}
\label{tab:01-coverage}
\setlength{\tabcolsep}{10pt}
\renewcommand{\arraystretch}{1.2}
\begin{tabular}{lccc}
\toprule
Method & \textsc{SST-2} & \textsc{MMLU} & \textsc{AG$\,$News} \\
\midrule
\evbaseline
& $0.013$~{\color{gray!90}\tiny$[0.005,\,0.034]$}
& $0.007$~{\color{gray!90}\tiny$[0.002,\,0.024]$}
& $0.007$~{\color{gray!90}\tiny$[0.002,\,0.024]$} \\
\nosurr
& $0.000$~{\color{gray!90}\tiny$[0.000,\,0.013]$}
& $0.000$~{\color{gray!90}\tiny$[0.000,\,0.013]$}
& $0.000$~{\color{gray!90}\tiny$[0.000,\,0.013]$} \\
\ourmethod
& $0.000$~{\color{gray!90}\tiny$[0.000,\,0.013]$}
& $0.000$~{\color{gray!90}\tiny$[0.000,\,0.013]$}
& $0.013$~{\color{gray!90}\tiny$[0.005,\,0.034]$} \\
\midrule
\oracleacq
& $0.000$~{\color{gray!90}\tiny$[0.000,\,0.013]$}
& $0.000$~{\color{gray!90}\tiny$[0.000,\,0.013]$}
& $0.000$~{\color{gray!90}\tiny$[0.000,\,0.013]$} \\
\bottomrule
\end{tabular}
\end{table}

\appsection{Certifiable and Label-efficient LLM evaluation on Cross-Entropy Loss}
\label{app:ce-results}

We follow the deployment-style protocol in Sec.~\ref{sec:experiments}, replacing the $0$--$1$ risk with the cross-entropy risk
% $\ell_{\mathrm{ce}}(f,\x,y)=-\log p_f(y\mid \x)$, which is bounded in $[L,U]$ by clipping at the maximum cross-entropy value observed over the pool. The sample-selection rule follows Strategy A in Rem.~\ref{rem:cross_entropy}, and all other hyperparameters are kept the same as in App.~\ref{app:setup}. Fig.~\ref{fig:width-ce} shows the resulting CI-width trajectories.
$\ell_{\mathrm{ce}}(f,\x,y)=-\log p_f(y\mid\x)$. The reported experiment uses the maximum cross-entropy observed over the pool as its upper bound $U$. This bounds the realized pool losses, but uses information from all labels and does not supply a bound fixed before observing pool losses as required by Assumption~\ref{ass:loss_bound}. We therefore report these results as an empirical comparison, rather than as verification of the population-risk guarantee. The sample-selection rule follows Strategy~A in Rem.~\ref{rem:cross_entropy}, with the other hyperparameters as in App.~\ref{app:setup}. Fig.~\ref{fig:width-ce} shows the resulting CI-width trajectories.

\begin{table}[h]
\centering
\small
\caption{Empirical miscoverage at the stopping time when the CI width reaches the target precision $\epsilon = 0.05$ under the \textbf{cross-entropy} loss, pooled across the six surrogate-target pairs within each dataset ($50$ seeds per pair). 
Point estimates are accompanied by Wilson $95\%$ intervals (in gray). 
% All $e$-process-based methods stay within the nominal $\alpha = 0.05$ on every dataset, consistent with the by-construction anytime-validity of $\cM_t^\pm$ in \eqref{eq:M_t} and of $C_t$ via Thm.~(\ref{thm:coverage_convergence}).
The observed miscoverage rates are below the nominal $\alpha=0.05$ on every dataset. These rates concern the empirical pool-bound protocol described above and do not establish the population-risk guarantee of Thm.~\ref{thm:coverage_convergence}.
}
\label{tab:ce-coverage}
\setlength{\tabcolsep}{10pt}
\renewcommand{\arraystretch}{1.2}
\begin{tabular}{l c c c}
\toprule
\textbf{Method} & \textsc{SST-2} & \textsc{MMLU} & \textsc{AG~News} \\
\midrule
\evbaseline
& $0.000$~{\color{gray!90}\tiny$[0.000,\,0.013]$}
& $0.000$~{\color{gray!90}\tiny$[0.000,\,0.013]$}
& $0.007$~{\color{gray!90}\tiny$[0.002,\,0.024]$} \\
\nosurr
& $0.000$~{\color{gray!90}\tiny$[0.000,\,0.013]$}
& $0.007$~{\color{gray!90}\tiny$[0.002,\,0.024]$}
& $0.000$~{\color{gray!90}\tiny$[0.000,\,0.013]$} \\
\ourmethod
& $0.000$~{\color{gray!90}\tiny$[0.000,\,0.013]$}
& $0.000$~{\color{gray!90}\tiny$[0.000,\,0.013]$}
& $0.000$~{\color{gray!90}\tiny$[0.000,\,0.013]$} \\
\midrule
\oracleacq
& $0.000$~{\color{gray!90}\tiny$[0.000,\,0.013]$}
& $0.000$~{\color{gray!90}\tiny$[0.000,\,0.013]$}
& $0.000$~{\color{gray!90}\tiny$[0.000,\,0.013]$} \\
\bottomrule
\end{tabular}
\end{table}

% [collaborator] warning fix: `h' float specifier changed to `ht' (layout shift after the T5 paragraph)
% \begin{figure}[h]
\begin{figure}[ht]
    \centering
    \includegraphics[width=\linewidth]{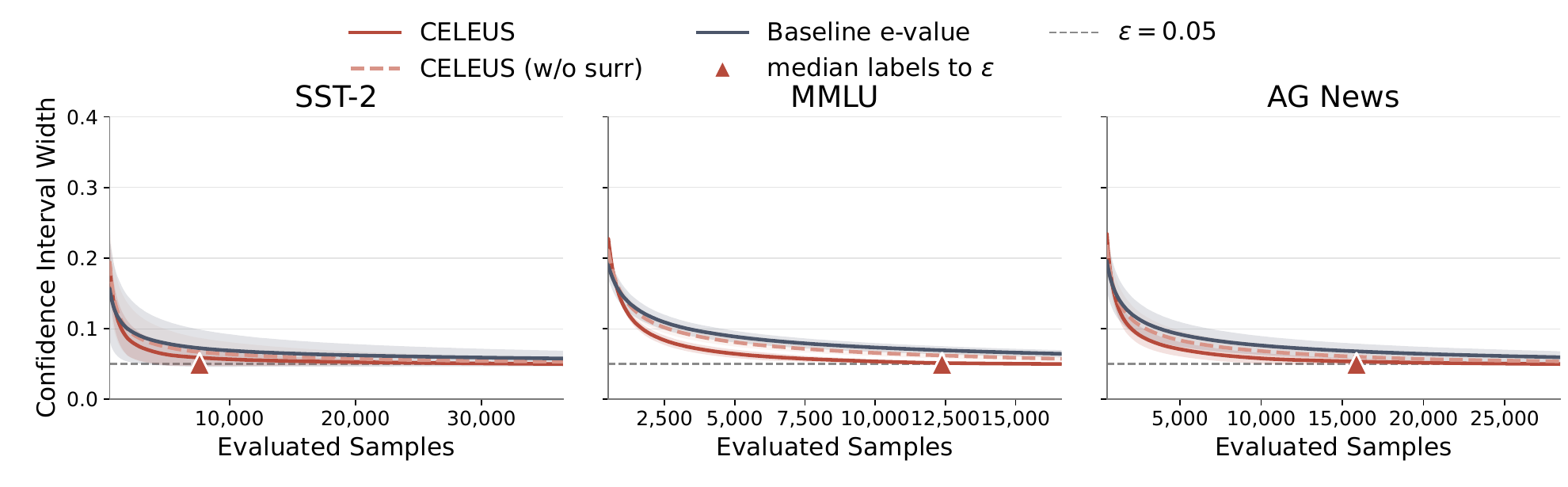}
%     \caption{Confidence interval width versus number of evaluated samples under cross-entropy loss on SST-2, MMLU, and AG News. Each curve is the mean width at each label budget, aggregated over $50$ seeds and the four surrogate-target pairs of App.~\ref{app:models}; the shaded band is $\pm 1\sigma$ across the per-pair means. The dashed line marks the target precision $\epsilon = 0.05$, and triangles mark the median samples-to-$\epsilon$ for \ourmethod, which reaches $\epsilon$ first on every benchmark.}
    \caption{Confidence interval width versus number of evaluated samples under cross-entropy loss on SST-2, MMLU, and AG News. Each curve is the mean width at each label budget, aggregated over $50$ seeds and the six surrogate-target pairs of App.~\ref{app:models}; the shaded band is $\pm 1\sigma$ across the per-pair means. The dashed line marks the target precision $\epsilon = 0.05$, and triangles mark the median samples-to-$\epsilon$ for \ourmethod, which reaches $\epsilon$ first on every benchmark.}
    \label{fig:width-ce}
\end{figure}

\textbf{Miscoverage rate at the stopping time.}
% Tab.~\ref{tab:ce-coverage} shows that the empirical miscoverage at the stopping time remains below the prescribed level $\alpha=0.05$ for all cross-entropy configurations, consistent with the $0$-$1$ risk results in Tabs.~\ref{tab:rq1-coverage} and~\ref{tab:01-coverage}. This follows from the $e$-process construction: the anytime-valid guarantee does not depend on the specific risk form, but only requires the boundedness condition of evaluation score in Assumption~\ref{ass:loss_bound}.
Tab.~\ref{tab:ce-coverage} reports empirical miscoverage below $\alpha=0.05$ for all displayed cross-entropy configurations. The theoretical construction applies to different loss functions when the boundedness, conditional-mean, sampling, and betting assumptions hold. Population inference additionally requires the stated population sampling assumptions. These conditions and exact test inversion are distinct from the empirical pool-bound and numerical-grid choices used here.

\textbf{Efficiency.}
% The relative ordering of methods at the stopping time matches the $0$-$1$ case: in Fig.~\ref{fig:width-ce}, \ourmethod~has smaller CI width than \nosurr, which in turn improves over \evbaseline, across the evaluation process. Two observations are worth noting. First, the required numbers of evaluated samples are larger, because cross-entropy risk has a wider range and the surrogate residual correspondingly contributes more to $\mathrm{Var}(\widehat{S}_t \mid \mathcal{F}_{t-1})$. Second, as in the $0$--$1$ setting, the benefit of surrogate-assisted approximation is largest on SST-2 and MMLU, where the entropy proxy is more aligned with the target cross-entropy risk, and smallest on AG$\,$News.
The relative ordering of methods at the stopping time matches the $0$-$1$ case: in Fig.~\ref{fig:width-ce}, \ourmethod~has smaller CI width than \nosurr, which in turn improves over \evbaseline, across the evaluation process. Two observations are worth noting. First, the required numbers of evaluated samples are larger, because cross-entropy risk has a wider range and the surrogate residual correspondingly contributes more to $\mathrm{Var}(\widehat{S}_t \mid \mathcal{F}_{t-1})$. Second, as in the $0$--$1$ setting, the gain of full \ourmethod~over uniform \nosurr~is largest on SST-2 and MMLU, where the entropy proxy is more aligned with the target cross-entropy risk, and smallest on AG News.

\appsection{Sample-selection strategy comparison}
\label{app:acquisition-strategy}

The deployed configurations in Sec.~\ref{sec:experiments} and Sec.~\ref{app:ce-results} use one strategy from each of the two label-free families: Strategy~A of Rem.~\ref{rem:0-1} for the $0$-$1$ loss and Strategy~A of Rem.~\ref{rem:cross_entropy} for cross-entropy. 
% This appendix compares the alternatives within each family, holding the trajectory machinery, $e$-process, and stopping rule fixed; only the formula that produces $q_t$ changes across runs.
This appendix compares the alternatives within each family while retaining the same $e$-process construction and stopping rule. The strategies vary the acquisition proxy and the surrogate score used in the residual correction, so this comparison does not isolate the effect of $q_t$ alone.

\begin{figure}[h]
    \centering
    \includegraphics[width=\linewidth]{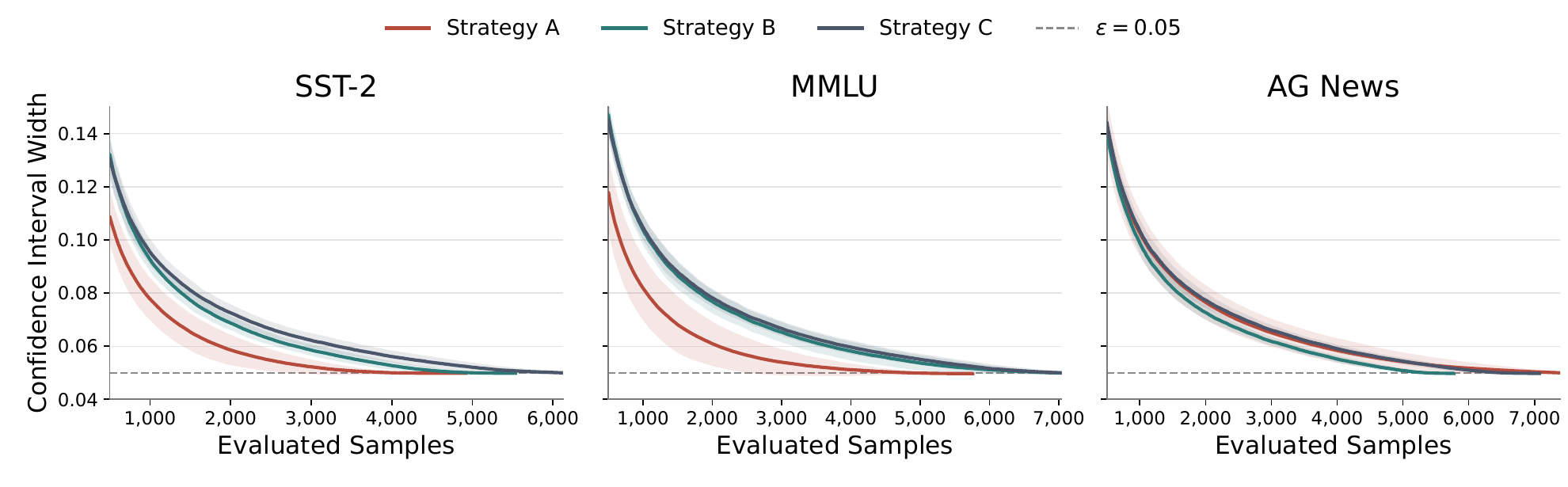}
    \caption{
    Sample-selection strategy comparison under $0$-$1$ risk. We compare the sample-selection strategies in Rem.~\ref{rem:0-1} by plotting CI width versus the number of evaluated samples on SST-2, MMLU, and AG$\,$News. Each curve is averaged over $50$ seeds, and the dashed line marks the target precision $\epsilon=0.05$.
    }
    \label{fig:acq-01}
\end{figure}

% [collaborator] warning fix: `h' float specifier changed to `ht' (layout shift after the T5 paragraph)
% \begin{figure}[h]
\begin{figure}[ht]
    \centering
    \includegraphics[width=\linewidth]{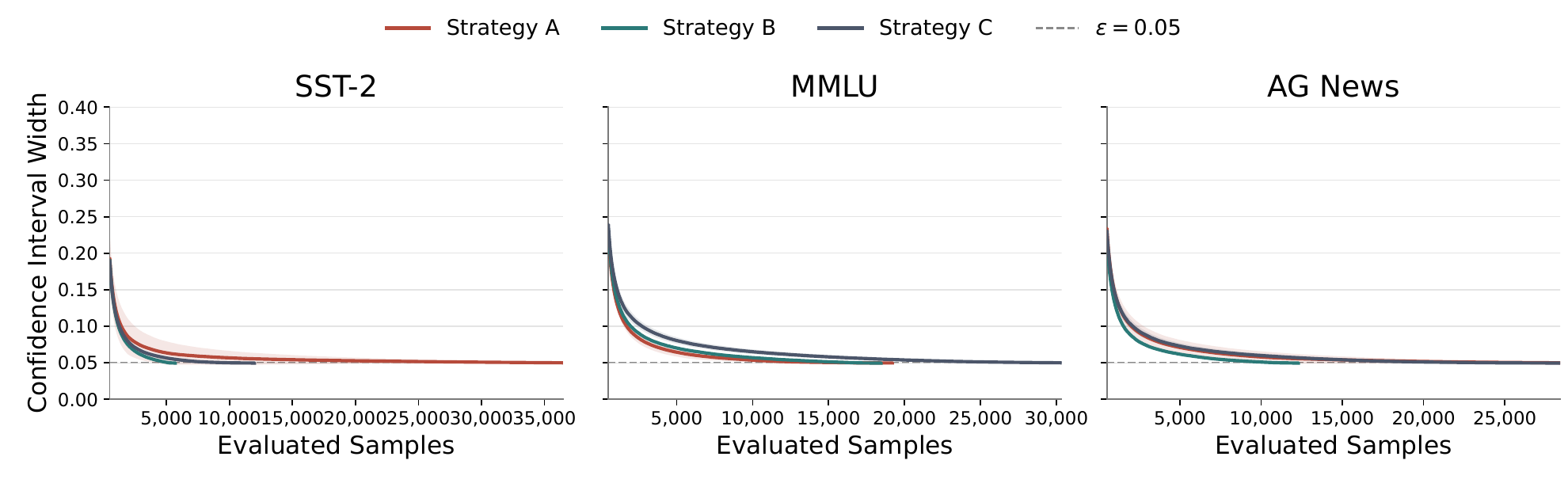}
    \caption{
    Sample-selection strategy comparison under cross-entropy risk. We compare the sample-selection strategies in Rem.~\ref{rem:cross_entropy} by plotting CI width versus the number of evaluated samples on SST-2, MMLU, and AG$\,$News. Each curve is averaged over $50$ seeds, and the dashed line marks the target precision $\epsilon=0.05$.
    }
    \label{fig:acq-ce}
\end{figure}

\textbf{$0$-$1$ loss family.} The strategies in Rem.~\ref{rem:0-1} differ in how they approximate the oracle residual $|\ell(f(\x_j),y_j)-\tilde{\ell}_f(\x_j)|$ using label-free quantities.
Strategy~A uses the hard pseudo-loss $\ell_{\mathrm{hard}}(f,g,\x_j)=\mathbf{1}\{\hat y_f(\x_j)\neq \hat y_g(\x_j)\}$ as the target-loss proxy and pairs it with the predictive dispersion score $\ell_{\mathrm{disp}}(f,\x_j)$ as the surrogate score.
Strategies~B and~C both use $\ell_{\mathrm{disp}}(f,\x_j)$ and the soft pseudo-loss $\ell_{\mathrm{soft}}(f,g,\x_j)=1-p_f(\hat y_g(\x_j)\mid \x_j)$, and therefore produce numerically identical sample-selection rankings.
% They differ only in which quantity plays the role of the target-loss proxy inside the estimator in~\eqref{eq:lure_estimator}: Strategy~B treats $\ell_{\mathrm{disp}}$ as the target-loss proxy and $\ell_{\mathrm{soft}}$ as the surrogate score, whereas Strategy~C reverses this assignment.
% This leaves the sampling order unchanged but changes the correction term, which can affect the conditional variance of $\widehat S_t$ and the cumulative-variance term in Thm.~\ref{thm:conf_sequen}.
They differ in the assignment of the target-loss proxy and surrogate score: Strategy~B uses $\ell_{\mathrm{disp}}$ as the proxy and $\ell_{\mathrm{soft}}$ as the surrogate, whereas Strategy~C reverses them. The proxy determines acquisition scores; the estimator in~\eqref{eq:lure_estimator} uses the observed target loss and the assigned surrogate score. The strategies give the same proposal for the same remaining pool, but different surrogate corrections. Their realized sampling orders coincide only if the sampling randomness is also coupled.
Fig.~\ref{fig:acq-01} compares the three strategies on SST-2, MMLU, and AG$\,$News.
Strategy~A yields smaller CI widths than Strategies~B and~C across evaluation budgets on all three datasets, while remaining within the prescribed miscoverage level $\alpha=0.05$.
This is intuitive for $0$-$1$ loss: the hard pseudo-loss preserves the discrete agreement structure of the target loss more directly, whereas the soft pseudo-loss spreads this information over $[0,1]$.
As a result, Strategy~A more effectively prioritizes samples with large proxy residuals, aligning more closely with the variance-optimal sampling rule in Thm.~\ref{thm:optimal}.

\textbf{Cross-entropy loss family.} The three strategies in Rem.~\ref{rem:cross_entropy} differ in the same structural way, with predictive entropy and mode loss used as label-free proxies for the unobserved cross-entropy loss.
Strategy~A treats the mode loss $\ell_{\mathrm{mode}}(f,\x_j)=-\log p_f(\hat y_f(\x_j)\mid\x_j)$ as the target-loss proxy and uses the predictive entropy $H(f,\x_j)$ as the surrogate score.
Strategy~B reverses this assignment, treating $H(f,\x_j)$ as the target-loss proxy and $\ell_{\mathrm{mode}}(f,\x_j)$ as the surrogate score; hence it has the same sample-selection ranking as Strategy~A but a different correction term in~\eqref{eq:lure_estimator}.
Strategy~C is the only strategy that uses a separate surrogate model: it compares the target predictive entropy $H(f,\x_j)$ with the surrogate predictive entropy $H(g,\x_j)$, prioritizing samples where the two predictive distributions differ most in entropy.
Fig.~\ref{fig:acq-ce} compares the three strategies.
The differences among strategies become smaller at large $t$, because their effect enters the CI width through the cumulative conditional-variance term in Thm.~\ref{thm:conf_sequen}; as more samples are evaluated, this cumulative averaging reduces the impact of early sample-selection differences.

\appsection{Hyperparameter sensitivity}
\label{app:sensitivity}

% \ourmethod\ has four configuration parameters that may affect the number of evaluated samples needed to reach the target precision, while leaving the anytime-valid coverage guarantee unchanged: the split of $\alpha$ between the $e$-process and the population correction, the lower bound $\beta$ on sampling probabilities, the hedge weight $\theta$, and the stabilizer $c$ in the predictable fractions. The default values
\ourmethod\ has four configuration parameters that may affect the number of evaluated samples needed to reach the target precision, while preserving the exact-inversion coverage guarantee when the stated assumptions remain satisfied: the split of $\alpha$ between the $e$-process and the population correction, the lower bound $\beta$ on sampling probabilities, the hedge weight $\theta$, and the stabilizer $c$ in the predictable fractions. The default values
\[
\alpha_1=\alpha_2=0.025,\qquad \beta=0.4,\qquad \theta=0.5,\qquad c=0.5
\]
% are fixed in advance and are not tuned on evaluation risks. We vary one parameter at a time while keeping the others at their default values. For $\beta$, we sweep over $\{0.1,0.2,0.3,0.4,0.5,0.6\}$. Tab.~\ref{tab:hp-sens} reports the number of evaluated samples needed to reach $\epsilon$ on four representative dataset-model configurations.
% [ZJ][M4Cr-W1]
% are fixed in advance and are not tuned on evaluation risks. We vary one parameter at a time while keeping the others at their default values. For $\beta$, we sweep over $\{0.1,0.2,0.3,0.4,0.5,0.6\}$. Tab.~\ref{tab:hp-sens} reports the number of evaluated samples needed to reach $\epsilon$ on four representative dataset-model configurations. This sensitivity study concerns practical sampling; the oracle benchmark uses $\beta=0$ and is not part of the positive-$\beta$ sweep.
are fixed in advance and are not tuned on evaluation risks. We vary one parameter at a time while keeping the others at their default values. For $\beta$, we sweep over $\{0.1,0.2,0.3,0.4,0.5,0.6\}$. Tab.~\ref{tab:hp-sens} reports the number of evaluated samples needed to reach $\epsilon$ on four representative dataset-model configurations.

% TODO (T6) [ZJ][M4Cr-W3]: the Table 8 caption (results/hparams-app/tab_oat.tex, line 7; now line 9) lists Llama-2-7B -> Mixtral-8x7B and Llama-2-7B -> Qwen-2.5-72B,
% which are not among the six pairs in Table~\ref{tab:main-config-summary}. Check the experiment logs: fix the labels, or describe them as additional pairs.
% RESOLVED (T6, collaborator, 2026-11-04): labels match the run config (configs/paper_experiment_v0504.yaml:249-253, hparam_sweep.cells); columns 1-2 are additional pairs, caption now says so
% [ZJ] warning fix: `h' float specifier changed to `ht' (LaTeX already placed this float as [ht])
% \begin{table}[h]
\begin{table}[ht]
\centering
\small
\caption{
Hyperparameter sensitivity.
We report the median number of evaluated samples required to reach $\epsilon=0.05$ over $50$ seeds.
% Each column corresponds to one dataset and surrogate-target pair: \ding{172} Llama-2-7B $\to$ Mixtral-8$\times$7B on SST-2; \ding{173} Llama-2-7B $\to$ Qwen-2.5-72B on MMLU; \ding{174} Llama-3-8B $\to$ DeepSeek-67B on AG$\,$News; and \ding{175} Llama-3-8B $\to$ Llama-3-70B on SST-2.
Each column corresponds to one dataset and surrogate-target pair: \ding{172} Llama-2-7B $\to$ Mixtral-8$\times$7B on SST-2; \ding{173} Llama-2-7B $\to$ Qwen-2.5-72B on MMLU; \ding{174} Llama-3-8B $\to$ DeepSeek-67B on AG$\,$News; and \ding{175} Llama-3-8B $\to$ Llama-3-70B on SST-2. The pairs in columns \ding{172} and \ding{173} are not among the six pairs in Table~\ref{tab:main-config-summary}.
The default setting is $(\alpha_1=0.025,\ \theta=0.5,\ c=0.5)$, chosen without using prior information about the target LLMs or the evaluation datasets.
}
% \resizebox{0.6\linewidth}{!}{%
\begin{tabular}{lrrrr}
\toprule
\textbf{Config} & \textbf{SST-2}/\ding{172} & \textbf{MMLU}/\ding{173} & \textbf{AG\,News}/\ding{174} & \textbf{SST-2}/\ding{175} \\
\midrule
\multicolumn{5}{l}{$\alpha$-split, with $\alpha_1 + \alpha_2 = 0.05$} \\
\midrule
baseline & 2,426 & 2,116 & 5,436 & 3,997 \\
$\alpha_1=0.005$ & 3,398 & 2,823 & 7,544 & 5,552 \\
$\alpha_1=0.045$ & \textbf{2,086} & \textbf{1,733} & \textbf{4,606} & \textbf{3,494} \\
\midrule
\multicolumn{5}{l}{$\theta$, with $0 < \theta < 1$} \\
\midrule
baseline & \textbf{2,426} & \textbf{2,116} & \textbf{5,436} & \textbf{3,997} \\
$\theta=0.25$ & 2,531 & 2,199 & 5,673 & 4,188 \\
$\theta=0.75$ & 2,570 & 2,218 & 5,685 & 4,250 \\
\midrule
\multicolumn{5}{l}{$c$, with $0<c<1$} \\
\midrule
baseline & 2,426 & 2,116 & 5,436 & 3,997 \\
$c=0.3$ & 2,533 & 2,141 & 5,709 & 4,266 \\
$c=0.7$ & \textbf{2,415} & \textbf{2,112} & \textbf{5,213} & \textbf{3,892} \\
\bottomrule
\end{tabular}
% }
\label{tab:hp-sens}
\end{table}

\textbf{Sensitivity to the $\alpha$-split.}
Tightening $\alpha_1$ from $0.025$ to $0.005$ increases the required number of evaluated samples by $40\%$, $33\%$, $39\%$, and $39\%$ on configurations 1-4, respectively. Loosening $\alpha_1$ to $0.045$ reduces the required number by $14\%$-$19\%$. This trend is consistent with the $e$-process construction: a smaller $\alpha_1$ raises the rejection threshold $1/\alpha_1$, so more evidence is needed before the CI can shrink. Although a larger $\alpha_1$ leaves a smaller $\alpha_2$ for the finite-pool-to-population correction, this correction remains small for the large pools used in our experiments. We keep the symmetric split as default because it is conservative and keeps the population correction stable for smaller pools.

\textbf{Sensitivity to the hedge weight $\theta$.}
Changing $\theta$ from $0.5$ to $0.25$ or $0.75$ changes the number of evaluated samples by less than $6\%$ on every configuration, and both perturbations slightly increase the sample budget. This suggests that the symmetric choice $\theta=0.5$ is near-optimal in our experiments, as it balances the upper and lower hedged capital processes in \eqref{eq:hedhedCP}.

\textbf{Sensitivity to the constant $c$.}
Changing $c$ to $0.3$ or $0.7$ has a mild effect. The choice $c=0.3$ increases the number of evaluated samples by $4\%$-$7\%$, while $c=0.7$ slightly reduces it on some configurations. Since $c$ controls the scale of the predictable fractions in Rem.~\ref{rem:betting-fractions}, this result indicates that the default $c=0.5$ provides a stable trade-off across datasets.

% TODO (T6) [ZJ][M4Cr-W3]: the legend of fig_beta.pdf shows Qwen2.5-72B -> Llama3-70B and DeepSeek-67B -> Llama3-70B, which are not among the six pairs;
% re-plot with the pairs actually used (the caption says two surrogate-target pairs per benchmark).
% RESOLVED (T6, collaborator, 2026-11-04): legend matches the run config (configs/paper_experiment_v0504.yaml:151-157, beta_sweep.pairs_representative; renderer scripts/paper_experiment/render_beta_appendix.py); no re-plot needed; caption now names the additional pairs
% [ZJ] warning fix: `h' float specifier changed to `ht' (LaTeX already placed this float as [ht])
% \begin{figure}[h]
\begin{figure}[ht]
    \centering
    \includegraphics[width=\linewidth]{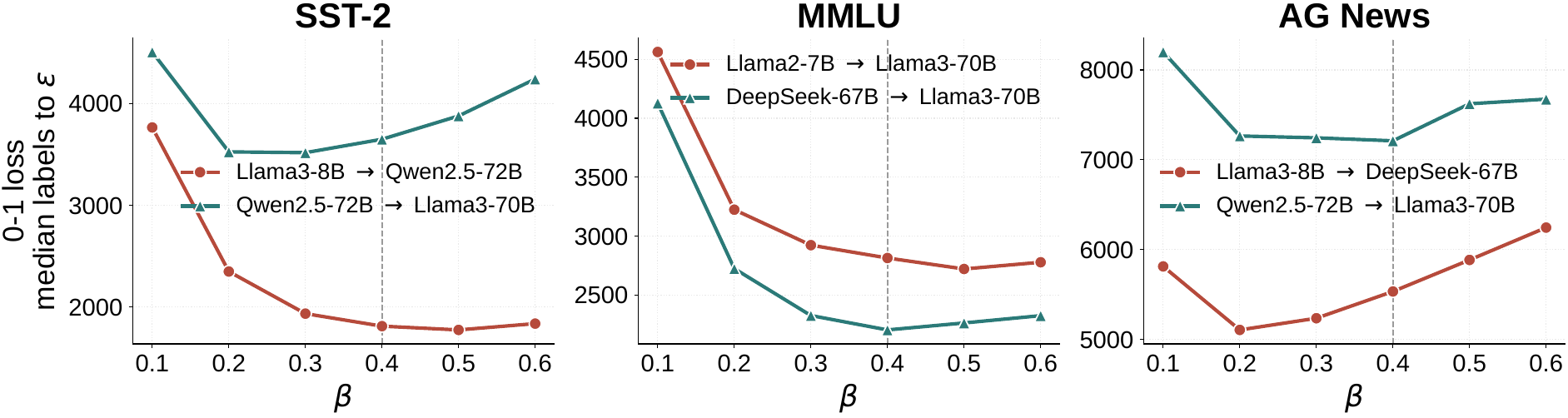}
%     \caption{Sensitivity of median samples-to-$\epsilon$ to the betting-magnitude floor $\beta$, on three benchmarks (SST-2, MMLU, AG News) and two surrogate-target pairs per benchmark.
    \caption{Sensitivity of median samples-to-$\epsilon$ to the sampling-probability floor $\beta$, on three benchmarks (SST-2, MMLU, AG News) and two surrogate-target pairs per benchmark.
    % [ZJ]
    % Each curve is the median number of evaluate samples used, over $50$ seeds at $\epsilon = 0.05$ under $0$-$1$ loss.
    Each curve is the median number of evaluated samples used, over $50$ seeds at $\epsilon = 0.05$ under $0$-$1$ loss.
    Three curves use pairs from Table~\ref{tab:main-config-summary} (Llama-3-8B $\to$ Qwen-2.5-72B on SST-2, Llama-2-7B $\to$ Llama-3-70B on MMLU, and Llama-3-8B $\to$ DeepSeek-67B on AG News); the other three curves use two additional pairs (Qwen-2.5-72B $\to$ Llama-3-70B on SST-2 and AG News, and DeepSeek-67B $\to$ Llama-3-70B on MMLU).
    The vertical dashed line marks the deployed value $\beta = 0.4$.
    Curves are L-shaped on the left and either flat or mildly increasing on the right; the deployed value is within $5$\% of the per-experimental configuration minimum on every curve.}
\label{fig:beta-sens}
\end{figure}

\textbf{Sensitivity to $\beta$.}
The parameter $\beta$ controls the lower bound on the sampling probabilities and hence prevents the sampling rule from assigning too little probability to any remaining sample. Fig.~\ref{fig:beta-sens} reports the results on six dataset-model configurations. The default value $\beta=0.4$ is at, or within $5\%$ of, the best value on every curve. When $\beta$ is too small, some sampling probabilities can become very small, which slows down the accumulation of evidence and leads to wider CIs. When $\beta$ is too large, the sampling rule is forced closer to uniform sampling, which can reduce the benefit of uncertainty-guided sampling. The best values mostly lie in $\beta\in[0.3,0.5]$, showing that the default choice is stable across benchmarks.

Overall, the sensitivity analysis shows that the performance of \ourmethod\ changes smoothly under these hyperparameter perturbations. The default choices provide stable performance across datasets without tuning on evaluation risks.

\appsection{Computational cost and break-even analysis}
\label{app:wallclock}

Tab.~\ref{tab:wallclock} reports the per-round wall-clock cost of \ourmethod\ and \evbaseline\ baseline on the same four (dataset, pair) experimental configurations used in App.~\ref{app:sensitivity}.
% The wall-clock figures exclude target-model inference, which is identical across methods at fixed acquisition order.
% The reported numbers therefore isolate the extra cost that \ourmethod\ pays for the e-process update, the LURE-weighted estimator of ~\eqref{eq:lure_estimator}, and the surrogate-driven proposal.
The wall-clock figures exclude target-model inference. The shared-output protocol computes target predictions over the full pool before sequential sampling, so this cost is common to the compared methods. The per-round measurements describe the remaining work reported for each method; they should not be interpreted as end-to-end evaluation times.

% Most of the per-round cost of \ourmethod\ comes from running the surrogate model on the remaining unevaluated samples, which takes time proportional to $N - t$.
% Because the surrogate is much smaller than the target, this extra work is cheap in absolute terms.
% Even on the slowest experimental configuration, mmlu/2, \ourmethod\ adds about $0.51$\,s per round.
% Combined with the $54$--$62$\,\% reduction in samples-to-$\epsilon$ reported in Tab.~\ref{tab:rq1-coverage}, the time saved by labelling fewer samples is much larger than the time added by the per-round overhead.
The surrogate scores are fixed before sequential evaluation. Constructing the proposal uses the scores of the remaining samples and does not require rerunning the surrogate model at every round. On the slowest listed configuration, mmlu/2, \ourmethod~takes about $0.51$ seconds per round, approximately $0.48$ seconds more than \evbaseline. Whether the $54$--$62\%$ reduction in target-loss queries in Table~\ref{tab:labels-to-eps} offsets this overhead depends on the cost of obtaining each loss, as quantified below.

% TODO (T6) [ZJ][M4Cr-W3]: same pair labels as Table 8 in the Table 9 caption (results/wallclock-app/tab_per_round.tex, line 6); fix together with Table 8.
% RESOLVED (T6, collaborator, 2026-11-04): labels match the run config (configs/paper_experiment_v0504.yaml:272-277, wallclock.cells); names unified; caption notes the additional pairs
\begin{table}[t]
\centering
\small
\caption{
Per-round wall-clock cost in seconds, averaged over $50$ seeds.
% Each column indexes a dataset-pair: \ding{172}: llama2-7b $\to$ Mixtral-8$\times$7B on SST-2, \ding{173}: llama2-7b $\to$ Qwen2.5-72B on MMLU, \ding{174}: llama3-8b $\to$ DeepSeek-67B on AG News, \ding{175}: llama3-8b $\to$ Llama-3-70B on SST-2. 
Each column indexes a dataset-pair: \ding{172}: Llama-2-7B $\to$ Mixtral-8$\times$7B on SST-2, \ding{173}: Llama-2-7B $\to$ Qwen-2.5-72B on MMLU, \ding{174}: Llama-3-8B $\to$ DeepSeek-67B on AG News, \ding{175}: Llama-3-8B $\to$ Llama-3-70B on SST-2. The pairs in columns \ding{172} and \ding{173} are not among the six pairs in Table~\ref{tab:main-config-summary}.
}
% \resizebox{0.6\linewidth}{!}{%
\begin{tabular}{lrrrr}
\toprule
\textbf{Method} & \textbf{SST-2}/\ding{172} & \textbf{MMLU}/\ding{173} & \textbf{AG\,News}/\ding{174} & \textbf{SST-2}/\ding{175} \\
\midrule
\ourmethod & 0.0388 & 0.5101 & 0.0963 & 0.0251 \\
\evbaseline & 0.0064 & 0.0305 & 0.0066 & 0.0018 \\
\bottomrule
\end{tabular}\label{tab:wallclock}
% }
\end{table}

\paragraph{Break-even analysis.}
Our primary empirical efficiency measure is target-loss-query efficiency, defined as the number of target losses that must be revealed before the prescribed precision is reached. On the labeled benchmarks considered here, each query corresponds to revealing the reference label needed to compute the target loss after the target-model output is available. In other applications, it may instead correspond to obtaining a human judgment or invoking a model judge. This quantity is distinct from the number of target-model forward passes and from end-to-end computational cost.

The practical acquisition rules evaluated in this paper use target predictive outputs over the full evaluation pool. In our offline timing protocol, these outputs are computed once and reused across the compared methods. We therefore treat the corresponding full-pool target-inference time, denoted by $T_{\mathrm{target}}$, as a common cost. Let $n_C$ and $n_E$ denote the numbers of target-loss queries required by \ourmethod~and \evbaseline~to reach the target CI width, respectively. Let $T_C$ and $T_E$ denote their remaining method-specific computation times accumulated up to their stopping times. For an end-to-end comparison, $T_C$ must include any method-specific full-pool surrogate inference as well as online computation, and $T_E$ must include the corresponding costs of \evbaseline. The numerical timing results below instead condition on precomputed outputs and surrogate scores, as detailed in Sec.~\ref{sec:re-new-cost}; in that online-only comparison, $T_C$ and $T_E$ contain only the measured online runtimes. Finally, let $p_L$ denote the time-equivalent marginal cost of revealing one target loss after the corresponding target output is available. The total times under this shared-output protocol are
\[
T^{\mathrm{tot}}_C(p_L)
=
T_{\mathrm{target}}+T_C+n_Cp_L,
\qquad
T^{\mathrm{tot}}_E(p_L)
=
T_{\mathrm{target}}+T_E+n_Ep_L.
\]
Because $T_{\mathrm{target}}$ is common, it does not affect the comparison. We therefore report the method-specific incremental times
\[
\widetilde T_C(p_L)=T_C+n_Cp_L,
\qquad
\widetilde T_E(p_L)=T_E+n_Ep_L.
\]
In all settings considered below, $n_E>n_C$ and $T_C>T_E$. The conditional break-even query cost is consequently
\[
p_L^\star
=
\frac{T_C-T_E}{n_E-n_C},
\]
and
\[
\widetilde T_C(p_L)<\widetilde T_E(p_L)
\quad\Longleftrightarrow\quad
p_L>p_L^\star.
\]
Thus, under the shared-output protocol, the query savings of \ourmethod~offset its additional surrogate and online computation when the marginal cost of revealing each target loss exceeds $p_L^\star$. More generally, if $n_E>n_C$ and $T_C\le T_E$, \ourmethod~has no greater incremental cost for every $p_L\ge0$, with a strict improvement whenever $T_C<T_E$ or $p_L>0$.

% We conduct this incremental timing study on SST-2, MMLU, and AG\,News using six random seeds. Each method is stopped when its CI width first reaches $\epsilon=0.05$, and all runs of \ourmethod~and \evbaseline~reach this width before exhausting the evaluation pool. Table~\ref{tab:break-even-cost} reports the conditional break-even query costs and the resulting incremental times under two illustrative values of $c_L$. Adding the common target-inference time $C_{\mathrm{target}}$ to both methods would change their absolute total times but would not change either the break-even point or their relative ordering.
% TODO (T7) [ZJ][M4Cr-W3]: "six random seeds per model pair" refers to the ten ordered pairs in App.~\ref{sec:re-new-cost}; see the TODO there.
% RESOLVED (T7, collaborator, 2026-11-04): text is correct: 10 pairs x 6 seeds (42-47); evidence: .rebuttal-cases/save-neurips2026-totalcost/analysis/audit.json, pair_summary.csv
The online timing study in Sec.~\ref{sec:re-new-cost} covers SST-2, MMLU, and AG News using six random seeds per model pair. Each method stops when its numerical CI width first reaches $\epsilon=0.05$, and all reported runs reach this width before pool exhaustion. Table~\ref{tab:break-even-cost} reports the break-even query costs and incremental times with outputs and surrogate scores precomputed. Adding a common preprocessing cost to both methods leaves their ordering unchanged; adding a method-specific surrogate-inference cost can change the break-even point.

\begin{table}[t]
\centering
\small
% \caption{Conditional break-even target-loss-query costs and incremental times required to reach CI width $\epsilon=0.05$, assuming that full-pool target outputs are precomputed and shared across methods. Break-even costs are reported as point estimate [$95\%$ interval] in milliseconds per query. Incremental times exclude the common target-inference time and are reported in seconds, with each entry showing \ourmethod{} / \evbaseline.}
\caption{Conditional break-even target-loss-query costs and online evaluation times required to reach width $\epsilon=0.05$, with target outputs and surrogate scores precomputed. Break-even costs are point estimates [$95\%$ intervals] in milliseconds per query. Times include measured online runtime and the stated query cost, exclude preprocessing, and are reported in seconds as \ourmethod{} / \evbaseline.}
\label{tab:break-even-cost}
\begin{tabular}{lccc}
\toprule
Dataset
&
% [ZJ] notation: C_target, C_C, C_E, c_L, c_f -> T_target, T_C, T_E, p_L, T_f; total times T_C(.) -> T^tot_C(.) (C_t is the CI, c the betting constant)
% $c_L^\star$ (ms/query)
% &
% $c_L=1$ ms/query
% &
% $c_L=10$ ms/query
$p_L^\star$ (ms/query)
&
$p_L=1$ ms/query
&
$p_L=10$ ms/query
\\
\midrule
SST-2
&
$0.92$ [$0.63$, $1.81$]
&
$10.83$ / $11.38$
&
$41.28$ / $107.55$
\\
MMLU
&
$2.24$ [$1.89$, $2.77$]
&
$16.12$ / $9.29$
&
$45.12$ / $87.86$
\\
AG\,News
&
$3.23$ [$2.37$, $4.37$]
&
$36.71$ / $16.66$
&
$96.78$ / $157.82$
\\
\bottomrule
\end{tabular}
\end{table}

% [ZJ] notation: C_target, C_C, C_E, c_L, c_f -> T_target, T_C, T_E, p_L, T_f; total times T_C(.) -> T^tot_C(.) (C_t is the CI, c the betting constant)
% At $c_L=1$ ms per query, \ourmethod~has lower incremental time on SST-2, whereas \evbaseline~has lower incremental time on MMLU and AG\,News. At $c_L=10$ ms per query, \ourmethod~has lower incremental time on all three datasets. Therefore, under the shared-output protocol, the comparison depends on the application-specific query cost: the reduction in queried target losses outweighs the additional computation when such queries are sufficiently costly, whereas \evbaseline~can be less expensive when they are inexpensive.
At $p_L=1$ ms per query, \ourmethod~has lower incremental time on SST-2, whereas \evbaseline~has lower incremental time on MMLU and AG\,News. At $p_L=10$ ms per query, \ourmethod~has lower incremental time on all three datasets. Therefore, under the shared-output protocol, the comparison depends on the application-specific query cost: the reduction in queried target losses outweighs the additional computation when such queries are sufficiently costly, whereas \evbaseline~can be less expensive when they are inexpensive.

% [ZJ] notation: C_target, C_C, C_E, c_L, c_f -> T_target, T_C, T_E, p_L, T_f; total times T_C(.) -> T^tot_C(.) (C_t is the CI, c the betting constant)
% When target outputs are not already available or shared, their method-specific inference costs should also be included. Let $C_{\mathrm{target}}^C$ and $C_{\mathrm{target}}^E$ denote the target-inference costs of \ourmethod~and \evbaseline, respectively. The corresponding total costs are
% \[
% T_C(c_L)
% =
% C_{\mathrm{target}}^C+C_C+n_Cc_L,
% \qquad
% T_E(c_L)
% =
% C_{\mathrm{target}}^E+C_E+n_Ec_L,
% \]
% which yield the general break-even point
% \[
% c_{L,\mathrm{general}}^\star
% =
% \frac{
% C_C-C_E+
% C_{\mathrm{target}}^C-C_{\mathrm{target}}^E
% }{
% n_E-n_C
% },
% \]
% whenever the numerator is positive. For example, if the practical acquisition rule of \ourmethod~requires target inference on all $N$ samples, while \evbaseline~performs target inference only on its $n_E$ queried samples, and $c_f$ denotes the average per-sample target-inference time, then
% \[
% C_{\mathrm{target}}^C=Nc_f,
% \qquad
% C_{\mathrm{target}}^E=n_Ec_f,
% \]
% and hence
% \[
% c_{L,\mathrm{general}}^\star
% =
% \frac{
% C_C-C_E+(N-n_E)c_f
When target outputs are not already available or shared, their method-specific inference costs should also be included. Let $T_{\mathrm{target}}^C$ and $T_{\mathrm{target}}^E$ denote the target-inference costs of \ourmethod~and \evbaseline, respectively. The corresponding total costs are
\[
T^{\mathrm{tot}}_C(p_L)
=
T_{\mathrm{target}}^C+T_C+n_Cp_L,
\qquad
T^{\mathrm{tot}}_E(p_L)
=
T_{\mathrm{target}}^E+T_E+n_Ep_L,
\]
which yield the general break-even point
\[
p_{L,\mathrm{general}}^\star
=
\frac{
T_C-T_E+
T_{\mathrm{target}}^C-T_{\mathrm{target}}^E
}{
n_E-n_C
},
\]
whenever the numerator is positive. For example, if the practical acquisition rule of \ourmethod~requires target inference on all $N$ samples, while \evbaseline~performs target inference only on its $n_E$ queried samples, and $T_f$ denotes the average per-sample target-inference time, then
\[
T_{\mathrm{target}}^C=NT_f,
\qquad
T_{\mathrm{target}}^E=n_ET_f,
\]
and hence
\[
p_{L,\mathrm{general}}^\star
=
\frac{
T_C-T_E+(N-n_E)T_f
}{
n_E-n_C
}.
\]
This expression separates the shared-output setting reported in Table~\ref{tab:break-even-cost} from deployment settings in which full-pool target inference constitutes an additional method-specific cost.

% The surrogate models used in our experiments are $7$--$8$B open-weight models without task-specific training or fine-tuning. Their measured inference time is included in $C_C$. Moreover, the surrogate-free variant \nosurr~requires fewer target-loss queries than \evbaseline~in the main experiments, indicating that the query-efficiency improvement does not rely entirely on full-pool surrogate inference. Under the acquisition rule evaluated here, however, \nosurr~may still use full-pool target predictive outputs; its result therefore concerns target-loss-query efficiency rather than target-inference or total-cost efficiency.
% The surrogate models used in our experiments are $7$--$8$B open-weight models without task-specific training or fine-tuning. Their measured inference time is included in $C_C$. Moreover, the surrogate-free variant \nosurr~requires fewer target-loss queries than \evbaseline~in the main experiments, indicating that the query-efficiency improvement does not rely entirely on full-pool surrogate inference. Its uniform sampler does not require target predictive outputs to select the next sample. However, the reported shared-output setting concerns target-loss-query efficiency rather than savings in target inference or total cost.
The surrogate models are $7$--$8$B open-weight models used without task-specific fine-tuning. Their inference cost must be added when it is incurred for the evaluation, since the reported online timing starts from precomputed scores. The uniform surrogate-free variant \nosurr~also requires fewer target-loss queries than \evbaseline~in the main experiments and does not require target outputs to select the next sample. These results concern query efficiency; end-to-end cost depends on the preprocessing and loss-acquisition costs of the deployment.

% A precise monetary comparison would require application-specific costs for target inference, surrogate inference, labels or judgments, and computational hardware. Accordingly, our primary empirical efficiency claim concerns the reduction in target-loss queries required to reach the prescribed precision. Table~\ref{tab:break-even-cost} characterizes when these savings also reduce incremental cost when target outputs are shared, while the general expression above covers settings with method-specific target-inference costs.
A precise monetary comparison would require application-specific costs for target inference, surrogate inference, labels or judgments, and computational hardware. Accordingly, our primary empirical efficiency claim concerns the reduction in target-loss queries required to reach the prescribed precision. Table~\ref{tab:break-even-cost} characterizes when these savings also reduce incremental cost conditional on precomputed target outputs and surrogate scores, while the general expression above covers settings with method-specific target-inference costs.

\appsection{Results for each surrogate-target model pair}
\label{app:per-configuration-efficiency}

% In addition to the dataset-level averages, we report the CI-width trajectories and evaluated samples-to-$\epsilon$ for each $(\textit{dataset}, \textit{surrogate}, \textit{target})$ configuration separately in Fig.~\ref{fig:per-pair-ci-width}. The results show that \ourmethod~reaches the target precision with fewer evaluated samples than \evbaseline~across all configurations, not only after averaging over surrogate-target pairs. The improvement is especially clear on AG$\,$News, where the gap between \ourmethod~and \evbaseline~is large for all six surrogate-target pairs, and \ourmethod~typically reaches $\epsilon=0.05$ at roughly half of the budget required by \evbaseline. Across the full grid, \oracleacq~acts as a lower envelope, while \ourmethod~consistently moves the CI-width trajectory closer to this oracle curve than the surrogate-free baselines.
In addition to the dataset-level averages, we report the CI-width trajectories and evaluated samples-to-$\epsilon$ for each $(\textit{dataset}, \textit{surrogate}, \textit{target})$ configuration separately in Fig.~\ref{fig:per-pair-ci-width}. The results show that \ourmethod~reaches the target precision with fewer evaluated samples than \evbaseline~across all configurations, not only after averaging over surrogate-target pairs. The improvement is especially clear on AG News, where the gap between \ourmethod~and \evbaseline~is large for all six surrogate-target pairs, and \ourmethod~typically reaches $\epsilon=0.05$ at roughly half of the budget required by \evbaseline. Across the full grid, \oracleacq~provides an idealized reference, while \ourmethod~consistently moves the CI-width trajectory closer to this oracle curve than the surrogate-free baselines.

% The comparison between \ourmethod~and \nosurr~shows the role of surrogate scoring. When the surrogate scores are more informative for the target risk, the gap between the two methods becomes larger; for example, in several Llama-3-8B surrogate configurations, \ourmethod~shrinks the CI visibly faster than \nosurr. In configurations where the surrogate-target relationship is weaker, the two curves are closer, indicating that the benefit of surrogate assistance is smaller but the anytime-valid guarantee is still preserved. This pattern is consistent with the variance bound in Thm.~\ref{thm:conf_sequen}: informative surrogate scores reduce the residual $\ell-\widetilde{\ell}_f$, which lowers the cumulative conditional variance and accelerates CI shrinkage. Thus, the per-configuration results support the same conclusion as the dataset-level summary: \ourmethod~improves efficiency through both uncertainty-guided sampling and surrogate-assisted approximation, while retaining certifiable coverage.
The comparison between \ourmethod~and \nosurr~shows the combined effects of surrogate assistance and uncertainty-guided sampling. When the surrogate scores are more informative for the target risk, the gap between the two methods becomes larger; for example, in several Llama-3-8B surrogate configurations, \ourmethod~shrinks the CI visibly faster than \nosurr. In configurations where the surrogate-target relationship is weaker, the two curves are closer, indicating that the gain from the full procedure is smaller but the anytime-valid guarantee is still preserved. This pattern is consistent with Thm.~\ref{thm:conf_sequen}, whose variance term depends on both the residuals and the sampling probabilities. These curves do not isolate either contribution; the controlled study in App.~\ref{app:surrogate_acquisition_quality} varies the two factors separately. Thus, the per-configuration results support the same conclusion as the dataset-level summary: \ourmethod~improves efficiency through both uncertainty-guided sampling and surrogate-assisted approximation, while retaining certifiable coverage.

\paragraph{A corner case.}
On the SST-2 evaluation of the Llama-3-8B $\to$ Mixtral-8$\times$7B row, the median samples-to-$\epsilon$ of \ourmethod\ ($2{,}422$) is slightly smaller than that of \oracleacq\ ($2{,}451$).
The difference is about $30$ evaluated samples and lies within the seed-to-seed spread of either method; it does not imply that \ourmethod~is better than \oracleacq.
% However, this case is still informative about what the oracle proposal in Thm.~\ref{thm:optimal} optimizes.
% [ZJ][M4Cr-W1]
% However, this case is still informative about the one-step objective of the $\beta=0$ oracle in Remark~\ref{rem:oracle-experiments}.
However, this case is still informative about the one-step objective of the oracle in Remark~\ref{rem:oracle-experiments}.
% The oracle proposal minimizes the conditional variance of the one-step signal $\widehat S_t$ given the past, but the reported metric is the stopping time determined by the entire CI trajectory.
% [ZJ][M4Cr-W1]
% The residual-proportional oracle minimizes the conditional variance of the one-step signal $\widehat S_t$ given the past, but the reported metric is the stopping time determined by the entire CI trajectory.
The oracle aims to reduce the conditional variance of the one-step signal $\widehat S_t$ given the past, but the reported metric is the stopping time determined by the entire CI trajectory.
These are closely related but not identical objectives.

On this pair, the target risk is small, with $R_N\approx0.104$, so the surrogate residuals are highly sparse and unevenly distributed across the pool.
In such a finite-sample regime, a proposal that is optimal for the next-step conditional variance need not produce a uniformly smaller stopping time on every random run, because the CI width also depends on the accumulated trajectory of signals, the scaling bounds, and the realized sampling path.
By contrast, \ourmethod's practical proposal is smoother and may occasionally produce a slightly shorter trajectory to the target precision on this specific configuration.
Therefore, this small reversal should be viewed as finite-sample variability rather than a violation of the oracle optimality result.
It also suggests that designing a label-free proposal that targets the full CI trajectory, rather than only the one-step variance, is an interesting direction for future work.

\begin{figure}
    \centering
    \includegraphics[width=\linewidth]{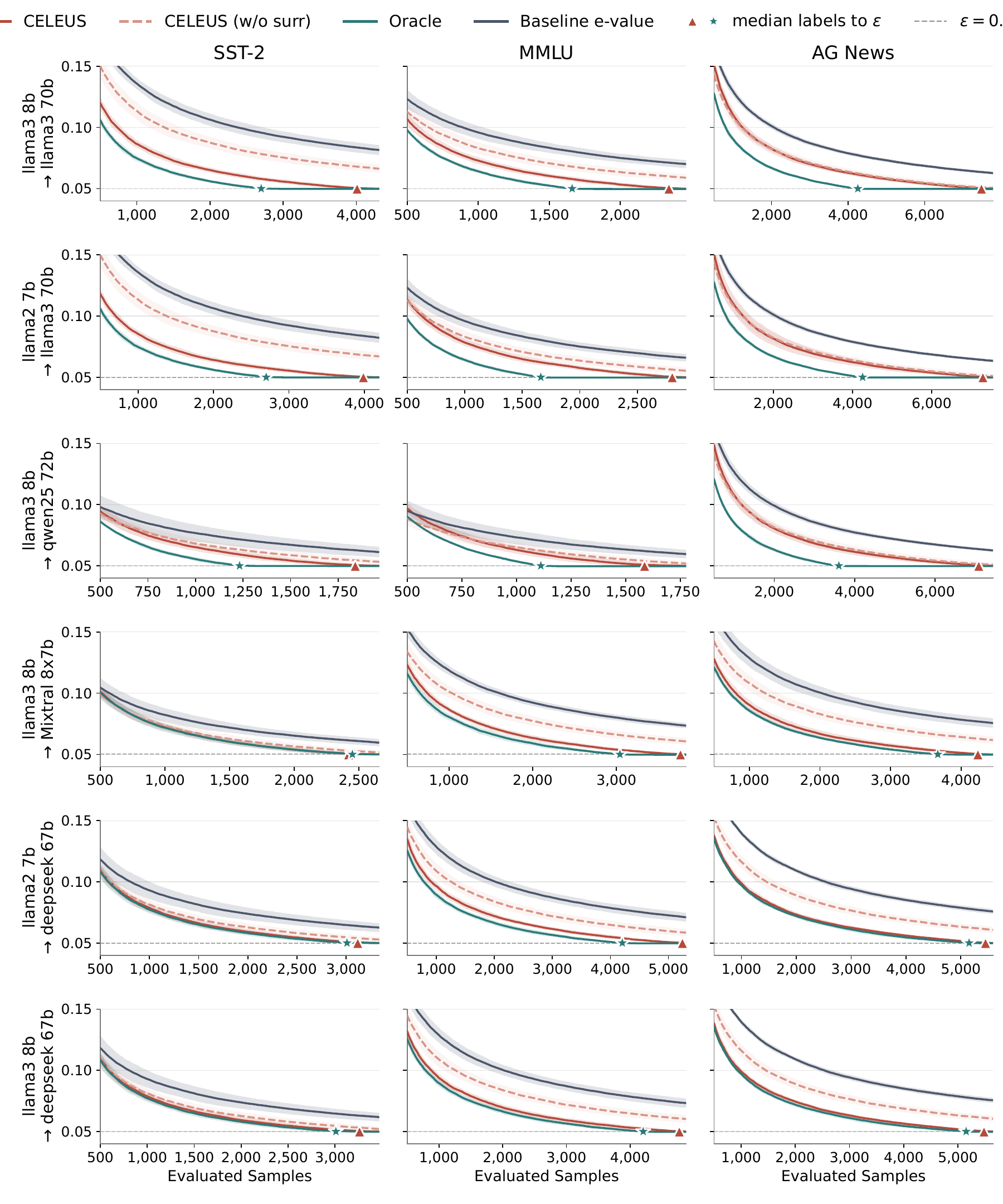}
%     \caption{Confidence interval width versus evaluated samples on each (\textit{dataset}, \textit{surrogate}, \textit{target}) experimental configuration. Rows index the six surrogate-target pairs of App.~\ref{app:models}, and columns index the three benchmarks. Each curve is the mean width over $50$ seeds. The dashed line marks the target precision $\epsilon = 0.05$, and triangles mark the median samples-to-$\epsilon$. The ordering \ourmethod\ $\leq$ \nosurr\ $\leq$ \evbaseline\ holds on every panel, and the gap between \ourmethod\ and \nosurr\ tracks the role of surrogate scoring.}
    \caption{Confidence interval width versus evaluated samples on each (\textit{dataset}, \textit{surrogate}, \textit{target}) experimental configuration. Rows index the six surrogate-target pairs of App.~\ref{app:models}, and columns index the three benchmarks. Each curve is the mean width over $50$ seeds. The dashed line marks the target precision $\epsilon = 0.05$, and triangles mark the median samples-to-$\epsilon$. The ordering \ourmethod\ $\leq$ \nosurr\ $\leq$ \evbaseline\ holds on every panel, and the gap between \ourmethod\ and \nosurr\ reflects the combined change in surrogate assistance and sampling.}
    \label{fig:per-pair-ci-width}
\end{figure}

% [ZJ]
% \section{Additional Empirical Results}
\appsection{Additional Empirical Results}
\label{sec:re-new-results}

% [ZJ]
% We examine label efficiency against an active evaluation baseline, extend evaluation to open-ended generation, and isolate the effects of surrogate accuracy and acquisition quality. We then study the choice of a cross-entropy loss bound and quantify when label savings translate into lower evaluation cost. Throughout, target precision refers to confidence-interval width: an evaluation stops at the first round at which the interval width is at most $\epsilon$.
We extend evaluation to open-ended generation. We then study the choice of a cross-entropy loss bound and quantify when label savings translate into lower evaluation cost. Throughout, target precision refers to confidence-interval width: an evaluation stops at the first round at which the interval width is at most $\epsilon$.

\appsubsection{Open-Ended Safety Evaluation}
\label{sec:re-new-safety}

We evaluate free-form generations on RealToxicityPrompts~\cite{Gehman:Gururangan:Sap:Choi:Smith2020}\footnote{\url{https://huggingface.co/datasets/allenai/real-toxicity-prompts}}, using a pool of $N=99{,}442$ prompts and toxicity scores in $[0,1]$. The target and surrogate are Qwen2.5-72B-Instruct and Qwen2.5-7B-Instruct~\cite{Yang:Yang:Zhang:Hui:Zheng:Yu2024}, respectively; each generates one continuation per prompt with greedy decoding and at most $64$ new tokens. ToxicBERT~\cite{Hanu:Thewlis:Unitary2020}\footnote{\url{https://huggingface.co/unitary/toxic-bert}} scores each prompt concatenated with a generated continuation: the risk value $\ell_j$ is the score with the target continuation, and the surrogate risk score $\tilde\ell_f(\x_j)$ is the score with the surrogate continuation. \ourmethod\ samples with the score $s_t(j)=\tilde\ell_f(\x_j)(1-\tilde\ell_f(\x_j))$, which favors prompts whose surrogate score is close to $1/2$. The full-pool mean toxicity is $R_N=0.04667$. We use $\alpha=0.05$ with $\alpha_1=\alpha_2=0.025$, $[L,U]=[0,1]$, and $K=2000$ grid points. We compare \ourmethod, its no-surrogate variant, and \evbaseline\ over ten matched seeds (42--51), with a maximum budget of 5,000 evaluations. \ourmethod\ uses $\beta=0.8$. This value was selected after running \ourmethod\ with $\beta\in\{0.1,0.2,0.4,0.6,0.8\}$ on the same pool and seeds; among these values, it required the fewest evaluations to reach $\epsilon=0.05$. Unlike the default $\beta=0.4$ in App.~\ref{app:sensitivity}, it is therefore not fixed in advance, and Table~\ref{tab:re-new-safety} may favor \ourmethod. The no-surrogate variant sets surrogate predictions to zero and samples uniformly; \evbaseline\ applies the baseline e-process to uniformly sampled losses.

\begin{table}[htbp]
\centering
\caption{Open-ended safety evaluation on RealToxicityPrompts. Entries are median evaluations to reach width $\epsilon=0.05$ over ten matched seeds.}
\label{tab:re-new-safety}
\begin{tabular}{@{}lr@{}}
\toprule
Method & Median evaluations \\
\midrule
\ourmethod & \textbf{529.0} \\
% [ZJ]
% \ourmethod without surrogate assistance & 765.0 \\
\nosurr & 765.0 \\
\evbaseline\ & 745.5 \\
\bottomrule
\end{tabular}
\end{table}

\ourmethod\ requires 29.0\% fewer evaluations than \evbaseline\ and 30.8\% fewer than its no-surrogate variant to reach the same precision (Table~\ref{tab:re-new-safety}). Its median estimate at stopping is 0.04708, close to the full-pool mean of 0.04667. The label savings therefore extend to evaluation of open-ended generations with a continuous safety score, in addition to classification and multiple-choice tasks.

\appsubsection{Sensitivity to a Pre-Specified Cross-Entropy Bound}
\label{sec:re-new-ce-bound}

% A valid bound for sequential evaluation need not equal the largest realized loss. We test this by separating dataset preprocessing from the bound supplied to the evaluation algorithm. Preprocessing removes examples with cross-entropy exceeding a pre-specified cutoff $C=3$ and freezes the retained pool. The retained losses are unchanged. We then set $U\in\{3.0,3.5,4.0\}$ before sequential evaluation. Every tested $U$ bounds the losses in this fixed pool because $U\geq C$; increasing $U$ changes the confidence-sequence bound without changing the evaluation examples or their losses. The evaluated risk is the mean loss for this curated population.
A valid bound for sequential evaluation need not equal the largest realized loss. We test this by separating dataset preprocessing from the bound supplied to the evaluation algorithm. Preprocessing removes examples with cross-entropy exceeding a pre-specified cutoff $C=3$ and freezes the retained pool. The retained losses are unchanged. We then set $U\in\{3.0,3.5,4.0\}$ before sequential evaluation. Every tested $U$ bounds the losses in this fixed pool because $U\geq C$; increasing $U$ changes the confidence-sequence bound without changing the evaluation examples or their losses. The immediate evaluation target is the mean loss over the retained pool. Filtering changes the target relative to the original benchmark and requires inspecting every loss before sequential evaluation. The reported query counts cover the subsequent sequential phase; they do not include this preprocessing. A population interpretation additionally requires assumptions on how the retained examples represent the filtered population.

% [ZJ]
% The study covers SST-2, MMLU, and AG News with four surrogate--target pairs. Two pairs evaluate Llama-3-70B using Llama-3-8B and Llama-2-7B as surrogates; the other two use Llama-3-8B to evaluate Qwen2.5-72B and Mixtral-8$\times$7B. We compare \ourmethod and \evbaseline\ over 50 seeds per setting at $\epsilon\in\{0.05,0.07,0.10\}$, using nominal miscoverage $\alpha=0.05$ with equal allocations to the finite-pool interval and population correction.
The study covers SST-2, MMLU, and AG News with four surrogate--target pairs. Two pairs evaluate Llama-3-70B using Llama-3-8B and Llama-2-7B as surrogates; the other two use Llama-3-8B to evaluate Qwen2.5-72B and Mixtral-8$\times$7B. We compare \ourmethod\ and \evbaseline\ over 50 seeds per setting at $\epsilon\in\{0.05,0.07,0.10\}$, using nominal miscoverage $\alpha=0.05$ with equal allocations to the finite-pool interval and population correction.

\begin{table}[htbp]
\centering
\small
\setlength{\tabcolsep}{3pt}
\caption{Labels to target precision under pre-specified CE bounds. For each model pair we take the median over 50 seeds, then average these medians equally across the four pairs. Brackets report the corresponding average percentage of the retained pool, computed within each pair before averaging. The preprocessing cutoff is fixed at $C=3$ throughout.}
\label{tab:re-new-ce-bound}
\begin{tabular}{@{}lrlrrr@{}}
\toprule
Dataset & $U$ & Method & $\epsilon=0.05$ & $\epsilon=0.07$ & $\epsilon=0.10$ \\
\midrule
SST-2 & 3.0 & \ourmethod & 27,245.6 [40.2\%] & 6,579.0 [9.7\%] & 2,101.9 [3.1\%] \\
 &  & \evbaseline\ & 46,158.5 [68.1\%] & 24,675.5 [36.4\%] & 5,715.9 [8.4\%] \\
SST-2 & 3.5 & \ourmethod & 42,063.8 [62.0\%] & 9,140.8 [13.5\%] & 2,580.6 [3.8\%] \\
 &  & \evbaseline\ & 53,982.3 [79.6\%] & 33,275.4 [49.1\%] & 6,375.4 [9.4\%] \\
SST-2 & 4.0 & \ourmethod & 55,629.0 [82.0\%] & 13,013.0 [19.2\%] & 3,142.1 [4.6\%] \\
 &  & \evbaseline\ & 67,760.5 [99.8\%] & 37,380.4 [55.1\%] & 7,523.6 [11.1\%] \\
\midrule
MMLU & 3.0 & \ourmethod & 10,823.6 [9.5\%] & 3,498.8 [3.1\%] & 1,514.3 [1.3\%] \\
 &  & \evbaseline\ & 48,332.4 [42.5\%] & 11,148.9 [9.8\%] & 3,114.3 [2.7\%] \\
MMLU & 3.5 & \ourmethod & 16,180.0 [14.2\%] & 4,373.5 [3.8\%] & 1,778.9 [1.6\%] \\
 &  & \evbaseline\ & 73,613.0 [64.7\%] & 13,391.8 [11.8\%] & 3,425.4 [3.0\%] \\
MMLU & 4.0 & \ourmethod & 25,915.9 [22.8\%] & 5,459.1 [4.8\%] & 2,076.0 [1.8\%] \\
 &  & \evbaseline\ & 99,151.9 [87.1\%] & 16,797.5 [14.8\%] & 3,776.8 [3.3\%] \\
\midrule
AG News & 3.0 & \ourmethod & 24,986.6 [23.7\%] & 6,956.0 [6.6\%] & 2,363.4 [2.3\%] \\
 &  & \evbaseline\ & 84,662.4 [80.6\%] & 20,540.4 [19.5\%] & 5,583.9 [5.3\%] \\
AG News & 3.5 & \ourmethod & 36,903.0 [35.1\%] & 8,623.3 [8.2\%] & 2,700.6 [2.6\%] \\
 &  & \evbaseline\ & 89,845.1 [85.9\%] & 24,497.6 [23.2\%] & 5,930.0 [5.6\%] \\
AG News & 4.0 & \ourmethod & 57,291.9 [54.7\%] & 10,971.8 [10.4\%] & 3,123.8 [3.0\%] \\
 &  & \evbaseline\ & 102,529.1 [99.0\%] & 30,459.1 [28.9\%] & 6,548.6 [6.2\%] \\
\bottomrule
\end{tabular}
\end{table}

% [ZJ] typo: \evbaseline\. -> \evbaseline. (\. is the dot accent: the period was dropped and the next letter got a dot)
% \ourmethod\ requires fewer labels in all 27 dataset--bound--precision summaries in Table~\ref{tab:re-new-ce-bound}. At $U=3$ and $\epsilon=0.05$, the mean pair-median label counts are 27,245.6 versus 46,158.5 on SST-2, 10,823.6 versus 48,332.4 on MMLU, and 24,986.6 versus 84,662.4 on AG News. All trajectories reach each target before exhausting the retained pool. Larger bounds increase the evaluation budget: for example, increasing $U$ from 3 to 4 raises the SST-2 count at $\epsilon=0.05$ from 27,245.6 to 55,629.0 for \ourmethod\ and from 46,158.5 to 67,760.5 for \evbaseline\. Thus, a tighter valid bound improves efficiency, while \ourmethod\ retains a label advantage under the more conservative bounds.
\ourmethod\ requires fewer labels in all 27 dataset--bound--precision summaries in Table~\ref{tab:re-new-ce-bound}. At $U=3$ and $\epsilon=0.05$, the mean pair-median label counts are 27,245.6 versus 46,158.5 on SST-2, 10,823.6 versus 48,332.4 on MMLU, and 24,986.6 versus 84,662.4 on AG News. All trajectories reach each target before exhausting the retained pool. Larger bounds increase the evaluation budget: for example, increasing $U$ from 3 to 4 raises the SST-2 count at $\epsilon=0.05$ from 27,245.6 to 55,629.0 for \ourmethod\ and from 46,158.5 to 67,760.5 for \evbaseline. Thus, a tighter valid bound improves efficiency, while \ourmethod\ retains a label advantage under the more conservative bounds.

The aggregate advantage coexists with variation across model pairs: \ourmethod\ uses fewer labels in 102 of the 108 pair--bound--precision comparisons. The six reversals occur for Llama-3-8B$\to$Mixtral-8$\times$7B on SST-2 at $\epsilon=0.07$ and $0.10$, across the three values of $U$. The complete grid therefore shows both the broad efficiency gain and the settings in which the baseline stops earlier.

% [ZJ]
% \subsection{Computational Cost and the Break-Even Label Price}
\appsubsection{Computational Cost and the Break-Even Label Price}
\label{sec:re-new-cost}

% TODO (T7) [ZJ][M4Cr-W3]: "ten ordered model pairs" here, "ten pairs" in the Table~\ref{tab:re-new-cost} caption and "30 dataset--pair settings" below
% conflict with the six pairs in Table~\ref{tab:main-config-summary}. List the pairs actually used, or correct the counts.
% Table~\ref{tab:re-new-cost} repeats the numbers of Table~\ref{tab:break-even-cost}; after the fix, consider keeping only one of them.
% RESOLVED (T7, collaborator, 2026-11-04): counts are correct (3 datasets x 10 pairs x 6 seeds x 2 methods = 360 runs, 30 settings; audit.json input_json_count=360, all reached eps); pairs now listed; Tables 10 and 13 come from the same run set (dataset_summary_primary.csv) - which table to keep is ZJ's decision
% Label efficiency translates into computational savings when the cost of obtaining the saved labels exceeds the additional algorithmic work. We measure each method in isolation until its first crossing of width $\epsilon=0.05$, covering three datasets, ten ordered model pairs, and six seeds per pair. All 360 runs reach the target before pool exhaustion. Timing begins with precomputed losses and surrogate scores and covers the online sampling, estimation, and confidence-sequence updates through stopping.
Label efficiency translates into computational savings when the cost of obtaining the saved labels exceeds the additional algorithmic work. We measure each method in isolation until its first crossing of width $\epsilon=0.05$, covering three datasets, ten ordered surrogate-target pairs, and six seeds per pair. The ten pairs are the six pairs in Table~\ref{tab:main-config-summary} and four additional pairs: Llama-2-7B $\to$ Mixtral-8$\times$7B, Llama-2-7B $\to$ Qwen-2.5-72B, DeepSeek-67B $\to$ Llama-3-70B, and Qwen-2.5-72B $\to$ Llama-3-70B. All 360 runs reach the target before pool exhaustion. Timing begins with precomputed losses and surrogate scores and covers the online sampling, estimation, and confidence-sequence updates through stopping.

% [ZJ] notation: method index m -> k; n_m, t_m, S_m, c_L, \mathcal C_m -> n_k, T_k, T^setup_k, p_L, \widetilde T_k as in App.~\ref{app:wallclock} (m, t, \bar S_t, c, C_t are used in the main text); typo: \evbaseline\. -> \evbaseline. (\. is the dot accent: the period was dropped and the next letter got a dot)
% Let $n_m$ be the number of labels and $t_m$ the measured online runtime for method $m\in\{\mathrm C,\mathrm E\}$, denoting \ourmethod\ and \evbaseline\. If one label costs $c_L$ seconds-equivalent, the online evaluation cost is
% \begin{equation}
% \mathcal C_m(c_L)=t_m+n_m c_L.
% \label{eq:re-new-online-cost}
% \end{equation}
% For equal-weight means across model pairs and seeds, the break-even price is
% \begin{equation}
% c_L^\star=\frac{\bar t_{\mathrm C}-\bar t_{\mathrm E}}{\bar n_{\mathrm E}-\bar n_{\mathrm C}},\qquad \bar n_{\mathrm E}>\bar n_{\mathrm C}.
% \label{eq:re-new-break-even}
% \end{equation}
% \ourmethod\ has lower mean online cost when $c_L>c_L^\star$. Table~\ref{tab:re-new-cost} reports this threshold and evaluates both methods at label costs of 1 and 10 milliseconds.
Let $n_k$ be the number of labels and $T_k$ the measured online runtime for method $k\in\{\mathrm C,\mathrm E\}$, denoting \ourmethod\ and \evbaseline. If one label costs $p_L$ seconds-equivalent, the online evaluation cost is
\begin{equation}
\widetilde T_k(p_L)=T_k+n_k p_L.
\label{eq:re-new-online-cost}
\end{equation}
For equal-weight means across model pairs and seeds, the break-even price is
\begin{equation}
p_L^\star=\frac{\bar T_{\mathrm C}-\bar T_{\mathrm E}}{\bar n_{\mathrm E}-\bar n_{\mathrm C}},\qquad \bar n_{\mathrm E}>\bar n_{\mathrm C}.
\label{eq:re-new-break-even}
\end{equation}
\ourmethod\ has lower mean online cost when $p_L>p_L^\star$. Table~\ref{tab:re-new-cost} reports this threshold and evaluates both methods at label costs of 1 and 10 milliseconds.

\begin{table}[htbp]
\centering
\small
\setlength{\tabcolsep}{5pt}
\caption{Online cost to width $\epsilon=0.05$. Break-even prices are in milliseconds per label; brackets give 95\% hierarchical bootstrap intervals (20,000 replicates, resampling model pairs and then paired seeds). Total costs are in seconds and include measured algorithm time plus the indicated label-cost term. Each total is an equal-weight mean across ten pairs and six seeds.}
\label{tab:re-new-cost}
\begin{tabular}{@{}lrrr@{}}
\toprule
 & Break-even price & \multicolumn{2}{c}{Total: \ourmethod\ / \evbaseline\ (s)} \\
\cmidrule(lr){3-4}
% [ZJ] notation: method index m -> k; n_m, t_m, S_m, c_L, \mathcal C_m -> n_k, T_k, T^setup_k, p_L, \widetilde T_k as in App.~\ref{app:wallclock} (m, t, \bar S_t, c, C_t are used in the main text)
% Dataset & $c_L^\star$ [95\% interval] & $c_L=1$ ms & $c_L=10$ ms \\
Dataset & $p_L^\star$ [95\% interval] & $p_L=1$ ms & $p_L=10$ ms \\
\midrule
SST-2 & 0.92 [0.63, 1.81] & 10.83 / 11.38 & 41.28 / 107.55 \\
MMLU & 2.24 [1.89, 2.77] & 16.12 / 9.29 & 45.12 / 87.86 \\
AG News & 3.23 [2.37, 4.37] & 36.71 / 16.66 & 96.78 / 157.82 \\
\bottomrule
\end{tabular}
\end{table}

The break-even prices range from 0.92 to 3.23 milliseconds per label. At 1 millisecond per label, \ourmethod\ is already cheaper on SST-2, while \evbaseline\ remains cheaper on MMLU and AG News. At 10 milliseconds per label, \ourmethod\ has lower cost on all three datasets, reducing the mean totals from 107.55 to 41.28 seconds on SST-2, from 87.86 to 45.12 seconds on MMLU, and from 157.82 to 96.78 seconds on AG News. The same ordering holds for the six-seed mean of each of the 30 dataset--pair settings at this price. These results quantify the regime in which the saved evaluations outweigh the online computational overhead.

% [ZJ] notation: method index m -> k; n_m, t_m, S_m, c_L, \mathcal C_m -> n_k, T_k, T^setup_k, p_L, \widetilde T_k as in App.~\ref{app:wallclock} (m, t, \bar S_t, c, C_t are used in the main text)
% When deployment also requires method-specific preparation, let $S_m$ denote its one-time cost, including surrogate inference where applicable. The complete cost becomes $S_m+t_m+n_m c_L$, and the corresponding mean break-even price is
% \begin{equation}
% c_{L,\mathrm{setup}}^\star=
% \frac{\bar S_{\mathrm C}-\bar S_{\mathrm E}+\bar t_{\mathrm C}-\bar t_{\mathrm E}}
% {\bar n_{\mathrm E}-\bar n_{\mathrm C}}.
% \label{eq:re-new-setup-break-even}
% \end{equation}
% The timing study measures $t_m$; $S_m$ captures deployment-dependent preparation. This separates the measured online trade-off from the additional cost of obtaining surrogate scores, and makes the benefit of reusing existing scores explicit.
When deployment also requires method-specific preparation, let $T^{\mathrm{setup}}_k$ denote its one-time cost, including surrogate inference where applicable. The complete cost becomes $T^{\mathrm{setup}}_k+T_k+n_k p_L$, and the corresponding mean break-even price is
\begin{equation}
p_{L,\mathrm{setup}}^\star=
\frac{\bar T^{\mathrm{setup}}_{\mathrm C}-\bar T^{\mathrm{setup}}_{\mathrm E}+\bar T_{\mathrm C}-\bar T_{\mathrm E}}
{\bar n_{\mathrm E}-\bar n_{\mathrm C}}.
\label{eq:re-new-setup-break-even}
\end{equation}
The timing study measures $T_k$; $T^{\mathrm{setup}}_k$ captures deployment-dependent preparation. This separates the measured online trade-off from the additional cost of obtaining surrogate scores, and makes the benefit of reusing existing scores explicit.

\appsection{Limitations and discussion}
\label{app:limitations}

We discuss the scope of the theoretical guarantees and several practical considerations for applying \ourmethod. The confidence sets in \eqref{eq:M_t} are defined by testing every candidate value in the continuous interval $[L,U]$. In practice, we evaluate the $e$-processes on a finite grid and return the interval between the smallest and largest unrejected grid points. Theorems~\ref{thm:conf_sequen} and~\ref{thm:coverage_convergence} apply to the exact confidence sets obtained from continuous test inversion, while the finite grid provides a computationally tractable approximation. Increasing the grid density or refining the grid near the estimated boundaries provides finer numerical resolution. App.~\ref{app:cell_inversion_discussion} provides a conservative alternative: it rejects whole grid cells using valid lower bounds and returns an outer enclosure of the exact confidence set. This enclosure inherits the coverage guarantee when the stated rounding rules are followed. This alternative is separate from the grid-point implementation used in the experiments. For grid-point inversion alone, a narrow unrejected region may fall between grid points and be missed, while the interval spanning the extreme unrejected grid points may include candidate values that would be rejected by exact inversion. Thus, the grid-based interval may be either narrower or wider than the exact confidence set. Assumption~\ref{ass:loss_bound} also requires deterministic bounds $L$ and $U$ to be specified before observing the pool losses. This requirement is immediate for naturally bounded metrics such as $0$--$1$ loss. For cross-entropy loss, a clipping level fixed before observing pool losses, as described in Remark~\ref{rem:cross_entropy}, makes the guarantees concern the mean clipped loss. The empirical pool-maximum bound in App.~\ref{app:ce-results} does not establish this population-level boundedness condition. The cutoff study in Sec.~\ref{sec:re-new-ce-bound} instead evaluates a retained pool after loss-based filtering and therefore changes the evaluation target. A larger clipping level retains more information from unusually large losses but increases the bounded range, whereas a smaller clipping level may yield tighter intervals but modifies the quantity being estimated. The clipping level should therefore be selected according to the evaluation objective before observing the pool losses.

% The finite-pool confidence sequence targets $R_N$, the average risk over the given evaluation pool, and remains valid under adaptive sampling without replacement. Extending this guarantee to the population risk $R$ requires the evaluation pool to consist of independent and identically distributed samples from the intended target population. Under this condition, the Hoeffding correction $\Delta_N(\alpha_2)$ controls the difference between $R_N$ and $R$. If a benchmark is instead treated as a fixed or deliberately curated collection, the finite-pool guarantee for $R_N$ remains valid and provides the natural interpretation of the result. A population-level guarantee would then require a correction suited to the process used to construct the benchmark. Because $\Delta_N(\alpha_2)$ depends on the full pool size $N$, rather than the number of evaluated samples $t$, evaluating more samples from the same pool does not reduce this correction. It becomes smaller only when the evaluation pool itself becomes larger and may therefore remain an appreciable part of the final interval for a small pool.
The finite-pool confidence sequence targets $R_N$, the average risk over the given evaluation pool, and remains valid under adaptive sampling without replacement. Extending this guarantee to the population risk $R$ requires the evaluation pool to consist of independent and identically distributed input-label pairs from the intended target population, with the target model fixed before the pool is drawn. Under this condition, the Hoeffding correction $\Delta_N(\alpha_2)$ controls the difference between $R_N$ and $R$. If a benchmark is instead treated as a fixed or deliberately curated collection, the finite-pool guarantee for $R_N$ remains valid and provides the natural interpretation of the result. A population-level guarantee would then require a correction suited to the process used to construct the benchmark. Because $\Delta_N(\alpha_2)$ depends on the full pool size $N$, rather than the number of evaluated samples $t$, evaluating more samples from the same pool does not reduce this correction. It becomes smaller only when the evaluation pool itself becomes larger and may therefore remain an appreciable part of the final interval for a small pool.

The validity of \ourmethod~does not require the surrogate or acquisition proxy to achieve any prescribed level of accuracy. The stated finite-pool theorems for practical sampling require that the surrogate scores are fixed before evaluation, acquisition scores depend only on observed information, the boundedness conditions hold, the proposal probabilities satisfy Assumption~\ref{ass:beta}, and the actual sampling probabilities are used in the inverse-probability correction. Nevertheless, their quality affects efficiency. More accurate surrogate scores can produce smaller residuals in the inferential signal, while acquisition scores that better approximate the residual magnitudes can direct evaluations toward more informative samples, as analyzed in App.~\ref{app:surrogate_acquisition_quality}. The proposal in Theorem~\ref{thm:optimal} is a positive-floor oracle benchmark that minimizes the conditional variance at a given round, conditional on the complete pool, preprocessing information, and previous selections. Because the true residuals are unknown before evaluation, practical strategies replace them with label-free uncertainty or disagreement scores, whose effectiveness may vary across tasks. The theorem therefore motivates these strategies but does not claim that they minimize the total number of evaluations over the entire sampling process or at the final stopping time. Surrogate scoring and online inference also incur computational costs. Reductions in target evaluations should therefore primarily be interpreted as improvements in label efficiency. Whether they also reduce total cost depends on the relative costs of surrogate inference, online computation, and target labels, as quantified by the break-even analysis in App.~\ref{app:wallclock}.

% [M4Cr-L(iv)]
The empirical evaluation covers classification and multiple-choice tasks and one open-ended generation task with a continuous safety score (App.~\ref{sec:re-new-safety}). Evaluation settings based on human-preference judgments are not tested in this paper.

Finally, several design choices affect the tightness of the resulting CIs. Time-uniform inference allows the intervals to be monitored repeatedly and the evaluation to be stopped based on the observed data without invalidating the coverage guarantee. The probability floor prevents inverse-probability corrections from becoming excessively large, while predictable scaling bounds and betting fractions ensure that each $e$-process update remains valid. Splitting the error budget between $\alpha_1$ and $\alpha_2$ separately controls uncertainty from evaluating only part of the pool and uncertainty from using a finite pool to represent the population. These safeguards provide finite-sample guarantees under broad conditions, but they may produce wider intervals than methods designed for a prespecified evaluation time or methods that impose stronger distributional assumptions. The width result in Theorem~\ref{thm:conf_sequen} should therefore be understood as a general upper bound on how the uncertainty decreases, rather than as an exact prediction of the number of evaluations required to reach the stopping criterion. The numerical error introduced by the finite grid is separate from this theoretical width bound.
% [TMnR-W3][BjfK-Q2]
In particular, the observed miscoverage below $\alpha$ in Tabs.~\ref{tab:rq1-coverage} and~\ref{tab:01-coverage} reflects two sources of conservativeness that require no additional distributional assumptions: the time-uniform $e$-process construction~\cite{Waudby-Smith:Ramdas2024} and the finite-pool-to-population correction based on Hoeffding's inequality~\cite{Hoeffding1963}. Compared with the $O(t^{-1/2})$ width typical of a CI designed for a prespecified evaluation time, the additional logarithmic factors in Thm.~\ref{thm:conf_sequen} are the price of validity at every round and under data-dependent stopping. Tighter constructions that retain finite-sample, anytime-valid guarantees are left for future work.

\appsubsection{Discussion: conservative numerical inversion}
\label{app:cell_inversion_discussion}

This subsection discusses an optional numerical alternative to the grid-point inversion in Alg.~\ref{alg:celeus}. It is not part of that algorithm, and we do not claim experimental results for this alternative. The test statistic remains $\cK_t^\pm(\upsilon)$ in \eqref{eq:hedhedCP}, and the continuous confidence set remains \eqref{eq:M_t}. The quantity $\underline{\cK}_{t,j}^{\mathrm{lo}}$ below is only a computational bound used to check a range of candidate values.

% [ZJ] notation: h, x, z removed (h denotes a model in Remark~\ref{rem:0-1}, z_t is the observed loss in Alg.~\ref{alg:celeus})
% Let $L=\upsilon_1<\cdots<\upsilon_K=U$ be a fixed equally spaced grid, with width $h=(U-L)/(K-1)$. A grid cell is the range $[\upsilon_j,\upsilon_{j+1}]$ between adjacent grid points. Write $\cK_{t,+}$ and $\cK_{t,-}$ for the two products in \eqref{eq:hedhedCP}, and define
% \begin{equation}\label{eq:cell_capital_bound}
% \underline{\cK}_{t,j}
% =
% \theta\cK_{t,+}(\upsilon_{j+1})
% +(1-\theta)\cK_{t,-}(\upsilon_j).
% \end{equation}
% With the fractions in Remark~\ref{rem:betting-fractions}, $\cK_{t,+}(\upsilon)$ cannot increase as $\upsilon$ increases, whereas $\cK_{t,-}(\upsilon)$ cannot decrease. To see this, write $x=\bar\upsilon_t$ and $z=\bar S_t$. Before clipping, the factors are $1+\tilde\lambda_t^+(z-x)$ and $1-\tilde\lambda_t^-(z-x)$; after clipping, they are $(z+c)/(x+c)$ and $(1-z+c)/(1-x+c)$, respectively. Each positive factor has the stated direction of change, and its branches agree at the clipping point. The products therefore have the same property. Consequently,
Let $L=\upsilon_1<\cdots<\upsilon_K=U$ be a fixed equally spaced grid with spacing $(U-L)/(K-1)$. A grid cell is the range $[\upsilon_j,\upsilon_{j+1}]$ between adjacent grid points. Write $\cK_{t,+}$ and $\cK_{t,-}$ for the two products in \eqref{eq:hedhedCP}, and define
\begin{equation}\label{eq:cell_capital_bound}
\underline{\cK}_{t,j}
=
\theta\cK_{t,+}(\upsilon_{j+1})
+(1-\theta)\cK_{t,-}(\upsilon_j).
\end{equation}
With the fractions in Remark~\ref{rem:betting-fractions}, $\cK_{t,+}(\upsilon)$ cannot increase as $\upsilon$ increases, whereas $\cK_{t,-}(\upsilon)$ cannot decrease. To see this, note that before clipping, the factors are $1+\tilde\lambda_t^+(\bar S_t-\bar\upsilon_t)$ and $1-\tilde\lambda_t^-(\bar S_t-\bar\upsilon_t)$; after clipping, they are $(\bar S_t+c)/(\bar\upsilon_t+c)$ and $(1-\bar S_t+c)/(1-\bar\upsilon_t+c)$, respectively. Each positive factor has the stated direction of change, and its branches agree at the clipping point. The products therefore have the same property. Consequently,
\[
\underline{\cK}_{t,j}^{\mathrm{lo}}
\le
\underline{\cK}_{t,j}
\le
\cK_t^\pm(\upsilon)
\text{ for every }\upsilon\in[\upsilon_j,\upsilon_{j+1}],
\]
where $\underline{\cK}_{t,j}^{\mathrm{lo}}$ is a computed lower bound guaranteed to remain valid despite rounding. If this lower bound is at least $1/\alpha_1$, the original test rejects every candidate in the cell. Otherwise, the cell is retained; it may contain values rejected by the original test, since the bound is conservative.

For this optional implementation, initialize $\mathcal A_0^{\mathrm{cell}}=\{1,\ldots,K-1\}$ and retain
\[
\mathcal A_t^{\mathrm{cell}}
=
\{j\in\mathcal A_{t-1}^{\mathrm{cell}}:
\underline{\cK}_{t,j}^{\mathrm{lo}}<T^{\mathrm{up}}\},
\]
where $T^{\mathrm{up}}$ is an upper bound on $1/\alpha_1$ that accounts for rounding. When cells remain, report their outside edges:
\[
I_t
=
[\upsilon_{\min\mathcal A_t^{\mathrm{cell}}},
\upsilon_{1+\max\mathcal A_t^{\mathrm{cell}}}].
\]
Every candidate in $\cM_t^\pm$ lies in a cell whose lower bound stayed below the rejection threshold at every preceding round. Thus,
\begin{equation}\label{eq:cell_enclosure}
\cM_t^\pm\subseteq I_t.
\end{equation}
The interval $I_t$ is an outer enclosure, meaning an interval containing the whole continuous confidence set. If no cell remains, use $I_t=[L,U]$ and do not declare the target width attained. Applying the original population correction to $I_t$ preserves coverage, but the wider numerical interval need not satisfy the exact-inversion width bound.

For finite-precision calculations, interval arithmetic propagates lower and upper bounds through each calculation. These bounds must include accumulated errors in the signal, scaling, betting fractions, and capital updates. Retain a cell whenever rounding prevents certified rejection. Use an upper bound on $\Delta_N$, round reported endpoints outward, and stop only when an upper bound on the full reported population-interval width is at most $\epsilon$. The score-normalization tolerance $\delta_{\mathrm{num}}$ does not control these inversion errors.

The cell check reuses the existing endpoint capitals. For fixed arithmetic precision, the grid computation per round and storage remain $O(K)$, with an additional constant-factor cost for interval arithmetic. The grid size and arithmetic precision should be recorded. A coarse grid or loose numerical bounds may delay stopping or prevent the target width from being reached. Stopping, coverage, and timing results would need to be evaluated for this optional implementation separately.

\newpage
\section*{NeurIPS Paper Checklist}

\begin{enumerate}

\item {\bf Claims}
    \item[] Question: Do the main claims made in the abstract and introduction accurately reflect the paper's contributions and scope?
    \item[] Answer: \answerYes{} % Replace by \answerYes{}, \answerNo{}, or \answerNA{}.
    \item[] Justification: The abstract and introduction clearly state the problem, the proposed e-process-based framework, the main theoretical guarantees, and the empirical claims. The claims are matched by the theoretical results and experiments reported in the paper.

\item {\bf Limitations}
    \item[] Question: Does the paper discuss the limitations of the work performed by the authors?
    \item[] Answer: \answerYes{} % Replace by \answerYes{}, \answerNo{}, or \answerNA{}.
    % [BjfK-L][M4Cr-L]
    % \item[] Justification: The paper discusses the bounded evaluation risk assumption and outlines future directions, including distribution-shift-aware evaluation.
    \item[] Justification: App.~\ref{app:limitations} discusses the limitations and the scope of the theoretical guarantees, including numerical test inversion on a finite grid, loss bounds fixed before observing the pool losses, the i.i.d. condition for population-risk guarantees, the computational costs of surrogate scoring and online inference, and the empirical scope of the experiments. The conclusion also outlines distribution-shift-aware evaluation as a future direction.

\item {\bf Theory assumptions and proofs}
    \item[] Question: For each theoretical result, does the paper provide the full set of assumptions and a complete (and correct) proof?
    \item[] Answer: \answerYes{} % Replace by \answerYes{}, \answerNo{}, or \answerNA{}.
    \item[] Justification: All theoretical results state their assumptions explicitly, and full proofs are provided in the appendix. The results are cross-referenced with the corresponding definitions, assumptions, and proof details.

    \item {\bf Experimental result reproducibility}
    \item[] Question: Does the paper fully disclose all the information needed to reproduce the main experimental results of the paper to the extent that it affects the main claims and/or conclusions of the paper (regardless of whether the code and data are provided or not)?
    \item[] Answer: \answerYes{} % Replace by \answerYes{}, \answerNo{}, or \answerNA{}.
    \item[] Justification: The paper specifies the datasets, surrogate-target model pairs, baselines, hyperparameters, stopping rule, evaluation metrics, and Monte Carlo protocols needed to reproduce the main experimental results.

\item {\bf Open access to data and code}
    \item[] Question: Does the paper provide open access to the data and code, with sufficient instructions to faithfully reproduce the main experimental results, as described in supplemental material?
    \item[] Answer: \answerYes{} % Replace by \answerYes{}, \answerNo{}, or \answerNA{}.
    % TODO (T2) [ZJ]: replace the anonymous link with the public code repository (same as neurips_2026.tex, Sec.~5).
    % RESOLVED (T2, collaborator, 2026-11-04): public repository https://github.com/zyecs/celeus (verified public via GitHub API, 2026-11-04)
%     \item[] Justification: We provide anonymized code and instructions in \url{https://anonymous.4open.science/r/celeus-24E0/}.
    \item[] Justification: We provide code and instructions in \url{https://github.com/zyecs/celeus}.

\item {\bf Experimental setting/details}
    \item[] Question: Does the paper specify all the training and test details (e.g., data splits, hyperparameters, how they were chosen, type of optimizer) necessary to understand the results?
    \item[] Answer: \answerYes{} % Replace by \answerYes{}, \answerNo{}, or \answerNA{}.
    \item[] Justification: The paper reports the datasets, evaluation pools, surrogate-target pairs, baselines, loss functions, hyperparameters, stopping criterion, and sample-selection strategies in the experimental setup and appendix.

\item {\bf Experiment statistical significance}
    \item[] Question: Does the paper report error bars suitably and correctly defined or other appropriate information about the statistical significance of the experiments?
    \item[] Answer: \answerYes{} % Replace by \answerYes{}, \answerNo{}, or \answerNA{}.
    \item[] Justification: The experiments report results over multiple independent seeds or Monte Carlo trials, including Wilson confidence intervals, mean trajectories, and variability bands where appropriate. The caption and text specify how the reported intervals or bands are computed.

\item {\bf Experiments compute resources}
    \item[] Question: For each experiment, does the paper provide sufficient information on the computer resources (type of compute workers, memory, time of execution) needed to reproduce the experiments?
    \item[] Answer: \answerYes{} % Replace by \answerYes{}, \answerNo{}, or \answerNA{}.
    % [ZJ]
    % \item[] Justification: The experiments are conducted on a high-performance computing cluster equipped with multiple NVIDIA A100 GPUs using the PyTorch framework. Each GPU has 64 GB of memory, and the available disk storage exceeds 4 TB.
    \item[] Justification: The experiments are conducted on a high-performance computing cluster equipped with multiple NVIDIA A100 GPUs using the PyTorch framework. Each GPU has 80 GB of memory, and the available disk storage exceeds 4 TB.
    
\item {\bf Code of ethics}
    \item[] Question: Does the research conducted in the paper conform, in every respect, with the NeurIPS Code of Ethics \url{https://neurips.cc/public/EthicsGuidelines}?
    \item[] Answer: \answerYes{} % Replace by \answerYes{}, \answerNo{}, or \answerNA{}.
    \item[] Justification: The research uses public datasets and pre-trained models for evaluation, does not involve collecting sensitive personal data, and is intended to improve reliable and statistically certified LLM evaluation.

\item {\bf Broader impacts}
    \item[] Question: Does the paper discuss both potential positive societal impacts and negative societal impacts of the work performed?
    \item[] Answer: \answerNA{} % Replace by \answerYes{}, \answerNo{}, or \answerNA{}.
    \item[] Justification: The paper discusses that certifiable LLM evaluation can improve reliable deployment decisions by providing anytime-valid confidence intervals, without any negative societal impacts.
    
\item {\bf Safeguards}
    \item[] Question: Does the paper describe safeguards that have been put in place for responsible release of data or models that have a high risk for misuse (e.g., pre-trained language models, image generators, or scraped datasets)?
    \item[] Answer: \answerNA{} % Replace by \answerYes{}, \answerNo{}, or \answerNA{}.
    \item[] Justification: The paper does not release new high-risk models, scraped datasets, or generative systems. It only evaluates existing models using public benchmarks.

\item {\bf Licenses for existing assets}
    \item[] Question: Are the creators or original owners of assets (e.g., code, data, models), used in the paper, properly credited and are the license and terms of use explicitly mentioned and properly respected?
    \item[] Answer: \answerYes{} % Replace by \answerYes{}, \answerNo{}, or \answerNA{}.
    \item[] Justification: The paper cites the public datasets and pre-trained models used in the experiments. Their usage follows the corresponding public access terms and licenses.

\item {\bf New assets}
    \item[] Question: Are new assets introduced in the paper well documented and is the documentation provided alongside the assets?
    \item[] Answer: \answerNA{} % Replace by \answerYes{}, \answerNo{}, or \answerNA{}.
    \item[] Justification: The paper does not introduce or release a new dataset, model, or benchmark asset.

\item {\bf Crowdsourcing and research with human subjects}
    \item[] Question: For crowdsourcing experiments and research with human subjects, does the paper include the full text of instructions given to participants and screenshots, if applicable, as well as details about compensation (if any)? 
    \item[] Answer: \answerNA{} % Replace by \answerYes{}, \answerNo{}, or \answerNA{}.
    \item[] Justification: The paper does not conduct new crowdsourcing experiments or human-subject studies. It uses existing public datasets and benchmarks.

\item {\bf Institutional review board (IRB) approvals or equivalent for research with human subjects}
    \item[] Question: Does the paper describe potential risks incurred by study participants, whether such risks were disclosed to the subjects, and whether Institutional Review Board (IRB) approvals (or an equivalent approval/review based on the requirements of your country or institution) were obtained?
    \item[] Answer: \answerNA{} % Replace by \answerYes{}, \answerNo{}, or \answerNA{}.
    \item[] Justification: The paper does not involve new human-subject data collection or intervention, so IRB approval is not applicable.

\item {\bf Declaration of LLM usage}
    \item[] Question: Does the paper describe the usage of LLMs if it is an important, original, or non-standard component of the core methods in this research? Note that if the LLM is used only for writing, editing, or formatting purposes and does \emph{not} impact the core methodology, scientific rigor, or originality of the research, declaration is not required.
    %this research? 
    \item[] Answer: \answerYes{} % Replace by \answerYes{}, \answerNo{}, or \answerNA{}.
    \item[] Justification: LLMs are central to the research setting and experiments. The paper evaluates target LLMs and uses smaller pre-trained surrogate LLMs to provide surrogate scores for unevaluated samples. These roles are described in the experimental setup, including the surrogate-target model pairs.

\end{enumerate}

\end{document}